\documentclass[mnsc,nonblindrev]{informs_nojournal}

\OneAndAHalfSpacedXI

\usepackage{natbib}
 \bibpunct[, ]{(}{)}{,}{a}{}{,}\def\bibfont{\small}

\TheoremsNumberedThrough     \ECRepeatTheorems

\EquationsNumberedThrough

\usepackage{graphicx}

\usepackage{amssymb}

\usepackage{mathtools}

\usepackage{silence}
\usepackage{subcaption}

\usepackage{hyperref}       \usepackage{url}            \usepackage{booktabs}       \usepackage{amsfonts}       \usepackage{nicefrac}       \usepackage{microtype}      \usepackage{xcolor}         \usepackage{multirow}

\usepackage[capitalise,noabbrev]{cleveref}

\usepackage{algorithm}

\usepackage{algorithmic}

\usepackage{lmodern}

\usepackage{enumitem}
\newcounter{subeqn}[section] \renewcommand{\thesubeqn}{\theequation\alph{subeqn}}\newcommand{\subeqn}{\refstepcounter{subeqn}\tag{\thesubeqn}}

\newcommand{\calH}{\mathcal{H}}

\newcommand{\bbR}{\mathbb{R}}

\newcommand{\bbE}{\mathbb{E}}
\newcommand{\bbP}{\mathbb{P}}

\newcommand{\otilde}{\mathcal{\tilde{O}}}

\newcommand{\dd}{\textnormal{d}}

\newcommand{\lrrr}[1]{\left\{ #1 \right\}}

\newcommand{\lspo}{\ell_{\textnormal{SPO}}}
\newcommand{\lspop}{\ell_{\textnormal{SPO+}}}
\newcommand{\rspo}{R_{\textnormal{SPO}}}
\newcommand{\rspop}{R_{\textnormal{SPO+}}}

\begin{document}

\RUNTITLE{Active Learning For Contextual Linear Optimization:  A Margin-Based Approach}

\TITLE{Active Learning For Contextual Linear Optimization:  A Margin-Based Approach}

\ARTICLEAUTHORS{\AUTHOR{Mo Liu}
\AFF{Department of Industrial Engineering and Operations Research, University of California, Berkeley, Berkeley, CA, 94720,  \EMAIL{mo\_liu@berkeley.edu}} \AUTHOR{Paul Grigas}
\AFF{Department of Industrial Engineering and Operations Research, University of California, Berkeley, Berkeley, CA, 94720,  \EMAIL{pgrigas@berkeley.edu}}
\AUTHOR{Heyuan Liu}
\AFF{Department of Industrial Engineering and Operations Research, University of California, Berkeley, Berkeley, CA, 94720,  \EMAIL{heyuan\_liu@berkeley.edu}}
\AUTHOR{Zuo-Jun Max Shen}
\AFF{Department of Industrial Engineering and Operations Research, University of California, Berkeley, Berkeley, CA, 94720,  \EMAIL{maxshen@berkeley.edu}} 
} 

\ABSTRACT{We develop the first active learning method for contextual linear optimization. Specifically, we introduce a label acquisition algorithm that sequentially decides whether to request the ``labels'' of feature samples arriving in an unlabeled data stream, where the labels correspond to the objective coefficients in the linear optimization problem. The label acquisition rule is guided by the decision loss induced by the predicted coefficients, known as the Smart Predict-then-Optimize (SPO) loss. Motivated by the structure of the SPO loss, our algorithm uses the distance to degeneracy as a margin-based criterion for label acquisition. The algorithm processes incoming feature samples in batches, acquiring within each batch only the labels of samples whose distance to degeneracy falls below a rejection threshold.
We derive non-asymptotic bounds for direct SPO-loss minimization over a general hypothesis class and for surrogate-loss minimization over a linear prediction class. These bounds demonstrate that, under certain natural conditions, our margin-based active learning algorithm requires fewer labels than supervised learning to achieve the same excess population risk for the chosen training loss. We provide concrete examples in which these conditions hold under continuous (sub-)Gaussian noise. Finally, we present numerical evidence demonstrating the practical value of our approach in applications such as personalized pricing and shortest-path problems.
}


\maketitle

\section{Introduction}

In many applications of operations research, decisions are made by solving optimization problems that involve some unknown parameters. Typically, machine learning tools are used to predict these unknown parameters, and then an optimization model is used to generate the decisions based on the predictions. For example, in the shortest path problem, we need to predict the cost of each edge in the network and then find the optimal path to route users. Another example is the personalized pricing problem, where we need to predict the purchase probability of a given customer at each possible price and then decide the optimal price. In this paper, we study the case where the downstream decision-making problem is a linear optimization with unknown coefficients in the objective function, a setting referred to as contextual linear optimization (CLO). In this setting, when generating prediction models, it is natural to consider the final decision error as a loss function to measure the quality of a model instead of standard notions of prediction error.
The loss function that directly considers the cost of the decisions induced by the predicted parameters, in contrast to the prediction error of the parameters, is called the {\em Smart Predict-then-Optimize (SPO)} loss as proposed by \citet{elmachtoub2022smart}. Prediction models trained with the SPO loss directly target the cost of the decisions induced by their predictions. 

In general, for a given feature vector $x$, calculating the SPO loss requires knowing the correct (in hindsight) optimal decision associated with the unknown parameters. However, a full observation of these parameters, also known as a label associated with $x$, is not always available.
For example, we may not observe the cost of all edges in the graph in the shortest path problem. In practice, acquiring the label of one feature vector instance could be costly, and thus acquiring the labels of all feature vectors in a given dataset would be prohibitively expensive and time-consuming. In such settings, it is essential to actively select the samples for which label acquisition is worthwhile.

Algorithms that make decisions about label acquisition lie in the area of \textit{active learning}. The goal is to learn a predictor while requesting a small number of labels from an unlabeled sample stream. In the CLO setting, confidence concerns the decision induced by a prediction, rather than certainty about the realized cost vector. If plausible predictions induce the same decision, acquiring another cost vector may be less useful for reducing decision error. This observation motivates selective label acquisition when observing cost vectors is expensive.

Traditional acquisition criteria often focus on prediction or classification error rather than the cost of downstream decisions. Incorporating the SPO loss into the acquisition rule can therefore change which features are selected for labeling. Whether this reduces the number of labels needed at a specified risk level depends on the loss, the hypothesis class, and the feature--cost distribution.

Considering active learning in the CLO can bridge the gap between active learning and operational decisions. However, there are two major challenges when designing algorithms to select samples. One is the computational issue due to the non-convexity and non-Lipschitzness of the SPO loss. When one is concerned with minimizing the SPO loss, existing active learning algorithms are computationally intractable. For example, the general importance weighted active learning (IWAL) algorithm proposed by \cite{beygelzimer2009importance} is impractical to implement, since calculating the ``weights'' of samples requires a large enumeration of all pairs of predictors. Other active learning algorithms that are designed for the classification problem cannot be extended to minimize the SPO loss directly. Another challenge is to derive bounds for the label complexity of the algorithms and to demonstrate the advantages over supervised learning. Label complexity refers to the number of labels that must be acquired to ensure that the risk of predictor is not greater than a desired threshold. To achieve the benefits of active learning, the label complexity must be smaller than the sample complexity of supervised learning while attaining the same risk level with respect to the loss function of interest (in our case, the SPO loss). \citet{kaariainen2006active} demonstrated that, without additional assumptions regarding the distributions of features and noise, active learning algorithms exhibit the same label complexity as supervised learning. Thus, finding some natural conditions on the noise and feature distributions that yield a smaller label complexity for an active learning algorithm is a nontrivial yet critical challenge in demonstrating the value of active learning. 

In this paper, we develop the first margin-based active learning method for CLO. Specifically, in the CLO setting, we address the problem of solving a linear optimization problem where the coefficients of the objective function, referred to as the unknown cost vector (or label vector), are predicted based on feature information. Our proposed algorithm, inspired by margin-based algorithms in active learning, uses a measure of ``confidence'' associated with the cost vector prediction of the current model to decide whether or not to acquire a label for a given feature. Specifically, the label acquisition decision is based on the notion of \textit{distance to degeneracy} introduced by \cite{el2022generalization}, which measures the distance from the prediction of the current model to the set of cost vectors that have multiple optimal solutions. These degenerate cost vectors divide the space of cost vectors into regions, where each region corresponds to one optimal extreme point of the polyhedral feasible region. Intuitively, the greater the distance to degeneracy, the more confident we can be that a small change in the prediction will not change the associated decision. Thus, in our margin-based active learning algorithm, we are more likely to acquire the labels of samples with a small distance to degeneracy.

The proposed algorithm, referred to as MBAL-SPO, begins with a warm-up period in which all labels are acquired. It then processes the incoming feature samples in batches whose sizes increase geometrically. Within each batch, the current predictor, the current hypothesis class, and the rejection threshold remain fixed. At the end of the batch, the prediction model and the hypothesis class are updated using the labels acquired in this batch. The working set is reset at the beginning of every batch, while the information from earlier batches is retained through the nested confidence sets. 

We establish non-asymptotic guarantees for the common parameter schedules using two distinct analyses. For direct SPO-loss minimization, decision stability applies to general hypothesis classes under the stated error bound and finite Natarajan dimension of the induced decision class. While nonoptimal predictors remain, retained predictors incur identical SPO losses on rejected features, allowing MBAL-SPO to inherit uniform supervised SPO-ERM risk bounds. For surrogate training, we study linear cost-vector predictors. A quantitative local information condition compares the selected feature distribution with the original distribution; squared loss admits a projected-noise argument under sub-Gaussian noise, while convex surrogate losses require population curvature and score-variability conditions. The resulting finite comparisons state when all acquired labels, including initialization and previous batches, suffice for a smaller same-risk budget than ordinary supervised empirical minimization. We provide concrete examples for continuous squared-loss and SPO+ loss. Numerical experiments in shortest path and personalized pricing problems examine the practical behavior. Our contributions are summarized below.

\begin{itemize}
    \item We study active learning for contextual linear optimization (CLO). To acquire labels for training a prediction model that supports downstream decisions, we propose a margin-based active learning algorithm that uses the distance to degeneracy as the criterion for label acquisition. The algorithm fixes the label acquisition rule within each batch, resets the working set at the beginning of every batch, and updates the prediction model only at the end of a batch.

    \item We establish direct SPO-risk guarantees over general hypothesis classes in Theorem \ref{thm:spo_loss_directly}, including transfer of uniform supervised SPO-ERM rates at a common feature horizon. For linear prediction with squared loss and convex surrogate losses, Theorems \ref{uniform_margin} and \ref{thm:alpha-two-spop-finite} give finite excess-risk and acquired-label bounds under the new conditional-information setting. 

    \item We compare MBAL-SPO with supervised learning using the guarantees established in Theorems \ref{thm:spo_loss_directly}, \ref{uniform_margin}, and \ref{thm:alpha-two-spop-finite}. Specifically, Proposition \ref{prop:sublinear_spo} considers direct SPO-loss minimization, while Propositions \ref{prop:alpha-two-sl-advantage} and \ref{prop:alpha-two-spop-comparison} consider squared-loss minimization and general surrogate-loss minimization, respectively. These propositions demonstrate that, under certain margin conditions, MBAL-SPO requires fewer labels than supervised learning to achieve the same excess training-risk guarantee. Section \ref{sec:phi_examples} further provides concrete examples in which these conditions hold.

    \item We demonstrate the practical value of our algorithm through numerical experiments in personalized pricing and shortest path problems. The reported experiments compare MBAL-SPO with supervised learning at common acquired-label budgets and provide empirical evidence of lower test decision risk in these settings.
\end{itemize}

\section{Motivating Example and Literature Review}

In this section, we first motivate active learning for CLO through a personalized pricing example and then review the related literature.
\subsection{Example Application: Personalized Pricing}

\begin{example}[Personalized pricing via customer surveys]\label{example1}
Suppose that a retailer needs to decide the prices of $\mathfrak{J}$ items for each customer, after observing the features (personalized information) of the customers. The feature vector of a generic customer is $x$, and the purchase probability of that customer for item $j$ is $d_j(p^j)$, which is a function of the price $p^j$. The retailer does not observe this purchase probability directly and may obtain a noisy estimate from a customer survey. Suppose the price for each item is selected from a candidate list $ \{p_1,p_2,...,p_\mathcal{I}\}$, which is sorted in ascending order. Then, the pricing problem can be formulated as
\begin{align}
    \max_{\mathbf{w}} & \quad \bbE[\sum_{j = 1}^{\mathfrak{J}}\sum_{i = 1}^{\mathcal{I}}d_j(p_i)p_i w_{i,j}|x] \label{obj:example} \\
    \text{s.t.} & \quad  \sum_{i = 1}^\mathcal{I} w_{i,j} = 1, \quad j = 1,2,...,\mathfrak{J}, \subeqn \label{con:1price}\\
    & \quad  \mathbf{A} \mathbf{w} \le b, \label{con:sum1} \subeqn\\
    & \quad w_{i,j} \in \{0,1\}, \quad  i = 1,2,...,\mathcal{I}, j = 1,2,...,\mathfrak{J}. \subeqn
\end{align}
Here, $\mathbf{w}$ encodes the decision variables with indices in the set $\mathcal{I}\times\mathfrak{J}$, where $w_{i,j}$ is a binary variable indicating which price for item $j$ is selected. Namely, $w_{i,j} = 1$ if item $j$ is priced at $p_i$, and otherwise $w_{i,j} = 0$. The objective \eqref{obj:example} is to maximize the expected total revenue of $\mathfrak{J}$ items by offering price $p_i$ for item $j$. Constraints \eqref{con:1price} require each item to have one price selected. In constraint \eqref{con:sum1}, $\mathbf{A}$ is a matrix with $\mathfrak{K}$ rows, and $b$ is a vector with $\mathfrak{K}$ dimensions. Each row of constraints \eqref{con:sum1} characterizes one rule for setting prices. For example, if the first row of $\mathbf{A} \mathbf{w}$ is $w_{i,j} - \sum_{i' = i}^\mathcal{I} w_{i',j+1}$ and the first entry in $b$ is zero, then this constraint further requires that if item $j$ is priced at $p_i$, then the price for the item $j+1$ must be no smaller than $p_i$. For another example, if the second row of $\mathbf{A} \mathbf{w}$ is $\sum_{i' = 1}^{i-1} \sum_{j = 1}^\mathfrak{J} w_{i',j}$, and the second entry of $b$ is 1, then it means that at most one item can be priced below the price $p_i$. Thus, constraints \eqref{con:sum1} can characterize different rules for setting prices for $\mathfrak{J}$ items.

Traditionally, the conditional expectation of revenue $\bbE[d_j(p_i)p_i|x]$ must be estimated from the purchasing behavior of the customers. In this example, we consider the case where the retailer can give the customers surveys to investigate their purchase probabilities. By analyzing the results of the surveys, the retailer can infer the purchase probability $d_j(p_i)|x$ for each price point $p_i$ and each item $j$ for this customer. Therefore, whenever a survey is conducted, the retailer acquires a noisy estimate of the revenue, denoted by $d_j(p_i)p_i|x$, at each price point $p_i$ and item $j$.

In personalized pricing, first, the retailer would like to build a prediction model to predict $\bbE[d_j(p_i)p_i|x]$ given the customer's feature vector $x$. Then, given the prediction model, the retailer solves the problem \eqref{obj:example} to obtain the optimal prices. In practice, when evaluating the quality of the prediction results of $d_j(p_i)p_i|x$, the retailer cares more about the expected revenue from the optimal prices based on this prediction, rather than the direct prediction error. Therefore, when building the prediction model for $d_j(p_i)p_i|x$, retailers are expected to be concerned with minimizing SPO loss, rather than minimizing prediction error.

One property of \eqref{obj:example} is that the objective is linear and can be further written as $ \max_{\mathbf{w}} \sum_{j = 1}^{\mathfrak{J}}\sum_{i=1}^{\mathcal{I}} \bbE[d_j(p_i)p_i|x] w_{i,j}$. By the linearity of the objective, the revenue loss induced by the prediction errors can be written in the form of the SPO loss considered in \cite{elmachtoub2022smart}. In general, considering the prediction errors when selecting customers may be inefficient, since smaller prediction errors do not always necessarily lead to smaller revenue losses, because of the properties of the SPO loss examined by \cite{elmachtoub2022smart}.
\hfill\Halmos
\end{example}

In Example \ref{example1}, in practice, there exists a considerable cost to investigate all customers, for example, the labor cost to collect the answers and incentives given to customers to fill out the surveys. Therefore, the retailer would rather intelligently select a limited subset of customers to investigate. This subset of customers should be ideally selected so that the retailer can build a prediction model with a small SPO loss, using a small number of surveys. 

Active learning is essential to help retailers select representative customers and reduce the number of surveys. Traditional active learning algorithms would select customers to survey based on model prediction errors, which are different from the final revenue of the retailer.
On the contrary, when considering the SPO loss, the final revenue is integrated into the learning and survey distribution processes.

\subsection{Literature Review}\label{sec:literature}

Contextual linear optimization (CLO) is a special case of predict-then-optimize. We review active learning and predict-then-optimize, focusing on decision-informed label acquisition.

\paragraph{Active learning.}
Active learning often targets classification or prediction error; see \citet{settles2009active} for a survey. \citet{cohn1994improving} establish substantial label savings in noiseless binary classification. In unrestricted noisy settings, however, \citet{kaariainen2006active} show that a better label-complexity order than supervised learning cannot be guaranteed uniformly. Uniform improvements therefore require distributional or noise conditions. Agnostic classification algorithms include \citet{hanneke2007bound,dasgupta2007general,hanneke2011rates,balcan2009agnostic,balcan2007margin}. In particular, \citet{balcan2007margin} analyze margin-based acquisition for linear separators under realizability and specified noise and feature-distribution assumptions. \citet{krishnamurthy2017active} study cost-sensitive classification, whereas \citet{gao2020cost} jointly select instances and labelers with different labeling costs and accuracies.
For real-valued losses, \citet{castro2005faster} study active regression rates, and \citet{sugiyama2009pool,cai2016batch} consider squared-loss settings. \citet{beygelzimer2009importance} propose importance-weighted active learning (IWAL) for general losses. Applying IWAL directly to SPO loss can be computationally demanding because acquisition compares predictor pairs and the loss is nonconvex. We instead use optimization geometry to acquire vector-valued cost labels.
\paragraph{Predict-then-optimize framework.}
Methods connecting prediction with downstream optimization include \citet{bertsimas2020predictive,kao2009directed,elmachtoub2022smart,zhu2022joint,donti2017task,ho2019data}. In the last few years, several complementary surveys have contributed towards organizing and synthesizing contextual stochastic optimization methods \citep{sadana2025survey} and decision-focused learning algorithms and benchmarks \citep{mandi2024decision}. The tutorial by \citet{liu2026tutorial} further focuses on contextual linear programming formulations within decision-focused learning.

We consider the CLO setting and adopt the SPO framework of \citet{elmachtoub2022smart}, who introduce the convex SPO+ surrogate for the generally nonconvex, non-Lipschitz SPO loss and establish consistency under distributional conditions. Computational methods and applications in this setting include \citet{elmachtoub2020decision,loke2022decision,demirovic2020dynamic,demirovic2019predict,mandi2020interior,mandi2019smart,tang2022pyepo}, among others. While a considerable amount of our focus is on the SPO+ loss, due in part to its pioneering and lasting role, our methodology is generic and other surrogate losses and related approaches could also be adapted to the active learning setting. For example, recently, \citet{huangdecision} propose perturbation gradient (PG) losses with asymptotically vanishing approximation error under suitable tuning, including under misspecification, and \citet{wan2026solverfree} develop a measure-transformation approach that avoids optimization-solver calls during training.

On the statistical side, \citet{el2022generalization} establish SPO generalization bounds, while \citet{ho2020risk,liu2021risk} study surrogate-risk guarantees and calibration. \citet{hu2020fast} derive fast CLO regret rates under margin conditions and identify settings where plug-in estimation attains faster rates than integrated learning. \citet{hu2024contextual} extend CLO to offline bandit feedback and \citet{benslimane2026decision} to online decision-making with bandit feedback; we instead select which arriving features receive full cost-vector labels. \citet{ahmed2026lookahead} distinguish true-model regret from hindsight regret through a policy-independent look-ahead premium. Our experiments evaluate regret against the conditional-mean optimal decision, excluding this premium.

Broader decision-focused methods differentiate through optimization or use related training approximations \citep{amos2017optnet,wilder2019melding,berthet2020learning,chung2022decision}; \citet{kotary2021end} survey end-to-end constrained optimization learning. For nonlinear objectives, contextual stochastic optimization can require more than conditional-mean prediction \citep{kallus2023stochastic,grigas2021integrated,elmachtoub2023estimate}.

\subsection{Organization}

The remainder of the paper is organized as follows. Section \ref{sec:preliminary} introduces CLO, the SPO loss, label complexity, and the SPO+ surrogate. Section \ref{sec:margin} explains distance to degeneracy and presents MBAL-SPO. Section \ref{sec:analysis} introduces the consistency and error-bound conditions, the near-degeneracy function, and the Natarajan dimension. We then analyze direct SPO minimization before linear prediction with squared loss and convex surrogate losses. Section \ref{sec:smalllabel} explains the distinct roles of the margin and information conditions and presents the continuous examples in Section \ref{sec:phi_examples}. The appendices contain the finite cutoffs, proofs, and separate Gaussian and non-degenerate benchmarks. Section \ref{sec:experiments} reports the shortest path and personalized pricing experiments, and Section \ref{sec:future} discusses further directions.

\section{Preliminaries}\label{sec:preliminary}
We begin by presenting some preliminaries related to active learning and CLO, including the SPO loss function, the objectives of active learning in the CLO setting, and the SPO+ surrogate loss.
    \subsection{Contextual Linear Optimization (CLO) and Active Learning}
    Let us begin by formally describing the setting of CLO and the ``Smart Predict-then-Optimize (SPO)'' loss function. 
    The downstream optimization problem is a linear optimization problem, but the cost vector of the objective, $c \in \mathcal{C} \subseteq \bbR^d$, is unknown when the problem is solved to make a decision, where $\mathcal{C}$ is an unbounded space unless stated otherwise. 
    We observe a feature vector, $x \in \mathcal{X} \subseteq \bbR^p$, which provides auxiliary information that can be used to predict the cost vector. The feature space $\mathcal X$ is assumed to be bounded.  Write $\mu(x):=\bbE[c\mid x]$ and noise term $\xi:= c-\mu(x)$. We assume there is a fixed but unknown distribution $\mathcal{D}$ over pairs $(x, c)$ living in $\mathcal{X} \times \mathcal{C}$. The marginal distribution of $x$ is denoted by $\mathcal{D}_{\mathcal{X}}$. We assume $\bbE\|\mu(X)\|_2<\infty$.  For the concentration results, we assume conditionally sub-Gaussian noise: there exists a constant $\sigma>0$ such that for any $v\in\bbR^d$,
\begin{equation}\label{eq:subgaussian-cost-noise}
 \bbE[\exp(v^T\xi)\mid x]\leq
 \exp\!\left(\frac{\sigma^2\|v\|_2^2}{2}\right).
\end{equation}
This condition permits Gaussian noise and does not require independent noise coordinates.

    Let $w\in S$ denote the downstream decision, where the feasible region $S\subseteq\bbR^d$ is a nonempty bounded polyhedron known to the decision-maker. Given an observed feature vector $x$, the ultimate goal is to solve the contextual stochastic optimization problem:
    \begin{equation}\label{eq:po-framework}
        \min_{w \in S} \bbE_c [c^T w \vert x] ~=~ \min_{w \in S} \bbE [c \vert x]^T w. 
    \end{equation}
    From the equivalence in \eqref{eq:po-framework}, we observe that solving CLO relies on the prediction of the conditional expectation $\bbE_c [c \vert x]$. 
    Given a prediction $\hat c$, we make a decision by solving the deterministic linear optimization problem:
    \begin{equation}\label{eq:opt}
        P(\hat{c}): \quad \min_{w \in S} \hat{c}^T w. 
    \end{equation}
    We use a measurable optimization oracle $w^*:\bbR^d\to S$ with fixed tie-breaking, so $w^*(\hat c)$ is an optimal solution of $P(\hat c)$. In the polyhedral setting, the oracle selects an optimal extreme point.  
    
    Our goal is to learn a cost vector predictor function $h: \mathcal{X} \rightarrow \bbR^d$, so that for any newly observed feature vector $x$, we first make prediction $h(x)$ and then solve the optimization problem $P(h(x))$ in order to make a decision. This CLO setting is prevalent in the application of machine learning to operations research problems.
    We assume the predictor function $h$ is within a compact hypothesis class $\mathcal{H}$ of functions on $\mathcal{X} \rightarrow \bbR^d$. We say that the class is well-specified when the function $x\mapsto\bbE[C\mid X=x]$ belongs to $\mathcal H$. Well-specification is required only where explicitly stated.
To select $h \in \mathcal{H}$, we minimize the empirical risk associated with an appropriately defined loss function.
In the CLO setting, our primary loss function of interest is the SPO loss, introduced by \citet{elmachtoub2022smart}, which characterizes the regret in decision error due to an incorrect prediction and is formally defined as 
    \begin{equation*}
        \lspo(\hat{c}, c) := c^T w^*(\hat{c}) - c^T w^*(c),
    \end{equation*}
    for any cost vector prediction $\hat{c}$ and realized cost vector $c$.
    We further define the SPO risk of a prediction function $h\in \mathcal{H}$  as $\rspo(h) := \bbE_{(x, c) \sim \mathcal{D}} [\lspo(h(x), c)]$, and the excess risk of $h$ as $\rspo(h) - \inf_{h'\in \mathcal{H}} \rspo(h')$. (Throughout, we typically remove the subscript notation from the expectation operator when it is clear from the context.) Notice that a guarantee on the excess SPO risk implies a guarantee that holds ``on average'' with respect to $x$ for \eqref{eq:po-framework}.

    As previously described, in many situations acquiring cost vector data may be costly and time-consuming. Active learning aims to choose which feature samples $x$ to label sequentially and interactively, in contrast to standard supervised learning which acquires the labels of all the samples before training the model. In the CLO setting, acquiring a ``label" corresponds to collecting the cost vector data $c$ associated with a given feature vector $x$.
    An active learner aims to acquire few labels while attaining a small excess risk. The optimal risk $R^*_{\mathrm{SPO}}:=\inf_{h\in\mathcal H}R_{\mathrm{SPO}}(h)\ge 0 $ is the best risk within the hypothesis class.  For a target error $\epsilon>0$ and $\delta\in(0,1)$, label complexity refers to the number of acquired cost vectors sufficient to return $\hat h$ satisfying $R_{\mathrm{SPO}}(\hat h)-R^*_{\mathrm{SPO}}\leq\epsilon$ with probability at least $1-\delta$. The same definition applies to a specified surrogate loss. 

    \subsection{Surrogate Loss Functions and SPO+}

    Due to the potential non-convexity and non-continuity of the SPO loss, a common approach is to consider surrogate loss functions $\ell$ that have better computational properties and are still (ideally) aligned with the original SPO loss. For a surrogate loss $\ell:\bbR^d\times\bbR^d\to\bbR_+$, its population risk is $R_\ell(h):=\bbE[\ell(h(X),C)]$. When the minimum is attained, we write $R_\ell^*:=\min_{h\in\mathcal H}R_\ell(h)$. Finiteness and attainment are imposed in the results that use them.

    As a special case of the surrogate loss function $\ell$,  \citet{elmachtoub2022smart} proposed a convex surrogate loss function, called the SPO+ loss, which is defined by 
    \begin{equation*}
        \lspop(\hat{c}, c) := \max_{w \in S} \lrrr{(c - 2 \hat{c})^T w} + 2 \hat{c}^T w^*(c) - c^T w^*(c),
    \end{equation*}
    and is an upper bound on the SPO loss, i.e., $\lspo(\hat{c}, c) \leq \lspop(\hat{c}, c)$ for any $\hat{c} \in \hat{\mathcal{C}}$ and $c \in \mathcal{C}$.
    \citet{elmachtoub2022smart} demonstrate the computational tractability of the SPO+ surrogate loss, conditions for Fisher consistency of the SPO+ risk with respect to the true SPO risk, as well as numerical evidence of its good performance with respect to the downstream optimization task. \citet{liu2021risk} further establish sufficient conditions under which a small excess surrogate SPO+ risk for a prediction function $h$ implies a small excess true SPO risk for $h$.
    This property holds not only for the SPO+ loss but also for other surrogate loss functions, such as the squared $\ell_2$ loss (see, for details, \citet{ho2020risk}). 
    Importantly, the SPO+ loss still accounts for the downstream optimization problem and the structure of the feasible region $S$, in contrast to losses like the $\ell_2$ loss that focus only on prediction error. 

\textbf{Notations. } Let $\|\cdot\|$ be a generic norm on $\bbR^d$, with dual norm $\|c\|_*:=\max_{\|w\|\leq1}c^Tw$. We denote the extreme points of $S$ by $\mathfrak S$ and its diameter by $D_S:=\sup_{w,w'\in S}\|w-w'\|$. The linear optimization gap is $\omega_S(c):=\max_{w\in S}c^Tw-\min_{w\in S}c^Tw$. For the conditional mean, write $\omega_S(\mu):=\operatorname*{ess\,sup}_x\omega_S(\mu(x))$ whenever this quantity is finite. The prediction range is $\hat{\mathcal C}:=\{h(x):h\in\mathcal H,\ x\in\mathcal X\}$, and $\rho(\hat{\mathcal C}):=\sup_{\hat c\in\hat{\mathcal C}}\|\hat c\|$ bounds predictions, not labels.
We use $\mathcal N_d(\mu,\Sigma)$ for a $d$-dimensional Gaussian distribution with mean $\mu$ and covariance matrix $\Sigma$, and $\mathcal N(\mu,\sigma^2)$ in the scalar case. We use $\bbR_+$ to denote $[0,+ \infty)$. When conducting the asymptotic analysis, we adopt the standard notations $\mathcal{O}(\cdot)$, $\Theta(\cdot)$ and $\Omega(\cdot)$. We further use $\otilde(\cdot)$ to suppress the logarithmic dependence.
We use $\mathbb{I}$ to refer to the indicator function, which outputs 1 if the argument is true and 0 otherwise.

\section{Margin-Based Algorithm}\label{sec:margin}

In this section, we introduce the margin-based algorithm designed to minimize the SPO loss (MBAL-SPO). To provide intuition and motivation, we begin by illustrating the algorithm using a geometric example.

\subsection{Illustration and Motivation}\label{sec:illustration_mbal}
Let us introduce the idea of MBAL-SPO with the following two examples, which illustrate the value of integrating the SPO loss into active learning. Given the current training set and predictor, some features can be more informative and therefore more valuable to label than others. In general, the value of labeling a feature depends on the prediction error in Figure \ref{fig:margin1} and the location of the prediction relative to the structure of the feasible region $S$ in Figure \ref{fig:margin2}.
In Figure \ref{fig:margin1}, the feasible region $S$ is polyhedral and the yellow arrow represents $-\hat h(x)$. For illustration, suppose that the hypothesis class is well-specified, so an optimal predictor is $h^*(x)=\bbE[c\mid x]$. Because $c\mid x$ is random, the empirical best predictor $\hat h$ based on the current training set may not be exactly equal to $h^*$. The green circles in Figure \ref{fig:margin1} represent confidence regions for $\bbE[c\mid x]$ centered at the current prediction $\hat h(x)$. Given one feature $x$, the prediction is $\hat c = \hat{h}(x)$, the negative of which is shown in Figures \ref{fig:margin1a} and \ref{fig:margin1b}. Intuitively, when the training set gets larger, the empirical best predictor $\hat{h}$ should get closer to $h^*$, and $\hat{h}(x)$ should get closer to $\bbE[c|x]$. Thus, we can construct a confidence region around $\hat{h}(x)$, such that $\bbE[c|x]$ is within this confidence region with some high probability. Examples of confidence regions for the estimation of $\bbE[c|x]$ given the current training set are shown in the green circles in Figure \ref{fig:margin1}. The optimal solution $w^\ast(\hat{c})$ is the extreme point indicated in Figure \ref{fig:margin1}, and the normal cone at $w^\ast(\hat{c})$ illustrates the set of all cost vectors whose optimal solution is also $w^\ast(\hat{c})$.

The cost vectors on the boundary of the normal cone can lead to multiple optimal decisions and will be defined as degenerate cost vectors in Definition \ref{def:distancedegen}. When the confidence region is large, as in Figure \ref{fig:margin1a}, it intersects the set of degenerate cost vectors, and predictions within the confidence region may induce different optimal decisions. When the confidence region is smaller, as in Figure \ref{fig:margin1b}, it does not intersect the set of degenerate cost vectors. In this case, if $\bbE[c\mid x]$ lies in the confidence region, then $\hat h(x)$ and $\bbE[c\mid x]$ induce the same optimal decision. Therefore, acquiring the label associated with $x$ is less informative for reducing the excess SPO risk.

\begin{figure}[ht]
	\centering
	\begin{subfigure}{0.5\textwidth}
    \centering
        \includegraphics[height=5cm]{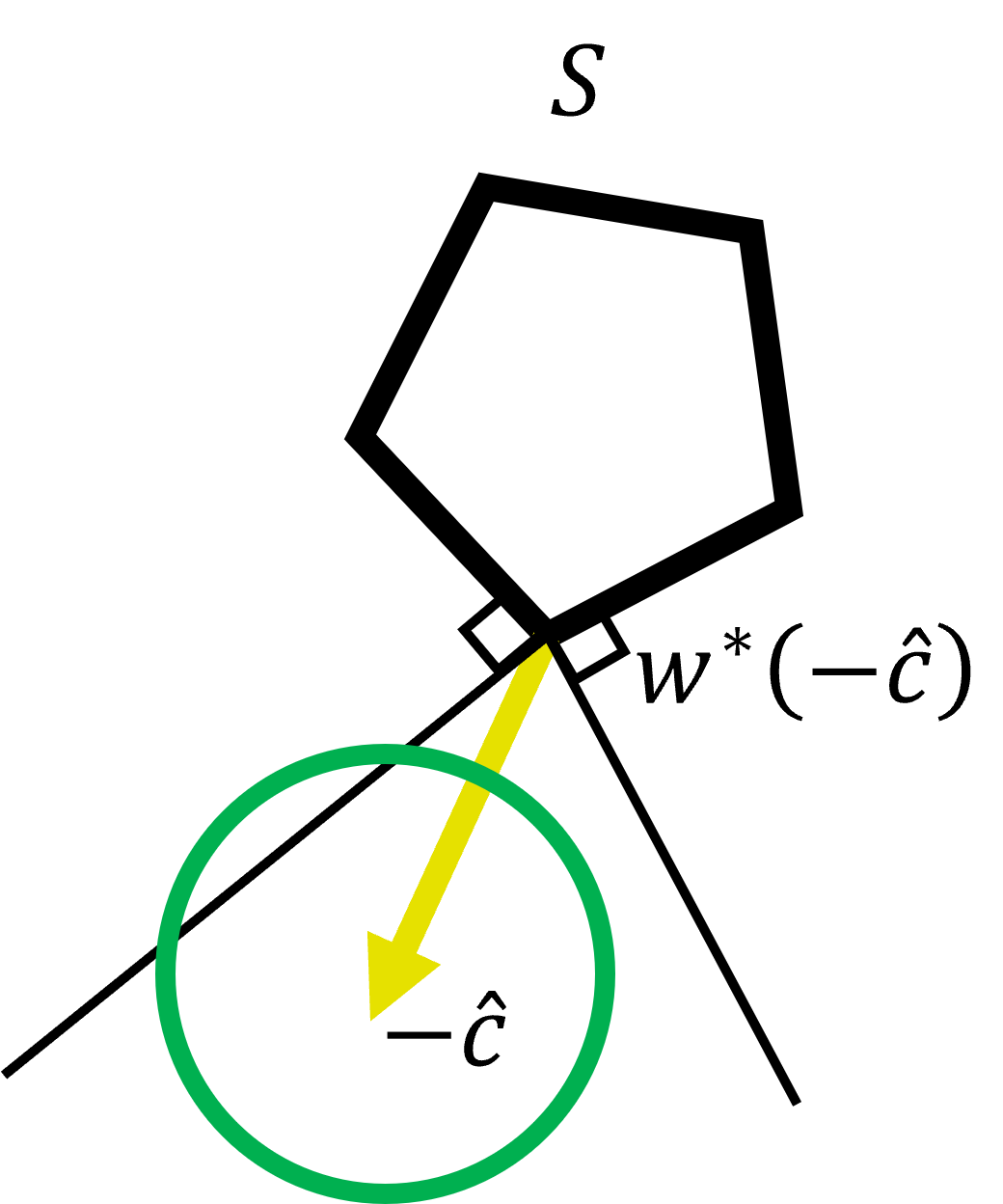}
        \caption{Confidence region is large}
        \label{fig:margin1a}
    \end{subfigure}\begin{subfigure}{0.5\textwidth}
    \centering
        \includegraphics[height=5cm]{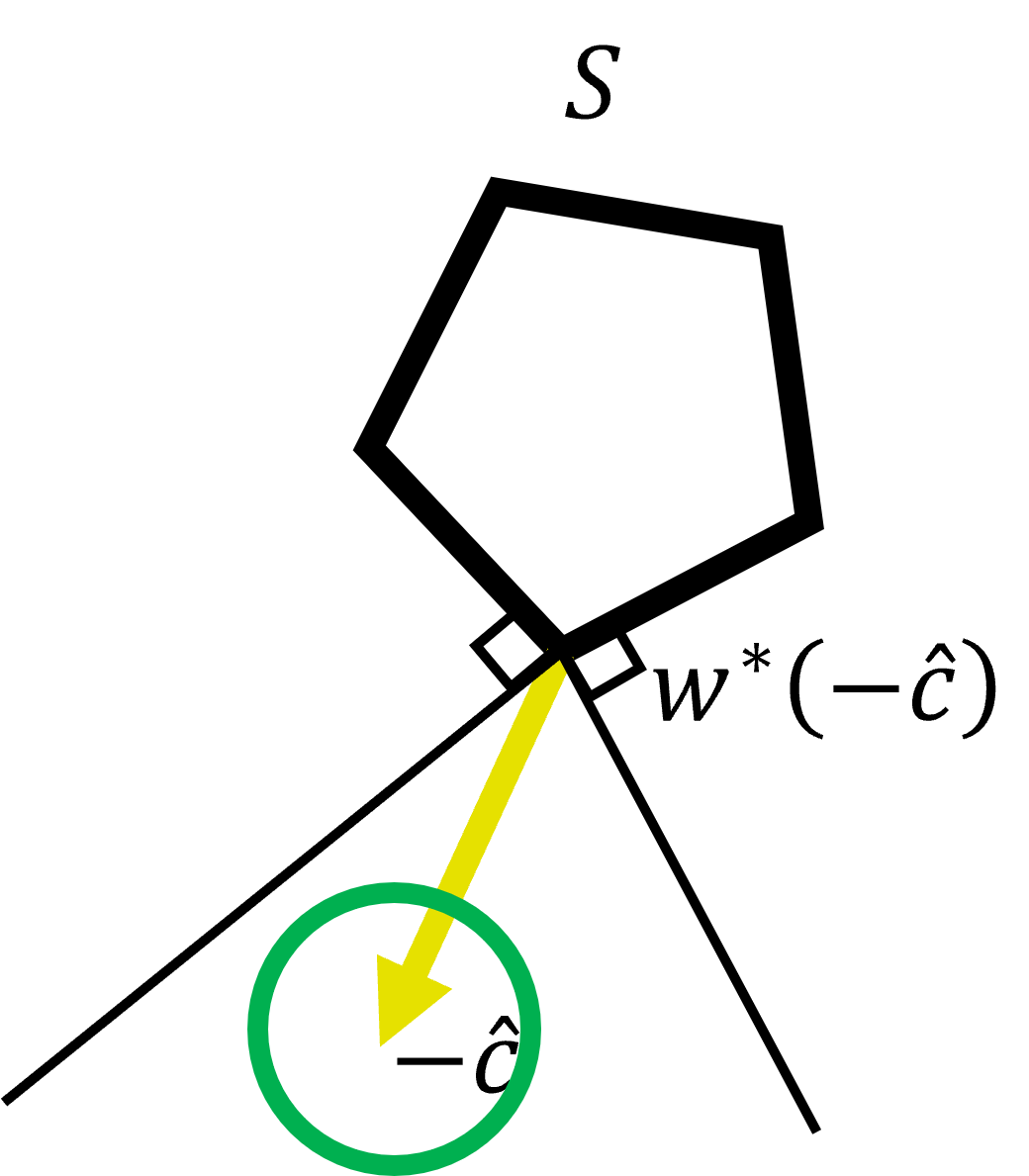}
        \caption{Confidence region is small}
        \label{fig:margin1b}
    \end{subfigure}\caption{Decision uncertainty at a fixed prediction as the confidence region changes.}
	\label{fig:margin1}
\end{figure}

Figure \ref{fig:margin2} further shows why the location of the prediction is important when the two confidence regions have the same size. In Figure \ref{fig:margin2a}, the confidence region is close to degeneracy, so a small change in the predicted cost vector may change the optimal decision. In Figure \ref{fig:margin2b}, the confidence region is far from the degenerate cost vectors, so acquiring the label for $x$ is less informative because the current prediction is more likely to induce the same optimal decision as the optimal predictor.

\begin{figure}[ht]
	\centering
	\begin{subfigure}{0.5\textwidth}
    \centering
        \includegraphics[height=5cm]{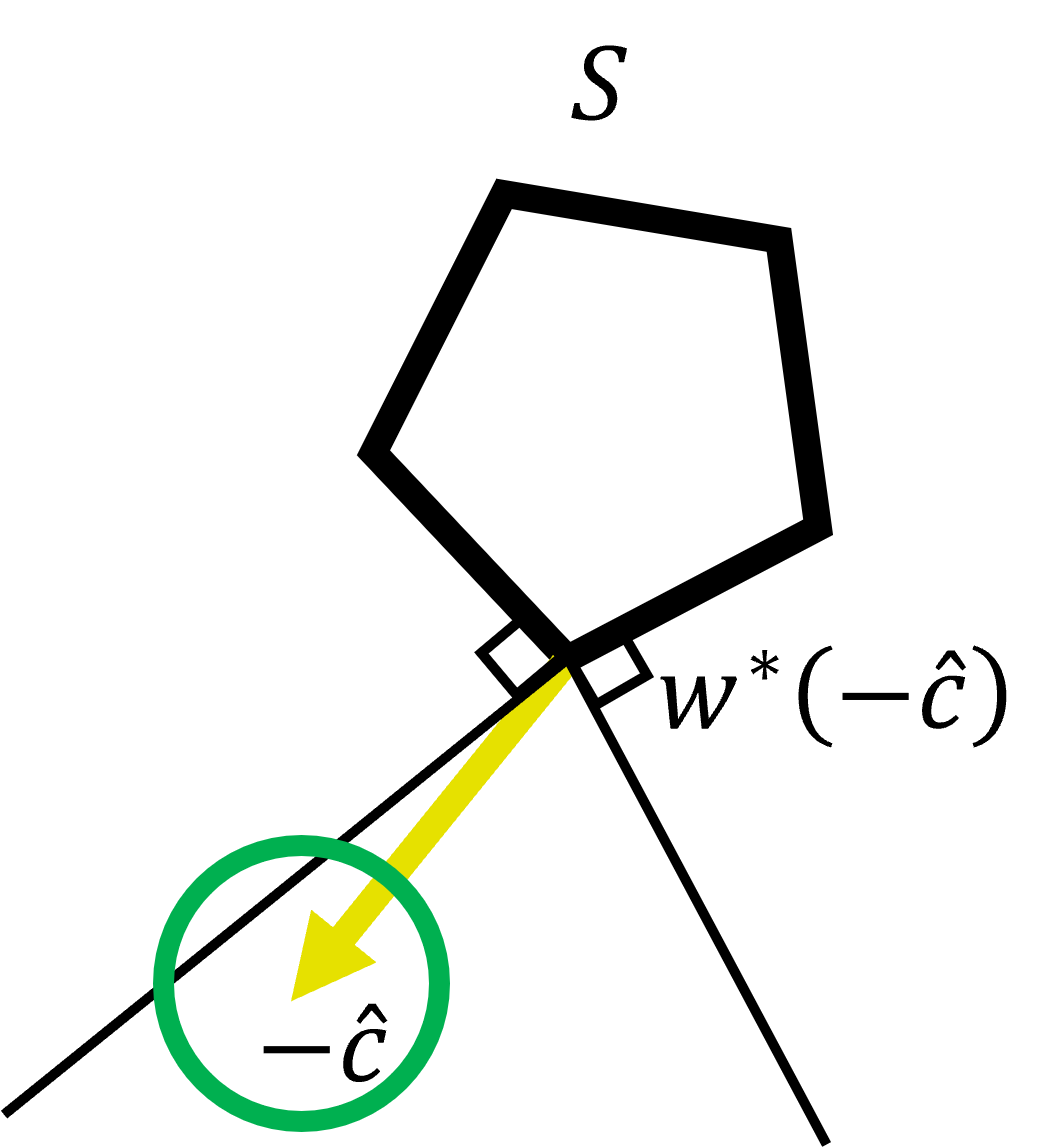}
        \caption{Prediction is close to  degenerate cost vectors}
        \label{fig:margin2a}
    \end{subfigure}\begin{subfigure}{0.5\textwidth}
    \centering
        \includegraphics[height=5cm]{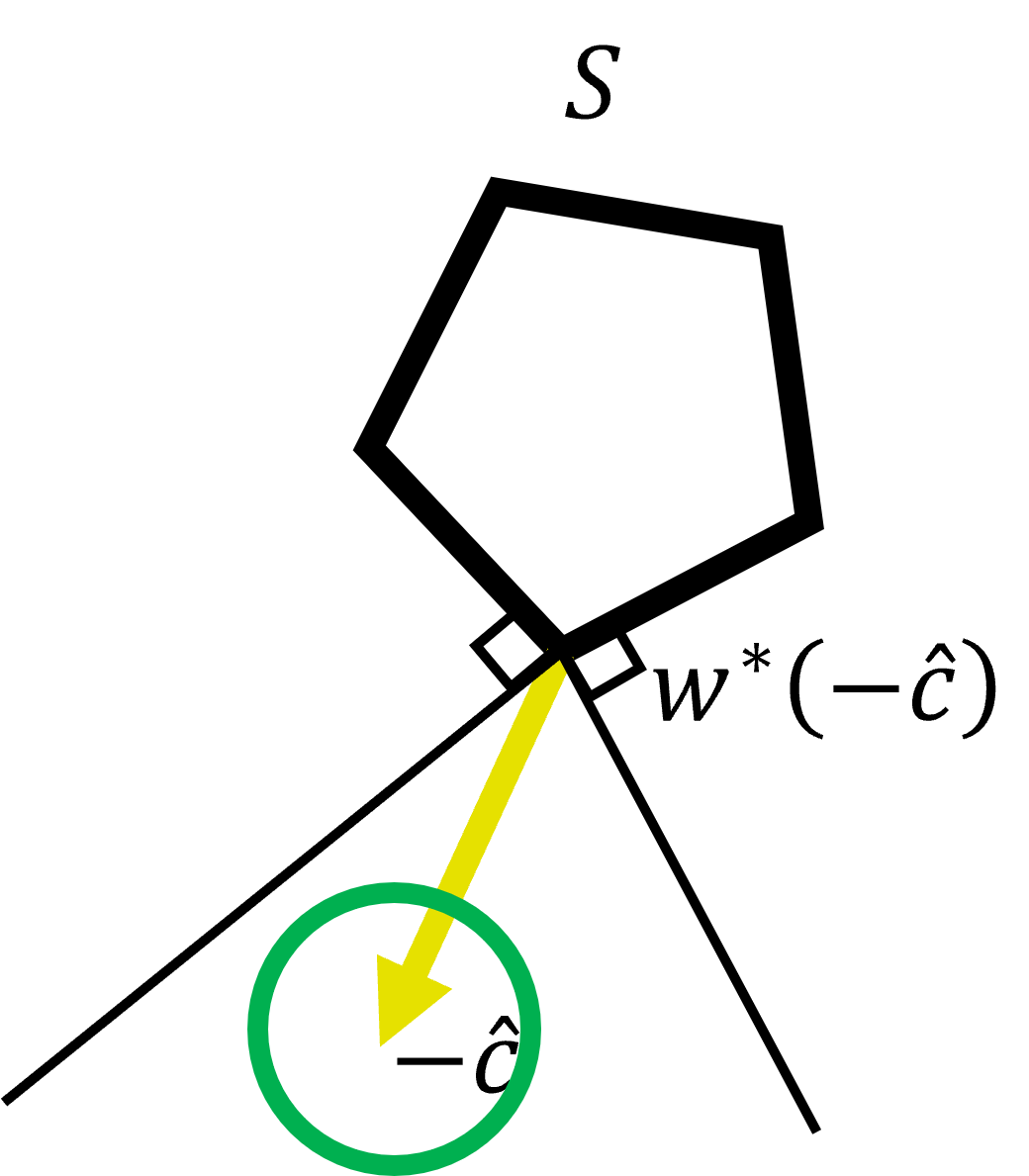}
        \caption{Prediction is far from  degenerate cost vectors}
        \label{fig:margin2b}
    \end{subfigure}\caption{Decision uncertainty for equally sized confidence regions at different distances to degeneracy.}
	\label{fig:margin2}
\end{figure}

The above two examples highlight that the confidence associated with the decision $w^*(\hat c)$ is important in determining whether it is valuable to acquire the true label $c$. This confidence depends on both the size of the confidence region and the location of the prediction relative to the structure of $S$. \citet{el2022generalization} introduced the notion of distance to degeneracy, which measures the distance of a prediction to the set of cost vectors with multiple optimal solutions. This distance provides a natural criterion for deciding whether a label should be acquired.

\begin{definition}\label{def:distancedegen} (Distance to Degeneracy, \cite{el2022generalization}). The set of degenerate cost vector predictions is $\mathcal C^o:=\{\hat c\in\bbR^d:P(\hat c)\text{ has multiple optimal solutions}\}$. Given a norm $\|\cdot\|$ on $\bbR^d$, the distance to degeneracy of a prediction $\hat c$ is $\nu_S(\hat c):=\inf_{c\in\mathcal C^o}\|c-\hat c\|$.
\hfill\Halmos
\end{definition}

The distance to degeneracy can be computed efficiently when the extreme points of the polyhedral feasible region are known. In particular, if $S=\operatorname{conv}(v_1,\ldots,v_K)$, Theorem 8 of \citet{el2022generalization} shows that
\begin{align}\label{equ:thm8}
    \nu_S(c)=\min_{j:v_j\neq w^*(c)}\{c^T(v_j-w^*(c))/\|v_j-w^*(c)\|_*\}.
\end{align}
The numerator in \eqref{equ:thm8} is the objective difference between the current optimal extreme point and an alternative extreme point, while the denominator converts this objective difference into a distance under the selected norm. Therefore, a larger distance to degeneracy $\nu_S(\hat c)$ indicates that a larger change in the predicted cost vector $c$ is required to change the optimal decision.

\subsection{MBAL-SPO Algorithm}

The distance to degeneracy determines which feature samples are labeled. Algorithm~\ref{alg:margin-based} combines this acquisition rule with batched model updates. We first collect $n_0\geq1$ fully labeled observations and obtain an initial predictor $h_0$ by empirical risk minimization. This warm-up period provides an initial estimate before selective acquisition begins; such initialization is also used in other active learning methods \citep{chen2023active}.

After warm-up, batch $k\geq1$ contains the $m_k:=2^{k-1}$ feature samples arriving at iterations $m_k,\ldots,2m_k-1$ from the distribution $\mathcal{D}_{\mathcal{X}}$. We set $m_0=n_0$ for the warm-up sample size. Within each batch, the predictor, rejection threshold, and hypothesis class remain fixed. We acquire the cost vector $c_t$ only when $\nu_S(h_{t-1}(x_t))<b_{t-1}$. The working set $W_t$ contains the labels acquired in the current batch and is reset at the start of the next batch. The rejection threshold $b_{t}$ controls acquisition, while the confidence-set slack $s_t$ protects the population-optimal predictor from exclusion.

The first batch uses the input threshold $b_0$. At the end of batch $k$, i.e., at iteration $2m_k-1$, we minimize the current empirical loss over the retained hypothesis class $\mathcal{H}_{t-1}$. The empirical loss is divided by the full batch size $m_k$, so rejected observations contribute zero.  

At a batch endpoint, the confidence-set slack decreases at order $\widetilde{\mathcal O}(m_k^{-1/2})$, while the rejection threshold decreases at order $\widetilde{\mathcal O}(m_k^{-1/4})$. The first transition from $b_0$ to $b_1$ need not decrease the threshold; subsequent thresholds decrease with the batch size. The updated confidence set $\mathcal{H}_t$ then retains the predictors whose empirical losses are within $s_t$ of the minimum. Thus, earlier observations enter future updates through the nested confidence sets, rather than a pooled empirical objective. If a batch acquires no labels, its empirical loss is zero for every predictor, and its confidence set is unchanged.
The parameter schedules depend on a confidence input $\delta\in(0,1)$ and scales $s_0,b_0>0$. Specifically, we use
\begin{equation}\label{eq:mbal-surrogate-schedules}
 s_{2m_k-1}=s_0\sqrt{\frac{\ln(16m_k^2/\delta)}{m_k}},\qquad
 b_{2m_k-1}=b_0\left\{\frac{\ln(16(2m_k)^2/\delta)}{2m_k}\right\}^{1/4}.
\end{equation}

\begin{algorithm}[!t]
\caption{Margin-Based Active Learning for SPO (MBAL-SPO)}
\label{alg:margin-based}
\begin{algorithmic}[1]
\STATE \textbf{Input:} Hypothesis class $\mathcal H$, training loss $\ell$, stream horizon $T$, warm-up size $n_0$, scales $b_0,s_0>0$, and confidence level $\delta\in(0,1)$.
\STATE Collect $W_0\gets\{(x_i,c_i)\}_{i=1}^{n_0}$, set $\mathcal H_0\gets\mathcal H$, and choose $h_0\in\arg\min_{h\in\mathcal H}n_0^{-1}\sum_{(x,c)\in W_0}\ell(h(x),c)$.
\FOR{$t=1,2,\ldots,T$}
    \STATE Set $k\gets1+\lfloor\log_2t\rfloor$ and $m_k\gets2^{k-1}$. If $t=m_k$, reset $W_t\gets\emptyset$; otherwise set $W_t\gets W_{t-1}$.
    \STATE Draw $x_t\sim\mathcal D_{\mathcal X}$. If $\nu_S(h_{t-1}(x_t))<b_{t-1}$, acquire $c_t$ and set $W_t\gets W_t\cup\{(x_t,c_t)\}$.
    \IF{$t=2m_k-1$}
        \STATE Set $\widehat\ell_t(h)\gets m_k^{-1}\sum_{(x,c)\in W_t}\ell(h(x),c)$ and choose $\widetilde h_t\in\arg\min_{h\in\mathcal H_{t-1}}\widehat\ell_t(h)$.
        \STATE Set $s_t\gets s_0\sqrt{\ln(16m_k^2/\delta)/m_k}$.
        \STATE Set $b_t\gets b_0\{\ln(16(2m_k)^2/\delta)/(2m_k)\}^{1/4}$.
        \STATE Set $\mathcal H_t\gets\{h\in\mathcal H_{t-1}:\widehat\ell_t(h)\leq\widehat\ell_t(\widetilde h_t)+s_t\}$ and $h_t\gets\widetilde h_t$.
    \ELSE
        \STATE Set $h_t\gets h_{t-1}$, $\mathcal H_t\gets\mathcal H_{t-1}$, and $b_t\gets b_{t-1}$.
    \ENDIF
\ENDFOR
\STATE \textbf{Return} $h_T$.
\end{algorithmic}
\end{algorithm}

Conditional on the history before a batch, the acquisition rule is fixed and the feature--cost observations are i.i.d. This permits concentration around the selected population objective. The theoretical analysis will show that the selected observations have enough information to identify the best prediction model under the full feature distribution. 

Algorithm~\ref{alg:margin-based} performs $\mathcal O(\log T)$ model updates over $T$ feature arrivals. Note that the last completed batch contains at most $(T+1)/2$ feature arrivals; its fitting objective does not use every previously acquired label. The computational cost depends on empirical minimization, historical confidence constraints, and distance-to-degeneracy evaluation. Appendix~\ref{app:adjacent-distance} shows that the minimum in \eqref{equ:thm8} can be restricted to vertices adjacent to the unique
optimal vertex and discusses computation from an optimal basis. 

\section{Guarantees and Analysis for the Margin-Based Algorithm}\label{sec:analysis}

Throughout the theoretical analysis, we assume Line 7 in Algorithm~\ref{alg:margin-based} selects a global empirical-risk minimizer over the current hypothesis class as a measurable function of the observed history.
We first relate prediction error to stability of the induced decision. We then use this relation to analyze direct SPO-loss minimization, followed by linear prediction with squared loss and other convex surrogate losses. These analyses address different statistical questions. For the SPO loss, decision stability makes loss differences vanish on rejected observations. For a surrogate loss, selected observations must also contain enough information to control risk under the original feature distribution.

The geometric illustrations in Section~\ref{sec:illustration_mbal} suggest that a prediction is reliable when its distance to degeneracy exceeds its possible prediction error. The following lemma formalizes this observation.
\begin{lemma}[Conditions for identical decisions]\label{lemma:identical}
Consider the linear optimization problem $P(c)$ over a nonempty bounded polyhedral feasible region $S$, with distance to degeneracy defined using $\|\cdot\|$. For any $c_1,c_2\in\bbR^d$, if $\|c_1-c_2\|<\max\{\nu_S(c_1),\nu_S(c_2)\}$, then $w^*(c_1)=w^*(c_2)$.
\end{lemma}
Lemma \ref{lemma:identical} implies that given one prediction of the cost vector, when its distance to degeneracy is larger than the radius of its confidence region, then all the predictions within this confidence region will lead to the same decision. Moreover, if the optimal prediction is also within this confidence region, the excess SPO loss of this prediction is zero.

To apply the lemma to a learned predictor, we need an error bound relating excess risk to prediction error. Write $\mathcal H^*:=\arg\min_{h\in\mathcal H}R_{\mathrm{SPO}}(h)$ and $\mathcal H_\ell^*:=\arg\min_{h\in\mathcal H}R_\ell(h)$ for the optimal predictor sets, with minimum risks $R_{\mathrm{SPO}}^*$ and $R_\ell^*$, respectively. For SPO+, we use $\mathcal H_{\mathrm{SPO}+}^*$ and $R_{\mathrm{SPO}+}^*$. Define $\|h\|_\infty:=\operatorname*{ess\,sup}_{x\in\mathcal X}\|h(x)\|$ under $\mathcal D_{\mathcal X}$ and $\mathrm{Dist}_{\mathcal G'}(h):=\inf_{g\in\mathcal G'}\|h-g\|_\infty$. The prediction norm is the same norm used in the distance to degeneracy.

We further define the conditional SPO risk on the selected feature space and retained confidence set of prediction models. For a measurable set $\mathcal A\subseteq\mathcal X$ with $\bbP(X\in\mathcal A)>0$, define $R_{\ell,\mathcal A}(h):=\bbE[\ell(h(X),C)\mid X\in\mathcal A]$ and $R_{\mathrm{SPO},\mathcal A}(h):=\bbE[\ell_{\mathrm{SPO}}(h(X),C)\mid X\in\mathcal A]$. For a nonempty compact predictor set $\mathcal G\subseteq\mathcal H$, let $R_{\ell,\mathcal A}^{\mathcal G,*}:=\min_{h\in\mathcal G}R_{\ell,\mathcal A}(h)$ and $\mathcal G_{\ell,\mathcal A}^*:=\arg\min_{h\in\mathcal G}R_{\ell,\mathcal A}(h)$. Define $R_{\mathrm{SPO},\mathcal A}^{\mathcal G,*}$ and $\mathcal G_{\mathrm{SPO},\mathcal A}^*$ analogously. Finally, set $\|h\|_{\infty,\mathcal A}:=\operatorname*{ess\,sup}_{x\in\mathcal A}\|h(x)\|$ and $\mathrm{Dist}_{\mathcal G',\mathcal A}(h):=\inf_{g\in\mathcal G'}\|h-g\|_{\infty,\mathcal A}$.
\begin{assumption}[Conditional consistency and error bound]\label{assumption:1}
The hypothesis class $\mathcal H$ is nonempty and compact in the
sup norm. The training loss $\ell:\bbR^d\times\bbR^d\to\bbR_+$
satisfies:
\begin{enumerate}[label={\rm (\arabic*)},ref={\theassumption.(\arabic*)}]
\item\label{assu:consistent}
For every measurable $\mathcal A\subseteq\mathcal X$ with
$\bbP(X\in\mathcal A)>0$, the set $\mathcal H_{\ell,\mathcal A}^*$
is nonempty and satisfies
$\mathcal H_{\ell,\mathcal A}^*
\subseteq\mathcal H_{\mathrm{SPO},\mathcal A}^*$.

\item\label{assumption:upper-bound-pointwise}
There is a nondecreasing function $\phi:\bbR_+\to\bbR_+$,
with $\phi(0)=0$, such that, for every nonempty compact
$\mathcal G\subseteq\mathcal H$ with
$\mathcal G\cap\mathcal H_\ell^*\neq\emptyset$,
every $h\in\mathcal G$, and every $\epsilon>0$,
$R_\ell(h)-R_\ell^*\leq\epsilon$ implies
$\mathrm{Dist}_{\mathcal G\cap\mathcal H_\ell^*}(h)
\leq\phi(\epsilon)$.
\end{enumerate}
\end{assumption}

Assumption~\ref{assu:consistent} concerns conditional surrogate
minimizers, whereas Assumption~\ref{assumption:upper-bound-pointwise}
uses full-population risk, uniformly over the stated predictor
restrictions. We focus on
$\phi(\epsilon)\leq C_\phi\sqrt{\epsilon}$.
For SPO training, distinct optima at positive sup distance then
imply a positive excess-risk gap for nonoptimal predictors. Example~\ref{example:SPO_ideal} verifies Assumption~\ref{assumption:upper-bound-pointwise} for the direct SPO loss minimization when the feature distribution has finite support. Appendix~\ref{appendix:Sec4} further gives a continuous example for direct SPO training. For squared loss and other convex surrogates, Sections \ref{sec:general_loss} -- \ref{sec:alpha-two-spop} instead use the global error bound defined in \eqref{eq:linear-global-error-bound}. For related calibration results linking excess surrogate risk to excess SPO risk, see \cite{ho2020risk}.

\subsection{Near-Degeneracy Function}\label{sec:near_degeneracy}

Decision stability explains when a label can be rejected. To bound how often labels are acquired, we also need the distribution of optimal predictions relative to degeneracy. Fix a training loss $\ell$ with nonempty $\mathcal H_\ell^*$ and select a population minimizer $h^*\in\mathcal H_\ell^*$ using a fixed population-level rule. Under Assumption~\ref{assu:consistent}, this predictor is also SPO-optimal.
\begin{definition}[Near-degeneracy function]\label{def:near}
Given an optimal predictor $h^*\in\mathcal H_\ell^*$, the near-degeneracy function is $\Psi(b):=\bbP(\nu_S(h^*(X))\leq b)$ for $b\geq0$.
\hfill\Halmos
\end{definition}
The function $\Psi$ describes the margins of this fixed predictor. Other decision-equivalent predictors can have different margins and do not enter the definition. The reference $h^*$ is used only in the analysis and need not be known to the algorithm.

If the current predictor is within the rejection threshold of $h^*$, any acquired feature is also near degeneracy under $h^*$. The Lipschitz property of distance to degeneracy gives the following label bound.
\begin{lemma}[Upper bound on the expected number of acquired labels]\label{lemma:label_c}
Suppose Algorithm \ref{alg:margin-based} uses $n_0$ warm-up labels. Suppose $\|h_t-h^*\|_\infty\leq b_t$ almost surely for every $t\geq0$, where $h^*$ is the reference in Definition~\ref{def:near}. If $N_T$ counts all labels acquired through stream iteration $T$, including warm-up, then
 $\bbE[N_T]-n_0\leq\bbE\!\left[\sum_{t=1}^T\Psi(2b_{t-1})\right]$.
\end{lemma}
The factor two within $\Psi(2b_{t-1})$ accounts for the acquisition threshold and the possible prediction error. The comparison must be with the same reference $h^*$. This geometric argument does not require Assumption~\ref{assumption:1} or the information condition introduced below.

\subsection{Natarajan Dimension}\label{sec:natarajan}

For direct SPO minimization, predictors with the same induced decision have the same loss. To analyze the concentration bound, we therefore measure the complexity of the decision class rather than uniformly approximate losses over the label space. Let $K:=|\mathfrak S|$ and $\Pi_{\mathcal H}:=\{x\mapsto w^*(h(x)):h\in\mathcal H\}$,
where the fixed oracle selects an extreme point. Thus $\Pi_{\mathcal H}$ takes values in the finite set $\mathfrak S$.

\begin{definition}[Natarajan dimension]\label{def:natarajan}
A decision class $\Pi$ Natarajan-shatters distinct features $x_1,\ldots,x_q$ if there are decisions $v_i\neq v_i'$ for each $i$ such that, for every subset $I\subseteq\{1,\ldots,q\}$, some $\pi\in\Pi$ satisfies $\pi(x_i)=v_i$ for $i\in I$ and $\pi(x_i)=v_i'$ otherwise. Its Natarajan dimension $d_{\mathrm N}(\Pi)$ is the largest such $q$, or infinity if no finite maximum exists. We write $d_{\mathrm N}:=d_{\mathrm N}(\Pi_{\mathcal H})$.
\hfill\Halmos
\end{definition}

Following \citet[Theorem 2 and Corollary 2]{el2022generalization}, finite $d_{\mathrm N}$ controls the number of decision vectors on a feature sample. For linear cost predictors with $p$ features and $d$ responses, their bound gives $d_{\mathrm N}\leq pd$; including an intercept by feature augmentation gives $d_{\mathrm N}\leq d(p+1)$. We use the same decision-counting argument with the conditional sub-Gaussian noise assumption. Appendix~\ref{appendix:Sec4} proves the resulting concentration bound for loss differences, without requiring bounded labels or a finite uniform loss cover. 

\subsection{Theoretical Guarantee with the SPO Loss}\label{sec:spo_loss}

We first analyze direct SPO-loss minimization over a general hypothesis class, not restricted to a linear prediction class. Run Algorithm~\ref{alg:margin-based} with $\ell=\ell_{\mathrm{SPO}}$. The fourth-root threshold in \eqref{eq:mbal-surrogate-schedules} matches the square root of the confidence-set slack, as required by the error bound $\phi(\epsilon)\leq C_\phi\sqrt\epsilon$. We use the batch endpoints and sizes defined in Section~\ref{sec:margin}. For a stream length $T\geq1$, write $\bar T:=2^{\lfloor\log_2(T+1)\rfloor}-1$ for the last completed endpoint. Because updates occur only at endpoints, $h_T=h_{\bar T}$. In Definition~\ref{def:near}, use this training loss and fix any $h^*\in\mathcal H^*$.
To compare MBAL-SPO with supervised learning, we use the following uniform risk bound.

\begin{definition}[Uniform i.i.d. risk bound]
Let $B_{\rm ERM}(m,\zeta)\geq0$ be deterministic and nonincreasing in $m$. With probability at least $1-\zeta$ over $m$ fully labeled i.i.d. observations,
$ \rspo(h)-\rspo^*\leq B_{\rm ERM}(m,\zeta)$
simultaneously for all $h\in\mathcal H$ whose empirical SPO risk is no larger than that of $h^*$.
\end{definition}

The formal condition on $B_{\rm ERM}(m,\zeta)$ is provided in \eqref{eq:supervised-erm-certificate} of Appendix~\ref{appendix:Sec4}.

\begin{theorem}[Risk and label guarantees under the SPO loss]\label{thm:spo_loss_directly}
Consider the MBAL-SPO algorithm with training loss $\ell=\ell_{\mathrm{SPO}}$. Suppose Assumption~\ref{assumption:1} holds, and $\phi(\epsilon)\leq C_\phi\sqrt\epsilon$. Assume the conditional mean is uniformly bounded, and the induced decision class has finite Natarajan dimension $d_{\mathrm N}$ as in Definition~\ref{def:natarajan}.
Fix a horizon $T\geq3$ before warm-up. For every $n_0\geq1$ and $\delta\in(0,1)$, there exist deterministic scales $s_0,b_0>0$, such that Algorithm~\ref{alg:margin-based} with confidence input $\delta/T$ satisfies:
\begin{enumerate}[label={\rm (\alph*)}]
\item With probability at least $1-\delta$,
\[
 R_{\mathrm{SPO}}(h_T)-R_{\mathrm{SPO}}^*
 \leq\min\left\{
 s_{\bar T},\Psi(2b_T)\omega_S(\mu),
 B_{\rm ERM}(T/4,\delta/2)
 \right\}.
\]
\item The total acquired-label count satisfies
\[
 \bbE[N_T]-n_0
 \leq\sum_{t=1}^T
 \sup_{g\in\mathcal H^*}
 \bbP\bigl(\nu_S(g(X))\leq2b_{t-1}\bigr)
 +\delta.
\]
\end{enumerate}
\end{theorem}

In Theorem \ref{thm:spo_loss_directly}, the scales $s_0$ and $b_0$
depend only on $n_0,\delta/T,C_\phi,\sigma,d_{\mathrm N}$, $\operatorname*{ess\,sup}_x\|\mu(x)\|_2$,
$D_{S}$ and $|\mathfrak S|$. 
Appendix~\ref{appendix:Sec4} provides sufficient scales and the proof. Part~(a) combines a $\widetilde{\mathcal O}(T^{-1/2})$ concentration bound, a near-degeneracy bound, and the uniform i.i.d. bound $B_{\rm ERM}(T/4,\delta/2)$ from supervised learning. The third follows from decision stability: on the high-probability event, either all retained predictors are optimal or their induced decisions coincide on rejected features. In the latter case, the selected empirical excess loss relative to $h^*$ equals its full-batch counterpart, allowing MBAL-SPO to inherit supervised risk bounds satisfying the uniform empirical-risk condition. Thus, MBAL-SPO preserves the corresponding supervised-learning risk order at a common feature horizon while potentially acquiring far fewer labels. Part~(b) bounds the expected number of labels acquired through time $T$, with a smaller bound when optimal predictions are less likely to lie near degeneracy. Proposition~\ref{prop:sublinear_spo} in Section~\ref{sec:smalllabel} establishes sublinear expected-label rates under the margin condition.
The following example verifies Assumption~\ref{assumption:1}
under finite feature support; a continuous example appears in
Appendix~\ref{appendix:Sec4}.

\begin{example}[An error bound for the SPO loss]\label{example:SPO_ideal}
Suppose $S$ is a nonempty bounded polyhedron, the optimization oracle selects an extreme-point optimum, $\mathcal H$ is nonempty and compact in the sup norm, and $\bbE\|C\|<\infty$. Suppose further that the class contains $h^*(x)=\bbE[c\mid x]$, the feature vector has finite support, and $w^*(\bbE[c\mid x])$ is unique at every support point. Then the SPO loss satisfies Assumption~\ref{assumption:1}, with $\phi(\epsilon)=k\sqrt\epsilon$ for some finite $k>0$.
\end{example}

\subsection{Linear Prediction with Squared Loss}\label{sec:general_loss}

MBAL-SPO allows surrogate losses for computational tractability. We first analyze its theoretical guarantees using the exact quadratic structure of squared loss, then extend the argument to other convex surrogate losses. Both analyses use Algorithm~\ref{alg:margin-based} with the schedules in \eqref{eq:mbal-surrogate-schedules}. Because surrogate and SPO risks differ, samples rejected far from degeneracy may still contribute to excess surrogate risk. We therefore focus on linear prediction models and characterize the information provided by selected samples to establish theoretical guarantees. Throughout Sections \ref{sec:general_loss}--\ref{sec:alpha-two-spop} and their proofs, we use the Euclidean norm for predictions and distance to degeneracy.

\paragraph{Linear prediction model setup.}
We consider the case where the feature vector $X\in\bbR^{1\times p}$ has a continuous distribution with $\Sigma_X:=\bbE[X^TX]\succ0$, and set the minimum eigenvalue $\eta:=\lambda_{\min}(\Sigma_X)$. Consider $h_\Theta(X)=X\Theta$ and $C=X\Theta^*+\xi$, where the centered noise satisfies \eqref{eq:subgaussian-cost-noise} with parameter $\sigma$. The coefficient set $\mathcal K$ is nonempty, compact, and convex, contains the Frobenius ball of radius $r_\Theta>0$ around $\Theta^*$, and permits all $pd$ local coefficient directions. We assume a bounded prediction range.

Take $\ell(\hat c,c)=\|\hat c-c\|_2^2$ and use $h^*=h_{\Theta^*}$ as the reference in Definition~\ref{def:near}. The excess risk is $R_\ell(h_\Theta)-R_\ell^*=\|\Sigma_X^{1/2}(\Theta-\Theta^*)\|_F^2$. For the linear-surrogate results, we require conditional consistency in Assumption~\ref{assu:consistent} and the global error bound
\begin{equation}\label{eq:linear-global-error-bound}
 \|h_\Theta-h^*\|_\infty\leq C_\phi\sqrt{R_\ell(h_\Theta)-R_\ell^*},\qquad \Theta\in\mathcal K.
\end{equation}
When the surrogate minimizer is unique, the square-root form of
Assumption~\ref{assumption:upper-bound-pointwise} is equivalent
to \eqref{eq:linear-global-error-bound}. The surrogate proofs do not require its uniform error bound on every feature and predictor restriction.
For $g\in\mathcal H$ and $r>0$, define $\mathcal A_{g,r}:=\{x:\nu_S(g(x))<r\}$. These are the feature regions selected by a fixed predictor and threshold.
\begin{assumption}[Quantitative local information under feature selection]\label{assumption:holder}
In the preceding linear model, there exist $0<\kappa<4$, $0<r_\circ\leq1$, $r_{\rm loc}>0$, $0<c_-\leq c_+<\infty$, and a positive function $a(r)$ on $(0,r_\circ]$ such that, for every $g\in\mathcal H$ with $\|g-h^*\|_\infty\leq r_{\rm loc}$ and every $0<r\leq r_\circ$, the set $\mathcal A=\mathcal A_{g,r}$ satisfies
\begin{equation}\label{eq:finite-local-information}
 c_-r^\kappa\leq\bbP(X\in\mathcal A)\leq c_+r^\kappa,\qquad
 \bbE[X^TX\mid X\in\mathcal A]\succeq\frac{a(r)}{\eta}\Sigma_X.
\end{equation}
The function $a(r)$ is nondecreasing as $r$ decreases; write $a:=a(r_\circ)>0$. The noise satisfies \eqref{eq:subgaussian-cost-noise} with the same parameter under feature restriction.
\end{assumption}

The probability bounds in \eqref{eq:finite-local-information} control how often features fall in the near-degeneracy set $\mathcal A_{g,r}$. The matrix inequality controls their information in every coefficient direction relative to the original feature distribution. Section \ref{sec:phi_examples} provides concrete examples for these conditions. The ratio $a(r)/\eta$ measures an information gain that is unchanged by feature rescaling. The eigenvalue $\eta$ is defined from the full feature distribution, not imposed as a lower bound for every selected distribution. The condition is local in $g$ and uniform over the stated neighborhood.

\paragraph{Finite risk and label guarantees.}
If $Z_1,\ldots,Z_j$ are independent standard normal variables, then $\sum_{i=1}^jZ_i^2$ has the chi-squared distribution with $j$ degrees of freedom. Write $\chi^2_{j,1-v}$ for its $1-v$ quantile. We also use $\Phi$ for the standard normal distribution function and $\varphi_\sigma(z):=(\sigma\sqrt{2\pi})^{-1}\exp\{-z^2/(2\sigma^2)\}$ for the centered Gaussian density with variance $\sigma^2$. In the surrogate guarantees, $0<\delta<1/8$ is the target failure probability, and Algorithm~\ref{alg:margin-based} uses confidence input $\delta/64$. Define
\begin{equation}\label{eq:squared-complexity-constants}
 C_{\rm MBAL}:=\frac{121}{81}\frac{c_+}{c_-}\frac{\eta}{a}
 \frac{2^{1-\kappa/4}}{1-2^{-(1-\kappa/4)}}\,
 \sigma^2\left\{pd+2\sqrt{pd\ln(2/\delta)}+2\ln(2/\delta)\right\}.
\end{equation}
When $\xi\sim\mathcal N_d(0,\sigma^2I_d)$ is independent of $X$, we use the sharper value of $C_{\rm MBAL}$ obtained by replacing the braces with $\chi^2_{pd,1-\delta/2}$. The theorem below holds with the corresponding value in either case.

\begin{theorem}[Finite label guarantee for squared loss]\label{uniform_margin}
Assume the preceding linear model, Assumption~\ref{assu:consistent}, the global error bound
\eqref{eq:linear-global-error-bound}, and Assumption~\ref{assumption:holder}. For every $n_0\geq1$ and $0<\delta<1/8$, there exist deterministic scales $s_0,b_0>0$, a finite initialization budget $N_{\rm init}\geq n_0$, and $\epsilon_{\rm A}>0$ with the following property. For every $0<\epsilon\leq\epsilon_{\rm A}$, there is a deterministic completed-batch horizon $2m_k-1$ at which
\[
 \bbP\!\left(h^*\in\mathcal H_{2m_k-1},\quad
 R_\ell(h_{2m_k-1})-R_\ell^*\leq\epsilon,\quad
 N_{2m_k-1}\leq N_{\rm init}+\left\lceil\frac{C_{\rm MBAL}}{\epsilon}\right\rceil\right)
 \geq1-\delta.
\]
\end{theorem}
In Theorem \ref{uniform_margin}, the scales, initialization budget $N_{\rm init}$, and accuracy cutoff $\epsilon_{\rm A}$ depend only on the dimension, sub-Gaussian parameter, warm-up size, confidence, prediction bound, global error-bound constant, global feature information, and the constants in Assumption~\ref{assumption:holder}; they do not depend on the realized stream or target accuracy. Appendix~\ref{app:information-alpha-two} specifies these quantities and the horizon in Theorem~\ref{uniform_margin}. Theorem~\ref{uniform_margin} bounds excess surrogate risk and the cumulative number of acquired labels for sufficiently small $\epsilon$. Section~\ref{sec:smalllabel} identifies when a sufficiently large information gain $a/\eta$ makes this sufficient label budget smaller than the supervised requirement at the same target risk and confidence.

\subsection{Extension to Convex Surrogate Losses}\label{sec:alpha-two-spop}

We extend the squared-loss analysis to convex surrogate losses using population curvature and subgradient variability. This allows nonsmooth sample losses such as SPO+. Retain the linear model, sub-Gaussian cost noise, Assumption~\ref{assu:consistent}, and the global error bound \eqref{eq:linear-global-error-bound}. Let $\theta=\operatorname{vec}(\Theta)\in\bbR^{pd}$ and let $v_*(X,C)\in\partial_\theta\ell(h_{\Theta^*}(X),C)$ be a fixed measurable subgradient. Define
\begin{equation}\label{eq:convex-hessian-variance}
 \begin{aligned}
 H_{\mathcal A}&:=\left.\nabla_\theta^2R_{\ell,\mathcal A}(h_\Theta)\right|_{\Theta=\Theta^*},
 &H&:=H_{\mathcal X},\\
 V_{\mathcal A}&:=\operatorname{Cov}(v_*(X,C)\mid X\in\mathcal A),
 &V&:=V_{\mathcal X}.
 \end{aligned}
\end{equation}
The Hessian is taken after expectation. The following assumption controls label acquisition and the information provided by selected observations.

\begin{assumption}[Local acquisition and surrogate regularity]
\label{assumption:surrogate-information}
Features are continuously distributed with $\|X\|_2\leq B_X$, and predictions are bounded on the initial compact convex coefficient class. There exist $0<\kappa<4$, $0<r_\circ\leq1$, $r_{\rm loc}>0$, and $0<c_-\leq c_+<\infty$ such that the following hold for every $g\in\mathcal H$ with $\|g-h^*\|_\infty\leq r_{\rm loc}$:
\begin{itemize}[leftmargin=*,nosep]
\item \emph{Acquisition probability.} For $0<r\leq r_\circ$,
$c_-r^\kappa\leq\bbP(X\in\mathcal A_{g,r})\leq c_+r^\kappa$.

\item \emph{Population curvature.} Each sample loss is finite and convex on a neighborhood of the coefficient class. Under the full feature law and every selected law $\mathcal A=\mathcal A_{g,r}$ with $r>0$, population risks are twice continuously differentiable and satisfy
$h_0I_{pd}\preceq\nabla_\theta^2R_{\ell,\mathcal A}(h_\Theta)\preceq\overline hI_{pd}$
throughout the class, with common constants $h_0,\overline h>0$. Their Hessians are operator-Lipschitz with common constant $L_H\geq0$.

\item \emph{Subgradient conditions.} The chosen subgradient satisfies $\bbE[v_*\mid X]=0$. Under the same laws, common constants $v_0>0$ and $M_3<\infty$ satisfy
$V_{\mathcal A}\succeq v_0I_{pd}$ and
$\bbE[\|V_{\mathcal A}^{-1/2}v_*\|_2^3\mid X\in\mathcal A]\leq M_3$.
The empirical subgradient bounds
\eqref{eq:convex-global-score}--\eqref{eq:convex-local-score}
in Appendix~\ref{app:spop-finite-cutoffs} hold with common constants.

\item \emph{Selected information.} For the selected laws with $0<r\leq r_\circ$, a common finite constant $J_\ell$ satisfies
\begin{equation}\label{eq:spop-information-comparison}
 \begin{aligned}
 &\lambda_{\max}(H_{\mathcal A}^{-1/2}HH_{\mathcal A}^{-1/2})
 \lambda_{\max}(H_{\mathcal A}^{-1/2}V_{\mathcal A}H_{\mathcal A}^{-1/2})
 \leq J_\ell,\\
 &v_{\rm SL}:=\lambda_{\min}(H^{-1/2}VH^{-1/2})>0.
 \end{aligned}
\end{equation}
\end{itemize}
\end{assumption}

Conditional centering and convexity make $h^*$ a common population minimizer under feature selection and give $V_{\mathcal A}=\bbE[v_*v_*^T\mid X\in\mathcal A]$. 
A smaller $J_\ell$ gives a smaller risk coefficient per acquired label.
The ratio $v_{\rm SL}/J_\ell$ is a conservative counterpart of $a/\eta$.
For squared loss under Assumption~\ref{assumption:holder} and independent
isotropic Gaussian noise, $v_{\rm SL}=2\sigma^2$ and one may take
$J_\ell=2\sigma^2\eta/a$.
Define the acquired-label coefficient
$C_{\rm MBAL,\ell}:=\frac{c_+}{c_-}
\frac{2^{1-\kappa/4}}{1-2^{-(1-\kappa/4)}}
J_\ell\chi^2_{pd,1-\delta/2}$.

\begin{theorem}[Finite label guarantee for convex surrogate losses]
\label{thm:alpha-two-spop-finite}
Assume the preceding model, Assumption~\ref{assu:consistent},
the global error bound \eqref{eq:linear-global-error-bound},
and Assumption~\ref{assumption:surrogate-information}.
For every $n_0\geq1$ and $0<\delta<1/8$, there exist deterministic
scales $s_0,b_0>0$, a finite initialization budget $N_{\rm init}\geq n_0$,
and $\epsilon_{\rm A,\ell}>0$ such that, for every
$0<\epsilon\leq\epsilon_{\rm A,\ell}$,
Algorithm~\ref{alg:margin-based}, with confidence input $\delta/64$,
has a deterministic completed-batch horizon $2m_k-1$ satisfying
\[
 \bbP\!\left(R_\ell(h_{2m_k-1})-R_\ell^*\leq\epsilon,\quad
 N_{2m_k-1}\leq N_{\rm init}
 +\left\lceil\frac{C_{\rm MBAL,\ell}}{\epsilon}\right\rceil\right)
 \geq1-\delta.
\]
\end{theorem}

The scales, initialization budget, and accuracy cutoff depend only
on $n_0,\delta$ and the stated model constants, not on the realized
stream or target accuracy.
Appendix~\ref{app:spop-finite-cutoffs} gives their construction and
the proof; Appendix~\ref{app:spop-verification} verifies the conditions
for the Gaussian box SPO+ example.
Theorems~\ref{uniform_margin} and~\ref{thm:alpha-two-spop-finite}
establish excess \emph{surrogate} risk guarantees, which yield downstream
SPO-risk guarantees under the calibration conditions of \cite{ho2020risk}.
The next section identifies conditions and examples giving lower label
complexity than supervised learning for the same training-risk guarantee.

\section{Smaller Label Complexity under Margin Conditions}\label{sec:smalllabel}

In this section, we identify conditions under which MBAL-SPO acquires fewer yet sufficiently informative labels, achieving lower label complexity than supervised learning. We first establish label rates for direct SPO-loss minimization, then consider surrogate losses.

\paragraph{Direct SPO-loss label rates.}
The two explicit risk bounds in Theorem~\ref{thm:spo_loss_directly}(a) concern the same loss. Combining them with the margin condition gives the following sufficient rate. The risk comparison uses the fixed reference $h^*$. The label bound uses the uniform margin condition stated below.

\begin{assumption}[Margin condition]\label{assumption:noise}
There exist constants $\bar b>0$ and $\kappa>0$ such that $\sup_{g\in\mathcal H^*}
\bbP(\nu_S(g(X))\leq r)
\leq (r/\bar b)^\kappa$, for any $r\geq 0$.
\end{assumption}

A larger margin exponent reduces this upper bound at small radii. Related low-noise conditions in decision-focused learning are discussed by \citet{hu2020fast,hu2024fast}. The scale $\bar b$ describes the population margins and is distinct from the algorithmic scale $b_0$. 

Assumption~\ref{assumption:noise} controls margins uniformly over SPO-risk minimizers. The surrogate analyses instead use
Assumptions~\ref{assumption:holder}
and~\ref{assumption:surrogate-information}, respectively,
with a unique population minimizer ensured by their
curvature assumptions. Appendix~\ref{appendix:Sec4} discusses the role of prediction scale in this distinction. The next two examples illustrate the margin condition without asserting the remaining hypotheses of the learning guarantees.

\begin{example}[Non-degenerate expectation]\label{example:2}
In the polyhedral CLO setting, suppose the hypothesis class is well-specified and has a unique optimal SPO predictor $h^*$, with $h^*(x)=\bbE[C\mid X=x]$ almost surely. Use this predictor in Definition~\ref{def:near}. If $\nu_S(h^*(X))\geq\underline k>0$ almost surely, then $\Psi(b)=0$ for $b<\underline k$, and Assumption~\ref{assumption:noise} holds for every $\kappa>0$ with $\bar b=\underline k$.
\hfill\Halmos
\end{example}
This separation concerns the conditional mean, not the realized noisy cost. It excludes the positive small-radius acquisition lower bound in Assumption~\ref{assumption:holder}. Appendix~\ref{appendix:nondegenerate-alternative} therefore analyzes it separately. Uniqueness is needed there for the predictor-identification argument; the statement about $\Psi$ alone remains valid when a separated predictor is selected from several optima.

\begin{example}[Induced general distribution]\label{example:3}
Consider a bounded polyhedral feasible region with at least two extreme points, and suppose $\mathcal H^*=\{h^*\}$. Use this predictor in Definition~\ref{def:near}. Suppose $U$ has a continuous density on the closed unit ball in $\bbR^d$. For $\kappa>1$, define $Y=0$ when $\nu_S(U)=0$ and $Y=[\nu_S(U)]^{1/\kappa-1}U$ otherwise. If $h^*(X)$ has the same distribution as $Y$, then Assumption~\ref{assumption:noise} holds with exponent $\kappa$ and some $\bar b>0$.
\hfill\Halmos
\end{example}
Positive homogeneity gives $\nu_S(Y)=\nu_S(U)^{1/\kappa}$ away from degeneracy. This converts a near-degeneracy probability of order at most $b$ into one of order at most $b^\kappa$. 

\begin{proposition}[Sublinear label acquisition under the SPO loss]\label{prop:sublinear_spo}
Use the setting and scales of Theorem~\ref{thm:spo_loss_directly}.
Fix a horizon $T\geq3$ before warm-up and $\delta\in(0,1)$. Suppose Assumption~\ref{assumption:noise} holds with exponent $\kappa>0$. With probability at least $1-\delta$, the excess SPO risk at $T$ is at most $\widetilde{\mathcal O}(T^{\min\{-1/2,-\kappa/4\}})$. The unconditional expected label count satisfies
\[
 \bbE[N_T]\leq
 \begin{cases}
 n_0+\widetilde{\mathcal O}(T^{1-\kappa/4})+\delta,&0<\kappa<4,\\
 n_0+\widetilde{\mathcal O}(1)+\delta,&\kappa\geq4.
 \end{cases}
\]
\end{proposition}
Proposition~\ref{prop:sublinear_spo} gives asymptotic bounds on excess SPO risk and the expected number of acquired labels. For $T\geq3$, the excess risk is also bounded by $B_{\rm ERM}(T/4,\delta/2)$ for any uniform supervised bound satisfying the conditions of Theorem~\ref{thm:spo_loss_directly}(a). MBAL-SPO attains this supervised risk bound with a sublinear expected label count: for fixed $n_0$ and $\delta$, $\bbE[N_T]/T\to0$ as the prescribed horizon grows. For $\kappa\geq4$, the expected-label bound is polylogarithmic in $T$.

\paragraph{Same-risk comparison for squared loss.}
For a specified loss and supervised empirical-minimizer rule, let $\widehat\Theta_n^{\rm SL}$ be fitted to $n$ fully labeled i.i.d. observations from the original distribution. Define its smallest deterministic label requirement by
\[
 n_{\rm SL}(\epsilon,\delta):=\min\{n\geq1:
 \bbP(R_\ell(h_{\widehat\Theta_n^{\rm SL}})-R_\ell^*\leq\epsilon)\geq1-\delta\}.
\]
This definition concerns the actual estimator, not a sufficient sample size inferred from an upper bound. For the following squared-loss comparison, assume $\xi\sim\mathcal N_d(0,\sigma^2I_d)$ independently of $X$, and use the sharper Gaussian value of $C_{\rm MBAL}$. Let the supervised estimator be ordinary constrained least squares over the same coefficient class. There is $\epsilon_{\rm S}>0$, depending only on the design bound, dimension, noise variance, $\eta,r_\Theta$, and $\delta$, such that, for every $0<\epsilon\leq\epsilon_{\rm S}$,
\begin{equation}\label{eq:supervised-label-requirement}
 \begin{aligned}
 \frac{\sigma^2\chi^2_{pd,1-2\delta}}{\epsilon}<n_{\rm SL}(\epsilon,\delta)
 &\leq\left\lceil\frac{\sigma^2\chi^2_{pd,1-\delta/2}}{\epsilon}\right\rceil.
 \end{aligned}
\end{equation}

The Gaussian assumption is used for these exact finite and asymptotic comparisons. The next proposition compares the Gaussian requirement with the corresponding active budget.

\begin{proposition}[Smaller label complexity than supervised least squares]\label{prop:alpha-two-sl-advantage}
Use the setting and scales of Theorem~\ref{uniform_margin} with the Gaussian noise and coefficient specified above. If $C_{\rm MBAL}<\sigma^2\chi^2_{pd,1-2\delta}$, there is $\epsilon_*>0$, depending additionally on the active initialization and information constants, such that, for every $0<\epsilon\leq\epsilon_*$, $
 N_{\rm init}+\left\lceil\frac{C_{\rm MBAL}}{\epsilon}\right\rceil
 <n_{\rm SL}(\epsilon,\delta)$.
\end{proposition}

The condition $C_{\rm MBAL}<\sigma^2\chi^2_{pd,1-2\delta}$ requires the information gain $a/\eta$ to offset the acquisition and batching factors in \eqref{eq:squared-complexity-constants}. Section \ref{sec:phi_examples} provides concrete examples for this condition. Appendix~\ref{app:information-alpha-two} gives the expanded information inequality, explicit accuracy cutoffs, and proof, including rank-deficient supervised samples. 

\paragraph{Same-risk comparison for convex surrogate losses.}
For a general convex surrogate, supervised estimation depends on both $H$ and $V$. With $Z\sim\mathcal N_{pd}(0,I_{pd})$, the score expansion gives
\begin{equation}\label{eq:convex-sl-asymptotic}
 n\{R_\ell(h_{\widehat\Theta_n^{\rm SL}})-R_\ell^*\}
 \ \Rightarrow\ \tfrac12Z^TH^{-1/2}VH^{-1/2}Z.
\end{equation}
To compare the minimum label requirement over \emph{all} positive sample sizes, we also need to control small samples. Assume the specified empirical-minimizer rule has a constant $D_{\rm ac}>0$ such that, for every $n\geq1$ and $0<\epsilon\leq h_0r_\Theta^2/8$,
\begin{equation}\label{eq:convex-small-sample}
 \bbP(R_\ell(h_{\widehat\Theta_n^{\rm SL}})-R_\ell^*\leq\epsilon)
 \leq\min\{1,nD_{\rm ac}\sqrt\epsilon\}.
\end{equation}
This condition excludes a positive atom at the true parameter arising from the empirical tie rule. Appendix~\ref{app:spop-verification} verifies it for the fixed lexicographic vertex rule in the Gaussian box SPO+ example. Together with the curvature, moment, and score conditions of Theorem~\ref{thm:alpha-two-spop-finite}, this gives $\epsilon_{\rm S,\ell}>0$, depending only on the dimension, confidence, coefficient neighborhood, curvature, moment, and score bounds, and $D_{\rm ac}$, such that, for every $0<\epsilon\leq\epsilon_{\rm S,\ell}$,
\begin{equation}\label{eq:convex-sl-finite-lower}
 n_{\rm SL}(\epsilon,\delta)>
 \frac{v_{\rm SL}\chi^2_{pd,1-4\delta}}{3\epsilon}.
\end{equation}

\begin{proposition}[Smaller label complexity with convex surrogate losses]\label{prop:alpha-two-spop-comparison}
Use the setting and scales of Theorem~\ref{thm:alpha-two-spop-finite}, and suppose the supervised empirical-minimizer rule satisfies \eqref{eq:convex-small-sample}. 
If $C_{\rm MBAL,\ell}<v_{\rm SL}\chi^2_{pd,1-4\delta}/3$, there exists $\epsilon_{*,\ell}>0$, depending additionally on the active initialization and information constants, such that, for every $0<\epsilon\leq\epsilon_{*,\ell}$, $
 N_{\rm init}+\left\lceil\frac{C_{\rm MBAL,\ell}}{\epsilon}\right\rceil
 <n_{\rm SL}(\epsilon,\delta)$.
\end{proposition}

As in Proposition~\ref{prop:alpha-two-sl-advantage}, the sufficient condition for lower label complexity requires a large ratio $v_{\rm SL}/J_\ell$, as illustrated in Section~\ref{sec:phi_examples}. Appendix~\ref{app:spop-finite-cutoffs} gives the explicit inequalities and cutoffs.

\subsection{Examples for Squared Loss and SPO+}\label{sec:phi_examples}

We provide examples satisfying the sufficient conditions for lower label complexity in Propositions~\ref{prop:alpha-two-sl-advantage} and~\ref{prop:alpha-two-spop-comparison}. For squared loss, the large information gain $a/\eta$ required by our
sufficient condition implies that an intercept in the linear model cannot be freely estimated.
Our examples therefore use linear predictors without an intercept. Nevertheless, the numerical results in Section~\ref{sec:experiments} show that MBAL-SPO can achieve lower SPO regret than supervised learning at matched label budgets even when the theoretical information gain is small.

\paragraph{A continuous design with two informative components.}
Let $B_1,B_2$ be independent uniform variables on $[1/40,1/20]$, let $A\sim\operatorname{Unif}[1,11/10]$, and let $U\sim\operatorname{Unif}[-1/4,1/4]$. Let $E,E_1,E_2$ be independent signs taking values $-1$ and $1$ with equal probability, independently of these variables. Define
\begin{equation}\label{eq:continuous-information-feature-law}
 X\sim\begin{cases}
 (E_1B_1,E_2B_2),&\text{with probability }199/200,\\
 (EA,U),&\text{with probability }1/400,\\
 (U,EA),&\text{with probability }1/400.
 \end{cases}
\end{equation}
Each component has a two-dimensional density. The first component has small feature magnitudes, whereas the other two contain larger features near the two coordinate axes. Use $S=[-1,1]^2$ and
\begin{equation}\label{eq:continuous-matrix-model}
 \begin{aligned}
 h_\Theta(X)&=X\Theta,\qquad
 \mathcal K=\{\Theta\in\bbR^{2\times2}:\|\Theta-I_2\|_{\max}\leq1/10\},\\
 C&=X\Theta^*+\xi,\qquad \Theta^*=I_2,\qquad
 \xi\sim\mathcal N_2(0,\sigma^2I_2),\quad \xi\perp X,
 \end{aligned}
\end{equation}
where $\|\cdot\|_{\max}$ is the largest absolute entry. The coefficient box is specified before observing data, and no equality between fitted entries is imposed. At the displayed true matrix, the fixed reference is $h^*(X)=X$, and acquisition near either axis can reveal a different coefficient direction. The selection bounds hold for every $\Theta\in\mathcal K$, not just for the true reference. The finite-budget arguments below also hold when the unknown true matrix ranges over $\|\Theta^*-I_2\|_{\max}\leq1/20$, with the same initial class and algorithmic scales.

Appendix~\ref{app:matrix-information-examples} verifies conditional consistency, the global error bound, and Assumption~\ref{assumption:holder} with $\kappa=1$. The sufficient information gain is $a/\eta=92.0568$, while $c_+/c_-=11/9$. In particular, the construction does not rely on canceling a known offset. 

\begin{example}[Squared loss without a fixed offset]\label{example:continuous-information-squared}
Use \eqref{eq:continuous-information-feature-law}--\eqref{eq:continuous-matrix-model} and squared loss. The model satisfies the statistical assumptions of Theorem~\ref{uniform_margin}, and its information gain satisfies the sufficient comparison in Proposition~\ref{prop:alpha-two-sl-advantage}. For a separate explicit finite-budget certificate, use $n_0=32$, $b_0=0.05$, $s_0=70\sigma^2$, and confidence input $0.05/256$ in Algorithm~\ref{alg:margin-based}. The verification over the whole initial coefficient class removes the need for a fully labeled localization phase after warm-up. At the deterministic horizons in Table~\ref{tab:finite-kappa-comparison}, Appendix~\ref{app:matrix-information-examples} proves jointly, with probability at least $0.95$, that the excess squared risk is at most the displayed $\epsilon$ and that the total acquired-label count is at most the displayed budget. Each budget is strictly smaller than the actual deterministic label requirement of constrained supervised least squares at the same risk and confidence.
\end{example}

\begin{table}[htbp]
\centering\small
\setlength{\tabcolsep}{5pt}
\begin{tabular}{rrrrrr}
\toprule
$\sigma$ & $\epsilon$ & Batch $k$ & Active budget & SL lower bound & Label ratio bound\\
\midrule
$0.1$ & $10^{-6}$ & $28$ & $10{,}225$ & $19{,}962$ & $0.513$\\
$0.1$ & $10^{-7}$ & $32$ & $75{,}338$ & $199{,}620$ & $0.378$\\
$1$ & $10^{-5}$ & $32$ & $75{,}338$ & $199{,}620$ & $0.378$\\
$1$ & $10^{-6}$ & $36$ & $606{,}439$ & $1{,}996{,}205$ & $0.304$\\
\bottomrule
\end{tabular}
\caption{Finite total-label comparisons for Example~\ref{example:continuous-information-squared}, with $95\%$ confidence and all $32$ warm-up labels included. The active budget is a sufficient upper bound, and the supervised lower bound is rounded down. Label-ratio bounds are rounded upward. The horizon is $T=2^k-1$; these are deterministic probability bounds. Lower is better.}
\label{tab:finite-kappa-comparison}
\end{table}

\begin{example}[SPO+ with controlled prediction curvature]\label{example:continuous-information-spop}
Use the same feature law and coefficient box, with Gaussian noise and SPO+ for $S=[-1,1]^2$. The conditional expected loss is minimized at $h^*(X)$. The reflected mean in its prediction Hessian satisfies $|2h_\Theta(X)_j-h^*(X)_j|\leq1.44$ throughout the initial box, including the stated neighborhood of true matrices. Consequently, $8\varphi_\sigma(1.44)$ is a uniform lower bound on prediction curvature. Appendices~\ref{app:spop-verification}--\ref{app:matrix-information-examples} verify the population and score conditions, together with the conditions on the supervised estimator, and Table~\ref{tab:spop-finite-ratios} evaluates the leading comparison at $\Theta^*=I_2$. The leading ratios compare the $1/\epsilon$ coefficients in the
sufficient active budget and the supervised lower bound.

The leading label-ratio bounds in Table~\ref{tab:spop-finite-ratios} are below one, so Proposition~\ref{prop:alpha-two-spop-comparison} gives a smaller sufficient total label budget for MBAL-SPO when $\epsilon$ is sufficiently small. Another separate finite comparison appears in
Appendix~\ref{app:spop-finite-label-comparison}. For that no-intercept
Gaussian family, with $n_0=32$, $b_0=30$, $s_0=200$, confidence input
$0.05/64$, and $\sigma=1$, MBAL-SPO attains excess SPO+ risk at most
$2\times10^{-5}$ with at most $54{,}461$ labels, jointly with
probability at least $0.95$, at $T=2^{29}-1$. This guarantee holds
for every true coefficient in the stated family and every exact
empirical-minimizer tie rule. In contrast, any procedure based on
fully labeled i.i.d. observations requires more than $2{,}600{,}000$
labels to attain the same risk and confidence uniformly over that
family. Therefore, the sufficient active budget in the example in Appendix~\ref{app:spop-finite-label-comparison} is less than $2.1\%$
of the uniform supervised requirement. \end{example}

\begin{table}[htbp]
\centering\small
\setlength{\tabcolsep}{8pt}
\begin{tabular}{rrrr}
\toprule
$\sigma$ & Curvature lower bound $h_0$ & $v_{\rm SL}/J_\ell$ & Leading label-ratio bound\\
\midrule
$1.0$ & $1.4125\times10^{-4}$ & $37.2044$ & $0.761$\\
$1.5$ & $1.6752\times10^{-4}$ & $54.1463$ & $0.523$\\
$2.0$ & $1.5370\times10^{-4}$ & $63.6146$ & $0.445$\\
\bottomrule
\end{tabular}
\caption{Population certificates for Example~\ref{example:continuous-information-spop}, with $\kappa=1$, $pd=4$, and $95\%$ confidence. The curvature lower bounds are rounded down and apply throughout the coefficient box under the full and specified selected feature laws. The leading ratios for Proposition~\ref{prop:alpha-two-spop-comparison} are rounded upward. Noise variance is fixed within each row. Lower is better.}
\label{tab:spop-finite-ratios}
\end{table}

In the above examples, these explicit guarantees involve late batch indices and large feature horizons, but they are sufficient bounds rather than minimum requirements. The numerical results in Section~\ref{sec:experiments} show that MBAL-SPO can achieve lower SPO regret at much smaller matched label budgets in the tested instances.

\captionsetup{font+={color=blue}}

\section{Numerical Experiments}\label{sec:experiments}

We examine the batched MBAL-SPO algorithm in shortest path and personalized pricing problems with synthetic data. We compare its downstream performance with supervised learning, given the same number of acquired labels, under the SPO+, squared, mean absolute error (MAE), and Huber surrogate losses. Each method pair uses the same hypothesis class and training loss.  We use the Euclidean norm and compute distance to degeneracy by enumerating feasible decisions, using \eqref{equ:thm8}.

\paragraph{Evaluation protocol.}
Let $\mu_c(x):=\bbE[c\mid x]$ denote the conditional mean cost. Since this function is known in the synthetic experiments, we evaluate the SPO regret
\begin{equation}\label{eq:numerical_spo_regret}
\mathcal R_{\mathrm{reg}}(h)
:=\bbE_X\!\left[\mu_c(X)^T\bigl(w^*(h(X))-w^*(\mu_c(X))\bigr)\right].
\end{equation}
This quantity is the population SPO risk minus its unrestricted Bayes minimum. It also equals excess SPO risk relative to the initial hypothesis class when that class contains a predictor inducing Bayes-optimal decisions, as in both shortest path settings below. We do not presume this equality for the misspecified pricing class. The training losses differ across comparisons, but every plotted regret and regret ratio uses \eqref{eq:numerical_spo_regret}. Full realized-cost SPO risks are reported separately in Appendix~\ref{appendix:numerical_experiments_results}.
Both learners share the warm-up observations and test set within each trial. Supervised learning fits a fully labeled i.i.d. prefix with exactly the number of labels acquired by MBAL-SPO, including warm-up.

The main comparisons use thirty paired trials per setting and training loss. The shortest path problems are evaluated at fixed acquired-label checkpoints, using the most recently completed update; pricing is evaluated at completed batches. Connecting lines are guides between evaluations. Label axes are linear, regret axes are logarithmic, and shading gives pointwise 90\% intervals from 10,000 bootstrap resamples of whole trial pairs.  Regret ratios divide arithmetic mean supervised-learning SPO regret by arithmetic mean MBAL-SPO regret; a value above one favors MBAL-SPO.

\subsection{Shortest Path Problem}\label{sec:shortest_path}

We consider $3\times3$ and $5\times5$ grid networks with edges directed north or east, and seek a path from the southwest corner to the northeast corner. The networks have 12 and 40 edges, and six and 70 feasible paths, respectively. Every path contains four edges in the first grid and eight in the second. Both settings use five features and affine conditional mean costs, with a strictly interior conditional-mean coefficient matrix in the fitted class.

The feature distributions combine regions with well-separated optimal decisions and components near optimal-path boundaries. For the $3\times3$ grid, a Gaussian mixture puts 97\% of the feature mass in three easy components and 3\% in six boundary components. Its independent coordinate noise is
\begin{equation}\label{eq:numerical_grid3_cost}
c_i=\mu_{c,i}(X)(1+\epsilon V_i),\qquad
V_i\overset{\mathrm{iid}}{\sim}\operatorname{Unif}[-1,1],\quad \epsilon=0.01.
\end{equation}
For the $5\times5$ grid, a bounded feature mixture puts 85\% of the mass in a small easy component and 15\% around 32 diverse optimal-path boundaries. Its costs have independent additive Gaussian noise:
\begin{equation}\label{eq:numerical_grid5_cost}
c_i=\mu_{c,i}(X)+\xi_i,\qquad
\xi_i\overset{\mathrm{iid}}{\sim}\mathcal N(0,0.02^2).
\end{equation}
The noise variables are independent of the features. The larger-grid construction gives boundary observations greater feature variation while preserving the full affine prediction class. Appendix~\ref{appendix:independent_shortest_path} specifies the distributions and coefficient matrices. Because the feature laws, conditional means, and noise laws differ, the comparison does not isolate the effect of grid size.

With SPO+ training, after tuning the parameters, the $3\times3$ setting uses $n_0=100$, $b_0=0.631675$, and $s_0=0.05$, and the $5\times5$ setting uses $n_0=200$, $b_0=0.208528$, and $s_0=0.025$. Each trial has 100,000 and 50,000 independent test observations, respectively. Both panels in Figure~\ref{fig:experiments} use the same range from zero to 5,000 acquired labels. We evaluate at warm-up, every 250 labels through 3,000, and at 4,000 and 5,000, without changing the completed-batch update rule. The generators and scales are frozen before thirty final streams, disjoint from their respective development samples.

At 3,000 acquired labels, the mean SPO regrets of MBAL-SPO and supervised learning are $2.259\times10^{-6}$ and $1.411\times10^{-4}$, respectively, for the $3\times3$ grid, yielding a supervised-to-MBAL regret ratio of 62.447. For the $5\times5$ grid, the corresponding mean regrets are $2.262\times10^{-4}$ and $4.185\times10^{-4}$, with a ratio of 1.850. At 5,000 labels, the regret ratio is approximately 108.419 for the $3\times3$ grid and is 1.664 for the $5\times5$ grid. These comparisons demonstrate that MBAL-SPO achieves lower SPO regret than supervised learning with the same cumulative number of acquired labels.

\begin{figure}[ht]
\centering
\begin{subfigure}[t]{0.49\columnwidth}
\centering
\includegraphics[width=\linewidth]{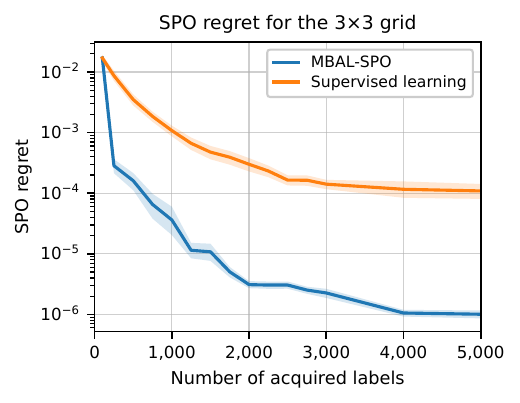}
\caption{$3\times3$ grid}
\label{fig:experiments_grid3}
\end{subfigure}\hfill
\begin{subfigure}[t]{0.49\columnwidth}
\centering
\includegraphics[width=\linewidth]{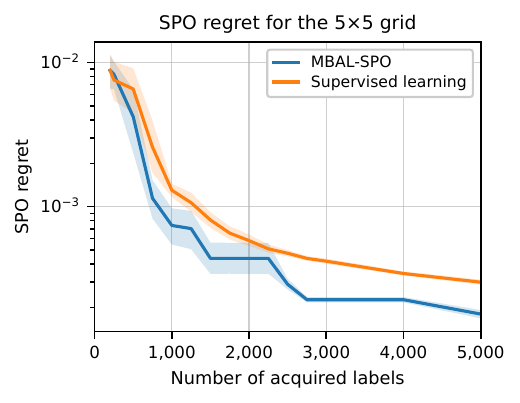}
\caption{$5\times5$ grid}
\label{fig:experiments_grid5}
\end{subfigure}
\caption{SPO regret for the shortest path problems under SPO+ training. Each panel shows means over 30 paired trials and pointwise 90\% bootstrap intervals, using the latest completed-batch predictor at warm-up, every 250 labels through 3,000, and at 4,000 and 5,000. Warm-up sizes are 100 and 200; acquired-label counts include warm-up and match exactly within each method pair. Both settings use the fixed fourth-root threshold and square-root slack. Panel (a) uses Gaussian features and independent multiplicative coordinate noise; panel (b) uses bounded features and independent additive Gaussian noise.}
\label{fig:experiments}
\end{figure}

These comparisons highlight the advantage of MBAL-SPO despite its batched model updates. Unlike supervised learning, which fits its model using all available labels at each evaluation point, MBAL-SPO updates its predictor only at the end of each batch. Consequently, evaluations conducted within a batch may place MBAL-SPO at a disadvantage, as some acquired labels have not yet been incorporated into its current predictor. For example, at 5,000 labels, an average of 1,774.1 and 1,126.5 acquired labels are still awaiting the next update in the $3\times3$ and $5\times5$ grids, respectively. Nevertheless, MBAL-SPO achieves lower mean SPO regret than supervised learning in both grids at 3,000 and 5,000 labels, demonstrating its advantage at larger cumulative label budgets. The main and appendix figures retain all evaluated checkpoints from the same final cohorts.

\subsection{Personalized Pricing Problem}\label{sec:personalizedpricing}

We next consider three items and candidate prices of \$60, \$80, and \$90. A six-dimensional customer feature vector determines demand intensities for all nine item--price combinations. A label reveals the cost vector of negative revenues, so minimizing the linear objective maximizes revenue. We select one price per item and impose $p^1\leq p^2\leq p^3$, leaving ten feasible decisions.

We use an exponential demand with independent uniform multiplicative noise in each item--price coordinate. The fitted predictors are affine with a signed coefficient box, so this experiment tests performance under \textit{misspecification}. Demand intensities need not lie in $[0,1]$ and are not interpreted as purchase probabilities. Appendix~\ref{appendix:pricing} specifies the Gaussian mixture, demand matrices, and clipping. A finite exponential-demand candidate comparison selects the original demand generator.

We use 200 warm-up labels, a post-warm-up horizon of $T=2047$, and 1,000 independent test observations per trial. The loss-specific hyperparameter settings are reported in Table~\ref{tab:pricing_parameters} in Appendix \ref{appendix:pricing}. Figure~\ref{fig:pricing} compares the SPO regret of MBAL-SPO and supervised learning under the SPO+ training loss across thirty independent final trials. At the last completed batch, MBAL-SPO acquires an average of 1,050.5 labels, approximately 47\% of the 2,247 available observations. At the same acquired-label budget, MBAL-SPO achieves a mean SPO regret of 0.01764, compared with 0.16048 for supervised learning, yielding a regret ratio of 9.10. Thus, MBAL-SPO acquires fewer than half of the available labels while achieving approximately 11\% of the mean SPO regret of supervised learning trained with the same number of labels.

\begin{figure}[ht]
\centering
\includegraphics[width=0.6\columnwidth]{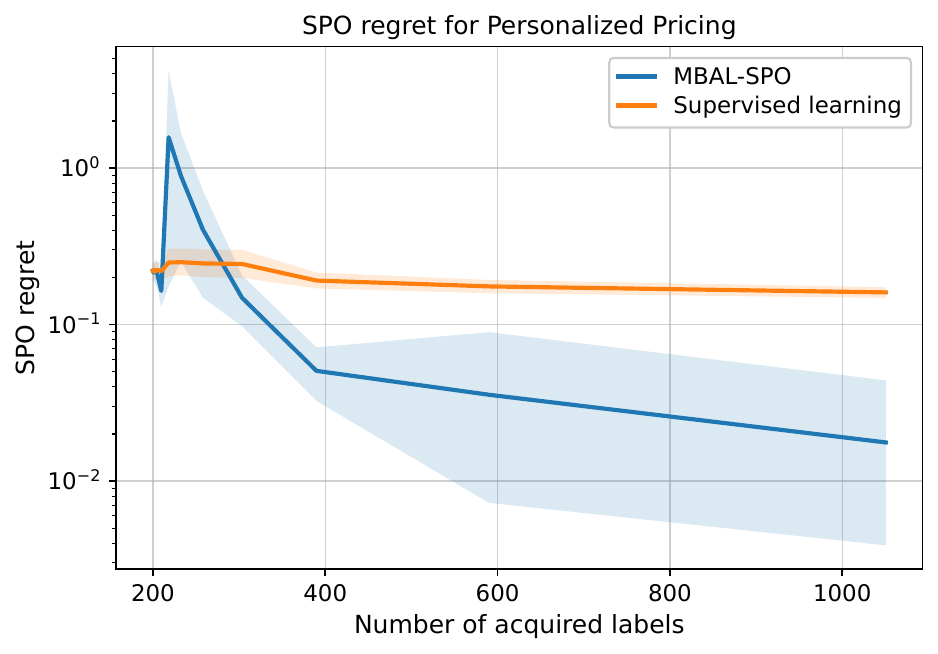}
\caption{SPO regret for personalized pricing under SPO+ training. Curves show means over 30 paired trials at warm-up and completed-batch endpoints $2m_j-1$, $j=1,\ldots,11$; shaded regions give pointwise 90\% bootstrap intervals. The horizontal coordinate is the mean acquired-label count, including 200 warm-up labels, with exactly matched counts within each method pair. The vertical axis is logarithmic.}
\label{fig:pricing}
\end{figure}

\subsection{Comparisons of Different Surrogate Loss Functions}

Table~\ref{tab:ratios} compares the four training losses using the same $3\times3$ and $5\times5$ shortest path settings in Section~\ref{sec:shortest_path} and the exponential-demand pricing setting in Section~\ref{sec:personalizedpricing}. Within each setting, all four losses share the same thirty final data streams and test observations. The feature distribution, conditional mean, cost-noise law, and fitted prediction class remain fixed across losses; initialization and numerical cost units follow the loss-specific settings in Appendix~\ref{appendix:setting_parameters} and Table~\ref{tab:pricing_parameters}. We report the grids at 5,000 acquired labels and pricing at the final completed batch $T=2047$.

\begin{table}[ht]
\centering
\small
\setlength{\tabcolsep}{5pt}
\begin{tabular}{lcccc}
\toprule
Surrogate loss & SPO+ & Squared & MAE & Huber \\
\midrule
Shortest path ($3\times3$) & \shortstack{108.419\\{\scriptsize [79.182, 145.766]}} & \shortstack{5.232\\{\scriptsize [3.778, 7.262]}} & \shortstack{4.605\\{\scriptsize [3.348, 6.347]}} & \shortstack{4.051\\{\scriptsize [3.103, 5.347]}} \\
Mean acquired labels & 5,000.0 & 5,000.0 & 5,000.0 & 5,000.0 \\
\addlinespace
Shortest path ($5\times5$) & \shortstack{1.664\\{\scriptsize [1.548, 1.789]}} & \shortstack{1.653\\{\scriptsize [1.421, 1.927]}} & \shortstack{1.837\\{\scriptsize [1.617, 2.077]}} & \shortstack{1.696\\{\scriptsize [1.446, 1.985]}} \\
Mean acquired labels & 5,000.0 & 5,000.0 & 5,000.0 & 5,000.0 \\
\addlinespace
Personalized pricing & \shortstack{9.096\\{\scriptsize [3.657, 40.637]}} & \shortstack{0.804\\{\scriptsize [0.463, 1.512]}} & \shortstack{5.055\\{\scriptsize [2.648, 13.233]}} & \shortstack{3.600\\{\scriptsize [1.848, 11.667]}} \\
Mean acquired labels & 1,050.5 & 288.0 & 408.5 & 308.9 \\
\bottomrule
\end{tabular}

\caption{SPO regret ratios under four training losses for the same data and settings used in Sections~\ref{sec:shortest_path} and~\ref{sec:personalizedpricing}. Each ratio divides mean supervised-learning SPO regret by mean MBAL-SPO regret, with a pointwise 90\% paired bootstrap interval in brackets. Each entry uses 30 pairs, evaluated at 5,000 acquired labels for both grids and at $T=2047$ for pricing. Counts include warm-up and match within each method pair; pricing counts differ across losses.}
\label{tab:ratios}
\end{table}

All four intervals at 5,000 labels favor MBAL-SPO in each shortest path setting. For squared, MAE, and Huber training, the regret ratios are 5.232, 4.605, and 4.051 in the $3\times3$ grid, and 1.653, 1.837, and 1.696 in the $5\times5$ grid. In pricing, the intervals for SPO+, MAE, and Huber training lie above one. Squared-loss training gives mean SPO regret 3.098 for MBAL-SPO and 2.491 for supervised learning, with regret ratio 0.804; this comparison does not establish an advantage for either method. The absolute pricing regrets differ substantially across training losses, as Table~\ref{tab:pricing_spo_regret} shows.
Appendix~\ref{appendix:numerical_experiments_results} gives the full four-loss trajectories and additional label budgets for the same grid data. 

\section{Conclusions and Future Directions}\label{sec:future}

We developed an active learning algorithm, MBAL-SPO, for efficient data collection in contextual linear optimization problems. MBAL-SPO uses the distance to degeneracy as a criterion for selecting or rejecting samples. We establish excess risk bounds and cumulative label bounds for MBAL-SPO under both the SPO loss and its surrogate loss. For direct SPO minimization, we transfer uniform supervised risk bounds while acquiring a sublinear expected number of labels. For surrogate minimization, we establish conditions for smaller label budgets at the same excess surrogate risk. We also provide concrete examples illustrating when these conditions hold. There are several intriguing future directions. One is to extend the results and algorithms to the case where the feasible region is a convex set. Another avenue is the incorporation of local properties, beyond the margin structure, into the active learning algorithm.  While our work focuses on batch-sized active learning, it is also worthwhile to consider pool-based active learning, where all feature vectors are revealed at once before training, in the future.

\ACKNOWLEDGMENT{PG acknowledges the support of NSF AI Institute for Advances in Optimization Award 2112533.}

\bibliography{reference}
\bibliographystyle{informs2014}

\newpage
\renewcommand{\theHsection}{A\arabic{section}}
\begin{APPENDICES}

\section{Refined Bounds for SPO+ Under Separability}\label{sec:hardseparability}

This appendix records a separate sufficient condition under which the direct SPO-loss argument extends to SPO+. It uses the common parameter schedules and a decision-stability argument rather than the information condition in Assumption~\ref{assumption:holder}. Separability controls realized cost vectors relative to optimal predictions; it is not satisfied by ordinary nondegenerate Gaussian noise. The continuous Gaussian-noise examples in the main body instead use conditional feature information and loss curvature.

Proposition \ref{prop:separable_bayes} shows that when the distance between a prediction $\bar h(x)$ and the realized cost vector is smaller than its distance to degeneracy, both the SPO and SPO+ risks are zero. The SPO+ loss generalizes the hinge and structured hinge losses, so this is the corresponding decision-focused separability property \citep{elmachtoub2022smart}.

\begin{proposition}[Zero SPO+ risk in the separable case]\label{prop:separable_bayes}
    Assume that there exists $\bar{h} \in \mathcal{H}$ and a constant $\varrho \in [0,1)$ such that $\|\bar{h}(x) - c \| \leq \varrho \nu_S(\bar{h}(x))$ with probability one over $(x,c) \sim \mathcal{D}$. It holds that $\rspop^* = \rspo^* = 0$ and $\bar{h}$ is a minimizer for both $\rspop$ and $\rspo$. 
\end{proposition}

Proposition \ref{prop:separable_bayes} intuitively illustrates the separability condition for a single predictor $\bar{h}$. When multiple optimal predictors exist, Assumption \ref{assu:bayes} formally defines the concept of separability by controlling the distance between the prediction $\bar{h}(x)$ and the realized cost vector $c$ relative to the distance to degeneracy of $\bar{h}(x)$.

\begin{assumption}[Strong separability condition]\label{assu:bayes}
The optimal set $\mathcal H_{\mathrm{SPO}+}^*$ is nonempty. There exist constants $\varrho\in[0,1)$ and $\tau\in(0,1]$ such that, on a common event of probability one over $(x,c)\sim\mathcal D$, every $h^*\in\mathcal H_{\mathrm{SPO}+}^*$ satisfies:
\begin{itemize}
    \item (1) $\|h^*(x) - c \| \leq \varrho \nu_S(h^*(x))$, and
    \item (2) $\nu_S(h^*(x))\ge \tau \left(\sup_{h' \in \calH_{\mathrm{SPO}+}^\ast}\{\nu_S(h'(x))\}\right)$. 
\end{itemize}
\end{assumption}

The second condition in Assumption \ref{assu:bayes} compares the distances to degeneracy of different surrogate-optimal predictions; it does not by itself bound distances between the predictors. For the batch result below, we use a unique optimal surrogate predictor, so this comparison is automatic with $\tau=1$. This result allows a nonconvex initial hypothesis class and uses Assumption \ref{assumption:1}; convexity of the SPO+ loss does not require convexity of that class.

For this separate bounded-cost result, we retain a uniform covering condition on the surrogate loss class. Its complexity need not be controlled by the induced decisions, since SPO+ can vary between predictions with the same decision.

\begin{definition}[Covering number]\label{def:covering-number}
For $\ell\circ\mathcal H:=\{(x,c)\mapsto\ell(h(x),c):h\in\mathcal H\}$ and $\varepsilon>0$, let $\widehat N_1(\varepsilon,\ell\circ\mathcal H)$ be the smallest cardinality of a finite set $V\subseteq\bbR^{\mathcal X\times\mathcal C}$ such that every $h\in\mathcal H$ has some $v\in V$ satisfying $\sup_{(x,c)\in\mathcal X\times\mathcal C}|v(x,c)-\ell(h(x),c)|\leq\varepsilon$. If no finite cover exists, the covering number is infinite.
\hfill\Halmos
\end{definition}

Multiplication by a fixed acquisition indicator does not increase this covering number. In Theorem~\ref{thm:sporiskcompare}, bounded cost and prediction ranges provide a finite loss bound for the uniform concentration argument.

\begin{theorem}[SPO+ loss under separability]\label{thm:sporiskcompare}
Fix $\delta\in(0,1)$ and suppose $S$ is a nonempty bounded polyhedron. Assume Assumption~\ref{assu:bayes} holds, the SPO+ minimizer $h^*$ is unique, and Assumption~\ref{assumption:1} holds with $\phi(\epsilon)\leq C_\phi\sqrt\epsilon$. Assume the cost and prediction ranges are bounded and $\ln\widehat N_1(\varepsilon,\ell_{\mathrm{SPO}+}\circ\mathcal H)=\mathcal O(\ln(1/\varepsilon))$ as $\varepsilon\downarrow0$, using Definition~\ref{def:covering-number}. Run Algorithm~\ref{alg:margin-based} with SPO+, confidence input $\delta$, and exact measurable empirical minimizers. There are deterministic scales $s_0,b_0>0$, depending on the stated class, geometry, cost and prediction bounds, confidence, and separability constants, such that, with probability at least $1-\delta$, for every $T\geq1$,
\[
 R_{\mathrm{SPO}+}(h_T)-R_{\mathrm{SPO}+}^*\leq s_{\bar T},\qquad
 R_{\mathrm{SPO}}(h_T)-R_{\mathrm{SPO}}^*\leq\omega_S(\mu)\Psi(2b_T).
\]
Moreover, $\bbE[N_T]\leq n_0+\sum_{t=1}^T\Psi(2b_{t-1})+\delta T$. If $T$ is fixed before warm-up, instead use confidence input $\delta/T$ and recalibrate the scales. For that run, the risk bounds hold through $T$ with probability at least $1-\delta/T\geq1-\delta$, and
\[
 \bbE[N_T]\leq n_0+\sum_{t=1}^T\Psi(2b_{t-1})+\delta,
\]
with thresholds from the recalibrated run.
\end{theorem}

This theorem uses a realized zero-loss property on rejected observations and does not require the local information and acquisition conditions of Assumption \ref{assumption:holder}. Its batch proof is provided in Appendix \ref{appendix:margin}; it does not treat the reset empirical objective as a cumulative objective and does not invoke an exploration parameter absent from the algorithm. If $\Psi(r)\leq(r/\bar b)^\kappa$, the known-horizon run, with the sufficient scales chosen in the proof, has the same sublinear expected-label orders as Proposition~\ref{prop:sublinear_spo}. This conclusion includes the logarithmic dependence of those scales on $T/\delta$. It does not assert finitely many unconditional acquisitions over an infinite stream at fixed confidence.

\section{Computing the Distance to Degeneracy}
\label{app:adjacent-distance}

The formula in \eqref{equ:thm8} minimizes over all competing extreme
points. We show that only vertices adjacent to the optimal vertex
are needed.

\begin{proposition}[Adjacent-vertex formula]
\label{prop:adjacent-distance}
Let $S$ be a nonempty bounded polyhedron with at least two extreme
points, and suppose $w=w^*(c)$ is its unique optimal solution.
Let $\operatorname{Adj}_S(w)$ denote the vertices joined to $w$
by an edge of $S$. Then
\begin{equation}\label{eq:adjacent-distance}
 \nu_S(c)=
 \min_{v\in\operatorname{Adj}_S(w)}
 \frac{c^T(v-w)}{\|v-w\|_*},
\end{equation}
where $\|\cdot\|_*$ is dual to the norm defining $\nu_S$.
No nondegeneracy assumption on a basic representation of $w$
is required.
\end{proposition}

\begin{proof}{\bfseries Proof of Proposition
\ref{prop:adjacent-distance}.}
The tangent cone at a polytope vertex is generated by its incident
edge directions. Thus, for every other extreme point $u$,
there are coefficients $\lambda_v\geq0$ such that
\[
 u-w=\sum_{v\in\operatorname{Adj}_S(w)}\lambda_v(v-w).
\]
Uniqueness of the optimum gives $c^T(v-w)>0$ for every adjacent
vertex. Put
\[
 r_*:=\min_{v\in\operatorname{Adj}_S(w)}
       \frac{c^T(v-w)}{\|v-w\|_*}>0.
\]
By the definition of $r_*$ and the triangle inequality,
\[
 \begin{aligned}
 c^T(u-w)
 &=\sum_{v\in\operatorname{Adj}_S(w)}\lambda_v c^T(v-w)\\
 &\geq r_*
       \sum_{v\in\operatorname{Adj}_S(w)}\lambda_v\|v-w\|_*\\
 &\geq r_*\|u-w\|_*.
 \end{aligned}
\]
Hence every ratio in \eqref{equ:thm8} is at least $r_*$.
Conversely, every adjacent vertex is included in that minimum,
so its value is at most $r_*$. Equation~\eqref{eq:adjacent-distance}
follows.
\hfill\Halmos
\end{proof}

\paragraph{Computation from an optimal basis.}
The ratio $c^Te/\|e\|_*$ is unchanged by positive scaling of an
edge direction $e$. Therefore, \eqref{eq:adjacent-distance} can
be evaluated from the incident edge directions without locating
the neighboring vertices. At a nondegenerate basic feasible
solution in a standard-form linear program, these directions
are obtained from the optimal basis by allowing each nonbasic
variable to enter. Boundedness ensures that each direction
reaches an adjacent vertex after a positive finite step.

At a primal-degenerate basic feasible solution, a single basis
need not expose every incident edge, and some pivot directions
permit only a zero-length step. The geometric formula
\eqref{eq:adjacent-distance} remains exact, but inspecting one
basis alone need not evaluate its minimum. This is distinct
from degeneracy of the cost vector: if $c$ has multiple optimal
solutions, then $\nu_S(c)=0$ by definition.

\paragraph{Implications for label acquisition.}
Minimizing over a nonempty subset of adjacent vertices gives
an upper bound on $\nu_S(c)$, which is exact if the subset contains
a minimizing neighbor. An overestimate can reject features that
the exact acquisition rule would label. Conversely, a certified
lower bound produces a conservative rule that labels at least
the features selected by the exact rule at the same predictor
and threshold. Such a modification does not automatically
preserve the label-budget guarantees. Scalable exact evaluation,
or certified approximation without exhaustive edge enumeration,
is therefore a useful computational direction, particularly
under primal degeneracy.

\section{A Continuous Linear Regression Example with Gaussian Features}
\label{appendix:continuous-gaussian-linear}

This appendix gives a concrete Gaussian-feature example for the general squared-loss error bound in Example~\ref{example:squared-loss}. It uses the common schedules in Algorithm~\ref{alg:margin-based}, with confidence input $\delta\in(0,1)$, and $b_0$ initialized by the empirical quantile $\tilde q$ of the warm-up distances. The unbounded prediction range and empirical-quantile initialization are analyzed directly, rather than covered by the bounded-prediction surrogate theorems. We first verify the conditional error bound, then analyze the acquired labels and empirical minimizer.

\paragraph{Model.}
Consider the scalar downstream problem \refstepcounter{equation}\label{equ:gaussian-linear-downstream}$\min_{w\in[-1,1]}cw$~(\theequation).
For nonzero cost $c$, its optimal decision is $w^*(c)=-\operatorname{sign}(c)$, with a fixed tie-breaking rule at zero. Its degenerate set is $\mathcal C^o=\{0\}$, and its distance to degeneracy is \refstepcounter{equation}\label{equ:gaussian-linear-distance}$\nu_S(c)=|c|$~(\theequation).
Let the feature have the Gaussian distribution $X\sim\mathcal N(1/2,1)$, with density \refstepcounter{equation}\label{equ:gaussian-linear-feature-density}$f_X(x)=(2\pi)^{-1/2}\exp\{-(x-1/2)^2/2\},\ x\in\bbR$~(\theequation).
The parameter interval is $\Theta=[1/4,3/4]$, and the prediction class is \refstepcounter{equation}\label{equ:gaussian-linear-hypothesis-class}$\mathcal H=\{h_\theta(x)=x-\theta:\theta\in\Theta\}$~(\theequation).
Each predictor is linear in the augmented feature $(x,1)^T$. Let the true parameter be \refstepcounter{equation}\label{equ:gaussian-linear-true-parameter}$\theta^*=1/2$~(\theequation).
The realized cost is generated using independent Gaussian noise: \refstepcounter{equation}\label{equ:gaussian-linear-data-model}$C=X-1/2+\xi,\ \xi\sim\mathcal N(0,\sigma^2),\ \xi\text{ independent of }X$~(\theequation).
Here $\sigma>0$. The training loss is $\ell(\hat c,c)=(\hat c-c)^2$. The feature, prediction, and cost ranges are unbounded, but their second moments are finite. Moreover, differences between predictors are constant in the feature vector, so their sup-norm distances are finite. The following analysis is specific to this interior-optimum benchmark and does not invoke the bounded-prediction setup or deterministic initialization of Theorem \ref{uniform_margin}.

For batch $k\geq1$, let $m_k=2^{k-1}$ be its number of feature samples, so its endpoint is $2m_k-1$. Let $\mathcal I_k:=\{m_k,\ldots,2m_k-1\}$ be its set of stream indices and let $D_i$ be the acquisition indicator. The number of acquired labels in batch $k$ is $n_k:=\sum_{i\in\mathcal I_k}D_i$. Let $\widehat\theta_k$ be the parameter of the returned predictor $h_{2m_k-1}$, and let $\Theta_k:=\{\theta\in\Theta:h_\theta\in\mathcal H_{2m_k-1}\}$ be the retained parameter set, with $\Theta_0=\Theta$. Write $\Phi$ for the standard normal cumulative distribution function.

\begin{proposition}[Properties of the Gaussian regression example]
\label{prop:continuous-gaussian-linear}
Consider the model in \eqref{equ:gaussian-linear-downstream}--\eqref{equ:gaussian-linear-data-model}. The following properties hold.
\begin{enumerate}[label={\rm (\alph*)}]
\item The unique squared-loss minimizer is $h_{\theta^*}$. Under every conditional feature distribution of positive probability, the excess squared-loss risk is $(\theta-\theta^*)^2$, and the sup-norm distance to $h_{\theta^*}$ is $|\theta-\theta^*|$. Taking the full feature distribution verifies
Assumption~\ref{assumption:upper-bound-pointwise} with
$\phi(\epsilon)=\sqrt\epsilon$.
The same conditional bound holds for every nonempty closed
restricted parameter interval, even when it excludes $\theta^*$.
\item The unique SPO-risk minimizer is also $h_{\theta^*}$. The near-degeneracy function is $\Psi(b)=2\Phi(b)-1$ for $b\geq0$, and it satisfies the margin condition with $\kappa=1$ and $\bar b=\sqrt{\pi/2}$. The excess SPO risk is $\sqrt{2/\pi}\{1-\exp[-(\theta-\theta^*)^2/2]\}$.
\item Run Algorithm~\ref{alg:margin-based} as specified at the start of this appendix on i.i.d. observations, with $n_0\geq1$ warm-up labels and empirical-quantile initialization. Fix $\delta\in(0,1)$, set $\delta_k:=\delta/2^{k+1}$ for $k\geq1$, and choose $s_0\geq2\sigma^2\sqrt{\ln(16/\delta)}$. With probability at least $1-\delta$, the true parameter belongs to every retained interval $\Theta_k$. On the same event, every batch with $n_k>0$ satisfies
\begin{equation}\label{equ:gaussian-linear-active-rates}
R_\ell(h_{\widehat\theta_k})-R_\ell(h_{\theta^*})
\leq\frac{2\sigma^2\ln(2/\delta_k)}{n_k},\qquad
R_{\mathrm{SPO}}(h_{\widehat\theta_k})-R_{\mathrm{SPO}}^*
\leq\sqrt{\frac2\pi}\frac{\sigma^2\ln(2/\delta_k)}{n_k}.
\end{equation}
If $n_k=0$, the retained interval is unchanged; no smaller prediction-error bound is asserted for an arbitrary empirical minimizer on that empty batch.
\item For fixed confidence input $\delta\in(0,1)$ and any initial slack $s_0\geq0$, let $N_T$ count all labels acquired in this Gaussian benchmark, including warm-up. Then its unconditional expectation satisfies
\begin{equation}\label{equ:gaussian-linear-label-bound}
\bbE[N_T]-n_0
\leq\sqrt{\frac2\pi}\sum_{t=1}^T\bbE[b_{t-1}]
=\mathcal O\!\left(\bbE[b_0]T^{3/4}(1+\ln T)^{1/4}\right).
\end{equation}
Here $\bbE[b_0]<\infty$ for empirical-quantile initialization with a quantile level in $(0,1]$. This bound does not require conditioning on the confidence event in part (c).
\end{enumerate}
\end{proposition}

\begin{proof}{\bfseries Proof of Proposition \ref{prop:continuous-gaussian-linear}.}
For a parameter $\theta\in\Theta$, let $\Delta_\theta:=\theta-\theta^*$ denote its error. Since $h_\theta(X)-C=-\Delta_\theta-\xi$, the pointwise excess squared loss is
\begin{equation}\label{equ:gaussian-linear-excess-loss}
\ell(h_\theta(X),C)-\ell(h_{\theta^*}(X),C)=\Delta_\theta^2+2\Delta_\theta\xi.
\end{equation}
The noise remains mean-zero under every conditional feature distribution, so
\begin{equation}\label{equ:gaussian-linear-squared-risk}
R_{\ell,\mathcal A}(h_\theta)-R_{\ell,\mathcal A}(h_{\theta^*})
=\Delta_\theta^2
=R_\ell(h_\theta)-R_\ell(h_{\theta^*})
\end{equation}
for every measurable $\mathcal A$ of positive probability. The prediction difference is constant in $x$, giving
\begin{equation}\label{equ:gaussian-linear-phi}
\|h_\theta-h_{\theta^*}\|_{\infty,\mathcal A}
=|\Delta_\theta|
=\sqrt{R_{\ell,\mathcal A}(h_\theta)-R_{\ell,\mathcal A}(h_{\theta^*})}.
\end{equation}
For any nonempty compact predictor restriction containing
$h_{\theta^*}$, this predictor is its unique conditional
squared-loss minimizer. Taking $\mathcal A=\mathcal X$ in
\eqref{equ:gaussian-linear-phi} verifies
Assumption~\ref{assumption:upper-bound-pointwise};
the conditional identity gives the stronger statement in part~(a). For a nonempty closed interval $\Theta'\subseteq\Theta$, its squared-loss minimizer is the projection $\theta'$ of $\theta^*$ onto $\Theta'$. The projection inequality gives $(\theta-\theta')^2\leq(\theta-\theta^*)^2-(\theta'-\theta^*)^2$ for $\theta\in\Theta'$, proving the same conditional error bound on the interval. At each fixed feature, the conditional excess squared risk is exactly the squared prediction difference, with coefficient one in this local quadratic identity. This proves part (a). For compact restrictions excluding $\theta^*$, take $\{\theta^*-a,\theta^*+a+e\}$ with fixed $0<a<1/8$ and $e\downarrow0$. The risk gap is $2ae+e^2$, whereas the distance to its unique minimizer is $2a+e$. Thus no uniform square-root bound holds over all such restrictions, which are not covered by Assumption~\ref{assumption:upper-bound-pointwise}; the algorithmic argument here still uses only intervals.

To establish part (b), write $Z=X-1/2\sim\mathcal N(0,1)$. The distance to degeneracy of the optimal prediction is $|Z|$, giving
\begin{equation}\label{equ:gaussian-linear-psi}
\Psi(b)=\bbP(|Z|\leq b)=2\Phi(b)-1\leq\sqrt{\frac2\pi}\,b,\qquad b\geq0.
\end{equation}
The excess SPO risk is the additional conditional decision cost integrated over the features between the true and predicted decision boundaries. By symmetry of the Gaussian feature density,
\begin{equation}\label{equ:gaussian-linear-spo-risk}
R_{\mathrm{SPO}}(h_\theta)-R_{\mathrm{SPO}}^*
=\frac{2}{\sqrt{2\pi}}\int_0^{|\Delta_\theta|}u e^{-u^2/2}\,\dd u
=\sqrt{\frac2\pi}\left(1-e^{-\Delta_\theta^2/2}\right)
\leq\frac{\Delta_\theta^2}{\sqrt{2\pi}}.
\end{equation}
This difference is positive unless $\theta=\theta^*$, proving uniqueness of the SPO-risk minimizer and justifying the expression for $\Psi$ at this fixed reference predictor.

For part (c), condition on the history before batch $k$ and all features in that batch. The acquisition indicators are then fixed, and the acquired noises are independent $\mathcal N(0,\sigma^2)$ random variables because the acquisition rule uses only the features and previous batches. For $n_k>0$, let $\bar\xi_k:=n_k^{-1}\sum_{i\in\mathcal I_k}D_i\xi_i$ be the average acquired noise. With conditional probability at least $1-\delta_k$, its squared magnitude is at most $2\sigma^2\ln(2/\delta_k)/n_k$. Taking expectations over the conditioning variables and then a union bound gives this inequality simultaneously for every nonempty batch with probability at least $1-\delta$, since $\sum_{k\geq1}\delta_k\leq\delta$.

We work on this simultaneous event. Initially, $\Theta_0=\Theta$ contains $\theta^*$. Every subsequent confidence set is an interval: the empirical squared loss is a convex quadratic in $\theta$ on a nonempty batch and is constant on an empty batch, and the update intersects a sublevel set with the preceding interval. Suppose that $\theta^*\in\Theta_{k-1}$. On a nonempty batch, the returned parameter is the projection of $\theta^*-\bar\xi_k$ onto $\Theta_{k-1}$, so $|\widehat\theta_k-\theta^*|\leq|\bar\xi_k|$. Moreover, completing the square in the empirical loss normalized by $m_k$ gives
\[
\widehat\ell_{2m_k-1}(h_{\theta^*})-\widehat\ell_{2m_k-1}(h_{\widehat\theta_k})
\leq\frac{n_k}{m_k}\bar\xi_k^2
\leq\frac{2\sigma^2\ln(2/\delta_k)}{m_k}.
\]
For this proof, put $u_k:=\ln(16m_k^2/\delta)$. Since $\ln(2/\delta_k)\leq u_k$ and $u_k/m_k\leq u_1$, the chosen scale gives
\[
 \frac{2\sigma^2\ln(2/\delta_k)}{m_k}
 \leq\frac{2\sigma^2u_k}{m_k}
 \leq s_0\sqrt{\frac{u_k}{m_k}}=s_{2m_k-1}.
\]
Thus, $\theta^*$ remains in the updated confidence set. If $n_k=0$, the whole preceding interval is retained, so the same conclusion holds. Induction proves retention in every batch. The acquired-noise bound, the projection inequality, and \eqref{equ:gaussian-linear-spo-risk} yield \eqref{equ:gaussian-linear-active-rates} whenever $n_k>0$.

For part (d), condition only on the history before the next feature is drawn. Every current predictor has the form $h_\theta(x)=x-\theta$ with $\theta\in\Theta$. Since the Gaussian density is bounded by $1/\sqrt{2\pi}$, the conditional probability of acquiring a label is
\[
\bbP(|X-\theta|<b_{t-1}\mid\text{past})
\leq\frac{2b_{t-1}}{\sqrt{2\pi}}.
\]
This holds for every current predictor, not merely on a confidence event. Taking expectations and summing gives the first bound in \eqref{equ:gaussian-linear-label-bound}. For $k\geq2$, the threshold is $b_0(u_k/m_k)^{1/4}$ and the batch contains $m_k=2^{k-1}$ feature samples; batch 1 uses $b_0$. At fixed $\delta$, $u_k$ is of order $k$. Summing the resulting geometric sequence, and bounding a partial last batch by a full batch, gives the stated order. Finally, every warm-up quantile is at most the maximum observed distance to degeneracy, and the warm-up parameter belongs to $\Theta$. Hence $b_0\leq\max_{1\leq i\leq n_0}|X_i-1/2|+1/4$, which has finite expectation. This completes the proof.
\hfill\Halmos
\end{proof}

\paragraph{The selected-sample objective.}
The acquired-label count in Proposition \ref{prop:continuous-gaussian-linear}(c) is important. Let $\mathcal F_{m_k-1}$ denote the history before batch $k$. The population excess objective corresponding to the empirical loss normalized by the full batch size is $\bbE[D_i\{\ell(h_\theta(X_i),C_i)-\ell(h_{\theta^*}(X_i),C_i)\}\mid\mathcal F_{m_k-1}]=\bbP(D_i=1\mid\mathcal F_{m_k-1})(\theta-\theta^*)^2$. Thus, converting this selected objective into the unconditional squared risk requires accounting for the acquisition probability. The proof above instead controls the average of the acquired noises directly. Conditional only on the past, $D_i\xi_i$ is a mixture of a point mass at zero and a Gaussian, not generally a Gaussian random variable. It is nevertheless centered and sub-Gaussian: conditional on either value of $D_i$, its exponential moment is at most $e^{\lambda^2\sigma^2/2}$, and averaging preserves that bound. Conditioning additionally on the batch features gives the Gaussian sample-mean bound used in the proof.

\paragraph{Comparison with supervised learning.}
From $n$ fully labeled observations, the supervised empirical minimizer is the projection onto $\Theta$ of $n^{-1}\sum_{i=1}^n(X_i-C_i)$. Since $X_i-C_i=\theta^*-\xi_i$, with probability at least $1-\delta$ its parameter error is at most $\sigma\sqrt{2\ln(2/\delta)/n}$. Applying \eqref{equ:gaussian-linear-spo-risk} gives
\begin{equation}\label{equ:gaussian-linear-epsilon-comparison}
R_{\mathrm{SPO}}(h_{\widehat\theta_n^{\mathrm{sup}}})-R_{\mathrm{SPO}}^*
\leq\sqrt{\frac2\pi}\frac{\sigma^2\ln(2/\delta)}{n}.
\end{equation}
Consequently, $n\geq\lceil\sqrt{2/\pi}\,\sigma^2\ln(2/\delta)/\epsilon\rceil$ labeled samples suffice for this supervised estimator to attain excess SPO risk at most $\epsilon$. The active-learning bound in \eqref{equ:gaussian-linear-active-rates} also depends on the number of acquired labels, with additional logarithmic factors for simultaneous batch guarantees. The sublinear number of labels relative to the feature stream in \eqref{equ:gaussian-linear-label-bound} therefore does not, by itself, establish an order-wise label-complexity improvement over the supervised estimator at the same target risk.

\paragraph{Label savings relative to full labeling.}
Proposition \ref{prop:continuous-gaussian-linear}(d) shows that MBAL-SPO acquires fewer labels than a supervised learner that labels every feature in the same stream. To connect this reduction to decision quality, we next establish that the excess SPO risk still converges to zero. The key additional step is a lower bound on the number of labels in a completed batch. Since every retained parameter lies in the fixed interval $\Theta$ and the Gaussian density is positive around that interval, the decreasing rejection threshold still selects sufficiently many samples for the empirical minimizer to converge.

\begin{proposition}[Label savings relative to full labeling]\label{prop:gaussian-label-comparison}
Consider the Gaussian regression model in \eqref{equ:gaussian-linear-downstream}--\eqref{equ:gaussian-linear-data-model}, with fixed noise variance $\sigma^2>0$ and warm-up size $n_0\geq1$. Run Algorithm~\ref{alg:margin-based} with the common schedules, confidence input $\delta$, and empirical-quantile initialization. Fix $\delta\in(0,1)$ and choose $s_0\geq4\sigma^2\sqrt{\ln(16/\delta)}$. Let $N_T$ be its total number of acquired labels, including warm-up, and let $N_T^{\mathrm{sup}}:=n_0+T$ be the number used by a supervised learner that labels all observations in the same feature stream. Then the following statements hold, with constants allowed to depend on the fixed model parameters, $n_0$, and $\delta$.
\begin{enumerate}[label={\rm (\alph*)}]
\item The expected proportion of labels acquired by MBAL-SPO satisfies
\begin{equation}\label{equ:gaussian-label-ratio}
\frac{\bbE[N_T]}{N_T^{\mathrm{sup}}}
=\mathcal O\!\left(\left(\frac{1+\ln T}{T}\right)^{1/4}\right)
\longrightarrow0.
\end{equation}
\item There is a finite constant $C_\delta>0$ such that, with probability at least $1-\delta$, the returned predictor satisfies
\begin{equation}\label{equ:gaussian-stream-risk}
R_{\mathrm{SPO}}(h_T)-R_{\mathrm{SPO}}^*
\leq C_\delta\left(\frac{1+\ln T}{T}\right)^{3/4}
\qquad\text{for every sufficiently large integer }T.
\end{equation}
\item For sufficiently small target risk $\epsilon>0$, choose the common feature horizon $T_\epsilon:=\lceil A_\delta\epsilon^{-4/3}(1+\ln(1/\epsilon))\rceil$, where $A_\delta>0$ is sufficiently large. MBAL-SPO then attains excess SPO risk at most $\epsilon$ with probability at least $1-\delta$, and its expected label count and the full-label count satisfy
\begin{align}\label{equ:gaussian-common-horizon-labels}
\bbE[N_{T_\epsilon}]&=\mathcal O\!\left(\epsilon^{-1}(1+\ln(1/\epsilon))\right),\\
N_{T_\epsilon}^{\mathrm{sup}}&=\Theta\!\left(\epsilon^{-4/3}(1+\ln(1/\epsilon))\right).\notag
\end{align}
In particular, $\bbE[N_{T_\epsilon}]/N_{T_\epsilon}^{\mathrm{sup}}=\mathcal O(\epsilon^{1/3})\to0$. The supervised empirical minimizer using all these observations also attains the target risk with probability at least $1-\delta$.
\end{enumerate}
The comparison in part (c) uses the same deterministic feature horizon for both methods; it does not compare their smallest possible label counts at a prescribed target risk.
\end{proposition}

\begin{proof}{\bfseries Proof of Proposition \ref{prop:gaussian-label-comparison}.}
Part (a) follows from the unconditional expected-label bound in \eqref{equ:gaussian-linear-label-bound}, since $n_0$ and $\bbE[b_0]$ are finite and fixed while $N_T^{\mathrm{sup}}=n_0+T$. We prove part (b) by showing that the acquired-label count in the most recently completed batch is large enough for \eqref{equ:gaussian-linear-active-rates} to apply with a vanishing upper bound.

First, the warm-up parameter is the projection onto $\Theta$ of $\theta^*-n_0^{-1}\sum_{i=1}^{n_0}\xi_i$, so it is independent of the warm-up features. Every empirical quantile of the warm-up distances is at least their minimum. Conditional on the warm-up parameter, boundedness of the Gaussian density and a union bound therefore give $\bbP(b_0<u)\leq n_0\sqrt{2/\pi}\,u$ for $u>0$. Set $\underline b:=\delta\sqrt{\pi/2}/(3n_0)$, which lies in $(0,1)$. It follows that $\bbP(b_0\geq\underline b)\geq1-\delta/3$. This bound uses the actual empirical-quantile initialization, rather than replacing it by a freely chosen threshold.

Consider batch $k\geq1$ and condition on the history before the batch. Its threshold is at least $b_0(k/m_k)^{1/4}$: batch 1 uses $b_0$, and for $k\geq2$ this follows from $u_k=\ln(16m_k^2/\delta)\geq k$. On the event $b_0\geq\underline b$, the acquired-feature interval contains an interval of half-length $\underline b(k/m_k)^{1/4}$ around the current parameter. Since $k/m_k\leq1$ and the current parameter belongs to $[1/4,3/4]$, this smaller interval lies within distance $5/4$ of the Gaussian mean $1/2$. Define the Gaussian density constant $c_0:=\sqrt{2/\pi}\,e^{-25/32}$. The Gaussian density lower bound on this interval gives
\begin{equation}\label{equ:gaussian-query-lower}
\bbP(D_i=1\mid\mathcal F_{m_k-1})\geq c_0\,\underline b\left(\frac{k}{m_k}\right)^{1/4},
\qquad
\bbE[n_k\mid\mathcal F_{m_k-1}]\geq c_0\,\underline b\,k^{1/4}m_k^{3/4}.
\end{equation}
These inequalities hold for every current predictor in $\mathcal H$, without assuming that its parameter is already close to $\theta^*$.

Conditional on the history, $n_k$ is binomial with $m_k$ trials and the success probability in \eqref{equ:gaussian-query-lower}. The binomial lower-tail bound gives
\[
\bbP\!\left(n_k<\tfrac12c_0\underline b\,k^{1/4}m_k^{3/4},\ b_0\geq\underline b\right)
\leq\exp\!\left(-\tfrac18c_0\underline b\,k^{1/4}m_k^{3/4}\right).
\]
Because the batch sizes grow geometrically, this upper bound is at most $\delta/(3\cdot2^{k+1})$ for every sufficiently large batch index $k$. A union bound shows that, except on an event of probability at most $\delta/3$, every such batch satisfies $n_k\geq c_0\underline b\,k^{1/4}m_k^{3/4}/2$ whenever $b_0\geq\underline b$. In particular, these batches are nonempty.

Repeat the retention and acquired-noise argument in Proposition~\ref{prop:continuous-gaussian-linear}(c) with batch failure probabilities $\delta/(3\cdot2^{k+1})$, keeping the algorithmic confidence input equal to $\delta$. The logarithm in that argument satisfies $\ln(3\cdot2^{k+2}/\delta)\leq u_k+\ln3\leq2u_k$. Hence the stated $s_0\geq4\sigma^2\sqrt{u_1}$ still makes the empirical improvement smaller than $s_{2m_k-1}$. Except on an additional event of probability at most $\delta/3$, the true parameter is retained in every batch and the risk bound in \eqref{equ:gaussian-linear-active-rates}, with these batch failure probabilities, holds for every nonempty batch. Intersecting this event with the preceding two events, we obtain probability at least $1-\delta$. On the intersection, for every sufficiently large batch index $k$,
\[
R_{\mathrm{SPO}}(h_{2m_k-1})-R_{\mathrm{SPO}}^*
\leq\frac{2\sqrt{2/\pi}\,\sigma^2\ln(3\cdot2^{k+2}/\delta)}
{c_0\underline b\,k^{1/4}m_k^{3/4}}
=\mathcal O\!\left(\frac{k^{3/4}}{m_k^{3/4}}\right).
\]
For an arbitrary stream length $T$, the returned predictor equals the predictor at the last completed batch. Its index is $K_T=\lfloor\log_2(T+1)\rfloor$, and its batch size satisfies $(T+1)/4<m_{K_T}\leq(T+1)/2$. Thus, the preceding bound implies \eqref{equ:gaussian-stream-risk} for every sufficiently large $T$, simultaneously on the same event. This proves part (b), including a bound that applies after empty early batches without making any assumption about their empirical tie-breaking choices.

For part (c), the specified horizon satisfies $((1+\ln T_\epsilon)/T_\epsilon)^{3/4}=\mathcal O(\epsilon/A_\delta^{3/4})$ as $\epsilon\downarrow0$. A sufficiently large fixed $A_\delta$ makes the bound in part (b) no greater than $\epsilon$, and $T_\epsilon$ lies in the range of part (b) for sufficiently small $\epsilon$. Substituting this horizon into \eqref{equ:gaussian-linear-label-bound} gives the expected MBAL-SPO label bound in \eqref{equ:gaussian-common-horizon-labels}. The supervised count is exactly $n_0+T_\epsilon$, so its stated order and the ratio of the two counts follow directly. Finally, the supervised risk bound \eqref{equ:gaussian-linear-epsilon-comparison}, with $n=n_0+T_\epsilon$, is also at most $\epsilon$ for sufficiently small $\epsilon$. The two risk guarantees each have probability at least $1-\delta$.
\hfill\Halmos
\end{proof}

Proposition \ref{prop:gaussian-label-comparison} shows that fewer acquired labels do not merely result from ceasing to learn. The prediction-error bound converges to zero because the growing batches compensate for the decreasing probability of acquiring a label. Relative to labeling the entire feature stream, MBAL-SPO acquires an asymptotically vanishing proportion of labels while achieving a vanishing excess SPO risk. At a common horizon sufficient for the prescribed target risk, part (c) quantifies this reduction by a label-count ratio of order $\epsilon^{1/3}$.

\paragraph{Interpretation of the benchmark.}
Proposition \ref{prop:gaussian-label-comparison} gives smaller label use than full labeling at a common feature horizon. As already shown in \eqref{equ:gaussian-linear-epsilon-comparison}, the supervised estimator can attain a prescribed decision risk with $\widetilde{\mathcal O}(\epsilon^{-1})$ labels by stopping earlier; the sufficient active label bound has the same order here. The residual $X-C=\theta^*-\xi$ has the same distribution at every feature value, so selection reduces the labeled fraction of the stream rather than making an individual residual more informative. Unlike the informative-feature examples in Section \ref{sec:phi_examples}, the residual has the same information at every feature. Its interior optimum and zero expected gradient do not establish a same-risk information gain; the unbounded benchmark is analyzed separately from the main finite linear-surrogate theorems.

\section{An Alternative Guarantee under Non-degenerate Expectation}
\label{appendix:nondegenerate-alternative}

Theorem \ref{uniform_margin} uses local conditional information under feature selection to learn the optimal constrained predictor from acquired labels. In this appendix, we establish an alternative sufficient setting based on Example \ref{example:2} and Assumption \ref{assumption:1}, without Assumption \ref{assumption:holder}. We retain the literal requirement in that example that the \emph{optimal SPO predictor}, not only its induced decision, is unique. Together with the square-root error bound, uniform separation from degeneracy then gives a positive gap between the optimal surrogate risk and every other attainable surrogate risk. The nested confidence sets allow the algorithm to exclude all other predictors before label acquisition stops.

The result uses the batch implementation and the parameter schedules of Algorithm \ref{alg:margin-based}, with sufficient initial scales specified below. It does not require the convex-parameter setting or the range $0<\kappa<4$ in Theorem~\ref{uniform_margin}. In fact, the optimal predictor must be isolated in the working class, so a nontrivial convex affine prediction class is not an example of this alternative setting. We give a disconnected, continuously parameterized class with ordinary multivariate Gaussian cost noise at the end of the appendix.

\subsection{A positive gap in the surrogate risk}

We use the $\ell_2$ norm. Let $S$ be a nonempty bounded polyhedron containing at least two points, and suppose $\bbE\|C\|_2<\infty$. Throughout this appendix the training loss is $\ell=\ell_{\mathrm{SPO}+}$. Assume the setting of Example \ref{example:2}: the class contains the conditional-mean predictor $h^*(x)=\bbE[C\mid X=x]$, this predictor is the unique SPO-risk minimizer, and $\nu_S(h^*(X))\geq\underline k>0$ almost surely. Assume also Assumption \ref{assumption:1}, with $\phi(\epsilon)\leq C_\phi\sqrt\epsilon$. Uniqueness and equality of predictors are understood up to feature-distribution null sets, as in the sup-norm conventions of the main analysis.

The following observation connects the non-degeneracy condition to surrogate training. It is a statement about \emph{excess} risk; neither the realized SPO+ loss nor its minimum population risk needs to vanish.

\begin{lemma}[A surrogate-risk gap under non-degenerate expectation]
\label{lemma:nondegenerate-risk-gap}
Under the setting above, including Assumption \ref{assumption:1}, $\mathcal H_\ell^*=\mathcal H^*=\{h^*\}$. Define the positive surrogate-risk gap $\Delta_\ell:=\underline k^2/C_\phi^2$. Every $h\in\mathcal H$ with $h\ne h^*$ satisfies $\|h-h^*\|_\infty\geq\underline k$ and $R_\ell(h)-R_\ell^*\geq\Delta_\ell$.
\end{lemma}

\begin{proof}{\bfseries Proof of Lemma \ref{lemma:nondegenerate-risk-gap}.}
Conditional consistency on the full feature space makes the nonempty set $\mathcal H_\ell^*$ a subset of $\mathcal H^*=\{h^*\}$. Thus the surrogate-risk minimizer is also uniquely $h^*$. If $\|h-h^*\|_\infty<\underline k$, then, almost surely, $\|h(X)-h^*(X)\|_2<\nu_S(h^*(X))$. Lemma \ref{lemma:identical} implies that the two predictions induce the same decision. Their realized SPO losses are therefore equal almost surely, so $h$ is SPO-optimal. Uniqueness forces $h=h^*$. This proves the prediction-distance claim.

Apply Assumption~\ref{assumption:upper-bound-pointwise}
with $\mathcal G=\mathcal H$. Its square-root bound gives $\|h-h^*\|_\infty\leq C_\phi\sqrt{R_\ell(h)-R_\ell^*}$. Combining this with the distance claim proves the risk gap. Equivalently, any predictor with excess surrogate risk strictly below $\Delta_\ell$ must be $h^*$.
\hfill\Halmos
\end{proof}

The lemma does not imply that the class is finite. It implies that $h^*$ is an isolated point of the class, while other predictors may form continuous subsets. Moreover, once every member of a retained confidence set has excess risk below $\Delta_\ell$, that set is the singleton $\{h^*\}$. This conclusion is stronger than selecting $h^*$ in one batch: a singleton retained set also protects later outputs when their working sets are empty.

\subsection{Initialization and the alternative risk guarantee}

We next specify conditions under which the current algorithm reaches such a confidence set. Assume the prediction range has finite maximum norm $\rho(\hat{\mathcal C})$. On the feature support used below, all prediction norms are bounded by this value. The SPO+ definition gives the prediction-Lipschitz constant $L_\ell:=2D_S$: for all predictions $\hat c,\hat c'$ and all realized costs $c$, $|\ell(\hat c,c)-\ell(\hat c',c)|\leq L_\ell\|\hat c-\hat c'\|_2$. Indeed, a prediction subgradient is $2\{w^*(c)-w^*(2\hat c-c)\}$ and its norm is at most $2D_S$. Consequently, $W:=2L_\ell\rho(\hat{\mathcal C})$ bounds the absolute loss difference $|\ell(h(X),C)-\ell(h^*(X),C)|$ uniformly over $h$. Here $W$ is a bound on \emph{loss differences}, not on absolute losses. In particular, Gaussian cost noise need not be truncated.

For the uniform concentration step, choose a finite collection $h^{(1)},\ldots,h^{(N)}\in\mathcal H\setminus\{h^*\}$ such that every nonoptimal $h$ is within prediction distance $\Delta_\ell/(4L_\ell)$ of one of these predictors, uniformly on one common feature set of probability one. This fixed-accuracy covering condition is automatic for finite classes and for classes compact in the pointwise supremum norm on that feature set: Lemma \ref{lemma:nondegenerate-risk-gap} makes $\mathcal H\setminus\{h^*\}$ closed and hence compact. We state the common-version requirement explicitly because essential-supremum compactness alone does not rule out arbitrary changes on sample-dependent null sets. A logarithmic growth condition on covering numbers is not needed here; only the finite cover at this fixed accuracy is used. The cover is for the proof and need not be computed by the algorithm. We treat the nontrivial case $\mathcal H\ne\{h^*\}$, so $N\geq1$ and $0<\Delta_\ell\leq W$.

Retain $m_k=2^{k-1}$ and, for this appendix, abbreviate the logarithmic term by $u_k:=\ln(16m_k^2/\delta)$. No convex coefficient parameterization is required. The empirical minimizers are exact and measurable. Choose deterministic initial scales satisfying
\begin{equation}\label{equ:nondegenerate-initial-scales}
 s_0\geq W\sqrt{2(1+\ln N)},
 \qquad
 b_0\geq2\rho(\hat{\mathcal C})
       \left(\frac{16s_0}{\Delta_\ell}\right)^{1/2}.
\end{equation}
Algorithm~\ref{alg:margin-based} uses $s_{2m_k-1}=s_0\sqrt{u_k/m_k}$ and $b_{m_k-1}=b_0(u_k/m_k)^{1/4}$ for $k\geq2$, with input threshold $b_0$ in batch 1. These are sufficient scale choices, not consequences of an arbitrary empirical-quantile initialization. They depend on bounds describing the prediction class and the positive risk gap, but are independent of $T$ and $\delta$. Their purpose is to retain the optimum and keep the initial batches fully labeled until a confidence set can exclude every other predictor. No acquisition lower bound near $h^*$ is imposed.

For this separate result, define $B_\ell:=\sup_{h\in\mathcal H}\{R_\ell(h)-R_\ell^*\}$ and $B_{\mathrm{SPO}}:=\sup_{h\in\mathcal H}\{R_{\mathrm{SPO}}(h)-R_{\mathrm{SPO}}^*\}$. These maximum excess risks are finite under the present assumptions; they do not refer to the constants in the linear-information theorem.

\begin{theorem}[SPO+ risk under non-degenerate expectation]
\label{thm:nondegenerate-alternative}
Suppose the setting of Lemma \ref{lemma:nondegenerate-risk-gap}, including the error bound in Assumption~\ref{assumption:1}, the bounded prediction range, and the fixed-accuracy covering condition above hold. Fix $\delta\in(0,1)$ and run Algorithm~\ref{alg:margin-based} with training loss SPO+, confidence input $\delta$, exact measurable empirical minimizers, and the scales in \eqref{equ:nondegenerate-initial-scales}. Include the $n_0$ warm-up labels in $N_T$. There is a finite constant $C\geq1$, independent of $T$ and $\delta$, such that, on one event of probability at least $1-\delta$, the following statements hold simultaneously for every $T\geq1$.
\begin{enumerate}[label={\rm (\alph*)}]
\item The optimal predictor is retained at every update, and the excess surrogate risk satisfies $R_\ell(h_T)-R_\ell^*\leq B_\ell\min\{1,\sqrt{C\ln(16/\delta)/(T+1)}\}$.
\item The excess SPO risk satisfies $R_{\mathrm{SPO}}(h_T)-R_{\mathrm{SPO}}^*\leq B_{\mathrm{SPO}}\min\{1,\sqrt{C\ln(16/\delta)/(T+1)}\}$.
\item The number of acquired labels satisfies $N_T\leq n_0+C\ln(16/\delta)$. In particular, only finitely many labels are acquired on this event. Unconditionally, $\bbE[N_T]\leq n_0+C\ln(16/\delta)+\delta T$ for every $T\geq1$.
\end{enumerate}
The constant $C$ may depend on the fixed class, loss, separation, and initial-scale constants. Neither Assumption~\ref{assumption:holder} nor a restriction on the margin exponent is required.
\end{theorem}

This alternative theorem states a conventional square-root risk order using an isolated-predictor argument, rather than the inverse-fitting-label argument in Theorem \ref{uniform_margin}. It uses Assumption~\ref{assumption:1} on its possibly disconnected predictor class. The initial fully labeled batches reduce the retained class to $\{h^*\}$. The threshold can then continue to decrease without losing accuracy, and uniform non-degeneracy eventually causes every subsequent label to be rejected. Thus, the lower acquisition bound used in Theorem \ref{uniform_margin} is not replaced by an unsupported lower bound on $\Psi$. It is unnecessary after the other predictors have been excluded.

\begin{proof}{\bfseries Proof of Theorem \ref{thm:nondegenerate-alternative}.}
For this proof, let $K$ be the first batch index satisfying $s_0\sqrt{u_K/m_K}\leq\Delta_\ell/8$.

\paragraph{A deterministic period of full labeling.} The ratio $u_k/m_k$ decreases to zero: $u_{k+1}=u_k+2\ln2$, $m_{k+1}=2m_k$, and $u_1>2\ln2$. Hence $K$ is deterministic and finite. Since $s_0\geq\sqrt2 W\geq\Delta_\ell$ and $u_1>1$, we have $K\geq2$. Minimality and $u_K>u_{K-1}$ give $u_K/m_K>\Delta_\ell^2/(128s_0^2)$. Therefore, for every $2\leq k\leq K$, \eqref{equ:nondegenerate-initial-scales} implies $b_{m_k-1}\geq2\rho(\hat{\mathcal C})$. The input threshold in batch 1 satisfies the same inequality.

Because $S$ contains at least two points, the zero cost vector is degenerate. Thus $\nu_S(h(x))\leq\|h(x)\|_2\leq\rho(\hat{\mathcal C})$ for every available predictor on the common full-probability feature set. The acquisition comparison is strict, so every feature in batches $1,\ldots,K$ is labeled, regardless of the warm-up predictor or preceding empirical choices. This proves full labeling before applying any concentration argument; it is not a conclusion conditional on later success. In these batches the algorithm's full-batch normalization is exactly the ordinary empirical-risk normalization.

\paragraph{Uniform comparison with the optimal predictor.}
Fix a fully labeled batch $k\leq K$. For every cover member $h^{(j)}$, the random loss difference $\ell(h^{(j)}(X),C)-\ell(h^*(X),C)$ lies in $[-W,W]$. Hoeffding's inequality and a union bound over the $N$ cover members show that, except with probability $e^{-u_k}$, every such member satisfies
\begin{align*}
 \widehat\ell_{2m_k-1}(h^{(j)})-\widehat\ell_{2m_k-1}(h^*)
 &\geq R_\ell(h^{(j)})-R_\ell^*
       -W\sqrt{\frac{2(u_k+\ln N)}{m_k}}.
\end{align*}
Indeed, the one-sided exceptional probability for each member is at most $\exp[-(u_k+\ln N)]$. For an arbitrary nonoptimal predictor $h$, choose its cover member. Prediction-Lipschitzness changes the empirical loss difference by at most $\Delta_\ell/4$, while Lemma \ref{lemma:nondegenerate-risk-gap} bounds the cover member's population excess risk below by $\Delta_\ell$. Consequently, on the same event,
\begin{equation}\label{equ:nondegenerate-empirical-gap}
 \widehat\ell_{2m_k-1}(h)-\widehat\ell_{2m_k-1}(h^*)
 \geq\frac{3\Delta_\ell}{4}
       -W\sqrt{\frac{2(u_k+\ln N)}{m_k}}
 \geq\frac{\Delta_\ell}{2}-s_{2m_k-1},
 \qquad h\ne h^*.
\end{equation}
For the second inequality, use $u_k>1$, so $u_k+\ln N\leq(1+\ln N)u_k$. The chosen $s_0$ gives $W\sqrt{2(u_k+\ln N)/m_k}\leq s_0\sqrt{u_k/m_k}=s_{2m_k-1}$; the stated bound follows, with an additional margin of $\Delta_\ell/4$.

The common-version covering condition ensures this inequality holds simultaneously for all predictors, including the data-dependent empirical minimizer. Summing the exceptional probabilities over $k\leq K$ gives at most $\sum_{k\geq1}e^{-u_k}=\delta/12\leq\delta$. Denote the resulting event by $E$. This union bound applies to the original i.i.d. batches, without conditioning on future retention.

\paragraph{Retention and exclusion of all other predictors.}
Initially, $h^*\in\mathcal H_0$ regardless of the warm-up estimate. Work on $E$, and suppose $h^*$ is retained before batch $k\leq K$. Equation \eqref{equ:nondegenerate-empirical-gap}, together with the zero loss difference for $h^*$ itself, implies that its empirical improvement deficit is at most $(s_{2m_k-1}-\Delta_\ell/2)_+\leq s_{2m_k-1}$. Thus the actual confidence-set update retains $h^*$. This proves retention through batch $K$ by induction.

At batch $K$, the slack is at most $\Delta_\ell/8$. Equation \eqref{equ:nondegenerate-empirical-gap} places every nonoptimal predictor's empirical loss at least $3\Delta_\ell/8$ above that of $h^*$. Therefore the empirical minimizer is $h^*$, and no other predictor satisfies the confidence-set inequality with slack $s_{2m_K-1}$. We obtain $\mathcal H_{2m_K-1}=\{h^*\}$. Every subsequent update is over this singleton and retains it, whether the working set is nonempty or empty. In particular, $h_T=h^*$ for all $T\geq 2m_K-1$ on $E$. This is the role of the nested classes in the alternative proof.

\paragraph{Length of the initial period and cessation of acquisition.}
For any fixed $c\in(0,1)$, the endpoint of the first batch with $u_k/m_k\leq c$, plus one, is at most $C\ln(16/\delta)$ for a finite constant depending on $c$. To verify the order and its confidence dependence, take $m=4\{\ln(16/\delta)+2\ln(4/c)\}/c$. Writing the braces as $H$ for this calculation, we have $\ln(16/\delta)+2\ln m=H+2\ln H\leq4H$. Rounding $m$ upward to a power of two increases its size by less than two, so the corresponding endpoint plus one is at most $16H/c$. This proves the stated logarithmic bound with a constant independent of $\delta$.

Apply the calculation with $c=\min\{(\Delta_\ell/(8s_0))^2,(\underline k/(2b_0))^4\}$. This is a fixed number in $(0,1)$. At the end of the first qualifying batch, the class has already become a singleton on $E$. In every later batch, its threshold is at most $\underline k/2$, whereas $\nu_S(h^*(X))\geq\underline k$ almost surely. Hence all those labels are rejected. Increasing $C$ if necessary, both the endpoint at which the class becomes a singleton and the last possible acquisition time on $E$ are bounded by $C\ln(16/\delta)-1$.

\paragraph{Risk and label bounds.}
The excess surrogate and SPO risks are bounded respectively by $B_\ell$ and $B_{\mathrm{SPO}}$ before the class becomes a singleton and are zero afterward on $E$. Comparing these two periods with $C\ln(16/\delta)$ gives parts (a) and (b) simultaneously for every $T$. The SPO+ loss-difference bound gives $B_\ell\leq W$, while bounded $S$ and $\bbE\|C\|_2<\infty$ give finite $B_{\mathrm{SPO}}$.

On $E$, at most $C\ln(16/\delta)$ stream labels are acquired. On its complement, at most $T$ labels are acquired by time $T$. Thus, pointwise, $N_T\leq n_0+C\ln(16/\delta)+T\mathbb I_{E^c}$. Taking expectations proves part (c), since $\bbP(E^c)\leq\delta$. This step does not assume independence after conditioning on $E$.
\hfill\Halmos
\end{proof}

\subsection{A Gaussian example and the scope of the comparison}

The alternative conditions are compatible with full-support Gaussian noise, but uniqueness of the SPO predictor requires a restriction on the prediction class. The following example retains a continuous set of nonoptimal alternatives while keeping the optimum isolated. It does not satisfy the local acquisition lower bound of Assumption \ref{assumption:holder}; its proof instead uses the isolated risk gap.

\begin{example}[An isolated optimal linear predictor with Gaussian costs]
\label{example:nondegenerate-gaussian-isolated}
Let $d\geq1$, let $X$ be uniform on $[1,2]^d$, and take $S=[-1,1]^d$. Generate $C=X+\xi$, where $\xi\sim\mathcal N_d(0,\Sigma)$ is independent of $X$ and $\Sigma\succ0$ may have nonzero off-diagonal entries. Consider the affine linear predictions $h_\theta(x)=x-\theta\mathbf1_d$ with the compact parameter set $\Theta=\{0\}\cup[3,4]$. The working class is well-specified, with unique optimal SPO and SPO+ predictor $h^*(x)=x$. The uniform separation condition holds with $\underline k=1$.

Write $\sigma_j^2=\Sigma_{jj}$, and let $\varphi_{\sigma_j}$ denote the centered Gaussian density with variance $\sigma_j^2$. The positive conditional curvature bound $m_+:=8\sum_{j=1}^d\min_{y\in[-7,2]}\varphi_{\sigma_j}(y)$ gives Assumption \ref{assumption:1} with $\phi(\epsilon)=\sqrt{2d\epsilon/m_+}$. The class admits the uniform cover required above, and Theorem \ref{thm:nondegenerate-alternative} applies with appropriate initial scales. The expected parameter subgradient at $\theta=0$ is zero on every feature set. This is compatible with an interior quadratic-risk argument, but the current class is disconnected and its optimal prediction is uniformly separated from degeneracy. In particular, the acquisition lower bound in Assumption \ref{assumption:holder} fails at $h^*$ for thresholds below one.
\end{example}

\begin{proof}{\bfseries Proof of Example \ref{example:nondegenerate-gaussian-isolated}.}
For every feature vector in $[1,2]^d$, the conditional mean $h^*(x)=x$ has positive coordinates and induces the decision $-\mathbf1_d$. Every alternative coefficient $\theta\in[3,4]$ produces negative coordinates and induces $\mathbf1_d$. Its excess conditional SPO risk is $2\sum_jx_j>0$. Thus $h^*$ is uniquely SPO-optimal, including under every positive-probability feature restriction. For the box decision set, the Euclidean distance to degeneracy is $\min_j|c_j|$, so $\nu_S(h^*(X))=\min_jX_j\geq1$. In particular, $\Psi(r)=0$ for $r<1$; more explicitly, $\Psi(r)=1-(2-r)^d$ for $1\leq r\leq2$, and $\Psi(r)=1$ for $r>2$.

The box SPO+ loss is the sum of its scalar coordinate losses. Conditional on $X=x$, differentiating the Gaussian expected loss with respect to $\theta$ gives $4\sum_j\{\Phi(x_j/\sigma_j)-\Phi((x_j-2\theta)/\sigma_j)\}$. This derivative is zero at $\theta=0$ and positive for $\theta>0$. Its second derivative is $8\sum_j\varphi_{\sigma_j}(x_j-2\theta)\geq m_+>0$ for $0\leq\theta\leq4$. These calculations use the Gaussian marginal distributions and therefore permit correlations between the cost coordinates. The sample parameter derivative is bounded by $4d$; its expectation may be interchanged with differentiation, and its conditional expectation at zero is zero. These derivative calculations establish the conditional risk ordering needed below. The local information theorem is not invoked: its positive acquisition lower bound fails at the separated optimal predictor.

For any nonempty compact restricted parameter set, let $a$ be its smallest coefficient. The conditional surrogate risk is increasing in $\theta\geq0$, so $h_a$ is its unique conditional minimizer. The preceding curvature bound and the nonnegative derivative at $a$ give $R_{\ell,\mathcal A}(h_\theta)-R_{\ell,\mathcal A}(h_a)\geq(m_+/2)(\theta-a)^2$. The conditional squared prediction distance is exactly $d(\theta-a)^2$. Taking $\mathcal A=\mathcal X$ verifies
Assumption~\ref{assumption:upper-bound-pointwise}.
The same calculation also provides the stronger conditional
bound for every compact restricted class, including disconnected
ones. 

Full-class conditional surrogate minimization selects $h_0$, which is conditionally SPO-optimal, proving the consistency clause as well.

The parameter set is compact, and $\|h_\theta-h_{\theta'}\|_\infty=\sqrt d\,|\theta-\theta'|$, with the same equality pointwise on the feature support. Consequently the class is compact and has a finite common-version uniform cover at every positive accuracy. All predictions are bounded and Gaussian costs have finite moments. The hypotheses of the alternative theorem are therefore jointly satisfied without the local-information or acquisition lower bound of Assumption \ref{assumption:holder}.
\hfill\Halmos
\end{proof}

\paragraph{Why the class restriction matters.}
The parent interval $[-1,5]$ contains the conditional-mean coefficient $\theta=0$ in its interior. The working set $\{0\}\cup[3,4]$ retains this predictor but removes its nearby alternatives. Without that removal, every sufficiently small coefficient perturbation would keep all predicted coordinates positive, induce the same optimal decision, and violate uniqueness of the SPO predictor. More generally, uniform non-degeneracy and uniqueness imply that no distinct predictor can approach $h^*$ in the sup norm. This is incompatible with a nontrivial convex affine class containing $h^*$, but not with a finite class or the disconnected continuous class above. The result is an alternative setting, not a claim for ordinary unrestricted continuous regression classes.

\paragraph{Comparison with supervised learning.}
For a prescribed confidence, the alternative theorem gives the same conventional $T^{-1/2}$ excess-risk order as full labeling of the feature stream, while the number of acquired labels is bounded independently of $T$ on the confidence event. The expected-label statement retains $\delta T$; it does not assert an unconditional finite number of labels over an infinite stream. For a prescribed finite horizon, using internal confidence $\delta/T$ gives an expected count of order $n_0+\mathcal O(\ln(T/\delta))$, with fixed model constants.

The comparison concerns full labeling of a common feature stream, not a smaller optimal target-risk label complexity than every supervised learner. Indeed, the same fixed-cover argument applied to $n$ fully labeled observations shows that $n\geq8W^2\Delta_\ell^{-2}\ln(N/\delta)$ suffices for the supervised empirical minimizer to equal $h^*$ with probability at least $1-\delta$: the cover error is at most $\Delta_\ell/4$, and the concentration error is at most $\Delta_\ell/2$, leaving a strictly positive empirical gap for every alternative. Supervised learning also benefits from the isolated optimum. The contribution of the present appendix is a complete alternative proof for MBAL-SPO, including the threshold and retention conditions that allow its nested class to preserve accuracy after label acquisition ceases.

\section{\texorpdfstring{Proofs for Decision Stability and Auxiliary Error Bounds}{Proofs for Decision Stability and Auxiliary Error Bounds}}\label{appendix:margin}

We first provide the direct SPO-loss and margin proofs, followed by the auxiliary loss-geometry calculations. The linear-surrogate finite-sample proofs are in Appendices \ref{app:information-alpha-two}--\ref{app:spop-verification}.

\subsection{Decision stability and direct SPO minimization}\label{appendix:Sec4}

\begin{proof}{\bfseries Proof of Lemma \ref{lemma:identical}}
Without loss of generality, assume that $\nu_S(c_1)\ge \nu_S(c_2)$, i.e., we have that $0 \leq \| c_1- c_2 \| < \nu_S(c_1)$. We also claim that $\nu_S(c_2) > 0$. Indeed, since $\nu_S$ is a $1$-Lipschitz distance function, it holds that $\nu_S(c_1) - \nu_S(c_2) \leq \| c_1- c_2 \| < \nu_S(c_1)$ and hence $\nu_S(c_2) > 0$.
As above, let $\{v_j : j = 1, ..., K\}$ be the extreme points of $S$, i.e., $S = \text{conv}(v_1, \ldots, v_K)$. Since $\nu_S(c_1) > 0$ and $\nu_S(c_2) > 0$, both $w^*(c_1)$ and $w^*(c_2)$ are extreme-point solutions, with $w^*(c_1)=v_{j_1}$ and $w^*(c_2)=v_{j_2}$ for some indices $j_1$ and $j_2$. 

We now prove the lemma by contradiction. If $w^*(c_1) \not = w^*(c_2)$, then by \eqref{equ:thm8}, the following two inequalities hold:
\begin{align*}
    \nu_S(c_1) \le  \frac{c_1^T (w^*(c_2) - w^*(c_1))}{\|w^*(c_2) - w^*(c_1)\|_*}, \quad \nu_S(c_2) \le  \frac{c_2^T (w^*(c_1) - w^*(c_2))}{\|w^*(c_1) - w^*(c_2)\|_*}.
\end{align*}
We add up both sides of the above two inequalities, and get 
\begin{align*}
    \nu_S(c_1) + \nu_S(c_2) \le \frac{(c_1 - c_2)^T (w^*(c_2) - w^*(c_1))}{\|w^*(c_2) - w^*(c_1)\|_*} \le \frac{\|c_1 - c_2\| \|w^*(c_2) - w^*(c_1)\|_*}{\|w^*(c_2) - w^*(c_1)\|_*} = \|c_1 - c_2\|,
\end{align*}
where the second inequality uses H\"older's inequality. Because $\| c_1- c_2 \| < \nu_S(c_1)$, we have that $\nu_S(c_1) + \nu_S(c_2) \le \|c_1 - c_2\| < \nu_S(c_1)$. This implies that $\nu_S(c_2) <0$, which contradicts that $\nu_S$ is a non-negative distance function. Thus, we conclude that $w^*(c_1) = w^*(c_2)$. \hfill\Halmos
\end{proof}

\begin{proof}{\bfseries Proof of Lemma \ref{lemma:label_c}}
Fix an iteration $t\geq1$, and condition on the observations and algorithmic decisions before the next feature is drawn. The predictor $h_{t-1}$ and rejection threshold $b_{t-1}$ are then fixed. The population reference predictor $h^*$ is fixed independently of the training sample. By the assumed prediction-error bound, $\|h_{t-1}-h^*\|_\infty\leq b_{t-1}$. Since the distance to degeneracy is $1$-Lipschitz, $\nu_S(h^*(x))\leq\nu_S(h_{t-1}(x))+\|h^*(x)-h_{t-1}(x)\|\leq\nu_S(h_{t-1}(x))+b_{t-1}$ almost surely in $x$. Thus, whenever a label is acquired, $\nu_S(h_{t-1}(x))<b_{t-1}$ implies $\nu_S(h^*(x))<2b_{t-1}$.

Consequently, conditional on the past, the probability of acquiring the next label is at most $\bbP(\nu_S(h^*(X))\leq2b_{t-1})=\Psi(2b_{t-1})$, where the probability is evaluated under the original feature distribution with the threshold fixed by that history. Taking expectations over the past and summing from $t=1$ to $T$ gives $\bbE[N_T]-n_0\leq\bbE[\sum_{t=1}^T\Psi(2b_{t-1})]$. If the thresholds are deterministic, the outer expectation on the right-hand side can be omitted.
\hfill\Halmos
\end{proof}

\paragraph{Technical convention.}
For the direct-SPO analysis, we assume that
$\|h(x)-g(x)\|\leq\|h-g\|_\infty$ holds for every
$h,g\in\mathcal H$ outside one fixed set of feature probability zero.
This pointwise property holds for continuous predictors on the
support of the feature distribution and for finite feature support.
The suprema in Lemma~\ref{lemma:natarajan-concentration}
are assumed measurable.

\paragraph{Concentration under finite Natarajan dimension.}
The SPO loss difference relative to the fixed reference depends only on the two induced decisions. This permits a uniform bound without restricting the realized labels. Write $D_{S,2}:=\sup_{w,w'\in S}\|w-w'\|_2$ and $M_\mu:=\operatorname*{ess\,sup}_x\|\mu(x)\|_2$. If $S$ is a singleton, any fixed positive scales suffice and the conclusions are immediate. The calculations below concern $D_{S,2}>0$.

\begin{lemma}[Concentration of selected SPO loss differences]\label{lemma:natarajan-concentration}
Suppose $d_{\mathrm N}<\infty$, $M_\mu<\infty$, and \eqref{eq:subgaussian-cost-noise} holds. Let $\mathcal A$ be any fixed measurable feature set, and let $(X_i,C_i)_{i=1}^m$ be i.i.d. observations. Define
\[
 \Delta_h(X,C):=\ell_{\mathrm{SPO}}(h(X),C)
                   -\ell_{\mathrm{SPO}}(h^*(X),C),
 \qquad
 z(m,\zeta):=1+d_{\mathrm N}\ln(e m K^2)+\ln(8/\zeta).
\]
For every $m\geq1$ and $\zeta\in(0,1)$, with probability at least $1-\zeta$,
\begin{equation}\label{eq:natarajan-loss-difference-bound}
 \sup_{h\in\mathcal H}\left|
 \frac1m\sum_{i=1}^m\mathbb I_{\mathcal A}(X_i)\Delta_h(X_i,C_i)
 -\bbE[\mathbb I_{\mathcal A}(X)\Delta_h(X,C)]\right|
 \leq8D_{S,2}(M_\mu+\sigma)\sqrt{\frac{z(m,\zeta)}m}.
\end{equation}
The same bound holds conditionally when $\mathcal A$ is determined by a preceding history independent of the current observations.
\end{lemma}

\begin{proof}{\bfseries Proof of Lemma \ref{lemma:natarajan-concentration}.}
Let $\pi_h(x):=w^*(h(x))$. The common hindsight term cancels, so $\Delta_h(X,C)=C^T\{\pi_h(X)-\pi_{h^*}(X)\}$. For this proof, let $\mathcal M(m)$ be the maximum number of vectors $(\pi_h(x_1),\ldots,\pi_h(x_m))$ on $m$ fixed features. The Natarajan counting bound used by \citet[proof of Theorem 2]{el2022generalization} gives
\[
 \mathcal M(m)\leq(e m K^2)^{d_{\mathrm N}}.
\]
This also covers $d_{\mathrm N}=0$, when every predictor has the same decision at every feature. Multiplication by the fixed acquisition indicator does not increase the number of sample vectors.

Decompose the selected difference into its conditional mean and centered noise:
\[
 f_h(X):=\mathbb I_{\mathcal A}(X)\mu(X)^T
              \{\pi_h(X)-\pi_{h^*}(X)\},\qquad
 g_h(X,\xi):=\mathbb I_{\mathcal A}(X)\xi^T
              \{\pi_h(X)-\pi_{h^*}(X)\}.
\]
The functions $f_h$ lie in $[-B,B]$, where $B=M_\mu D_{S,2}$. Let $P_m$ and $P$ denote empirical and population averaging in this paragraph. Symmetrization with independent signs $\varepsilon_i$, followed by the exponential-moment bound for the at most $2\mathcal M(m)$ signed sample vectors, gives
\[
 \bbE\sup_h|(P_m-P)f_h|
 \leq2\bbE\sup_h\left|\frac1m\sum_{i=1}^m\varepsilon_i f_h(X_i)\right|
 \leq2B\sqrt{\frac{2\ln(2\mathcal M(m))}{m}}.
\]
Changing one feature changes the supremum on the left by at most $2B/m$. The bounded-differences inequality therefore gives, outside probability $\zeta/2$,
\[
 \sup_h|(P_m-P)f_h|
 \leq2B\sqrt{\frac{2\ln(2\mathcal M(m))}{m}}
       +B\sqrt{\frac{2\ln(2/\zeta)}m}.
\]
When $B=0$, this term vanishes.

Conditional on all current features, the noises are independent and satisfy \eqref{eq:subgaussian-cost-noise}. For each of the at most $\mathcal M(m)$ decision vectors, its empirical noise average is centered sub-Gaussian with variance proxy at most $\sigma^2D_{S,2}^2/m$. A conditional union bound, valid also when noise coordinates are dependent, yields
\[
 \sup_h\left|\frac1m\sum_{i=1}^mg_h(X_i,\xi_i)\right|
 \leq\sigma D_{S,2}
 \sqrt{\frac{2\ln(4\mathcal M(m)/\zeta)}m}
\]
outside probability $\zeta/2$. The population mean of $g_h$ is zero for every fixed $h$. Adding the two bounds and using the growth bound proves \eqref{eq:natarajan-loss-difference-bound}. The numerical constant $8$ exceeds the sum required by these inequalities. The argument controls a measurable supremum and neither truncates labels nor substitutes a sample-dependent representative into an unconditional concentration bound. Conditioning first on a preceding history proves the final assertion.
\hfill\Halmos
\end{proof}

\paragraph{Batch errors and sufficient scales.}
Allocate $\delta_k:=\delta/2^{k+1}$ for $k\geq0$, including warm-up, and set $m_0=n_0$. Define
\begin{equation}\label{equ:batch-bound}
 r_k:=16D_{S,2}(M_\mu+\sigma)
 \sqrt{\frac{1+d_{\mathrm N}\ln(e m_kK^2)+\ln(8/\delta_k)}{m_k}},
 \qquad k\geq0.
\end{equation}
Lemma~\ref{lemma:natarajan-concentration} bounds each selected empirical-minus-population difference relative to $h^*$ by $r_k/2$. Subtracting two such comparisons bounds the error for any pair of predictors by $r_k$. These are the comparisons used in retention and empirical optimality.

In this proof, let $u_k:=\ln(16m_k^2/\delta)$. Finite $d_{\mathrm N}$ gives $C_r:=\sup_{k\geq1}r_k\sqrt{m_k/u_k}<\infty$. Choose
\begin{equation}\label{eq:direct-spo-initial-scales}
 s_0\geq C_r+M_\mu D_{S,2},\qquad
 b_0\geq\max\{C_\phi\sqrt{r_0},
              \,4\,2^{1/4}C_\phi\sqrt{2s_0}\}.
\end{equation}
Since $\omega_S(\mu)\leq M_\mu D_{S,2}$, this choice also gives
$s_0\geq C_r+\omega_S(\mu)$.
Since $u_k/m_k$ decreases and $u_{k+1}\geq u_k$, these choices give $r_k\leq s_{2m_k-1}$, $r_1+\omega_S(\mu)\leq s_1$, and
\[
 \frac{C_\phi\sqrt{2s_{2m_k-1}}}{b_{2m_k-1}}
 =\frac{C_\phi\sqrt{2s_0}}{b_0}
 \left(\frac{2u_k}{u_{k+1}}\right)^{1/4}\leq\frac14.
\]
The constants depend on the noise parameter, conditional-mean bound, decision geometry, $d_{\mathrm N}$, confidence, and warm-up size, not on label support or the realized stream.

\paragraph{Dependence on the confidence input.}
For the known-horizon results, let $\gamma\in(0,1)$ denote the algorithmic confidence in this calculation and replace $\delta$ by $\gamma$ above. Choose $s_0,b_0$ equal to the sufficient lower bounds in \eqref{eq:direct-spo-initial-scales}. The numerator inside the square root in \eqref{equ:batch-bound} is at most $C u_k$ for $k\geq1$, where $C$ depends only on $d_{\mathrm N}$ and $K$. Thus, for fixed model bounds and $n_0$,
\[
 C_r=\mathcal O(1),\qquad s_0=\mathcal O(1),\qquad
 r_0=\mathcal O\!\left(\sqrt{\ln(16/\gamma)}\right),\qquad
 b_0=\mathcal O\!\left(\{\ln(16/\gamma)\}^{1/4}\right).
\]
The constants in these bounds are independent of $\gamma$. In particular, substituting $\gamma=\delta/T$ changes the scales only by logarithmic factors. The rate statements use these sufficient choices.

\paragraph{Supervised empirical-risk bounds.}
Let $\widehat R_m^{\rm full}$ be the empirical SPO risk on $m$ fully labeled i.i.d. observations from the original distribution. The supervised bounds used here have the following uniform form: for each $m\geq1$ and $\zeta\in(0,1)$, a deterministic $B_{\rm ERM}(m,\zeta)\geq0$ satisfies
\begin{equation}\label{eq:supervised-erm-certificate}
 \bbP\!\left(
 \forall h\in\mathcal H:\ 
 \widehat R_m^{\rm full}(h)\leq\widehat R_m^{\rm full}(h^*)
 \ \Longrightarrow\
 R_{\mathrm{SPO}}(h)-R_{\mathrm{SPO}}^*
 \leq B_{\rm ERM}(m,\zeta)\right)\geq1-\zeta.
\end{equation}
This condition applies to any fully supervised empirical minimizer. For the horizon-based statement in Theorem~\ref{thm:spo_loss_directly}(a), take $B_{\rm ERM}$ nonincreasing in its first argument. A valid bound can be enlarged to its nonincreasing envelope $\sup_{n\geq m}B_{\rm ERM}(n,\zeta)$, with the supremum over integers; noninteger arguments are rounded up. For example, Lemma~\ref{lemma:natarajan-concentration} with $\mathcal A=\mathcal X$ supplies
$B_{\rm ERM}(m,\zeta)=8D_{S,2}(M_\mu+\sigma)\sqrt{z(m,\zeta)/m}$.
A sharper bound can replace it whenever \eqref{eq:supervised-erm-certificate} is established under the same distributional assumptions. A risk guarantee for one particular supervised tie rule alone need not have this uniform form.

\begin{lemma}[Transfer of supervised SPO-ERM risk bounds]\label{lemma:spo-erm-transfer}
Fix a batch $k\geq2$ and an event $\mathcal E$ with $\bbP(\mathcal E^c)\leq\gamma$. Suppose that on $\mathcal E$, either the preceding retained class is contained in $\mathcal H^*$, or it contains $h^*$ and all its predictors induce the same decision on rejected features. If $B_{\rm ERM}$ satisfies \eqref{eq:supervised-erm-certificate}, then
\[
 \bbP\!\left(R_{\mathrm{SPO}}(h_{2m_k-1})-R_{\mathrm{SPO}}^*
       \leq B_{\rm ERM}(m_k,\zeta)\right)\geq1-\gamma-\zeta .
\]
\end{lemma}

\begin{proof}{\bfseries Proof of Lemma \ref{lemma:spo-erm-transfer}.}
On $\mathcal E$, an all-optimal preceding class gives zero excess
risk by nesting. Otherwise, supply latent costs for rejected
features in the current batch without changing the observed
process or its output. Under the technical convention, decision
agreement on rejected features holds simultaneously for all
retained predictors at the sampled features, including the
empirical minimizer selected using the current batch.
Thus, if $D_i$ is the acquisition indicator, for every retained $h$,
\begin{equation}\label{eq:spo-full-batch-excess-identity}
 \begin{aligned}
 \widehat R_{m_k}^{\rm full}(h)-\widehat R_{m_k}^{\rm full}(h^*)
 &=\frac1{m_k}\sum_{i\in\mathcal I_k}
 D_i\{\ell_{\mathrm{SPO}}(h(X_i),C_i)
          -\ell_{\mathrm{SPO}}(h^*(X_i),C_i)\}\\
 &=\widehat\ell_{2m_k-1}(h)-\widehat\ell_{2m_k-1}(h^*).
 \end{aligned}
\end{equation}
The right side is nonpositive at the returned empirical minimizer because $h^*$ belongs to the preceding retained class. Conditional on the preceding history, the full current batch is i.i.d. with the original law. The uniform event in \eqref{eq:supervised-erm-certificate} therefore fails with probability at most $\zeta$ and covers this data-dependent output. Intersect it with $\mathcal E$ and use a union bound. No independence between these two events, conditioning on future success, or full-class empirical optimality of $h_{2m_k-1}$ is required.
\hfill\Halmos
\end{proof}

\begin{proof}{\bfseries Proof of Theorem \ref{thm:spo_loss_directly}.}
Fix the prescribed horizon $T\geq3$, warm-up size $n_0\geq1$, and $\delta\in(0,1)$. Apply the batch-error and scale constructions in \eqref{equ:batch-bound}--\eqref{eq:direct-spo-initial-scales} with confidence input $\delta/T$. In particular, use $\delta_k:=\delta/(T2^{k+1})$ for $k\geq0$, including warm-up, so $\sum_{k\geq0}\delta_k=\delta/T$. All thresholds and confidence-set slacks below use this confidence input. The optimal set $\mathcal H^*$ may contain several predictors; $h^*\in\mathcal H^*$ is the fixed reference in Definition~\ref{def:near}.

\emph{A consequence of the error bound on two-element classes.}
Take any $h\notin\mathcal H^*$ and $g^*\in\mathcal H^*$.
The compact set $\mathcal G=\{h,g^*\}$ satisfies
$\mathcal G\cap\mathcal H^*=\{g^*\}$.
Assumption~\ref{assumption:upper-bound-pointwise}, applied to
full-population risk, gives
\begin{equation}\label{equ:spo-pairwise-optimal-error}
 \|h-g^*\|_\infty
 \leq\phi\bigl(R_{\mathrm{SPO}}(h)-R_{\mathrm{SPO}}^*\bigr)
 \leq C_\phi\sqrt{R_{\mathrm{SPO}}(h)-R_{\mathrm{SPO}}^*},
 \qquad h\notin\mathcal H^*,\quad g^*\in\mathcal H^*.
\end{equation}
This inequality is not asserted between two optimal predictors, whose distance can be positive despite zero excess risk.

\emph{Uniform concentration and parameter comparisons.}
For batch $k$, let $\mathcal I_k=\{m_k,\ldots,2m_k-1\}$ and let $\mathcal F_{m_k-1}$ contain the observations and algorithmic decisions preceding that batch. Conditional on this history, the reference predictor, threshold, and retained class are fixed. Define
\[
 R_{\mathrm{SPO},k}^{a}(h):=\bbE\!\left[
 \ell_{\mathrm{SPO}}(h(X),C)
 \mathbb I\{\nu_S(h_{m_k-1}(X))<b_{m_k-1}\}
 \mid\mathcal F_{m_k-1}\right].
\]
Lemma~\ref{lemma:natarajan-concentration} gives
\[
 \sup_{h\in\mathcal H}\left|
 \{R_{\mathrm{SPO},k}^{a}(h)-R_{\mathrm{SPO},k}^{a}(h^*)\}
 -\{\widehat\ell_{2m_k-1}(h)-\widehat\ell_{2m_k-1}(h^*)\}
 \right|\leq r_k/2
\]
with conditional failure probability at most $\delta_k$. Subtracting two comparisons bounds every pairwise deviation by $r_k$. The fully labeled warm-up satisfies the corresponding centered bound with $r_0/2$ and failure probability $\delta_0$. Let $\mathcal E$ be the event that these bounds hold for warm-up and every completed batch through $T$. A union bound gives $\bbP(\mathcal E^c)\leq\delta/T$.

The sufficient scales imply $r_k\leq s_{2m_k-1}$, $r_1+\omega_S(\mu)\leq s_1$, and $C_\phi\sqrt{2s_{2m_k-1}}\leq b_{2m_k-1}/4$. On $\mathcal E$, warm-up empirical optimality gives $R_{\mathrm{SPO}}(h_0)-R_{\mathrm{SPO}}^*\leq r_0$. If $h_0$ is nonoptimal, \eqref{equ:spo-pairwise-optimal-error} and the initial threshold choice give $\|h_0-h^*\|_\infty\leq C_\phi\sqrt{r_0}\leq b_0$. No distance to this particular reference is needed when $h_0$ is optimal.

\emph{Retention and decision stability.}
We first work on $\mathcal E$. For any fixed predictors $h,g$, conditioning on the feature cancels the common hindsight term:
\begin{equation}\label{eq:spo-conditional-loss-difference}
 \bbE[\ell_{\mathrm{SPO}}(h(X),C)-\ell_{\mathrm{SPO}}(g(X),C)\mid X]
 =\mu(X)^T\{w^*(h(X))-w^*(g(X))\}.
\end{equation}
Its absolute value is at most $\omega_S(\mu)$, although the realized loss need not be bounded. At the first endpoint, the rejected population difference is therefore bounded by $\omega_S(\mu)$. Pairwise concentration and population optimality give
\[
 \widehat\ell_1(g^*)-\widehat\ell_1(h_1)
 \leq r_1+\omega_S(\mu)\leq s_1,
 \qquad g^*\in\mathcal H^*.
\]
Hence $\mathcal H^*\subseteq\mathcal H_1$. Every retained predictor has excess risk at most $r_1+s_1+\omega_S(\mu)\leq2s_1$, and the returned empirical minimizer has excess risk at most $r_1+\omega_S(\mu)\leq s_1$.

We prove inductively that, after every completed update, either the retained class consists entirely of population minimizers, or it contains $\mathcal H^*$ and every retained predictor has excess risk at most twice the current slack. In the first case, nesting keeps all subsequent returned predictors optimal. For the second case, consider batch $k\geq2$ and put $e=C_\phi\sqrt{2s_{m_k-1}}$. Equation~\eqref{equ:spo-pairwise-optimal-error} places every nonoptimal retained predictor within $e$ of every population minimizer. Two nonoptimal predictors are within $2e$ by comparison with $h^*$; two optimal predictors are within $2e$ by comparison with any nonoptimal retained predictor. Thus the preceding retained class has diameter at most $2e\leq b_{m_k-1}/2$.

On a rejected feature $x$ outside the fixed exceptional set,
$\nu_S(h_{m_k-1}(x))\geq b_{m_k-1}$.
The technical convention, the diameter bound, and
Lemma~\ref{lemma:identical} imply that all retained predictors
induce the same decision there. Their realized SPO loss
differences therefore vanish on rejected observations.
Consequently, for every retained $h$,
\[
 R_{\mathrm{SPO},k}^{a}(h)-R_{\mathrm{SPO},k}^{a}(h^*)
 =R_{\mathrm{SPO}}(h)-R_{\mathrm{SPO}}^*.
\]
Thus selected-loss concentration controls full-population
excess risk, and no conditional error bound is needed.
Pairwise concentration and population optimality now give
\[
 \widehat\ell_{2m_k-1}(g^*)
 -\widehat\ell_{2m_k-1}(h_{2m_k-1})
 \leq r_k\leq s_{2m_k-1},
 \qquad g^*\in\mathcal H^*.
\]
All population minimizers are retained. Empirical optimality and the confidence-set update bound the excess risk of every retained predictor by $r_k+s_{2m_k-1}\leq2s_{2m_k-1}$, and that of the returned predictor by $r_k\leq s_{2m_k-1}$. This completes the induction. If an update removes every nonoptimal predictor, all later outputs remain optimal without requiring the fixed reference $h^*$ to remain in the class.

\emph{Risk bounds at the prescribed horizon.}
The predictor is unchanged between completed updates, so the induction gives
\[
 R_{\mathrm{SPO}}(h_T)-R_{\mathrm{SPO}}^*\leq s_{\bar T}
\]
on $\mathcal E$. If $h_T\in\mathcal H^*$, the near-degeneracy bound is immediate. Otherwise, \eqref{equ:spo-pairwise-optimal-error} gives
\[
 \|h_T-h^*\|_\infty
 \leq C_\phi\sqrt{s_{\bar T}}\leq b_T/4.
\]
Lemma~\ref{lemma:identical} makes their loss difference zero wherever $\nu_S(h^*(X))\geq2b_T$. On the remaining features, its conditional expectation is at most $\omega_S(\mu)$ by \eqref{eq:spo-conditional-loss-difference}. Thus
\[
 R_{\mathrm{SPO}}(h_T)-R_{\mathrm{SPO}}^*
 \leq\Psi(2b_T)\omega_S(\mu).
\]

Let $k$ index the last completed batch. Since $T\geq3$, we have $k\geq2$. On $\mathcal E$, the retention induction establishes the hypotheses of Lemma~\ref{lemma:spo-erm-transfer}. Intersect $\mathcal E$ with the full-batch uniform event in \eqref{eq:supervised-erm-certificate}, using $\zeta=\delta/2$. The transfer lemma gives
\[
 R_{\mathrm{SPO}}(h_T)-R_{\mathrm{SPO}}^*
 \leq B_{\rm ERM}(m_k,\delta/2)
 \leq B_{\rm ERM}(T/4,\delta/2).
\]
The last inequality follows from $2m_k-1\leq T<4m_k-1$, hence $m_k\geq\lceil T/4\rceil$, together with monotonicity and the rounding convention for $B_{\rm ERM}$. All three risk bounds hold on this intersection, whose failure probability is at most $\delta/T+\delta/2\leq\delta$. This proves part~(a).

\emph{The expected number of acquired labels.}
Let $D_t$ be the acquisition indicator. Before drawing $X_t$, let $\mathcal E_{t-1}$ be the event that warm-up concentration and all completed batch bounds through iteration $t-1$ have succeeded. This event is measurable with respect to the preceding history and satisfies $\bbP(\mathcal E_{t-1}^c)\leq\delta/T$.

On $\mathcal E_{t-1}$, a nonoptimal current predictor satisfies $\|h_{t-1}-h^*\|_\infty\leq b_{t-1}$ by the warm-up bound or the retention induction. Since distance to degeneracy is $1$-Lipschitz, acquisition then implies $\nu_S(h^*(X_t))<2b_{t-1}$. If the current predictor is instead some $g^*\in\mathcal H^*$, its conditional acquisition probability is at most $\bbP(\nu_S(g^*(X))\leq2b_{t-1})$. The next feature is independent of the preceding history, and $h^*\in\mathcal H^*$, so both cases are bounded by the supremum over individual population minimizers. On the complementary event, the acquisition probability is at most one. Therefore,
\[
 \bbE[D_t]
 \leq\sup_{g\in\mathcal H^*}
 \bbP\bigl(\nu_S(g(X))\leq2b_{t-1}\bigr)+\frac\delta T,
 \qquad 1\leq t\leq T.
\]
Summing and using $N_T=n_0+\sum_{t=1}^TD_t$ gives
\[
 \bbE[N_T]-n_0
 \leq\sum_{t=1}^T\sup_{g\in\mathcal H^*}
 \bbP\bigl(\nu_S(g(X))\leq2b_{t-1}\bigr)+\delta,
\]
which proves part~(b). This expectation argument uses only preceding-history events; it does not condition future features on simultaneous success or assume that an excluded optimum can return to a nested confidence set.
\hfill\Halmos
\end{proof}

\paragraph{The strength of the retained error-bound hypothesis.}
Assumption~\ref{assumption:upper-bound-pointwise} is stronger
than an error bound to the full optimal set. If
$g_1^*,g_2^*\in\mathcal H^*$ have positive sup distance, then
\eqref{equ:spo-pairwise-optimal-error} and the triangle inequality give
\[
 R_{\mathrm{SPO}}(h)-R_{\mathrm{SPO}}^*
 \geq\frac{\|g_1^*-g_2^*\|_\infty^2}{4C_\phi^2},
 \qquad h\notin\mathcal H^*.
\]
Hence the condition excludes classes containing distinct optimal
predictors together with nonoptimal predictors of arbitrarily
small excess risk. Admitting decision-equivalent rescalings
can therefore make the condition restrictive.
A unique minimizer can instead coexist with arbitrarily small
positive excess risk, as the following example shows.

\begin{example}[Continuous features with fixed prediction scale]
\label{example:spo-continuous-fixed-scale}
Let $X\sim\operatorname{Unif}[-1,1]$, $S=[-1,1]$, and
\[
 \mathcal H=\{h_\theta(x)=x-\theta:
                 \theta\in[-1/2,1/2]\},\qquad
 C=X-\theta^*+\xi,
\]
where $\theta^*\in[-1/4,1/4]$ is unknown and
$\xi\sim\mathcal N(0,\sigma^2)$ is independent of $X$.
Use the absolute-value norm and fix $w^*(0)=-1$.
Then $\mathcal H^*=\{h_{\theta^*}\}$,
Assumption~\ref{assumption:1} holds with
$\phi(\epsilon)=\sqrt{2\epsilon}$, and the induced decision
class has Natarajan dimension one.
Moreover, $\Psi(r)\leq\min\{r,1\}$, so
Assumption~\ref{assumption:noise} holds with $\kappa=1$
and $\bar b=1$.
The conditional mean is bounded, and nonoptimal predictors
have arbitrarily small positive excess SPO risk.
Thus Theorem~\ref{thm:spo_loss_directly} applies without an
excess-risk gap.
\end{example}

\begin{proof}{\bfseries Proof of Example
\ref{example:spo-continuous-fixed-scale}.}
The class is compact in the sup norm, and its predictors are
continuous. The conditional-mean predictor $h_{\theta^*}$
minimizes conditional SPO risk on every positive-probability
feature set, verifying Assumption~\ref{assu:consistent}.
The decisions of $h_\theta$ and $h_{\theta^*}$ differ only
between their thresholds. Since the feature density is $1/2$,
\[
 R_{\mathrm{SPO}}(h_\theta)-R_{\mathrm{SPO}}^*
 =\int_{\min\{\theta,\theta^*\}}^{\max\{\theta,\theta^*\}}
   |x-\theta^*|\,dx
 =\frac{(\theta-\theta^*)^2}{2}.
\]
Hence $h_{\theta^*}$ is the unique population minimizer and
\[
 \|h_\theta-h_{\theta^*}\|_\infty
 =|\theta-\theta^*|
 =\sqrt{2\{R_{\mathrm{SPO}}(h_\theta)-R_{\mathrm{SPO}}^*\}}.
\]
Every compact restriction containing a population minimizer
contains $h_{\theta^*}$, proving
Assumption~\ref{assumption:upper-bound-pointwise}.

The induced decisions form a threshold class with Natarajan
dimension one. Also, $\nu_S(c)=|c|$ gives
$\Psi(r)=\bbP(|X-\theta^*|\leq r)\leq\min\{r,1\}$.
The conditional mean has magnitude at most $5/4$, and the
Gaussian noise satisfies \eqref{eq:subgaussian-cost-noise}.
For any finite history, empirical objectives and historical
confidence constraints are constant on finitely many cells
determined by observed features. A deterministic choice among
cell representatives and endpoints therefore supplies a
measurable exact empirical minimizer. The one-sided continuity
of threshold losses also makes the empirical-process suprema
measurable.
Finally, the excess-risk identity tends to zero as
$\theta\to\theta^*$ through nonoptimal parameters.
Fixing the coefficient of $x$ excludes distinct positive
rescalings within the class.
\hfill\Halmos
\end{proof}

\paragraph{Why one reference does not control all optimal outputs.}
Consider a single feature value, $S=[-1,1]$, deterministic cost $C=1$, and $\mathcal H=\{h^{(0)},h^{(1)}\}$ with predictions zero and one, respectively. Fix the optimization oracle to return $-1$ also at cost zero. Both predictors have zero realized SPO loss. Every nonempty restricted class therefore consists only of its own minimizers, so the strong error-bound hypothesis holds, and the induced decision class has Natarajan dimension zero. Choose $h^*=h^{(1)}$ for Definition \ref{def:near}. Then $\Psi(r)=0$ for $r<1$, although an exact empirical tie rule returning $h^{(0)}$ acquires every label because its distance to degeneracy is zero. All positive schedules in the theorem permit this rule, and the confidence sets remain unchanged. Consequently, $N_T=n_0+T$ even though the risk is already optimal.  The supremum in Theorem \ref{thm:spo_loss_directly}(b), or a uniform optimal-predictor margin bound for its rate consequence, is therefore necessary for this general empirical-minimizer formulation. It does not redefine $\Psi$.

\paragraph{Prediction scale and the margin condition.}
Assumption~\ref{assumption:noise} requires common margin constants
for all SPO-risk minimizers, not merely the fixed reference used
in $\Psi$. This distinction matters because positive scaling preserves
optimal decision sets but changes the distance to degeneracy.
For a nontrivial polyhedral feasible region, the degenerate set is
a cone, so
\[
 \nu_S(ac)=a\nu_S(c),\qquad a>0.
\]
Consequently, for any predictor $g$ and $r>0$,
\[
 \bbP\bigl(\nu_S(ag(X))\leq r\bigr)
 =\bbP\bigl(\nu_S(g(X))\leq r/a\bigr).
\]
Thus, if $g$ satisfies a margin bound with scale $\bar b$ and exponent
$\kappa$, its positive rescaling satisfies the corresponding bound
with scale $a\bar b$ and the same exponent.

To see the implication for the uniform condition, suppose
$g^*\in\mathcal H^*$ induces a unique optimal decision almost surely
and $a_jg^*\in\mathcal H$ for a sequence $a_j\downarrow0$.
Positive scaling preserves these decisions, so
$a_jg^*\in\mathcal H^*$ for every $j$. For any fixed $r>0$,
\[
 \sup_{g\in\mathcal H^*}
 \bbP\bigl(\nu_S(g(X))\leq r\bigr)
 \geq
 \lim_{j\to\infty}
 \bbP\bigl(\nu_S(g^*(X))\leq r/a_j\bigr)
 =1.
\]
This contradicts Assumption~\ref{assumption:noise} whenever
$0<r<\bar b$. The almost-sure uniqueness used here is itself
necessary under that assumption, since the margin bound excludes
positive probability of degeneracy. Hence the uniform condition
rules out hypothesis classes admitting arbitrarily small positive
rescalings of an optimal predictor. It is a substantive restriction
on the optimal predictor set, not a consequence of SPO optimality.

This restriction does not require margins to be bounded away from
zero almost surely. It controls the probability of small margins
uniformly over optimal predictors and permits arbitrarily small
positive margins. A positive lower bound on the prediction norm
alone would also be insufficient, since a nonzero prediction can
lie arbitrarily close to a decision boundary.

The surrogate analyses avoid the same rescaling ambiguity under
their population-curvature assumptions. For squared loss,
$\Sigma_X\succ0$ identifies the coefficient minimizer; for the general
convex-surrogate result, positive population curvature provides
uniqueness. A rescaling that changes this predictor therefore cannot
remain surrogate-optimal. Assumptions~\ref{assumption:holder}
and~\ref{assumption:surrogate-information} nevertheless impose
local acquisition and information conditions around the
respective minimizer. Neither convexity alone nor uniqueness
ensures those bounds or keeps predictions away from zero or degeneracy.
The surrogate results thus replace uniform control over all SPO-risk
minimizers with conditions tied to the identified training-risk
minimizer; they do not dispense with margin assumptions.

\subsection{Margin examples and direct-SPO label counts}

\begin{proof}{\bfseries Proof of Examples \ref{example:2} and \ref{example:3}}
For Example \ref{example:2}, choose $\bar b=\underline k$. If $b<\underline k$, the near-degeneracy probability is zero; if $b\geq\underline k$, it is at most one and hence at most $(b/\underline k)^\kappa$. This proves the statement for every $\kappa>0$.

For Example \ref{example:3}, continuity of the density on the closed unit ball gives a finite density upper bound. The degenerate set of a bounded polyhedron is contained in the union of the finitely many hyperplanes $\{c:c^T(v_i-v_j)=0\}$ determined by distinct extreme points. Within a fixed bounded ball, the Lebesgue measure of the Euclidean $b$-neighborhood of each such hyperplane is at most a constant times $b$. Norm equivalence gives the same order for the norm defining $\nu_S$. A finite union and the density upper bound therefore give $\bbP(\nu_S(U)\leq b)\leq Lb$ for a finite $L\geq1$ and every $b\geq0$. This argument covers the stated continuous density; it does not require $U$ to be uniform. In particular, $\bbP(\nu_S(U)=0)=0$.

The degenerate set is a cone, so the distance is positively homogeneous: $\nu_S(ac)=a\nu_S(c)$ for $a>0$. Consequently, $\nu_S(Y)=\nu_S(U)^{1/\kappa}$ almost surely. Thus $\Psi(b)=\bbP(\nu_S(U)\leq b^\kappa)\leq Lb^\kappa$. Taking $\bar b=L^{-1/\kappa}$ and using $\mathcal H^*=\{h^*\}$ proves the uniform margin condition. This is a distributional construction for $\Psi$; it does not assert that the transformed predictions are bounded or that the other learning assumptions hold.
\hfill\Halmos
\end{proof}

\begin{proof}{\bfseries Proof of Proposition \ref{prop:sublinear_spo}.}
Fix the prescribed horizon $T$ and run the algorithm with confidence input $\delta/T$. Write $A:=\ln(16T/\delta)$ only in this proof. The confidence calculation above gives $s_0=\mathcal O(1)$ and $b_0=\mathcal O(A^{1/4})$, with constants independent of $T$ and $\delta$. For every batch that meets the stream through $T$, its logarithmic factor is $u_k=\ln(16Tm_k^2/\delta)\leq3A$. The same bound applies to the next-batch factor assigned at the last completed endpoint, since its batch size is at most $T+1$. The most recently completed batch has size at least $(T+1)/4$. Hence
\[
 s_{\bar T}=\mathcal O(A^{1/2}T^{-1/2}),\qquad
 b_T=\mathcal O(A^{1/2}T^{-1/4}).
\]
Theorem~\ref{thm:spo_loss_directly}(a), applied at the internal confidence $\delta/T$, bounds the same excess SPO risk by $s_{\bar T}$ and by $\omega_S(\mu)\Psi(2b_T)$. Assumption~\ref{assumption:noise} implies the margin bound at the fixed reference, so their minimum is
\[
 \mathcal O\!\left(\min\{A^{1/2}T^{-1/2},\ A^{\kappa/2}T^{-\kappa/4}\}\right)
 =\widetilde{\mathcal O}(T^{\min\{-1/2,-\kappa/4\}}).
\]
This holds with probability at least $1-\delta/T\geq1-\delta$. For $T\geq3$, intersecting with the same final-batch certificate event used in Theorem~\ref{thm:spo_loss_directly}(a) gives the additional $B_{\rm ERM}(T/4,\delta/2)$ bound at confidence at least $1-\delta$. Thus the explicit margin rate need not be the best available risk bound. The error bound and decision-stability induction do not require convexity of the retained predictor classes.

For labels, use the known-horizon bound in Theorem~\ref{thm:spo_loss_directly}(b) and the uniform margin condition. Batch~1 contributes at most one label. In every later batch, $m_k\geq t/2$ at stream iteration $t$, so $b_{t-1}\leq C A^{1/2}t^{-1/4}$. Thus
\[
 \bbE[N_T]-n_0
 \leq1+\sum_{t=2}^T\min\{1,C A^{\kappa/2}t^{-\kappa/4}\}+\delta,
\]
where $C$ may depend on the fixed model and margin constants, but not on $T$ or $\delta$. For $0<\kappa<4$, summing $t^{-\kappa/4}$ gives $\mathcal O(A^{\kappa/2}T^{1-\kappa/4})$. At $\kappa=4$, the sum is at most $\mathcal O(A^2\ln(T+2))$. If $\kappa>4$, split the capped sum at a sufficiently large constant multiple of $A^2$. Its initial part is $\mathcal O(A^2)$, and an integral bound on the remaining power-law tail gives the same order. These are the stated $\widetilde{\mathcal O}$ bounds, and each is $o(T)$ for fixed $\delta$ and $n_0$. A margin condition only on the fixed reference would not bound acquisition after an arbitrary different optimal predictor is returned; the supremum over optimal predictors remains essential.
\hfill\Halmos
\end{proof}

\begin{proof}{\bfseries Proof of Example \ref{example:SPO_ideal}.}
The conditional-mean predictor minimizes SPO risk on every
positive-probability feature restriction, so
Assumption~\ref{assu:consistent} holds.
Finite feature support and an extreme-point oracle yield finitely
many decision functions and hence finitely many full-population
SPO risks. If nonoptimal predictors exist, let $\Delta>0$
be their smallest positive excess risk.
Compactness gives a finite class diameter
$D_{\mathcal H}:=\sup_{h,g\in\mathcal H}\|h-g\|_\infty$.

For any nonempty compact $\mathcal G\subseteq\mathcal H$
containing a population minimizer, an optimal $h\in\mathcal G$
has zero distance to $\mathcal G\cap\mathcal H^*$.
Otherwise, its excess risk is at least $\Delta$ and its distance
to that set is at most $D_{\mathcal H}$.
Thus Assumption~\ref{assumption:upper-bound-pointwise} holds
with $\phi(\epsilon)=D_{\mathcal H}\sqrt{\epsilon/\Delta}$.
If every predictor is optimal, the condition is immediate.
\hfill\Halmos
\end{proof}

\subsection{Auxiliary error bounds for SPO+ and squared loss}\label{appendix:auxiliary-error-bounds}

The following sufficient condition is used only in the auxiliary error-bound calculations. The main risk and label theorems do not impose it as a separate assumption.
\begin{assumption}[Recoverability]\label{assumption:recover}
There exists a constant $\varkappa>0$ such that, for every $h\in\mathcal H$, every optimal predictor $h^*\in\mathcal H^*$, and almost every feature vector $x'$, the pointwise squared prediction error satisfies $\|h(x')-h^*(x')\|_2^2\leq\varkappa\bbE[\|h(X)-h^*(X)\|_2^2]$.
\end{assumption}

The results in this subsection describe risk geometry under their own distributional assumptions. They are not substituted for the finite-sample information and score conditions in the main surrogate analysis. Recoverability in Assumption \ref{assumption:recover} converts an expected squared prediction error to a sup-norm error under the distribution on which that condition is verified. Restricted conditional minimizers require the corresponding pairwise conditional property.

For the distributional conditions, let $\mathcal P_{\mathrm{cont,symm}}$ denote the joint distributions whose conditional cost density is continuous and centrally symmetric about the conditional mean. As in the existing SPO+ risk analysis, write $\bbP(c\mid x)$ for this density. For constants $M\geq1$, $0<\alpha\leq1$, and $\beta>0$ (here $\alpha$ is a density-domination weight and $\beta$ bounds the conditional mean relative to the noise scale), let $\mathcal P_{M,\alpha,\beta}$ be the distributions in $\mathcal P_{\mathrm{cont,symm}}$ such that, for each $x$ and conditional mean $\bar c=\bbE[c\mid x]$, there exists a noise scale $\sigma\in(0,M]$ with $\|\bar c\|_2\leq\beta\sigma$ and $\bbP(c\mid x)\geq\alpha\mathcal N(\bar c,\sigma^2I)$ for every $c\in\bbR^d$. Here, the Gaussian expression denotes its density. This is the distributional lower bound used in \citet{liu2021risk}; it permits unbounded costs, so the following result concerns risk geometry rather than the exponential-moment concentration argument used for direct SPO learning.

For a bounded polyhedral feasible region, recall its diameter $D_S$ and define its width constant by $d_S:=\min_{\|v\|_2=1}\{\max_{w\in S}v^Tw-\min_{w\in S}v^Tw\}$. A nonempty interior ensures $d_S>0$. Also define the geometric constant $\Xi_S:=(1+2\sqrt{3}D_S/d_S)^{1-d}$. These quantities describe how the feasible region enters the conversion from surrogate-risk error to prediction error.

\begin{lemma}[Example of $\phi$]\label{lemma:phi-poly}
Use the $\ell_2$ norm. Suppose that Assumption \ref{assumption:recover} holds, that $S$ is a bounded polyhedron with width $d_S>0$, and that the maximum prediction norm $\rho(\hat{\mathcal C})$ is finite. Suppose further that the hypothesis class contains the conditional-mean predictor $h^*(x)=\bbE[c\mid x]$ and that $\mathcal D\in\mathcal P_{M,\alpha,\beta}$. Then there is a finite constant $k>0$ such that, for every $h\in\mathcal H$ and $\epsilon>0$, the inequality $R_{\mathrm{SPO}+}(h)-R_{\mathrm{SPO}+}(h^*)\leq\epsilon$ implies $\|h-h^*\|_\infty\leq k\sqrt{\epsilon}$.
\end{lemma}

Lemma \ref{lemma:phi-poly} explains the square-root order through two separate properties: the surrogate risk controls expected squared prediction error relative to the conditional-mean predictor, and recoverability controls its pointwise size. The same argument applies under a conditional feature distribution when recoverability holds there with the same constant. For a restricted class containing $h^*$, the risk comparison remains with that predictor. For a compact convex restricted class not containing $h^*$, one must additionally establish curvature relative to its constrained minimizer and uniform recoverability for the predictors being compared. The continuous examples in Section~\ref{sec:phi_examples} establish the global error bound and corresponding bounds on locally selected regions, not a common error-bound constant over every measurable feature restriction.

\begin{example}[Error bound for the squared loss]\label{example:squared-loss}
Use the $\ell_2$ norm and the squared loss $\ell(\hat c,c)=\|\hat c-c\|_2^2$. Suppose that $\bbE\|C\|_2^2<\infty$, that all predictors have finite second moments, that the conditional-mean predictor $h^*(x)=\bbE[C\mid X=x]$ belongs to $\mathcal H$, and that Assumption \ref{assumption:recover} holds with constant $\varkappa$. Then $R_\ell(h)-R_\ell(h^*)=\bbE\|h(X)-h^*(X)\|_2^2$ for every $h\in\mathcal H$. Consequently, an excess squared-loss risk of at most $\epsilon$ implies $\|h-h^*\|_\infty\leq\sqrt{\varkappa\epsilon}$, so the error bound function can be chosen as $\phi(\epsilon)=\sqrt{\varkappa\epsilon}$. The same conclusion holds under a conditional feature distribution when the recoverability condition holds under that distribution. For a nonempty compact convex restricted predictor set, the corresponding error bound relative to its conditional squared-loss minimizer also holds when recoverability holds uniformly between the predictors being compared.
\end{example}
The identity in Example \ref{example:squared-loss} does not by itself provide conditional recoverability. Appendix \ref{appendix:continuous-gaussian-linear} verifies it directly for its scalar Gaussian-feature benchmark, whose prediction differences are constant in the feature. Section \ref{sec:phi_examples} verifies it for a different bounded continuous slope model, where informative feature selection supports the same-risk label comparison.

\begin{proof}{\bfseries Proof of Lemma \ref{lemma:phi-poly}.}
Fix a feature vector $x$, let $\bar c:=\bbE[c\mid x]$ be its conditional mean, and let $\Delta:=h(x)-h^*(x)$ be its prediction error. The conditional SPO+ identity used in \citet{elmachtoub2022smart} gives $\bbE[\ell_{\mathrm{SPO}+}(h(x),c)-\ell_{\mathrm{SPO}+}(h^*(x),c)\mid x]=\bbE[(c+2\Delta)^T\{w^*(c)-w^*(c+2\Delta)\}\mid x]$.

The integrand in this identity is nonnegative by optimality of $w^*(c+2\Delta)$. We may therefore first lower-bound the original cost density by $\alpha$ times the Gaussian density in $\mathcal P_{M,\alpha,\beta}$ and then perform the geometric calculation under that Gaussian law. This step does not infer a conditional-coordinate density bound directly from a joint-density bound.

For $d\geq2$, write the error size as $\kappa:=\|\Delta\|_2$ locally in this proof. After an orthogonal change of coordinates, $\Delta=\kappa e_d$ and $c=(c',\xi)$, with conditional means $\bar c'$ and $\bar\xi$. Apply the Gaussian case of Lemma 7 of \citet{liu2021risk} and let its truncation radius tend to infinity. The normalized incomplete-gamma factor then tends to one. Including the density-domination weight once, the resulting lower bound is $\alpha\tilde\kappa\kappa\Xi_Sd_S\exp\{-(3\tilde\kappa^2+3\bar\xi^2+\|\bar c'\|_2^2)/(2\sigma^2)\}/(4\sqrt{2\pi}\sigma)$ for $0<\tilde\kappa\leq\kappa$. When $\kappa=0$, the conclusion is immediate.

Set $\tilde\kappa:=\min\{\kappa,\sigma\}$. The mean bound $\|\bar c\|_2\leq\beta\sigma$ and noise-scale bound $\sigma\leq M$ show that the conditional excess SPO+ risk is at least $\alpha\Xi_Sd_S\min\{\kappa^2/M,\kappa\}/(4\sqrt{2\pi}e^{3(1+\beta^2)/2})$. Since both predictions belong to the prediction range, their distance satisfies $\kappa\leq2\rho(\hat{\mathcal C})$. It follows that $\min\{\kappa^2/M,\kappa\}\geq\kappa^2/\max\{M,2\rho(\hat{\mathcal C})\}$.

Taking expectations over $X$ gives $R_{\mathrm{SPO}+}(h)-R_{\mathrm{SPO}+}(h^*)\geq(\varkappa/k^2)\bbE[\|h(X)-h^*(X)\|_2^2]$, where the error bound constant is $k^2:=[4\sqrt{2\pi}\varkappa/(\alpha\Xi_Sd_S)]\,\allowbreak\max\{M,2\rho(\hat{\mathcal C})\}\,\allowbreak e^{3(1+\beta^2)/2}$. Recoverability then gives $\|h-h^*\|_\infty^2\leq k^2\{R_{\mathrm{SPO}+}(h)-R_{\mathrm{SPO}+}(h^*)\}$ for $d\geq2$.

For $d=1$, the feasible set is an interval of length $D_S=d_S>0$, and the same conclusion follows directly without a dimension-zero gamma factor. Under a Gaussian cost with mean $\bar c$ and standard deviation $\sigma$, the derivative of the scalar expected SPO+ loss at $p$ is $2D_S\{\Phi((2p-\bar c)/\sigma)-\Phi(\bar c/\sigma)\}$. Since $|\bar c|/\sigma\leq\beta$, integration of this derivative from $\bar c$ to $\bar c+\Delta$ gives excess risk at least $D_S\varphi(\beta+1)\min\{\Delta^2/\sigma,|\Delta|\}$, where $\varphi$ is the standard normal density. Indeed, the absolute derivative along that segment is at least $2D_S\varphi(\beta+1)\min\{2t/\sigma,1\}$ at displacement $t\geq0$. Multiplying by the density weight $\alpha$, using $\sigma\leq M$ and $|\Delta|\leq2\rho(\hat{\mathcal C})$, and applying recoverability gives a finite square-root error-bound constant also in the scalar case.
\hfill\Halmos
\end{proof}

\begin{proof}{\bfseries Proof of Example \ref{example:squared-loss}.}
Let $h^*(X)=\bbE[C\mid X]$ denote the conditional-mean predictor. Expanding the squared losses gives
\[
\|h(X)-C\|_2^2-\|h^*(X)-C\|_2^2
=\|h(X)-h^*(X)\|_2^2
+2\{h(X)-h^*(X)\}^T\{h^*(X)-C\}.
\]
The cross term has zero conditional expectation given $X$. The assumed second moments make all terms integrable, so taking expectations yields $R_\ell(h)-R_\ell(h^*)=\bbE\|h(X)-h^*(X)\|_2^2$. Recoverability then gives $\|h-h^*\|_\infty^2\leq\varkappa\{R_\ell(h)-R_\ell(h^*)\}$, proving the original-distribution bound.

For any measurable feature set $\mathcal A$ with positive probability, the cross term also has zero expectation conditional on $X\in\mathcal A$. Thus, the same risk identity holds under that conditional distribution, and conditional recoverability gives the same error bound function.

For the restricted-class statement, let $\mathcal G$ be a nonempty compact convex predictor set and let $g\in\mathcal G_{\ell,\mathcal A}^*$ be a conditional squared-loss minimizer. The risk is continuous on a sup-norm compact set of square-integrable predictors, so a minimizer exists. For every $h\in\mathcal G$, convexity of $\mathcal G$ and optimality of $g$ imply $\bbE[(g(X)-C)^T(h(X)-g(X))\mid X\in\mathcal A]\geq0$. Expanding the risk difference therefore gives
\[
R_{\ell,\mathcal A}(h)-R_{\ell,\mathcal A}(g)
\geq\bbE[\|h(X)-g(X)\|_2^2\mid X\in\mathcal A].
\]
If recoverability holds uniformly for the compared predictors, namely $\|h-g\|_{\infty,\mathcal A}^2\leq\varkappa\bbE[\|h(X)-g(X)\|_2^2\mid X\in\mathcal A]$, it follows that $\mathrm{Dist}_{\mathcal G_{\ell,\mathcal A}^*,\mathcal A}(h)\leq\sqrt{\varkappa\{R_{\ell,\mathcal A}(h)-R_{\ell,\mathcal A}^{\mathcal G,*}\}}$. This proves the asserted restricted-class bound. Its constant is uniform over the feature sets and predictor sets precisely when the recoverability constant is uniform over those sets.
\hfill\Halmos
\end{proof}

\subsection{An auxiliary quadratic comparison for a well-specified SPO+ model}\label{appendix:quadratic-comparison}
The following comparison bounds conditional excess SPO+ risk from above by a squared prediction error. It is retained as an auxiliary risk-geometry result, rather than as Assumption \ref{assumption:holder}. Its strongly convex feasible region differs from the polyhedral setting of the main algorithm, and the conditional-mean predictor is assumed to belong to the working class. The scalar constant denoted by $\eta$ in this auxiliary lemma is local to that comparison; it is not the feature-eigenvalue parameter of Section \ref{sec:general_loss}.

\begin{lemma}[Existence of $\eta$ for $\ell_{\mathrm{SPO}+}$]\label{lemma:holder}
Use the $\ell_2$ norm, and let $f:\bbR^d\to\bbR$ be $\mu_S$-strongly convex and $L_S$-smooth, with $L_S\geq\mu_S>0$. Suppose that $S=\{w:f(w)\leq r\}$, where $r>f_{\min}:=\min_wf(w)$, and that the hypothesis class contains $h^*(x)=\bbE[c\mid x]$. Suppose also that the prediction range is bounded and that, for every feature vector $x$, the conditional cost distribution is rotationally symmetric about its mean and satisfies $\bbP(\|c\|_2\geq\beta\mid x)=1$ for a constant $\beta>0$. Then there is a finite constant $\eta>0$ such that $|\bbE[\ell_{\mathrm{SPO}+}(\hat c,c)-\ell_{\mathrm{SPO}+}(h^*(x),c)\mid x]|\leq\eta\|\hat c-h^*(x)\|_2^2$ for every $x$ and every $\hat c\in\hat{\mathcal C}$.
\end{lemma}

\begin{proof}{\bfseries Proof of Lemma \ref{lemma:holder}.}
Fix a feature vector $x$, write its conditional mean as $\bar c:=h^*(x)$, and let $\Delta:=\hat c-\bar c$ be the prediction error. By conditional symmetry, the conditional excess SPO+ risk is $\bbE[(c+2\Delta)^T\{w^*(c)-w^*(c+2\Delta)\}\mid x]$. The integrand is nonnegative by optimality of $w^*(c+2\Delta)$.

For nonzero cost vectors $c_1,c_2$, the upper bounds in Lemmas 1 and 2 of \citet{liu2021risk} give
\begin{align*}
c_1^T\{w^*(c_2)-w^*(c_1)\}
&\leq\frac{L_S\|c_1\|_2}{2\sqrt{2\mu_S(r-f_{\min})}}
\|w^*(c_2)-w^*(c_1)\|_2^2,\\
\|w^*(c_2)-w^*(c_1)\|_2
&\leq\frac{\sqrt{2L_S(r-f_{\min})}}{\mu_S}
\left\|\frac{c_1}{\|c_1\|_2}-\frac{c_2}{\|c_2\|_2}\right\|_2.
\end{align*}
Squaring the second inequality before substitution gives the coefficient $L_S^2\|c_1\|_2\sqrt{r-f_{\min}}/(\sqrt2\mu_S^{5/2})$. Set $c_1=c+2\Delta$ and $c_2=c$. For $c+2\Delta\ne0$,
\[
 \left\|\frac c{\|c\|_2}-\frac{c+2\Delta}{\|c+2\Delta\|_2}\right\|_2
 \leq\frac{4\|\Delta\|_2}{\|c\|_2},\qquad
 \frac{\|c+2\Delta\|_2}{\|c\|_2^2}
 \leq\frac1\beta+\frac{4\rho(\hat{\mathcal C})}{\beta^2}.
\]
The second inequality uses $\|c\|_2\geq\beta$ and $\|\Delta\|_2\leq2\rho(\hat{\mathcal C})$. If $c+2\Delta=0$, the original objective difference is zero. Substituting and taking the conditional expectation therefore proves the result with
\[
 \eta:=\frac{16L_S^2\sqrt{r-f_{\min}}}{\sqrt2\mu_S^{5/2}}
 \left\{\frac1\beta+\frac{4\rho(\hat{\mathcal C})}{\beta^2}\right\}.
\]
No upper bound on the realized cost vector is needed.
\hfill\Halmos
\end{proof}

\begin{proof}{\bfseries Proof of Proposition \ref{prop:separable_bayes}.}
For any realized pair satisfying the proposition's condition, first suppose $\nu_S(\bar h(x))>0$. Then both $c$ and $2\bar h(x)-c$ are within $\varrho\nu_S(\bar h(x))<\nu_S(\bar h(x))$ of $\bar h(x)$. Lemma \ref{lemma:identical} gives identical optimal decisions for these three cost vectors. The identity $\ell_{\mathrm{SPO}+}(\hat c,c)=(2\hat c-c)^T\{w^*(c)-w^*(2\hat c-c)\}$ therefore makes the SPO+ loss zero. If $\nu_S(\bar h(x))=0$, the assumed inequality forces $c=\bar h(x)$; the two reflected costs coincide, and the same identity again gives zero loss. Nonnegativity and $\ell_{\mathrm{SPO}}\leq\ell_{\mathrm{SPO}+}$ prove both zero risks and optimality of $\bar h$.
\hfill\Halmos
\end{proof}

\paragraph{Covering bound for the bounded-cost SPO+ result.}
Here $\ell=\ell_{\mathrm{SPO}+}$ and the cost and prediction ranges are bounded, as assumed in Theorem~\ref{thm:sporiskcompare}. Choose $W>0$ such that $0\leq\ell(h(x),c)\leq W$ on these ranges; it also bounds differences between two losses. Retain $\delta_k:=\delta/2^{k+1}$, $m_0=n_0$, and $m_k=2^{k-1}$. In this proof, define the batch error by
\begin{equation}\label{eq:spop-cover-batch-bound}
 r_k:=2W\left\{
 \sqrt{\frac{\ln\{2\widehat N_1(W/m_k,\ell\circ\mathcal H)/\delta_k\}}{2m_k}}
 +\frac4{m_k}\right\},\qquad k\geq0.
\end{equation}
The covering number is the uniform loss cover in Definition~\ref{def:covering-number}, not the Natarajan dimension of the induced decisions.

Condition on a batch's preceding history. Choose a uniform cover of radius $W/m_k$ and one loss function from each nonempty covering ball. These measurable representatives form a $2W/m_k$-cover without increasing its cardinality. Every representative multiplied by the fixed acquisition indicator lies in $[0,W]$. Hoeffding's inequality and a union bound give a deviation at most
$W\sqrt{\ln\{2\widehat N_1(W/m_k,\ell\circ\mathcal H)/\delta_k\}/(2m_k)}$
for every selected representative, outside probability $\delta_k$. The empirical and population approximation errors together are at most $4W/m_k$. Consequently,
\[
 \sup_{h\in\mathcal H}
 |R^a_{\ell,k}(h)-\widehat\ell_{2m_k-1}(h)|\leq r_k/2,
\]
where $R^a_{\ell,k}$ is the conditional population objective with the same fixed acquisition indicator. This is the uniform-cover argument for the separate bounded-cost result.

\begin{proof}{\bfseries Proof of Theorem \ref{thm:sporiskcompare}.}
Use the bounded-loss concentration bound \eqref{eq:spop-cover-batch-bound}. Let $u_k=\ln(16m_k^2/\delta)$ and $C_r=\sup_{k\geq1}r_k\sqrt{m_k/u_k}$. Logarithmic covering growth makes $C_r$ finite. Sufficient scales are $s_0\geq C_r+W$ and
\[
 b_0\geq\max\{\phi(r_0),\,2^{1/4}2(3-\varrho)C_\phi\sqrt{2s_0}/(1-\varrho)\}.
\]
The constant $W$ bounds the rejected-loss difference. Separability gives zero realized loss at $h^*$, and bounded costs imply a bounded conditional mean.

Use Assumption \ref{assumption:1}. The full-class comparisons below are valid even when $\mathcal H$ is nonconvex; convexity of the sample SPO+ loss does not change the geometry of that class. With the allocation $\delta_k=\delta/2^{k+1}$, the selected-loss deviations are at most $r_k/2$ simultaneously, together with the fully labeled warm-up deviation, outside probability $\delta$. The initial comparisons and $u_{k+1}\geq u_k$ imply $r_k\leq s_{2m_k-1}$ and $C_\phi\sqrt{2s_{2m_k-1}}\leq a_\varrho b_{2m_k-1}$, where $a_\varrho:=(1-\varrho)/(2(3-\varrho))$. The warm-up empirical minimizer is within $\phi(r_0)\leq b_0$ of $h^*$.

For batch 1, the rejected-risk difference is bounded in absolute value by $W$. Empirical optimality therefore retains $h^*$ because its empirical improvement deficit is at most $r_1+W\leq s_1$. Every newly retained predictor has population excess risk at most $r_1+s_1+W\leq2s_1$, so Assumption \ref{assumption:1} and uniqueness place it within $a_\varrho b_1$ of $h^*$. This starts the induction.

Suppose all predictors in the preceding retained class are within $D:=C_\phi\sqrt{2s_{m_k-1}}\leq a_\varrho b$ of $h^*$, where $b=b_{m_k-1}$. If a feature is rejected by the reference predictor $g$, the Lipschitz property of distance to degeneracy gives $\nu_S(h^*(x))\geq b-D$. For any retained $h$, separability then gives $\|2h(x)-C-h^*(x)\|\leq2D+\varrho\nu_S(h^*(x))$. Since $D<(1-\varrho)b/(3-\varrho)$, we have $2D<(1-\varrho)(b-D)\leq(1-\varrho)\nu_S(h^*(x))$. Both $C$ and $2h(x)-C$ therefore induce the same decision as $h^*(x)$. The SPO+ loss of every retained predictor is zero on the rejected feature--cost pairs.

Accordingly, the rejected excess surrogate risk is zero from this batch onward. The selected uniform deviation and empirical optimality retain $h^*$ with deficit at most $r_k\leq s_{2m_k-1}$, bound every retained predictor's excess risk by $r_k+s_{2m_k-1}\leq2s_{2m_k-1}$, and bound the output's excess risk by $r_k\leq s_{2m_k-1}$. Assumption \ref{assumption:1} restores the distance comparison for the next batch. This completes the induction for the reset working sets. For batch 1, the output risk is bounded by $r_1+W\leq s_1$; for $k\geq2$, it is bounded by $r_k\leq s_{2m_k-1}$. The distance bound gives identical decisions outside $\{\nu_S(h^*(X))<2b_T\}$. Applying \eqref{eq:spo-conditional-loss-difference} on the remaining features yields the SPO-risk conclusion with $\omega_S(\mu)$. On the simultaneous success event $\mathcal E$, each acquisition indicator is bounded by $\mathbb I\{\nu_S(h^*(X_t))\leq2b_{t-1}\}$. Pointwise, including failed histories,
\[
 N_T\leq n_0+\sum_{t=1}^T\mathbb I\{\nu_S(h^*(X_t))\leq2b_{t-1}\}
        +T\mathbb I_{\mathcal E^c}.
\]
Taking expectations without conditioning the feature indicators on $\mathcal E$ proves the fixed-confidence bound, since $\bbP(\mathcal E^c)\leq\delta$.

For a horizon fixed before warm-up, replace the confidence input by $\delta/T$ in the same construction and choose the sufficient scales at that confidence. The event now has failure probability at most $\delta/T$, so the last term contributes at most $\delta$ to expectation, and the risk bounds hold through $T$ with probability at least $1-\delta/T$. Logarithmic covering growth in \eqref{eq:spop-cover-batch-bound} gives $C_r=\mathcal O(1)$ and $r_0=\mathcal O(\sqrt{\ln(16T/\delta)})$. Choosing the displayed sufficient scales with equality therefore gives $s_0=\mathcal O(1)$ and $b_0=\mathcal O(\{\ln(16T/\delta)\}^{1/4})$. The loss bound $W$ and the separability constant $\varrho$ are fixed, so they change only the multiplicative constants. Under $\Psi(r)\leq(r/\bar b)^\kappa$, the capped-sum calculation in the proof of Proposition~\ref{prop:sublinear_spo} therefore proves the asserted sublinear expected-label orders. No recovery after a failed nested-set update is assumed.
\hfill\Halmos
\end{proof}

\section{Finite comparisons for squared loss}
\label{app:information-alpha-two}

We distinguish the full-population eigenvalue $\eta$ from the information on a selected feature set. The initial fully labeled batches use only the former. After that initial period, nesting preserves a fixed neighborhood in which the quantitative version of Assumption \ref{assumption:holder} applies. The inequalities defining every cutoff involve known bounds and population constants; qualitative convergence without a rate is not used to produce a numerical cutoff.

\paragraph{Proof notation and initialization.}
In this appendix, set $\delta_0:=\delta/64$ and abbreviate the logarithmic term by $u_k:=\ln(16m_k^2/\delta_0)$. Thus Algorithm~\ref{alg:margin-based}, with confidence input $\delta_0$, uses $s_{2m_k-1}=s_0\sqrt{u_k/m_k}$ and $b_{m_k-1}=b_0(u_k/m_k)^{1/4}$ for $k\geq2$; its first batch uses $b_0$. Let $n_k$ be the number of labels acquired in batch $k$, let $\widehat\Theta_k$ be the coefficient of $h_{2m_k-1}$, and write $\mathcal E(\Theta):=\|\Sigma_X^{1/2}(\Theta-\Theta^*)\|_F^2$ for squared excess risk. The count $N_{2m_k-1}$ includes the warm-up labels and all acquisitions through $2m_k-1$, whether or not they enter the last fitting objective.
For the final fit, abbreviate $Q_\delta:=pd+2\sqrt{pd\ln(2/\delta)}+2\ln(2/\delta)$. In the independent isotropic Gaussian specialization, use $Q_\delta:=\chi^2_{pd,1-\delta/2}$ instead. The main coefficient and all active cutoffs in this appendix use the corresponding value.

Let $L:=\operatorname*{ess\,sup}X\Sigma_X^{-1}X^T$ and let $B<\infty$ bound $\nu_S(h(X))$ uniformly over $h\in\mathcal H$. The global error bound gives $L\leq C_\phi^2$, as shown below. A known upper bound, including $C_\phi^2$, may be used in place of $L$ in every cutoff. The bounded prediction range supplies $B$. Choose $s_0\geq4\sigma^2(pd+1)\sqrt{u_1}$. Since $u_k/m_k$ decreases from $u_1$, this choice controls retention even in the early batches. Define $K_0$ as the smallest integer $k\geq2$ satisfying the first two inequalities below, and then choose $b_0$ using the third:
\begin{equation}\label{eq:squared-initial-cutoff}
 m_k\geq32Lu_k,\qquad
 \sqrt{\frac{u_k}{m_k}}\leq\frac{r_{\rm loc}^2}{18Ls_0},\qquad
 b_0>2B\max\{1,(m_{K_0}/u_{K_0})^{1/4}\}.
\end{equation}
The resulting thresholds label every feature through batch $K_0$. This fully labeled period localizes the retained class before the proof uses the selected-information condition.

Define $K_1$ as the first $k\geq K_0+1$ satisfying
\begin{equation}\label{eq:squared-information-cutoff}
 b_0(u_k/m_k)^{1/4}\leq r_\circ,\qquad
 c_-b_0^\kappa m_k^{1-\kappa/4}u_k^{\kappa/4}
 \geq1000\max\{1,L\eta/a\}\,u_k.
\end{equation}
Let $K$ be the first $k\geq K_1$ with $(u_k/u_{k-1})^{\kappa/4}\leq11/10$, and set $N_{\rm init}:=n_0+2^{K_1-1}-1$. This budget bounds every label before batch $K_1$. Since $0<\kappa<4$, all these integer cutoffs are finite, and each defining inequality remains valid at larger indices.

The coefficient $C_{\rm MBAL}$ is defined in \eqref{eq:squared-complexity-constants}, with its stated Gaussian refinement when applicable. A sufficient active accuracy cutoff is
\begin{equation}\label{eq:squared-active-cutoff}
 \epsilon_{\rm A}:=\frac{100\sigma^2\eta Q_\delta}
 {162a c_-b_0^\kappa m_K^{1-\kappa/4}u_K^{\kappa/4}}>0.
\end{equation}
For a fixed $0<\epsilon\leq\epsilon_{\rm A}$, choose the first deterministic $k\geq K$ for which
\begin{equation}\label{eq:squared-target-horizon}
 \frac{100\sigma^2\eta Q_\delta}
 {81a c_-b_0^\kappa m_k^{1-\kappa/4}u_k^{\kappa/4}}\leq\epsilon.
\end{equation}
The proof below establishes the joint risk and total-label guarantee at $2m_k-1$. None of the scales or cutoffs preceding this last choice depends on $\epsilon$.

\paragraph{Bounded design and projected sub-Gaussian noise.}
For any coefficient direction, the interior ball permits a sufficiently small perturbation in that direction. Applying \eqref{eq:linear-global-error-bound} and cancelling its scale gives $\operatorname*{ess\,sup}(Xv)^2\leq C_\phi^2v^T\Sigma_Xv$. First take a countable dense set of $v$, then extend on their common probability-one event. Thus $L=\operatorname*{ess\,sup}X\Sigma_X^{-1}X^T\leq C_\phi^2$. In particular, $L\geq p$ and $\Sigma_X^{-1/2}X^TX\Sigma_X^{-1/2}\preceq LI_p$ almost surely. This bound follows from the full feature distribution; it does not assert any conditional lower eigenvalue.

Condition on the previous history and all current features in batch $k$. Write $Z_k$ for the acquired design, $G_k=Z_k^TZ_k$, and $P_k$ for projection onto the column space of $Z_k$. Let $\zeta_k:=\operatorname{vec}(\xi_k)$ be the vector of acquired noise entries. Conditional on the history and current features, independence across observations and \eqref{eq:subgaussian-cost-noise} give $\bbE[e^{v^T\zeta_k}\mid\text{history, features}]\leq e^{\sigma^2\|v\|_2^2/2}$. The projected noise energy is $e_k^2=\|P_k\xi_k\|_F^2$. The matrix $I_d\otimes P_k$ is an orthogonal projection of rank $q_k=d\operatorname{rank}(Z_k)\leq pd$. The quadratic-form inequality of \citet{hsu2012tail} yields, for any $z>0$,
\begin{equation}\label{eq:projected-subgaussian-tail}
 \bbP\!\left(e_k^2>\sigma^2\{q_k+2\sqrt{q_kz}+2z\}
 \mid\text{history, features}\right)\leq e^{-z}.
\end{equation}
This only needs the vector sub-Gaussian bound, not independent response coordinates. Since $u_k>1$, taking $z=u_k$ gives, outside an event of probability $e^{-u_k}$,
\[
 e_k^2\leq4\sigma^2(pd+1)u_k
 \leq s_0\sqrt{m_ku_k}=m_ks_{2m_k-1},
\]
where the second inequality uses $u_k/m_k\leq u_1$ and the chosen $s_0$. The sum of these exceptional probabilities over geometric batches is $\delta_0/12$. Empty batches have $e_k=0$.

If the preceding retained class contains the truth, constrained first-order optimality gives $\|Z_k(\widehat\Theta_k-\Theta^*)\|_F^2\leq e_k^2$. Completing the square bounds the possible empirical improvement over the truth by $e_k^2/m_k$, so the chosen slack retains the truth. Furthermore, every member of the updated retained class satisfies
\begin{equation}\label{eq:finite-retained-radius}
 \|Z_k(\Theta-\Theta^*)\|_F
 \leq e_k+\sqrt{e_k^2+m_ks_{2m_k-1}}\leq3\sqrt{m_ks_{2m_k-1}}=3\sqrt{s_0}\,(m_ku_k)^{1/4}.
\end{equation}
This follows by expanding its empirical loss relative to the truth, using its sublevel-set membership, and solving the resulting quadratic inequality. It includes singular and empty batches. The coefficient classes stay nonempty, compact, and convex. For a prescribed last fit, \eqref{eq:projected-subgaussian-tail} with $z=\ln(2/\delta)$ gives $e_k^2\leq\sigma^2Q_\delta$ outside conditional probability $\delta/2$. If the noise is independent isotropic Gaussian and the acquired design is full rank, then $e_k^2/\sigma^2$ has exactly $pd$ chi-squared degrees of freedom, justifying the sharper Gaussian choice. Both calculations are conditional on the history and features before intersecting with the learning events, not conditional on eventual success.

\begin{proof}{\bfseries Proof of Theorem \ref{uniform_margin}.}
The prescribed $b_0$ makes all batches through $K_0$ fully labeled, since their thresholds exceed $B$ and $u_k/m_k$ decreases. The matrix Chernoff lower-tail bound for whitened features \citep{tropp2012user} gives
\[
 \bbP(G_{K_0}\not\succeq m_{K_0}\Sigma_X/2)
 \leq p\exp\{-m_{K_0}/(8L)\}\leq p e^{-4u_{K_0}}.
\]
On its complement and the preceding retention event, \eqref{eq:finite-retained-radius} gives $\|\Sigma_X^{1/2}(\Theta-\Theta^*)\|_F^2\leq18s_{2m_{K_0}-1}=18s_0\sqrt{u_{K_0}/m_{K_0}}$ for every retained coefficient. The pointwise leverage bound converts this into a sup-norm error at most $r_{\rm loc}$ by \eqref{eq:squared-initial-cutoff}. All subsequent retained classes are subsets of this class. Hence every later reference is in the required fixed neighborhood, including when an intervening batch is empty. No acquisition lower bound is needed before this neighborhood has been reached.

For $k\geq K_1$, condition on any prior history in this neighborhood. Let the conditional mean acquired count be $\bar n=m_k\bbP(\mathcal A)$. The assumption bounds it between $c_-b_0^\kappa m_k^{1-\kappa/4}u_k^{\kappa/4}$ and the analogous expression with $c_+$. The multiplicative binomial inequalities give
\[
 \bbP(n_k<0.9\bar n)\leq e^{-\bar n/200},\qquad
 \bbP(n_k>1.1\bar n)\leq e^{-\bar n/300}.
\]
The cutoff \eqref{eq:squared-information-cutoff} makes these probabilities at most $e^{-5u_k}$ and $e^{-(10/3)u_k}$. Conditional additionally on any positive count $n$, the acquired rows are i.i.d. with the distribution of $X$ restricted to the fixed set $\mathcal A$. Their whitened second moment is at least $(a/\eta)I_p$, and the matrices are bounded above by $LI_p$. Matrix Chernoff gives
\[
 \bbP(G_k\not\succeq0.9n(a/\eta)\Sigma_X\mid\text{past},n_k=n)
 \leq p\exp\{-na/(200L\eta)\}.
\]
On the count event this is at most $p e^{-4.5u_k}$. Since $L\geq p$ and $m_{K_0}\geq32Lu_{K_0}$, we have $p\leq m_k$ for $k\geq K_0$. Thus $p e^{-4u_k}\leq e^{-u_k}$, which also controls the later design tails without a dimension factor in $u_k$. Each bound is conditional on a prior good history. Summing probabilities of first failures, rather than conditioning on future success, makes the total learning/count/design failure less than $4\delta_0<\delta/8$.

The intermediate risk and count inequalities used in the target-accuracy argument are, for each prescribed $k\geq K$,
\begin{align}\label{eq:finite-batch-risk}
 \mathcal E(\widehat\Theta_k)
 &\leq\frac{10\sigma^2\eta Q_\delta}{9a n_k}
 \leq\frac{100\sigma^2\eta Q_\delta}
 {81a c_-b_0^\kappa m_k^{1-\kappa/4}u_k^{\kappa/4}},\\
 N_{2m_k-1}&\leq N_{\rm init}+\frac{11c_+b_0^\kappa}
 {10\{1-2^{-(1-\kappa/4)}\}}m_k^{1-\kappa/4}u_k^{\kappa/4}.
 \label{eq:finite-batch-labels}
\end{align}
They hold jointly on the learning, design, count, and final-noise events considered here, whose total failure probability is less than $\delta$.

For a prescribed last fit, the additional noise-quantile failure is at most $\delta/2$. The empirical-minimizer inequality $\|Z_k(\widehat\Theta_k-\Theta^*)\|_F^2\leq e_k^2$ does not contain the slack, so the retained-class radius does not enter this final-fit bound. On the resulting event, $\mathcal E(\widehat\Theta_k)\leq\sigma^2Q_\delta/(0.9n_k a/\eta)$, proving both bounds in \eqref{eq:finite-batch-risk}. The number of labels before batch $K_1$ is at most $N_{\rm init}$. For later batches, write $\gamma=1-\kappa/4>0$ only for this summation. Since $u_j\leq u_k$ and $m_j=2^{j-k}m_k$,
\[
 \sum_{j=K_1}^k m_j^\gamma u_j^{\kappa/4}
 \leq\frac{m_k^\gamma u_k^{\kappa/4}}{1-2^{-\gamma}}.
\]
The count upper bounds prove \eqref{eq:finite-batch-labels}. Every cutoff condition persists for later indices: $u_k/m_k$ and $u_k/u_{k-1}$ decrease, while $m_k^\gamma u_k^{\kappa/4}/u_k=(m_k/u_k)^\gamma$ increases. Thus the integer minima defined above are finite and sufficient for all their claimed later uses.

For the target-risk statement, put $H_k=m_k^\gamma u_k^{\kappa/4}$ only in this proof. It is increasing and satisfies $H_k/H_{k-1}\leq1.1\,2^\gamma$ for $k\geq K$. The rightmost risk bound is $D/H_k$, with $D=100\sigma^2\eta Q_\delta/(81a c_-b_0^\kappa)$. Since $\epsilon\leq\epsilon_{\rm A}=D/(2H_K)$, the first index with $D/H_k\leq\epsilon$ is strictly greater than $K$. Its $H_k$ is at most $1.1\,2^\gamma D/\epsilon$. Substituting this into \eqref{eq:finite-batch-labels} gives the displayed $C_{\rm MBAL}$ exactly: the product of $1.1$, $1.1$, and $100/81$ is $121/81$. The index is deterministic, so the same fixed-final-fit probability calculation applies. A hard cap at the stated budget can only change the output on the complement of the joint count/risk event. This proves all conclusions.
\hfill\Halmos
\end{proof}

\paragraph{Quantiles and the supervised sample cutoff.}
\label{app:chi-squared-quantiles}
The notation $\chi^2_{j,1-v}$ denotes the $1-v$ quantile of a chi-squared random variable with $j$ degrees of freedom. Its full definition, not needed as an extra algorithm parameter, is
\begin{equation}\label{eq:quantile-definition}
 \chi^2_{j,1-v}:=\inf\left\{z\geq0:
 \frac{1}{2^{j/2}\Gamma(j/2)}\int_0^z t^{j/2-1}e^{-t/2}\,dt\geq1-v\right\}.
\end{equation}
For $j=1$, it equals $\{\Phi^{-1}(1-v/2)\}^2$. The inequality
$\chi^2_{j,1-v}\leq j+2\sqrt{j\ln(1/v)}+2\ln(1/v)$
is a permissible simplification of an upper bound only. It must not replace the quantile in a lower bound for supervised sample size.

\paragraph{Supervised comparison and accuracy cutoffs.}
For the rest of this supervised comparison, assume $\xi\sim\mathcal N_d(0,\sigma^2I_d)$ independently of $X$ and use the Gaussian refinement of $C_{\rm MBAL}$. A sub-Gaussian upper-tail proxy alone does not supply these lower bounds: zero noise satisfies such a proxy with any larger positive parameter.
To control random-design deviations, define
\begin{equation}\label{eq:supervised-design-tolerance}
 t_\delta:=\min\left\{\frac12,
 \frac12\left(\frac{\chi^2_{pd,1-3\delta/2}}{\chi^2_{pd,1-2\delta}}-1\right),
 \frac12\left(1-\frac{\chi^2_{pd,1-3\delta/4}}{\chi^2_{pd,1-\delta/2}}\right)\right\}>0.
\end{equation}
This constant depends only on dimension and confidence. Set
\begin{align}\label{eq:finite-sl-cutoff}
 n_{\rm D}&:=\left\lceil\frac{3L}{t_\delta^2}\ln(8p/\delta)\right\rceil,\\
 \epsilon_{\rm S}&:=\min\left\{
 \frac{\eta r_\Theta^2}{4},\quad
 \frac{\sigma^2\chi^2_{1,(1-\delta)/2}}{n_{\rm D}L},\quad
 \frac{\sigma^2\chi^2_{pd,1-\delta/2}}{n_{\rm D}}
 \right\}>0.\notag
\end{align}
For $n\geq n_{\rm D}$, sample-covariance concentration supports the two finite bounds in \eqref{eq:supervised-label-requirement}. The second term in $\epsilon_{\rm S}$ also controls every smaller sample size, including rank-deficient designs.

Substitution of \eqref{eq:squared-complexity-constants} shows that the strict coefficient gap in Proposition~\ref{prop:alpha-two-sl-advantage} is equivalent to
\begin{equation}\label{eq:squared-explicit-advantage}
 \frac{a}{\eta}>
 \frac{121}{81}\frac{c_+}{c_-}
 \frac{2^{1-\kappa/4}}{1-2^{-(1-\kappa/4)}}
 \frac{\chi^2_{pd,1-\delta/2}}{\chi^2_{pd,1-2\delta}}.
\end{equation}
When this condition holds, define
\begin{equation}\label{eq:finite-advantage-range}
 \epsilon_*:=\min\left\{\epsilon_{\rm A},\epsilon_{\rm S},
 \frac{\sigma^2\chi^2_{pd,1-2\delta}-C_{\rm MBAL}}
 {2(N_{\rm init}+1)}\right\}>0.
\end{equation}
The last term reserves half of the coefficient gap for initialization and the integer ceiling. It makes the same-risk comparison apply to all paid labels, not only the leading $1/\epsilon$ term.

\begin{proof}{\bfseries Proof of Proposition \ref{prop:alpha-two-sl-advantage}.}
For $n\geq n_{\rm D}$ independent fully labeled observations, two matrix Chernoff tails imply
$(1-t_\delta)n\Sigma_X\preceq G\preceq(1+t_\delta)n\Sigma_X$
except with probability at most $2p\exp\{-nt_\delta^2/(3L)\}\leq\delta/4$. The design is then full rank. Conditional on it, the unconstrained least-squares error is Gaussian with covariance $\sigma^2G^{-1}$ in each response column. Its squared population error is bounded below and above, in stochastic order, by $\sigma^2\chi^2_{pd}/\{(1+t_\delta)n\}$ and $\sigma^2\chi^2_{pd}/\{(1-t_\delta)n\}$, respectively.

Suppose a constrained estimator has risk at most $\epsilon\leq\eta r_\Theta^2/4$. It lies within Frobenius distance $r_\Theta/2$ of the true coefficient, and is therefore interior to the original coefficient class. Its normal equations are the unconstrained ones. If $n\leq\sigma^2\chi^2_{pd,1-2\delta}/\epsilon$, its success probability on the good-design event is at most
$\bbP(\chi^2_{pd}\leq(1+t_\delta)\chi^2_{pd,1-2\delta})\leq1-3\delta/2$.
Adding design failure leaves at most $1-5\delta/4<1-\delta$. Thus every such $n\geq n_{\rm D}$ is insufficient.

For $1\leq n<n_{\rm D}$, continuity of the feature distribution gives a nonzero design rank almost surely. Every nonzero eigenvalue of the whitened Gram matrix is at most $nL$. On a small-risk outcome the same interior normal equations hold. Along one positive-eigenvalue direction, in one response coordinate, they force a Gaussian coefficient component with variance at least $\sigma^2/(nL)$. Hence the small-risk event has probability at most
$\bbP(\chi_1^2\leq nL\epsilon/\sigma^2)\leq(1-\delta)/2$
when $\epsilon\leq\epsilon_{\rm S}$. This proves the strict lower bound in \eqref{eq:supervised-label-requirement} for all positive sample sizes, including singular designs and any measurable empirical tie choice.

For the upper bound, take $n=\lceil\sigma^2\chi^2_{pd,1-\delta/2}/\epsilon\rceil$. The last term in \eqref{eq:finite-sl-cutoff} ensures $n\geq n_{\rm D}$. On the good design, a Gaussian event of probability $1-3\delta/4$ gives unconstrained risk at most
$\sigma^2\chi^2_{pd,1-3\delta/4}/\{(1-t_\delta)n\}\leq\epsilon$.
The first term in that cutoff puts this estimator inside the coefficient class, so it is also the constrained estimator. With design failure included, its success probability is at least $1-\delta$. This proves the finite upper bound.

We next establish the asymptotic label requirement. The empirical
second moment converges to $\Sigma_X$, and the coefficient constraint
is inactive with probability tending to one because $\Theta^*$ is
interior. Conditional Gaussianity therefore gives
\[
 \frac{n\mathcal E(\widehat\Theta_n^{\rm SL})}{\sigma^2}
 \Rightarrow\chi^2_{pd}.
\]
At every fixed positive sample size, exact recovery of $\Theta^*$
has probability zero: the interior normal equations would require
a nontrivial Gaussian score to vanish. Hence
$n_{\rm SL}(\epsilon,\delta)\to\infty$ as $\epsilon\downarrow0$.
Continuity and strict monotonicity of the limiting chi-squared
distribution then imply, for fixed $\delta$,
\[
 n_{\rm SL}(\epsilon,\delta)
 \sim\frac{\sigma^2\chi^2_{pd,1-\delta}}{\epsilon},
 \qquad \epsilon\downarrow0.
\]
This is the asymptotic requirement of the actual supervised estimator,
rather than a sufficient sample size defined by an upper bound.

Finally, the parameter inequality \eqref{eq:squared-explicit-advantage} makes $\sigma^2\chi^2_{pd,1-2\delta}-C_{\rm MBAL}>0$ by substitution of the proved active constant. For $\epsilon$ in \eqref{eq:finite-advantage-range}, the active budget is at most $N_{\rm init}+1+C_{\rm MBAL}/\epsilon$. Dividing by the strict supervised lower bound makes the ratio at most
$C_{\rm MBAL}/\{\sigma^2\chi^2_{pd,1-2\delta}\}
+\epsilon(N_{\rm init}+1)/\{\sigma^2\chi^2_{pd,1-2\delta}\}<1$.
This proves the explicit advantage claim without assuming a comparison between two undefined sample-complexity constants.
\hfill\Halmos
\end{proof}

\paragraph{Why $\kappa>4$ does not imply arbitrary accuracy from finitely many labels.}
Under the fourth-root threshold, the series of the local conditional label means is bounded by a constant times
$\sum_k 2^{k(1-\kappa/4)}u_k^{\kappa/4}$.
It converges when $\kappa>4$. A predictable indicator that counts acquisition only while the reference belongs to the fixed neighborhood therefore has finite expected total and finite total almost surely. On a path that eventually remains in that neighborhood, only finitely many later labels are acquired. This is a consequence about acquisition, not a consistency statement. In particular, a margin law at $h^*$ alone does not prove that the learned references eventually stay there.

In a fixed-noise interior regression family containing the independent Gaussian submodel, a uniform guarantee for arbitrarily small risk cannot be deduced from such a finite transcript. More quantitatively, bounded leverage implies that one acquired label supplies at most a constant times its paid count in Gaussian likelihood information. For two true coefficients
$\Theta_\pm=\Theta_0\pm2\sqrt\epsilon\,\Sigma_X^{-1/2}ve_1^T$,
where $\|v\|_2=1$, a procedure attaining risk at most $\epsilon$ with probability at least $1-\delta$ under both models provides a test with error at most $\delta$. The relative entropy of the entire adaptive transcript is at most
$8L\epsilon\,\bbE_-[N]/\sigma^2$,
because queried labels have Gaussian mean differences and the same variance, while unqueried features have the same law. Coarsening to the test outcome lower-bounds that relative entropy by $(1-2\delta)\ln((1-\delta)/\delta)$. Thus
\[
 \sup_{\Theta\in\{\Theta_-,\Theta_+\}}\bbE_\Theta[N]
 \geq\frac{\sigma^2(1-2\delta)\ln((1-\delta)/\delta)}{8L\epsilon}.
\]
The conclusion concerns a statistical family containing these nearby unknown coefficients; it does not rule out exceptional isolated known targets or the earlier boundary-identification assumptions. It distinguishes a finite sufficient count for each fixed $\epsilon$ from one finite count attaining every $\epsilon>0$. At $\kappa=4$, the geometric sum does not converge; the expected labels per late batch have order $u_k$ and the cumulative order is quadratic in the batch index. Neither boundary case is obtained by substituting into the $0<\kappa<4$ theorem.

\paragraph{The role of a quantitative assumption.}
A qualitative limit for the selected probability or covariance does not specify a numerical batch cutoff: distributions with the same limits can approach them at different rates. The bounds on a fixed neighborhood in Assumption \ref{assumption:holder}, together with the explicitly counted fully labeled period, make $K$ and $\epsilon_*$ computable. The fixed concentration allowances account for finite-sample deviations; they do not identify a sample covariance with its expectation.

\section{Finite analysis for a convex surrogate loss}
\label{app:spop-finite-cutoffs}

We now give the loss-specific quantitative conditions and the proof of Theorem \ref{thm:alpha-two-spop-finite} and Proposition \ref{prop:alpha-two-spop-comparison}. These are conditions on the fixed population laws and the initial coefficient class, not on realized MBAL histories. They extend the original SPO+ argument without assuming that every convex loss has the same regularity. The Gaussian box loss is verified in Appendix \ref{app:spop-verification}.

\paragraph{Empirical subgradient conditions.} In derivative expressions, we regard the risk as a function of the coefficient vector. For $n$ i.i.d. observations under one of these laws, let $\widehat R_{\ell,\mathcal A,n}$ be the empirical average loss and $\overline v_*$ the average subgradient at $\Theta^*$. There are constants $C_0\geq e$ and $B_0>0$ such that, for every $n\geq1$, $0<t<1$, and fixed deterministic $0<R\leq1$, each of the following bounds holds with probability at least $1-t$:
\begin{align}
 \sup_{\Theta}\sup_{v\in\partial\widehat R_{\ell,\mathcal A,n}(\Theta)}
 \|v-\nabla R_{\ell,\mathcal A}(\Theta)\|_2
 &\leq B_0\sqrt{\frac{\ln(C_0n/t)}{n}},\label{eq:convex-global-score}\\
 \sup_{\|\Theta-\Theta^*\|_F\leq R}\sup_{v\in\partial\widehat R_{\ell,\mathcal A,n}(\Theta)}
 \|v-\overline v_*-H_{\mathcal A}\operatorname{vec}(\Theta-\Theta^*)\|_2
 &\leq B_0\left\{R^2+\sqrt{\frac{R\ln(C_0n/t)}{n}}+
 \frac{\ln(C_0n/t)}{n}\right\}.\label{eq:convex-local-score}
\end{align}
Both suprema are restricted to the initial coefficient class. The first bound controls estimation away from the minimizer; the second provides a local score approximation, including all subgradient choices at sample kinks. Such local expansions underlie convex-estimation theory \citep{hjort2011asymptotics}.

\paragraph{Loss-specific conditions.}
The loss-specific population and empirical-score conditions are stated in Section \ref{sec:alpha-two-spop}. The proof below uses those same constants under the full feature law and under the locally selected laws.

\paragraph{Proof notation and fitting cutoff.}
For this appendix, set $\delta_0:=\delta/64$ and $u_k:=\ln(16m_k^2/\delta_0)$. Let $n_k$ denote the acquired count in batch $k$ and let $B<\infty$ bound the distance to degeneracy over the initial predictor class. The confidence input in Algorithm~\ref{alg:margin-based} is $\delta_0$. All constants below are constructed from the model and the population and score bounds in Section~\ref{sec:alpha-two-spop}.

Let $h_*:=\lambda_{\min}(H)>0$, and define
\begin{align}\label{eq:convex-fitting-cutoff}
 R_0&:=\min\{1,r_\Theta/2,h_*/(10L_H)\},
 &\ell_n&:=\ln(128C_0n/\delta),\\
 R_n&:=\frac{2B_0}{h_0}\sqrt{\ell_n/n},
 &e_n&:=B_0\{R_n^2+\sqrt{R_n\ell_n/n}+\ell_n/n\}.\notag
\end{align}
Interpret $h_*/(10L_H)$ as $+\infty$ when $L_H=0$. Let $n_\ell(\delta)$ be the smallest integer $n\geq\lceil e^4(128C_0/\delta)\rceil$ satisfying
\begin{equation}\label{eq:convex-fitting-tests}
 R_n\leq R_0,\qquad
 \sqrt n\,e_n\leq\frac{h_0\sqrt{v_0}}{20\overline h},\qquad
 \frac{400(pd)^{1/4}M_3}{\sqrt n}\leq\frac{\delta}{32}.
\end{equation}
The first inequality places the estimator in the region used by the local expansion. The second controls its remainder, and the third controls the convex-set normal approximation for the standardized score \citep{bentkus2003dependence}. The left sides decrease to zero in the stated range: the nonconstant terms in $\sqrt n\,e_n$ have orders $\ell_n/\sqrt n$ and $\ell_n^{3/4}/n^{1/4}$. Hence $n_\ell(\delta)$ is finite and the tests remain valid for every larger sample size.

\paragraph{Initialization and the target horizon.}
Choose
\begin{equation}\label{eq:convex-slack}
 s_0\geq\frac{2B_0^2(2+\ln C_0)\sqrt{u_1}}{h_0}.
\end{equation}
Define $K_0$ as the first $k\geq2$ satisfying the first inequality below, and then choose $b_0$ by the second:
\[
 \sqrt{\frac{u_k}{m_k}}\leq\frac{h_0r_{\rm loc}^2}{9B_X^2s_0},\qquad
 b_0>2B\max\{1,(m_{K_0}/u_{K_0})^{1/4}\}.
\]
This choice labels every feature through batch $K_0$. Define $K_1$ as the first $k\geq K_0+1$ satisfying $b_0(u_k/m_k)^{1/4}\leq r_\circ$ and $c_-b_0^\kappa m_k^{1-\kappa/4}u_k^{\kappa/4}\geq1000u_k$. Finally, let $K$ be the first $k\geq K_1$ satisfying
\begin{equation}\label{eq:convex-batch-cutoff}
 (u_k/u_{k-1})^{\kappa/4}\leq11/10,\qquad
 0.9c_-b_0^\kappa m_k^{1-\kappa/4}u_k^{\kappa/4}\geq n_\ell(\delta).
\end{equation}
These indices are finite because $0<\kappa<4$. Set $N_{\rm init}:=n_0+2^{K_1-1}-1$ and
\begin{equation}\label{eq:convex-active-cutoff}
 \epsilon_{\rm A,\ell}:=
 \frac{J_\ell\chi^2_{pd,1-\delta/2}}
 {3c_-b_0^\kappa m_K^{1-\kappa/4}u_K^{\kappa/4}}>0.
\end{equation}
For a fixed $0<\epsilon\leq\epsilon_{\rm A,\ell}$, choose the first deterministic $k\geq K$ satisfying
\begin{equation}\label{eq:convex-target-horizon}
 \frac{20J_\ell\chi^2_{pd,1-\delta/2}}
 {27c_-b_0^\kappa m_k^{1-\kappa/4}u_k^{\kappa/4}}\leq\epsilon.
\end{equation}
This defines the horizon in Theorem~\ref{thm:alpha-two-spop-finite}. The scales, initialization budget, and accuracy cutoff depend on the stated population constants and confidence, not on $\epsilon$ or a realized learning history.

\paragraph{General Hessian and covariance formulas.}
At differentiable prediction points, write the prediction score as $s_*(X,C)$ and the linear prediction Jacobian as $J_X=I_d\otimes X$. The coefficient score is $v_*=J_X^Ts_*$. More generally, a compatible measurable subgradient gives the same expression. If the conditional expected prediction loss has Hessian $B_\ell(X)$ at $h^*(X)$ and differentiation may pass through the bounded-feature expectation, then $H_{\mathcal A}=\bbE[J_X^TB_\ell(X)J_X\mid\mathcal A]$. The score covariance is $V_{\mathcal A}=\operatorname{Cov}(J_X^Ts_*\mid\mathcal A)$. Conditional centering removes the covariance-of-the-conditional-mean term; without it that term must be included. These definitions do not equate $V_{\mathcal A}$ with a noise variance times $H_{\mathcal A}$. For squared loss, $B_\ell(X)=2I_d$. Independent isotropic Gaussian noise gives $V_{\mathcal A}=4\sigma^2(I_d\otimes\Sigma_{X,\mathcal A})$; under only \eqref{eq:subgaussian-cost-noise}, this is an upper matrix bound rather than an equality.

\begin{proof}{\bfseries Proof of Theorem \ref{thm:alpha-two-spop-finite}.}
First apply \eqref{eq:convex-global-score} conditionally on a preceding good history and an acquired count $n>0$, with failure probability $e^{-u_k}$. Because $n\leq m_k$, its logarithm is at most $(2+\ln C_0)u_k$. Write $e$ for the score-deviation bound. The population risk is $h_0$-strongly convex and has zero gradient at $\Theta^*$. Integrating the empirical score difference along a feasible segment gives
$\widehat R(\Theta)-\widehat R(\Theta^*)\geq h_0\|\Theta-\Theta^*\|_F^2/2-e\|\Theta-\Theta^*\|_F$.
The largest empirical improvement is therefore at most $e^2/(2h_0)$. Since $e^2\leq B_0^2(2+\ln C_0)u_k/n$, the scale in \eqref{eq:convex-slack} and $u_k/m_k\leq u_1$ give
\[
 \frac{ne^2}{2m_kh_0}\leq\frac{s_{2m_k-1}}4.
\]
Thus the update retains $\Theta^*$. For a retained coefficient, the same quadratic lower bound is at most $m_ks_{2m_k-1}/n$, because its empirical loss exceeds the minimum by at most that amount. Solving this inequality gives
\[
 \|\Theta-\Theta^*\|_F
 \leq3\sqrt{\frac{m_ks_{2m_k-1}}{nh_0}}.
\]
Empty batches leave the set unchanged.

The large initial threshold labels all observations through $K_0$. There $n=m_{K_0}$, so the last bound gives $\|h_\Theta-h^*\|_\infty^2\leq9B_X^2s_{2m_{K_0}-1}/h_0$. The $K_0$ condition therefore puts the entire retained class in the fixed $r_{\rm loc}$ prediction neighborhood. Nesting preserves this property. The loss regularity then applies to all subsequent selected laws, including radii larger than $r_\circ$. Summing first failures on preceding good histories gives a retention/localization failure probability below $\delta_0$. This induction does not condition current data on future success.

For $k\geq K_1$, condition on a preceding successful history, so the retained class lies in the $r_{\rm loc}$ neighborhood. Since $b_{m_k-1}\leq r_\circ$, the acquisition-probability bounds in Assumption~\ref{assumption:surrogate-information} and the binomial tail inequalities used in the squared-loss proof give
\[
 0.9c_-b_0^\kappa m_k^{1-\kappa/4}u_k^{\kappa/4}
 \leq n_k\leq
 1.1c_+b_0^\kappa m_k^{1-\kappa/4}u_k^{\kappa/4}.
\]
Summing the conditional failure bounds over these batches gives total failure probability below $\delta_0$. Thus every $k\geq K$ satisfies $n_k\geq n_\ell(\delta)$ on the joint event. Fix one such deterministic final index and condition on its previous history and count. Its acquired observations are i.i.d. with the selected law. Apply the global and local score inequalities with $t=\delta/128$. Convex constrained optimality supplies an empirical subgradient having nonpositive inner product with $D=\operatorname{vec}(\widehat\Theta-\Theta^*)$. The global inequality and strong convexity give $\|D\|_2\leq R_n$.

The local inequality at that subgradient gives
\[
 \|H_{\mathcal A}^{1/2}D\|_2
 \leq \|H_{\mathcal A}^{-1/2}\overline v_*\|_2+e_n/\sqrt{h_0}.
\]
Put $j=\lambda_{\max}(H_{\mathcal A}^{-1/2}V_{\mathcal A}H_{\mathcal A}^{-1/2})$ only in this calculation. It is at least $v_0/\overline h$. The Berry--Esseen bound and \eqref{eq:convex-fitting-tests} imply $\|H_{\mathcal A}^{-1/2}\overline v_*\|_2\leq\sqrt{j\chi^2_{pd,1-\delta/2}/n}$ except with probability at most $\delta/2+\delta/32$. The remainder is no larger than one tenth of this upper scale; $\chi^2_{pd,1-\delta/2}\geq1$ for $\delta<1/8$ and $h_0\leq\overline h$ suffice. Hence $\|H_{\mathcal A}^{1/2}D\|_2^2\leq1.21j\chi^2_{pd,1-\delta/2}/n$.

The Hessian-Lipschitz bound and $R_n\leq R_0$ give $R_\ell(h_{\widehat\Theta})-R_\ell^*\leq0.55D^THD$. Combining with \eqref{eq:spop-information-comparison} yields an upper bound $0.6655J_\ell\chi^2_{pd,1-\delta/2}/n<2J_\ell\chi^2_{pd,1-\delta/2}/(3n)$. Adding all process, normal-approximation, retention, and count failures gives total failure below $\delta$. Historical confidence constraints need not be inactive anywhere in this argument. Once they retain $\Theta^*$ and remain in the fixed neighborhood, the final empirical-optimality and score bounds depend on $n$, not on the confidence-set slack.

Summing these acquired-count bounds geometrically gives \eqref{eq:finite-batch-labels}. The last-fit risk after the count lower bound is at most $20J_\ell\chi^2_{pd,1-\delta/2}/\{27c_-b_0^\kappa m_k^{1-\kappa/4}u_k^{\kappa/4}\}$. The stated $\epsilon_{\rm A,\ell}$ is less than half this bound at $K$, so the first qualifying index exceeds $K$. Rounding it multiplies the deterministic count equivalent by at most $1.1\,2^{1-\kappa/4}$. The resulting label constant is $(121/135)(c_+/c_-)\{2^{1-\kappa/4}/[1-2^{-(1-\kappa/4)}]\}J_\ell\chi^2_{pd,1-\delta/2}$, at most the simpler $C_{\rm MBAL,\ell}$ in the theorem. It does not remove a confidence or batching cost. The chosen final index is deterministic, and all previous labels are included. A hard budget cap changes the output only outside the joint event already bounded.
\hfill\Halmos
\end{proof}

\paragraph{Supervised comparison and accuracy cutoffs.}
For the empirical-minimizer rule satisfying \eqref{eq:convex-small-sample}, define
\begin{equation}\label{eq:convex-supervised-cutoff}
 \epsilon_{\rm S,\ell}:=\min\left\{
 h_0r_\Theta^2/8,
 \left(\frac{1-\delta}{2D_{\rm ac}n_\ell(\delta)}\right)^2
 \right\}>0.
\end{equation}
This cutoff combines the local population expansion with control of every smaller sample size. Its positivity does not imply a practically large accuracy range.

Substitution of $ C_{\rm MBAL,\ell}$ makes the strict coefficient gap in Proposition~\ref{prop:alpha-two-spop-comparison} equivalent to
\begin{equation}\label{eq:convex-explicit-advantage}
 \frac{v_{\rm SL}}{J_\ell}>
 3\frac{c_+}{c_-}
 \frac{2^{1-\kappa/4}}{1-2^{-(1-\kappa/4)}}
 \frac{\chi^2_{pd,1-\delta/2}}{\chi^2_{pd,1-4\delta}}.
\end{equation}
When this inequality holds, define
\begin{equation}\label{eq:spop-advantage-cutoff}
 \epsilon_{*,\ell}:=\min\left\{\epsilon_{\rm A,\ell},\epsilon_{\rm S,\ell},
 \frac{v_{\rm SL}\chi^2_{pd,1-4\delta}/3-C_{\rm MBAL,\ell}}
 {2(N_{\rm init}+1)}\right\}>0.
\end{equation}
The last term controls initialization and the integer ceiling in the total active label budget.

\begin{proof}{\bfseries Proof of Proposition \ref{prop:alpha-two-spop-comparison}.}
For $n\geq n_\ell(\delta)$ independent supervised observations, the same process inequalities hold under the full feature law. On a small-risk outcome with $\epsilon\leq h_0r_\Theta^2/8$, the minimizer lies within $r_\Theta/2$ of the true coefficient and therefore is interior to the original constraint. Zero is then one of its empirical subgradients. The local inequality implies
$\|H^{1/2}D\|_2\geq\|H^{-1/2}\overline v_*\|_2-e_n/\sqrt{h_*}$.
Moreover, $\|H^{-1/2}\overline v_*\|_2\geq\sqrt{v_{\rm SL}}\|V^{-1/2}\overline v_*\|_2$. The standardized score norm exceeds $\sqrt{\chi^2_{pd,1-4\delta}/n}$ with probability at least $4\delta-\delta/32$. The fitting cutoff makes the remainder less than a tenth of that scale: $\chi^2_{pd,1-4\delta}\geq\chi^2_{1,1/2}>0.4$ and the eigenvalue bounds suffice. The local population lower Taylor bound gives risk at least $0.45(0.9)^2v_{\rm SL}\chi^2_{pd,1-4\delta}/n>v_{\rm SL}\chi^2_{pd,1-4\delta}/(3n)$ on that event. Subtracting the two score-process failures leaves probability strictly greater than $\delta$. Hence no $n\leq v_{\rm SL}\chi^2_{pd,1-4\delta}/(3\epsilon)$ in this range attains confidence $1-\delta$.

For $n<n_\ell(\delta)$, \eqref{eq:convex-small-sample} and \eqref{eq:convex-supervised-cutoff} bound the risk-success probability by $(1-\delta)/2<1-\delta$. This supplies the required control of smaller sample sizes; it cannot be replaced by an asymptotic normal approximation. We have therefore proved \eqref{eq:convex-sl-finite-lower} for the minimum over all positive integer sample sizes.

The explicitly stated inequality \eqref{eq:convex-explicit-advantage} makes $v_{\rm SL}\chi^2_{pd,1-4\delta}/3-C_{\rm MBAL,\ell}$ positive by algebra. The bound on $\epsilon$ in \eqref{eq:spop-advantage-cutoff} makes the initialization and ceiling correction at most half this gap. Dividing the active sufficient budget by \eqref{eq:convex-sl-finite-lower} proves the strict ratio below one, without defining or assuming an unrelated ``required-label constant.''

Finally, the same score expansions imply consistency and
$\sqrt n\operatorname{vec}(\widehat\Theta_n^{\rm SL}-\Theta^*)=-H^{-1}\sqrt n\,\overline v_*+o_{\mathbb P}(1)$.
To check the remainder directly, the global score bound gives $\|D\|=O_{\mathbb P}(\sqrt{\ln n/n})$; substituting this radius into the local bound makes $\sqrt n$ times its right side tend to zero. The constraint is inactive with probability tending to one, so the equality follows from the empirical first-order condition. The multivariate central limit theorem and a population Taylor expansion prove \eqref{eq:convex-sl-asymptotic}. Its limiting quadratic form has a continuous strictly increasing distribution on $(0,\infty)$ because $H,V\succ0$. Condition \eqref{eq:convex-small-sample} prevents any fixed finite sample size from attaining arbitrarily small risk at the given confidence. Quantile convergence thus gives the stated leading form of the actual $n_{\rm SL}(\epsilon,\delta)$.
\hfill\Halmos
\end{proof}

\section{Verification for SPO+ and the continuous examples}
\label{app:spop-verification}

We now verify all the loss-specific hypotheses of the convex-surrogate-loss theorem for the Gaussian box model. The feature examples are two-dimensional and continuous, with four free matrix coefficients. For a more general bounded affine model the same formulas verify the score regularity, while conditional consistency, the global error bound \eqref{eq:linear-global-error-bound}, and the local acquisition/information bounds still require their own verification. The sample SPO+ loss is piecewise affine, so the argument controls all empirical subgradients, not an almost-everywhere zero sample Hessian.

\paragraph{Population constants.}
Put $q=pd$ in this appendix only, and let $D_\Theta$ be the Frobenius diameter of the initial coefficient polytope. Let $B_X$ bound $\|X\|_2$, let $B_P$ bound all coordinate prediction magnitudes on that polytope, and let $M$ bound the coordinate magnitudes of $h^*(X)$. These are finite bounds required in this extension. The following constants involve only the model and Assumption \ref{assumption:holder}:
\begin{align}\label{eq:spop-population-constants}
 m_X&:=\min\{\eta,\ a\min(1,c_-r_\circ^\kappa)\},
 &h_0&:=8\varphi_\sigma(2B_P+M)m_X,\\
 \overline h&:=8\varphi_\sigma(0)B_X^2,
 &v_0&:=32\Phi(-M/\sigma)m_X,\\
 B_{\rm sc}&:=4\sqrt d B_X,
 &L_H&:=16\sqrt d\,B_X^3\varphi(1)/\sigma^2.\notag
\end{align}
The lower and upper curvature constants and the score-variance lower bound are positive. In $L_H$, $\varphi$ denotes the standard normal density. The constant $h_0$ lower-bounds the population coefficient Hessian throughout the initial polytope under the full feature law and every selected law whose reference is in the fixed neighborhood, including radii larger than $r_\circ$. Indeed, a set at such a larger radius contains the set at $r_\circ$; positive-semidefinite averaging gives the feature lower bound $a c_-r_\circ^\kappa$. The full feature law has lower bound $\eta$. The constants $\overline h$ and $v_0$ respectively upper-bound the Hessian at the truth and lower-bound its score covariance for these laws. The bound $B_{\rm sc}$ controls each sample score norm, and $L_H$ is a uniform operator-Lipschitz constant for the population coefficient Hessian. These are distinct from $\eta$, which remains the global feature minimum eigenvalue.

Define two deterministic score-envelope constants
\begin{align}\label{eq:spop-envelope-constants}
 C_0&:=\{64(1+D_\Theta)(1+\sqrt q)\}^{4q+4},\\
 B_0&:=10^4(1+q)^3(1+d)^2(1+B_X)^4
       \{1+\varphi_\sigma(0)+L_H+\overline h\}.\notag
\end{align}
They are bounds for a proof, not tuning parameters or new algorithmic inputs. A tighter empirical-process calculation could substantially reduce them. We state them to make the finite applicability range unambiguous.

\begin{lemma}[Uniform score bounds for the Gaussian box loss]
\label{lem:spop-finite-score}
Take $n\geq1$ i.i.d. observations from the full feature distribution or one of the selected distributions just specified. Write $\widehat V_*$ for the empirical average sample score at $\Theta^*$, $H_{\mathcal A}$ for the population Hessian there, and $\ell_n(t)=\ln(C_0n/t)$ for $0<t<1$. With probability at least $1-t$,
\begin{equation}\label{eq:spop-global-score}
 \sup_{\Theta}\ \sup_{v\in\partial\widehat R(\Theta)}
 \|v-\nabla R_{\mathcal A}(\Theta)\|_2
 \leq B_0\sqrt{\ell_n(t)/n}.
\end{equation}
For any fixed deterministic $0<R\leq1$, a second event of probability at least $1-t$ has
\begin{equation}\label{eq:spop-local-score}
 \sup_{\|\Theta-\Theta^*\|_F\leq R}\ \sup_{v\in\partial\widehat R(\Theta)}
 \|v-\widehat V_*-H_{\mathcal A}\operatorname{vec}(\Theta-\Theta^*)\|_2
 \leq B_0\left\{R^2+\sqrt{R\ell_n(t)/n}+\ell_n(t)/n\right\}.
\end{equation}
The suprema are restricted to the initial coefficient polytope. No assumption on a learned constraint being inactive is made.
\end{lemma}

\begin{proof}{\bfseries Proof of Lemma \ref{lem:spop-finite-score}.}
For the box, the SPO+ definition and subgradient formula of \citet{elmachtoub2022smart} give prediction score in coordinate $j$ equal to $2\{\operatorname{sign}(2h_j-C_j)-\operatorname{sign}(C_j)\}$, with the full interval of subgradient choices at zero. After multiplying by the feature, its vectorized coefficient score has norm at most $B_{\rm sc}$. A score change between coefficients separated by $s$ is nonzero only if at least one $C_j$ lies between their two reflected predictions. The conditional Gaussian density bound and a union bound give probability at most $2d B_X\varphi_\sigma(0)s$. The squared norm of the difference is at most $16dB_X^2$. Thus its expected squared norm is at most $32d^2B_X^3\varphi_\sigma(0)s$. All these bounds hold under every selected feature law because selection depends only on features.

Cover the coefficient polytope by Euclidean balls of radius $n^{-2}$ with centers chosen from the polytope. A grid on a containing cube gives at most $\{1+4(1+D_\Theta)\sqrt q\,n^2\}^q$ centers after a harmless factor-two enlargement. For the local bound, centers are chosen in the intersection with the radius-$R$ ball; a sample increment at a center then has the preceding variance bound with $s\leq R$. Apply the scalar Bernstein inequality to each of the $q$ coordinates and each center, with tail parameter at most $8(q+1)\ell_n(t)$. A union bound gives a centered increment bound of the form
\[
  100(q+1)^2\{1+d^2B_X^3\varphi_\sigma(0)+B_{\rm sc}\}
  \left\{\sqrt{R\ell_n(t)/n}+\ell_n(t)/n\right\}.
\]
For the global score, use bounded-variable concentration rather than the local variance bound, giving the same constant times $\sqrt{\ell_n(t)/n}+\ell_n(t)/n$.

It remains to extend from centers to every coefficient and every subgradient choice. A coefficient within $n^{-2}$ of its center can change a sign only on a Gaussian label strip of width at most $4B_Xn^{-2}$ around that center's reflected prediction. Each observation belongs to the union of these $d$ strips with probability at most $4dB_X\varphi_\sigma(0)n^{-2}$. Another scalar Bernstein bound on the strip counts, over the same centers, bounds their empirical proportions by a constant times $\ell_n(t)/n$. Their expected proportions are bounded by the same expression. Every possible choice at a sample kink is covered by such a strip. Multiplication by $2B_{\rm sc}$ therefore bounds the empirical and population interpolation errors by the same constant times $\ell_n(t)/n$. The choices of grid size and tail parameters above make the sum of all failures at most $t$; allocating half the probability to centers and half to interpolation can be absorbed by the displayed factor $100$.

Finally, differentiating the Gaussian conditional expectation gives Hessian blocks $8\bbE[\varphi_\sigma(2h_j-h_j^*)X^TX\mid\mathcal A]$. Since $\sup_y|\varphi_\sigma'(y)|=\varphi(1)/\sigma^2$, integration of the Hessian difference bounds the population gradient remainder by $L_HR^2/2$. All the preceding constants are smaller than $B_0$ in \eqref{eq:spop-envelope-constants}. This proves \eqref{eq:spop-local-score}. For \eqref{eq:spop-global-score}, if $\ell_n(t)/n\leq1$, its extra linear term is no larger than its square root. If $\ell_n(t)/n>1$, the deterministic bound $2B_{\rm sc}$ on the score difference is already smaller than the stated right side. This completes the proof, including small sample sizes.
\hfill\Halmos
\end{proof}

\paragraph{Covariance, moment, and small-sample conditions.}
For clarity, the coefficient Hessian and score covariance for the Gaussian box loss are block diagonal. Writing $[H_{\mathcal A}]_{jj}$ and $[V_{\mathcal A}]_{jj}$ for their $j$-th $p\times p$ diagonal blocks,
\begin{equation}\label{eq:spop-hessian-variance}
 \begin{aligned}
 [H_{\mathcal A}]_{jj}&=8\bbE[\varphi_\sigma(h_j^*(X))X^TX\mid X\in\mathcal A],\\
 [V_{\mathcal A}]_{jj}&=32\bbE[\Phi(-|h_j^*(X)|/\sigma)X^TX\mid X\in\mathcal A],
 \quad j=1,\ldots,d.
 \end{aligned}
\end{equation}
All off-diagonal blocks are zero. Removing the conditioning gives the full-population matrices $H,V$.

At the true prediction, the scalar box score is zero if $|\xi_j|<|h_j^*(X)|$ and has magnitude four otherwise. Its conditional mean is zero, and its variance is $32\Phi(-|h_j^*(X)|/\sigma)$. Independent coordinate noises eliminate cross-response covariances; differentiating the conditional expected score gives $8\varphi_\sigma(h_j^*(X))$. This proves \eqref{eq:spop-hessian-variance}. With the constants above, one may use $M_3=B_{\rm sc}^3/v_0^{3/2}$. The global and local score lemmas establish \eqref{eq:convex-global-score}--\eqref{eq:convex-local-score} with the explicitly supplied $C_0,B_0$. Thus the general fitting cutoff uses no unspecified normal-approximation error.

For the finite supervised comparison, use the fixed lexicographically smallest empirical SPO+ minimizer on the initial coefficient polytope. It is a vertex of the polyhedral minimizing set, and the rule does not depend on the true coefficient. On a risk event with $\epsilon\leq h_0r_\Theta^2/8$, population strong convexity puts the selected minimizer inside the initial polytope. It must lie on a sample kink hyperplane $2h_\Theta(X_i)_j-C_{ij}=0$ for some $i,j$: otherwise it is in an open affine cell, and an interior affine minimum would make an open neighborhood consist of minimizers, contradicting the vertex selection. Consequently,
$|h_j^*(X_i)-\xi_{ij}|\leq2B_X\sqrt{2\epsilon/h_0}$
for some sample response. A conditional Gaussian density bound and a union bound give probability at most $4ndB_X\varphi_\sigma(0)\sqrt{2\epsilon/h_0}$. Hence \eqref{eq:convex-small-sample} holds with the supplied constant
$D_{\rm ac}=4dB_X\varphi_\sigma(0)\sqrt{2/h_0}$.
This also explains why merely choosing an arbitrary point on an empirical flat face is insufficient for the claimed smallest-sample comparison.

The continuous matrix examples are verified in Appendix \ref{app:matrix-information-examples}, using the preceding Gaussian box constants in their matrix form.

\paragraph{Other convex losses.}
The definitions of $H,V$ and the proof of the general theorem are not restricted to SPO+. Smooth squared loss gives the exact matrices stated in the main analysis; bounded-design covariance concentration and conditional Gaussian score concentration verify its global score inequality, and its local empirical gradient remainder is a sample-covariance fluctuation times the coefficient error. A finite third moment follows from bounded $X$ and Gaussian noise. A small-risk supervised outcome is interior; along one positive-design direction the normal equations give probability at most $\sqrt{2nL\epsilon/\pi}/\sigma\leq n\sqrt{2L/\pi}\sqrt\epsilon/\sigma$, verifying the analogous anti-concentration property. The sharper squared-loss theorem avoids the generic normal-approximation constants. For any other convex loss, the displayed common-target, population-curvature, moment, and score-process conditions, together with the conditions on the supervised estimator, must be checked; neither convexity alone nor a formal assignment of $H,V$ verifies them.

\paragraph{Reading the example tables.}
Table~\ref{tab:finite-kappa-comparison} reports finite total-label bounds from the model-specific count and Gaussian-design proof in Appendix~\ref{app:matrix-information-examples}. Table~\ref{tab:spop-finite-ratios} reports population curvature and Gaussian SPO+ leading ratios, not a finite-budget comparison. The direct SPO+ certificate \eqref{eq:example-spop-certificate} is separate from the generic normal-approximation cutoff. Both examples use the single feature distribution with $\kappa=1$.

\section{Verification of the continuous matrix examples}\label{app:matrix-information-examples}
We verify Examples~\ref{example:continuous-information-squared} and~\ref{example:continuous-information-spop}, distinguish their population comparisons from their finite-budget certificates, and count all acquired labels. The algorithm and its common schedules are unchanged. The initial coefficient box, rather than an estimated neighborhood obtained after a long fully labeled period, supplies the uniform selection geometry used below.

\paragraph{Two restrictions of the uniform information condition.}
Suppose an augmented feature has a constant coordinate, and let $e$ select it. Then $e^T\Sigma_Xe=e^T\bbE[X^TX\mid\mathcal A]e=1$ for every positive-probability selected set. The matrix inequality in \eqref{eq:finite-local-information} therefore gives $a(r)/\eta\leq1$. When this coordinate represents a freely estimated intercept, the sufficient strict advantage in Proposition~\ref{prop:alpha-two-sl-advantage} cannot hold. This is a limitation of that sufficient matrix comparison, not an impossibility result for every active learner or every risk-weighted information comparison.

Next, take $S=\operatorname{conv}\{w_1,w_2\}$ with $w_1\ne w_2$, a homogeneous conditional mean $h^*(X)=X\Theta^*$, and $v=\Theta^*(w_1-w_2)$. Its distance to degeneracy is $|Xv|/\|w_1-w_2\|_2$. If $v\ne0$, selection at $h^*$ gives
\[
 v^T\bbE[X^TX\mid\mathcal A_{h^*,r}]v
 \leq r^2\|w_1-w_2\|_2^2.
\]
Since $v^T\Sigma_Xv>0$, this contradicts any fixed positive lower bound on $a(r)/\eta$ as $r\downarrow0$. If $v=0$, the acquisition probability is one at every positive radius, contradicting the probability upper bound with $\kappa>0$. An affine offset can change this argument, but is not needed in the four-vertex construction below: its two decision boundaries select different informative directions.

\paragraph{Feature density and population information.}
Every component of \eqref{eq:continuous-information-feature-law} has a density on a union of rectangles. Thus the mixture is absolutely continuous in two dimensions. The independent signs give zero cross moments, and symmetry under exchange of the two coordinates gives
\begin{equation}\label{eq:continuous-matrix-covariance}
 \Sigma_X=\eta I_2,\qquad
 \eta=\frac{199}{200}\frac7{4800}
       +\frac1{400}\left(\frac{331}{300}+\frac1{48}\right)
       =\frac{4091}{960000},\qquad
 L=\frac{509}{400\eta}.
\end{equation}
Here $\bbE[B_i^2]=7/4800$, $\bbE[A^2]=331/300$, $\bbE[U^2]=1/48$, and $B_X^2:=509/400$ is the exact essential supremum of $\|X\|_2^2$. In particular, $L=298.6067\ldots$ and all four coefficient directions are identifiable. The predictor image of the coefficient box is compact and convex. The true matrix $I_2$ has an interior Frobenius ball of radius $1/10$; every matrix with $\|\Theta^*-I_2\|_{\max}\leq1/20$ has such a ball of radius $1/20$. The initial box and its four free coordinates are the same for both learners and do not depend on observations.

At $\Theta^*=I_2$, the small-radius acquisition event on the horizontal component is $|U|<r$, and the analogous event holds on the vertical component. The background component is excluded for $0<r\leq7/400$. Consequently,
\begin{equation}\label{eq:continuous-selected-information}
 \bbP(\mathcal A_{h^*,r})=\frac r{50},\qquad
 \bbE[X^TX\mid\mathcal A_{h^*,r}]
 =\frac12\left(\frac{331}{300}+\frac{r^2}{3}\right)I_2,
 \qquad 0<r\leq\frac7{400}.
\end{equation}
Both axes contribute to the conditional covariance; neither coordinate loses all its variation as the threshold vanishes.

\paragraph{Selection for every fitted coefficient in the initial box.}
Write a fitted coefficient as $\Gamma\in\mathcal K$. Its diagonal entries lie in $[0.9,1.1]$ and its off-diagonal entries in $[-0.1,0.1]$. On the background component, every predicted coordinate has absolute value at least
$0.9(1/40)-0.1(1/20)=7/400$.
On the horizontal component, the first predicted coordinate has absolute value at least $0.9-0.1/4=0.875$. Hence, for $r\leq7/400$, this component is acquired precisely when
\[
 \left|U+EA\frac{\Gamma_{12}}{\Gamma_{22}}\right|
 <\frac r{\Gamma_{22}}.
\]
This interval is entirely inside $[-1/4,1/4]$, because
$|(EA)\Gamma_{12}/\Gamma_{22}|+r/\Gamma_{22}\leq(0.11+0.0175)/0.9<1/4$.
The vertical component satisfies the same statement with the indices exchanged. Their full-distribution acquisition probabilities are respectively $r/(100\Gamma_{22})$ and $r/(100\Gamma_{11})$. Therefore valid constants are
\begin{equation}\label{eq:continuous-local-constants}
 \kappa=1,\quad r_\circ=\frac7{400},\quad r_{\rm loc}=1,\quad
 c_-=\frac1{55},\quad c_+=\frac1{45},\quad
 a(r)=a:=\frac{1324}{3375}.
\end{equation}
We verify the covariance constant next. These probability bounds hold for every $\Gamma\in\mathcal K$, independently of whether its prediction is close to a particular true coefficient.

Conditional on horizontal acquisition, $A$ and $E$ retain their original distributions, and
\[
 X=EA\left(1,-\frac{\Gamma_{12}}{\Gamma_{22}}\right)+(0,V),
 \qquad V\sim\operatorname{Unif}[-r/\Gamma_{22},r/\Gamma_{22}],
\]
where $V$ is independent of $(E,A)$. The vertical component has the corresponding representation. Their weights conditional on acquisition are $\Gamma_{11}/(\Gamma_{11}+\Gamma_{22})$ and $\Gamma_{22}/(\Gamma_{11}+\Gamma_{22})$, each at least $9/20$. Let $B$ in this calculation have columns $(1,-\Gamma_{12}/\Gamma_{22})^T$ and $(-\Gamma_{21}/\Gamma_{11},1)^T$. Since $\|B-I_2\|_{\rm op}\leq1/9$, its smallest singular value is at least $8/9$. The independent residual coordinates contribute positive-semidefinite matrices, so
\begin{equation}\label{eq:matrix-local-covariance-proof}
 \bbE[X^TX\mid\mathcal A_{h_\Gamma,r}]
 \succeq \frac9{20}\frac{331}{300}BB^T
 \succeq \frac9{20}\frac{331}{300}\left(\frac89\right)^2I_2
 =aI_2=\frac a\eta\Sigma_X.
\end{equation}
The constant function $a(r)$ satisfies the required monotonicity. Equations~\eqref{eq:continuous-local-constants}--\eqref{eq:matrix-local-covariance-proof} verify Assumption~\ref{assumption:holder}, with $a/\eta=92.0568\ldots$ and $c_+/c_-=11/9$. For every true matrix in the stated inner box, $\|h_\Gamma-h^*\|_\infty\leq0.3B_X<1$, so the whole initial class lies in the stated neighborhood. This does not require an oracle estimate of the unknown true matrix.

\paragraph{A deterministic upper bound on all acquisition probabilities.}
The finite certificates must also count batches whose threshold exceeds $r_\circ$. For any $\Gamma\in\mathcal K$, background acquisition requires, for at least one coordinate,
$B_i<(r+0.005)/0.9$.
Its probability is at most $\min\{1,2[(r-0.0175)/0.0225]_+\}$, where $[x]_+:=\max\{x,0\}$. On either informative component, the low-coordinate interval has probability at most $\min\{1,r/0.225\}$. The high coordinate cannot be selected below $0.875$; above that radius the last bound is already one. Thus, for every $r>0$ and every reference in the initial class,
\begin{equation}\label{eq:example-acquisition-upper}
 \bbP(\mathcal A_{h_\Gamma,r})\leq
 p_+(r):=\frac{199}{200}
 \min\left\{1,2\left[\frac{r-0.0175}{0.0225}\right]_+\right\}
 +\frac1{200}\min\left\{1,\frac r{0.225}\right\}.
\end{equation}
In particular, $p_+(r)=c_+r$ for $r\leq r_\circ$. This bound holds even on histories where a confidence constraint has excluded the truth, because every output still belongs to $\mathcal K$.

\paragraph{Finite count certificates.}
For the explicit certificates only, fix $\delta=0.05$, $\delta_0=\delta/256$, $n_0=32$, $b_0=0.05$, and $u_j:=\ln(16m_j^2/\delta_0)$. In this proof write $b^{(1)}=b_0$ and $b^{(j)}=b_0(u_j/m_j)^{1/4}$ for the threshold used within batch $j\geq2$. For a prescribed final batch $k$ with $b^{(k)}\leq r_\circ$, define
\begin{align}\label{eq:example-finite-counts}
 z&:=\ln(8/\delta),&
 \mu_k^-&:=c_-m_kb^{(k)},&
 n_k^-&:=\left\lfloor\mu_k^--\sqrt{2\mu_k^-z}\right\rfloor,\\
 M_k&:=\sum_{j=1}^km_jp_+(b^{(j)}),&
 N_k^+&:=n_0+\left\lceil M_k+\sqrt{2M_kz}+\frac{2z}{3}\right\rceil.\notag
\end{align}
Conditional on the preceding history, the final acquired count is binomial with mean at least $\mu_k^-$. The usual lower-tail exponential bound therefore gives $\bbP(n_k<n_k^-)\leq\delta/8$. For the total count, iterated conditioning of the Bernoulli moment generating function bounds it by $\exp\{(e^t-1)M_k\}$ for $t>0$, because \eqref{eq:example-acquisition-upper} bounds every predictable acquisition probability. The resulting upper-tail bound gives $\bbP(N_{2m_k-1}>N_k^+)\leq\delta/8$. No independence across batches or conditioning on future success is used. Every early acquisition and every warm-up label is included in $N_k^+$.

\paragraph{Squared-loss retention and the final fit.}
The excess squared risk is $\mathcal E(\Theta)=\eta\|\Theta-\Theta^*\|_F^2$. Bounded leverage gives the global error bound with $C_\phi=\sqrt L$. Conditional centering and nonsingular conditional feature moments give conditional consistency on every positive-probability feature restriction. No error bound uniform over arbitrary feature restrictions is needed.

Use $s_0=70\sigma^2$. The projected-noise calculation in Appendix~\ref{app:information-alpha-two} gives
$e_j^2\leq20\sigma^2u_j\leq m_js_{2m_j-1}$
outside conditional probability $e^{-u_j}$, since $20\sqrt{u_1}<70$ and $u_j/m_j$ decreases. Starting from $\mathcal H_0=\mathcal H$, completing the square therefore retains $\Theta^*$ at every update, with total failure probability at most $\delta_0/12$. Empty batches leave the class unchanged. The warm-up estimate need not be close to the truth: the selection inequalities hold for the entire initial class.

Take a final $k$ with $n_k^->0$, and set
\[
 d_k:=\sqrt{\frac{2B_X^2\ln(16/\delta)}{a n_k^-}},\qquad
 E_k^{\rm sq}:=\frac{\eta\sigma^2\chi^2_{4,1-\delta/8}}
 {(1-d_k)a n_k^-},
\]
provided $d_k<1$. Conditional on the past and a count $n\geq n_k^-$, the acquired feature rows are i.i.d. and have second moment at least $aI_2$. The matrix Chernoff lower-tail bound \citep{tropp2012user} gives
\[
 \bbP\{Z_k^TZ_k\not\succeq(1-d_k)naI_2
       \mid\text{past},n_k=n\}
 \leq2\exp\{-d_k^2na/(2B_X^2)\}\leq\delta/8.
\]
Conditional on a full-rank acquired design, the projected Gaussian noise energy has distribution $\sigma^2\chi^2_4$. With another failure probability at most $\delta/8$, constrained empirical optimality and retention give
\[
 \mathcal E(\widehat\Theta_k)
 \leq\frac{\eta e_k^2}{(1-d_k)a n_k}
 \leq E_k^{\rm sq}.
\]
Together with \eqref{eq:example-finite-counts}, this proves
\begin{equation}\label{eq:example-squared-certificate}
 \bbP\{\mathcal E(\widehat\Theta_k)\leq E_k^{\rm sq},\quad
                N_{2m_k-1}\leq N_k^+\}
 \geq1-\delta_0/12-4\delta/8>0.95.
\end{equation}
Unlike the generic initialization construction, this calculation does not first require a shrinking confidence set to establish the selection geometry.

\paragraph{A finite lower bound for the actual supervised estimator.}
We next prove a lower bound without the generic random-design cutoff. Suppose
$0<\epsilon<\eta/400$ and $\epsilon\leq\sigma^2\eta/B_X^2$.
A constrained supervised least-squares estimator with risk at most $\epsilon$ lies inside the initial coefficient box, even when the true matrix ranges over its stated inner box. Its normal equations are therefore the unconstrained normal equations.

For every $n\geq2$, the feature design has rank two almost surely. Write $G$ for its Gram matrix. Since $\bbE\operatorname{tr}(G)=2n\eta$, Markov's inequality gives $\bbP(G\preceq3n\eta I_2)\geq1/3$. Conditional on such a design, the unconstrained squared risk stochastically dominates $\sigma^2\chi^2_4/(3n)$. Consequently, if
$n\leq\sigma^2\chi^2_{4,0.8}/(3\epsilon)$,
the unconditional probability of risk larger than $\epsilon$ is at least $(1/3)(1/5)=1/15>0.05$. This also bounds failure of the constrained estimator: a small-risk constrained outcome would have to coincide with the unconstrained solution.

For $n=1$, the same interior normal equations force $X(\widehat\Theta-\Theta^*)=\xi$. Hence its squared risk on a small-risk outcome is at least $\eta\|\xi\|_2^2/B_X^2$. The assumed accuracy range bounds its success probability by $\bbP(\chi^2_2\leq1)<0.4$. Thus every positive sample size up to the preceding bound is insufficient. Writing $B_{\rm SL}:=\sigma^2\chi^2_{4,0.8}/3$, we have proved
\begin{equation}\label{eq:example-total-label-ratio}
 n_{\rm SL}(\epsilon,0.05)>\frac{B_{\rm SL}}\epsilon,
 \qquad
 \frac{N_k^+}{n_{\rm SL}(\epsilon,0.05)}
 <\frac{\epsilon N_k^+}{B_{\rm SL}}
 \quad\text{whenever }E_k^{\rm sq}\leq\epsilon.
\end{equation}
All rows of Table~\ref{tab:finite-kappa-comparison} satisfy the stated accuracy restrictions. They use $\chi^2_{4,0.8}/3=1.99620556\ldots$, with the exact count expressions above. This is a lower bound for the specified supervised estimator at each sample size, not a minimax lower bound or a comparison of two sufficient budgets. The table reports its total-label ratio, including every early acquisition, and does not assert that the generic $\epsilon_*$ has become a practically large cutoff.

\paragraph{SPO+ population and selected information.}
At the displayed true matrix $I_2$, independence of the coordinate noises and the sign symmetries of the feature law give
\begin{equation}\label{eq:matrix-spop-information}
 \begin{aligned}
 H&=\operatorname{diag}(H_1,H_2,H_2,H_1),&
 H_1&=8\bbE[X_1^2\varphi_\sigma(X_1)],&
 H_2&=8\bbE[X_2^2\varphi_\sigma(X_1)],\\
 V&=\operatorname{diag}(V_1,V_2,V_2,V_1),&
 V_1&=32\bbE[X_1^2\Phi(-|X_1|/\sigma)],&
 V_2&=32\bbE[X_2^2\Phi(-|X_1|/\sigma)].
 \end{aligned}
\end{equation}
All four diagonal entries are positive. These expectations reduce to one-dimensional integrals over the three uniform distributions. For example, for an even function $f$, independence gives
\[
 \begin{aligned}
 \bbE[X_1^2f(X_1)]
 &=\frac{199}{200}\bbE[B_1^2f(B_1)]
   +\frac1{400}\{\bbE[A^2f(A)]+\bbE[U^2f(U)]\},\\
 \bbE[X_2^2f(X_1)]
 &=\frac{199}{200}\bbE[B_2^2]\bbE[f(B_1)]
   +\frac1{400}\{\bbE[U^2]\bbE[f(A)]+\bbE[A^2]\bbE[f(U)]\}.
 \end{aligned}
\]
This gives the population values in Table~\ref{tab:spop-finite-ratios} without simulation.

For every locally selected law, $|h_j^*(X)|=|X_j|\leq1.1$ at this true matrix. The Gaussian box formulas in \eqref{eq:spop-hessian-variance} give
$H_{\mathcal A}\succeq8\varphi_\sigma(1.1)(I_2\otimes\Sigma_{X,\mathcal A})$.
Moreover,
\[
 \frac{4\Phi(-|x|/\sigma)}{\varphi_\sigma(x)}
 =4\sigma\int_0^\infty e^{-(|x|/\sigma)t-t^2/2}\,\dd t
 \leq2\sigma\sqrt{2\pi}.
\]
It follows pointwise, and hence after conditional expectation, that
$V_{\mathcal A}\preceq2\sigma\sqrt{2\pi}H_{\mathcal A}$.
Combining these inequalities with \eqref{eq:matrix-local-covariance-proof} gives valid constants
\begin{equation}\label{eq:matrix-spop-comparison-constants}
 J_\ell:=\frac{\max\{H_1,H_2\}}
 {8\varphi_\sigma(1.1)a}\,2\sigma\sqrt{2\pi},\qquad
 v_{\rm SL}:=\min\{V_1/H_1,V_2/H_2\}.
\end{equation}
These are bounds for all selected references in the box, not a substitution of the oracle reference into the acquisition rule. The leading ratio in Table~\ref{tab:spop-finite-ratios} is
\[
 3\frac{11}{9}\frac{2^{3/4}}{1-2^{-3/4}}
 \frac{\chi^2_{4,0.975}}{\chi^2_{4,0.8}}
 \frac{J_\ell}{v_{\rm SL}}.
\]
It is less than one at each displayed variance. The resulting population comparison does not specify a useful finite crossover budget for the scales below.

\paragraph{Curvature throughout the coefficient class.}
For any fitted $\Theta\in\mathcal K$ and any true matrix in the stated inner box,
\[
 |2h_\Theta(X)_j-h^*(X)_j|
 \leq |X_j|+\frac14(|X_1|+|X_2|)
 \leq1.4375<1.44.
\]
The expected prediction Hessian therefore has each diagonal entry at least
$m_+(\sigma):=8\varphi_\sigma(1.44)$.
The population score is zero at the conditional mean. Thus the full-population global error bound holds with $C_\phi=\sqrt{2L/m_+(\sigma)}$, and conditional consistency holds on every positive-probability feature restriction. On every small-radius selected law the coefficient Hessian is at least $a m_+(\sigma) I_4$ throughout the class.

For completeness, a lower bound valid also under the full law and all larger selected radii is
\[
 m_X:=\min\{\eta,a c_-r_\circ\}=0.0001248215\ldots,
 \qquad h_0:=m_+(\sigma)m_X.
\]
Indeed, a selected set with radius larger than $r_\circ$ contains the selected set at $r_\circ$, whose unnormalized second moment is at least $a c_-r_\circ I_2$. Dividing by an acquisition probability at most one preserves this lower bound. A valid upper Hessian bound is $8\varphi_\sigma(0)B_X^2$. Also, $|h_j^*(X)|<1.17$ throughout the inner true-parameter box, so $v_0=32\Phi(-1.17/\sigma)m_X>0$ is a valid score-covariance lower bound. The bounded score, Hessian-Lipschitz bound, and empirical-score arguments in Appendix~\ref{app:spop-verification} verify the remaining conditions with these constants. Its vertex-selection argument supplies the finite-sample anti-concentration condition for the lexicographically smallest supervised minimizer.

At $\sigma=1$, the coefficient lower bound above is $1.41258\ldots\times10^{-4}$, while the selected small-radius coefficient curvature is at least $0.4439$. These bounds avoid a Gaussian density evaluated thousands of standard deviations from the relevant reflected means. Nevertheless, the generic score-envelope and normal-approximation cutoffs can still be very conservative. We do not use them to claim a numerical finite-budget result; the following independent argument supplies the stated active certificate.

\paragraph{SPO+ retention without an inverse-curvature slack.}
Put $G:=4\sqrt2 B_X$ in this proof. Every sample SPO+ loss is $G$-Lipschitz in the coefficient vector, including at its kinks. For a true matrix in the inner box, $\|\Theta-\Theta^*\|_F\leq0.3$ throughout $\mathcal K$, so every loss difference from the truth has absolute value at most $M:=0.3G$. This bound is independent of the realized Gaussian label.

Condition on the history before batch $j$. Cover $\mathcal K$, which is contained in a Frobenius ball of radius $0.2$, by at most $(1+2m_j)^4$ balls of radius $0.2/m_j$ with centers in $\mathcal K$. The selected loss difference at a center lies in $[-M,M]$. Hoeffding's inequality and a union bound show that its empirical deviation is at most
$M\sqrt{2\{u_j+\ln[2(1+2m_j)^4]\}/m_j}$
for every center, except with probability $e^{-u_j}$. The expression in braces is at most $4u_j$. Lipschitzness adds at most $0.4G/m_j\leq2M/m_j$ for empirical and population interpolation together. Consequently the uniform deviation is at most
\[
 5M\sqrt{u_j/m_j}<10\sqrt{u_j/m_j}=s_{2m_j-1}
\]
when $s_0=10$. The selected population loss is minimized at the truth, so its empirical improvement deficit is no larger than this uniform deviation. The actual confidence update retains $\Theta^*$ at every batch, outside total failure probability $\delta_0/12$. This argument does not use $h_0^{-1}$ or the generic $B_0$ in the slack. It is an analysis of the same confidence-set update, not a change to the algorithm.

\paragraph{A finite SPO+ risk certificate without normal approximation.}
Fix a final small-radius batch and condition on its past and acquired count $n\geq1$. Let $\widehat R_n$ be the acquired-sample average loss, and let $R_{\mathcal A}$ be its population counterpart. For a fixed radius $R>0$, a cover of $\mathcal K\cap\{\|\Theta-\Theta^*\|_F\leq R\}$ with mesh $R/n$ has at most $(1+2n)^4$ centers. The same Hoeffding and interpolation argument gives, with probability at least $1-\delta/8$,
\[
 \sup_{\|\Theta-\Theta^*\|_F\leq R}
 \left|\{\widehat R_n(\Theta)-\widehat R_n(\Theta^*)\}
       -\{R_{\mathcal A}(\Theta)-R_{\mathcal A}(\Theta^*)\}\right|
 \leq GR\,z_n,
\]
where
\[
 z_n:=\sqrt{\frac{2\ln\{2(1+2n)^4/(\delta/8)\}}n}+\frac2n.
\]
The supremum is restricted to $\mathcal K$. Let $h_a:=a m_+(\sigma)$. If $R_n:=3Gz_n/h_a<0.3$, apply this bound with the deterministic radius $R_n$. On its boundary, population strong convexity gives an empirical loss difference at least
$h_aR_n^2/2-GR_nz_n>0$.
If a constrained empirical minimizer lay outside this ball, the segment from the truth to that minimizer would remain in the preceding convex retained class and have a boundary point with empirical loss no larger than at the truth. This contradicts the displayed strict inequality. Thus its coefficient error is at most $R_n$. If $R_n\geq0.3$, the initial coefficient box already gives the error bound $0.3$.

The full-population coefficient Hessian is at most $8\varphi_\sigma(0)\eta I_4$. Taylor's formula and the zero population gradient therefore give
\begin{equation}\label{eq:example-spop-certificate}
 R_\ell(h_{2m_k-1})-R_\ell^*
 \leq4\varphi_\sigma(0)\eta
 \min\left\{0.3^2,
       \left(\frac{3Gz_{n_k^-}}{a m_+(\sigma)}\right)^2\right\},
 \qquad N_{2m_k-1}\leq N_k^+.
\end{equation}
Here $z_n$ decreases for $n\geq1$, so the count lower bound in \eqref{eq:example-finite-counts} justifies using $n_k^-$. Combining retention, the final count bound, the conditional empirical-process bound, and the total count bound gives joint probability at least $1-\delta_0/12-3\delta/8>0.95$. This proof needs no generic fitting cutoff and no claim that historical constraints are inactive.

For $\sigma=1$ and $k=48$, the count formulas give $n_k^-=109{,}813{,}175$ and $N_k^+=328{,}968{,}805$. The right side of \eqref{eq:example-spop-certificate} is less than $1.903\times10^{-5}$. This proves the stated finite active guarantee with $s_0=10$. Its risk order in the fitting count is $O(\ln n/n)$, rather than the sharper leading constant of Theorem~\ref{thm:alpha-two-spop-finite}. We do not transfer the leading ratio in Table~\ref{tab:spop-finite-ratios} to this different finite bound, and we do not claim a practically useful finite same-risk SPO+ advantage from it. The large reported stream horizon is an explicit limitation of the certificate.

\paragraph{Scope of the error bound.}
The global error bound need not extend uniformly to all measurable
feature restrictions; Assumption~\ref{assumption:upper-bound-pointwise}
does not require such an extension.
For instance, restrict to the positive horizontal rectangle and take
\[
 \mathcal B_e=\{1<X_1<1.1,\ |X_2|<e\}
 \ \cup\ \{1<X_1<1.1,\ 0.1<X_2<0.1+e^3\},\qquad 0<e<0.05.
\]
The conditional second moment of $X_2$ is $O(e^2)$, whereas its conditional essential supremum is bounded below by $0.1$. On the coefficient segment $\Theta=I_2+tE_{21}$, $0\leq t\leq0.05$, the conditional squared excess risk at a fixed positive $t$ tends to zero as $e\downarrow0$, while its prediction sup distance from $h^*$ does not. The bounded upper prediction curvature gives the same counterexample for SPO+. Thus the surrogate analysis uses the full-population error bound
and separate selected-information conditions, rather than a
uniform conditional error bound.

\subsection{A finite SPO+ label comparison}\label{app:spop-finite-label-comparison}

We give a finite comparison that does not depend on an empirical tie rule. The active guarantee holds throughout a specified family of unknown coefficients. The supervised lower bound applies to every procedure based on fully labeled i.i.d. observations that attains the same guarantee uniformly over that family. In particular, we first verify that no fixed predictor already meets the target. This uniform requirement differs from the pointwise requirement of a specified supervised estimator in Proposition~\ref{prop:alpha-two-spop-comparison}; the notation below keeps them separate.

Let $q=10^{-4}$ and define
\begin{equation}\label{eq:spop-finite-ratio-feature-law}
 X=\begin{cases}
 (A,B),&\text{with probability }1-q,\\
 (V,Z),&\text{with probability }q,
 \end{cases}
 \qquad
 \begin{aligned}
 A&\sim\operatorname{Unif}[100,101],&
 B&\sim\operatorname{Unif}[-10^{-3},10^{-3}],\\
 V&\sim\operatorname{Unif}[-0.1,0.1],&
 Z&\sim\operatorname{Unif}[0.9,1.1].
 \end{aligned}
\end{equation}
Both components have densities and independent coordinates. For $S=[-100,100]$, use
\begin{equation}\label{eq:spop-finite-ratio-model}
 h_\theta(X)=X\theta,\qquad
 \mathcal K=[0.95,1.05]\times[-0.2,0.2],\qquad
 C=X\theta^*+\xi,\quad \xi\sim\mathcal N(0,1),\quad \xi\perp X.
\end{equation}
The unknown true coefficient belongs to
$\Theta_*:=[0.99,1.01]\times[-0.1,0.1]$.
The feature law, initial class, and algorithmic scales do not depend on $\theta^*$. Both slope coefficients are free in the fitted class, and every true coefficient is interior to $\mathcal K$. MBAL-SPO may use any exact measurable empirical-minimizer rule.

Write $\mathcal E_{\theta^*}(h):=R_{\ell,\theta^*}(h)-R_{\ell,\theta^*}^*$ for the excess SPO+ risk under the distribution indexed by $\theta^*$. Gaussian symmetry makes $h_{\theta^*}$ conditionally optimal, so the same optimal risk is attained over $\mathcal H=\{h_\theta:\theta\in\mathcal K\}$ and over all measurable predictors. For this comparison, define
\begin{equation}\label{eq:spop-uniform-supervised-requirement}
 n_{\rm SL}^{\rm unif}(\epsilon,\delta):=
 \min\left\{n\geq0:\ \exists\widehat h_n,\quad
 \sup_{\theta^*\in\Theta_*}
 \bbP_{\theta^*}\!\left(\mathcal E_{\theta^*}(\widehat h_n)>\epsilon\right)
 \leq\delta\right\}.
\end{equation}
Here $n$ is an integer, and $\widehat h_n$ is any measurable, possibly randomized, procedure based on $n$ fully labeled i.i.d. observations. It may know the feature law and $\Theta_*$, but not $\theta^*$. Its output need not be an empirical minimizer or even belong to the linear class. Independent unlabeled observations do not distinguish the coefficients because their feature law is common. Labels chosen selectively after observing features are not i.i.d. labeled sampling and are not covered by this supervised requirement.

For $S=[-100,100]$, the scalar SPO+ loss is $100$ times
\[
 L(z,c):=|c-2z|-2z\operatorname{sign}(c)+|c|.
\]
We use this normalized loss only in the proof; the algorithm trains on the original loss. The decision set fixes risk units and does not introduce an intercept.

\begin{proposition}[A uniform finite SPO+ label advantage]\label{prop:spop-finite-label-ratio}
Under \eqref{eq:spop-finite-ratio-feature-law}--\eqref{eq:spop-finite-ratio-model}, run Algorithm~\ref{alg:margin-based} with $n_0=32$, $b_0=30$, $s_0=200$, confidence input $0.05/64$, and $T=2^{29}-1$. For every $\theta^*\in\Theta_*$ and every exact measurable empirical-minimizer rule,
\[
 \bbP_{\theta^*}\!\left(
 \mathcal E_{\theta^*}(h_T)\leq2\times10^{-5},\quad
 N_T\leq54{,}461\right)\geq0.95.
\]
In contrast, the uniform supervised requirement satisfies
\[
 n_{\rm SL}^{\rm unif}(2\times10^{-5},0.05)>2{,}600{,}000.
\]
Consequently,
\begin{equation}\label{eq:spop-explicit-finite-ratio}
 \frac{54{,}461}{n_{\rm SL}^{\rm unif}(2\times10^{-5},0.05)}
 <\frac{54{,}461}{2{,}600{,}000}<0.021.
\end{equation}
The active budget includes all warm-up and post-warm-up labels. No procedure using zero labels attains the risk target with probability $0.95$ uniformly over $\Theta_*$.
\end{proposition}

\begin{proof}{\bfseries Proof of Proposition \ref{prop:spop-finite-label-ratio}.}
Fix $\epsilon=2\times10^{-5}$ and write $\varphi,\Phi$ for the standard normal density and distribution function. Let $f_\mu(z):=\bbE[L(z,\mu+\xi)]$. Direct differentiation gives
\[
 f_\mu'(z)=4\{\Phi(2z-\mu)-\Phi(\mu)\},\qquad
 f_\mu''(z)=8\varphi(2z-\mu).
\]
Thus $f_\mu$ is strictly convex and minimized at $z=\mu$. In coefficient form,
\begin{equation}\label{eq:spop-finite-ratio-hessian}
 \nabla_\theta^2\bbE[L(X\theta,C)\mid X]
 =8\varphi(2X\theta-X\theta^*)X^TX.
\end{equation}
These formulas also show that the conditional excess loss is nonnegative on each feature component.

\emph{Separation of the target-risk sets.}
Take $a=0.1$ and the two possible true coefficients $\theta^+=(1,a)$ and $\theta^-=(1,-a)$, both in $\Theta_*$. Their conditional means are $\mu_\pm(x)=x_1\pm ax_2$. For each feature $x$, the derivative of
$\{f_{\mu_+(x)}(z)+f_{\mu_-(x)}(z)\}/2$
vanishes at $z=x_1$. Convexity therefore makes $h^\circ(x)=x_1$ minimize the average of the two population risks, even over arbitrary measurable predictors.

On the informative component, along the segment joining $x_1$ to either conditional mean, $|2z-\mu_\pm(x)|\leq0.1+0.1(1.1)=0.21$. Taylor's formula and $\bbE[Z^2]=301/300$ imply, for every measurable predictor $h$,
\begin{equation}\label{eq:spop-target-separation}
 \begin{aligned}
 \frac{\mathcal E_{\theta^+}(h)+\mathcal E_{\theta^-}(h)}2
 &\geq\frac{\mathcal E_{\theta^+}(h^\circ)
            +\mathcal E_{\theta^-}(h^\circ)}2\\
 &\geq400q\varphi(0.21)a^2\frac{301}{300}
 >1.566\times10^{-4}>\epsilon.
 \end{aligned}
\end{equation}
The background excess terms are nonnegative and are dropped in the last inequality. Hence no predictor belongs to both target-risk sets. In particular, every fixed predictor fails the target at one of the two coefficients. For a randomized zero-label procedure, its output law is the same at both coefficients, and its two success events are disjoint. At least one success probability is therefore at most $1/2$.

\emph{A lower bound for every supervised procedure.}
Fix any procedure $\widehat h_n$ based on $n$ fully labeled i.i.d. observations. From its output, define a test that selects $\theta^+$ when $\mathcal E_{\theta^+}(\widehat h_n)\leq\epsilon$ and selects $\theta^-$ otherwise. By \eqref{eq:spop-target-separation}, this test is correct whenever the procedure meets the target under the true coefficient. Its error probability under each coefficient is therefore at most the corresponding risk-failure probability.

Conditional on the observed features, put $S_n=\sum_{i=1}^nX_{i2}^2$ and $Y_i=C_i-X_{i1}$. Under the two hypotheses, $Y_i=\pm aX_{i2}+\xi_i$. The likelihood ratio is increasing in $\sum_iX_{i2}Y_i$, whose conditional distributions are $\mathcal N(\pm aS_n,S_n)$. With equal prior probabilities, choosing the larger conditional density minimizes the average testing error; its error is exactly $\Phi(-a\sqrt{S_n})$. Since the feature law is identical under the two hypotheses, every test, including randomized tests, has average error at least $\bbE[\Phi(-a\sqrt{S_n})]$.

The function $g(s)=\Phi(-a\sqrt s)$ is convex on $[0,\infty)$: for $s>0$,
\[
 g''(s)=\frac{a\varphi(a\sqrt s)}{4s^{3/2}}(1+a^2s)>0.
\]
Moreover,
\[
 \bbE[X_2^2]=(1-q)\frac{10^{-6}}3+q\frac{301}{300}
 =0.000100666633333\ldots.
\]
Jensen's inequality and the reduction from risk estimation to testing give
\begin{equation}\label{eq:spop-uniform-testing-lower}
 \begin{aligned}
 \sup_{\theta^*\in\Theta_*}
 \bbP_{\theta^*}\!\left(\mathcal E_{\theta^*}(\widehat h_n)>\epsilon\right)
 &\geq\bbE[\Phi(-a\sqrt{S_n})]\\
 &\geq\Phi\!\left(-a\sqrt{n\bbE[X_2^2]}\right).
 \end{aligned}
\end{equation}
For every integer $0\leq n\leq2{,}600{,}000$, the right side is at least $0.052851\ldots>0.05$. This proves the uniform supervised lower bound. It uses the information in all observations, including the small second-coordinate variation on the background component, and does not depend on empirical optimization or tie-breaking.

\emph{Acquisition through the prescribed horizon.}
Put $\gamma=0.05/64$ and $u_j=\ln(16m_j^2/\gamma)$ in this proof. The final completed batch is $j=29$, with $m_{29}=2^{28}$. Every retained class is a subset of $\mathcal K$. On background features,
\[
 |h_\theta(X)|\geq0.95(100)-0.2(10^{-3})=94.9998,
 \qquad\theta\in\mathcal K,
\]
whereas informative features satisfy $|h_\theta(X)|\leq1.05(0.1)+0.2(1.1)=0.325$. Distance to degeneracy is $|h_\theta(X)|$. The acquisition thresholds through $T$ lie between $0.619$ and $46.267$: the first batch uses $30$, and batch $j\geq2$ uses $30(u_j/m_j)^{1/4}$. Therefore every informative observation is labeled and every background observation is rejected, for every history and every empirical-minimizer rule. The acquisition indicators are i.i.d. Bernoulli variables with mean $q$.

\emph{Confidence constraints.}
For any $\theta,\theta'\in\mathcal K$, on acquired features,
\[
 |L(X\theta,C)-L(X\theta',C)|
 \leq4|X(\theta-\theta')|
 \leq4\{0.1(0.1)+1.1(0.4)\}=1.8.
\]
Let $n_j$ be the acquired count in batch $j$. The normalized empirical loss oscillation is at most $1.8n_j/m_j$. A Bernoulli upper-tail bound gives
\[
 \frac{n_j}{m_j}\leq q+\sqrt{\frac{2qu_j}{m_j}}+\frac{2u_j}{3m_j}
\]
except with probability $e^{-u_j}$. Write $z_j=u_j/m_j$ only for this calculation; it decreases with $j$. If $z_j\geq1$, then $1.8n_j/m_j\leq1.8<2\sqrt{z_j}$. If $z_j<1$, use $q/\sqrt{z_{29}}<0.24$ to obtain
\[
 \frac{1.8n_j/m_j}{\sqrt{z_j}}
 \leq1.8\left(0.24+\sqrt{2q}+\frac23\right)<2.
\]
The algorithmic slack divided by the loss scale $100$ is $2\sqrt{z_j}$ because $s_0=200$. It therefore covers the entire empirical loss oscillation. Starting from $\mathcal H_0=\mathcal H$, every retained coefficient set remains equal to $\mathcal K$ through $T$ on this event. Its failure probability is at most
\[
 \sum_{j=1}^{29}e^{-u_j}\leq\frac\gamma{16}
 \sum_{j=1}^\infty4^{-(j-1)}=\frac\gamma{12}.
\]
The true coefficient is retained without requiring the warm-up estimate to be close to it.

\emph{Counts and the final empirical fit.}
Let $z=\ln160$ and $m=2^{28}$. The last acquired count is binomial with mean $qm$, and the total post-warm-up count is binomial with mean $qT$. Their tails give
\begin{align}\label{eq:spop-finite-ratio-counts}
 n_{29}&\geq\left\lfloor qm-\sqrt{2qmz}\right\rfloor=26{,}321,\\
 N_T&\leq32+\left\lceil qT+\sqrt{2qTz}+\frac{2z}{3}\right\rceil
 =54{,}461,\notag
\end{align}
each except with probability $1/160$. All early acquisitions and warm-up labels are included.

Set $r=0.033$ and $r_-=r-0.06(0.1)/0.9=0.026333\ldots$. For comparison coefficients $\theta=(\theta_1^*,\theta_2^*+t)$ with $|t|\leq r_-$, informative features satisfy
\[
 |2X\theta-X\theta^*|
 \leq1.01(0.1)+(0.1+2r_-)(1.1)<0.27.
\]
Hence $c:=8\varphi(0.27)(301/300)>3.0875$ lower-bounds the normalized population second derivative in $\theta_2$ along these comparison segments. For every admissible first coefficient and every informative feature,
\[
 \begin{aligned}
 V\theta_1+Z(\theta_2^*+r)&\geq X\theta^*+Zr_-,\\
 V\theta_1+Z(\theta_2^*-r)&\leq X\theta^*-Zr_-.
 \end{aligned}
\]
The normalized sample derivative in the second coefficient is
$2Z\{\operatorname{sign}(2X\theta-C)-\operatorname{sign}(C)\}$.
It is nondecreasing in the prediction and belongs to $[-4.4,4.4]$. At the fixed comparison coefficients $(\theta_1^*,\theta_2^*\pm r_-)$, its expected values are at least $cr_-$ and at most $-cr_-$, respectively. These comparison derivatives exist almost surely because the cost has a continuous conditional density.

Conditional on a final acquired count $n\geq26{,}321$, the acquired observations are i.i.d. under the informative law. Hoeffding's inequality bounds the probability of either wrong derivative sign by
\[
 2\exp\left\{-\frac{n(cr_-)^2}{32(1.1)^2}\right\}<0.022358.
\]
The preceding pointwise comparisons extend the signs to $\theta_2^*\pm r$ for every first coefficient, including all subgradient choices at sample kinks. Both second-coordinate endpoints lie in $[-0.2,0.2]$. On the confidence-set event, convexity therefore forces every final empirical minimizer to satisfy
\[
 |\widehat\theta_2-\theta_2^*|\leq0.033,
 \qquad |\widehat\theta_1-\theta_1^*|\leq0.06.
\]
This step uses no empirical tie rule.

On the informative component, \eqref{eq:spop-finite-ratio-hessian} is bounded above using $\varphi(0)$; also $\bbE[VZ]=0$, $\bbE[V^2]=1/300$, and $\bbE[Z^2]=301/300$. On the background component, the reflected mean is at least $0.89(100)-0.5(10^{-3})>88$ throughout the fitted and true coefficient boxes. The resulting normalized excess-risk contribution is less than $10^{-100}$. Taylor's formula yields
\[
 \begin{aligned}
 \mathcal E_{\theta^*}(h_T)
 &\leq100\left[
 4q\varphi(0)\left\{\frac{0.06^2}{300}
       +\frac{301}{300}(0.033)^2\right\}+10^{-100}\right]\\
 &<1.763\times10^{-5}<2\times10^{-5}.
 \end{aligned}
\]
Combining the confidence-set, two count, and derivative events gives joint probability at least
$1-\gamma/12-2/160-0.022358>0.9650$.
All bounds hold for every $\theta^*\in\Theta_*$ with the same constants. No final sample is conditioned on future retention. This proves the active guarantee; division by the uniform supervised lower bound gives \eqref{eq:spop-explicit-finite-ratio}.
\hfill\Halmos
\end{proof}

The comparison concerns a uniform guarantee over the stated unknown-parameter family. It does not assert that every supervised procedure requires this many labels at each fixed parameter: a rule hard-coded to that parameter could succeed without data. Instead, no supervised procedure based on at most $2{,}600{,}000$ i.i.d. labeled observations meets the target at every parameter in the family, whereas MBAL-SPO has the stated guarantee at every parameter with at most $54{,}461$ labels on its success event. A hard label cap at this budget changes its output only outside that event. The lower bound covers center-seeking empirical tie rules and arbitrary other estimators, and \eqref{eq:spop-target-separation} rules out any zero-label solution to this uniform task.

This is a separate finite-sample result for the same algorithmic updates, not an evaluation of the leading ratio or generic cutoff in Proposition~\ref{prop:alpha-two-spop-comparison}. It does not assume the uniform small-radius matrix-information condition, which may fail for homogeneous two-decision models. The initial coefficient box is substantive prior information, shared with the supervised learner. The feature horizon $2^{29}-1$ remains large; the established saving concerns acquired labels, not feature processing or computational time.

\captionsetup{font+={color=blue}}
\section{Details for Numerical Experiments}\label{appendix:experiments}

We specify the data models, initialization, numerical oracles, and sensitivity analyses used in Section~\ref{sec:experiments}. Every reported comparison uses the fourth-root threshold and square-root slack in \eqref{eq:mbal-surrogate-schedules}. We distinguish the two independent-coordinate-noise grids and the misspecified exponential-demand pricing model. The four-loss comparisons use the same final data within each setting. The experiments illustrate behavior in developed synthetic settings; they do not verify every population assumption of the theoretical results.

\subsection{Shortest Path Data and Implementation}\label{appendix:independent_shortest_path}

Let $\widetilde X=(X^T,1)^T$. In each grid, we use a fixed matrix $\Theta^*$ and conditional mean
\[
\mu_c(X)=2\mathbf1_d+(\Theta^*)^T\widetilde X.
\]
The full fitted classes are $\Theta\in[-2,2]^{6\times12}$ and $\Theta\in[-0.5,0.5]^{6\times40}$ for the two grids. All entries, including intercepts, are free. The fixed matrices, component centers, and covariances are supplied with the numerical files.

\paragraph{The $3\times3$ grid.}\label{appendix:grid3_setup}
We use a nine-component Gaussian feature mixture. Three easy components have total probability 0.97, equal component weights, centers of Euclidean norm 8, and covariance $0.2^2I_5$. Six components centered on genuine optimal-path boundaries have equal weights $0.03/6$ and covariance
\[
1.5^2(I_5-u_ku_k^T)+0.15^2u_ku_k^T,
\]
where $u_k$ is the unit normal to the corresponding boundary in feature space. The conditional-mean matrix is constructed with fixed coefficient seed 6101205: a Gaussian QR factor gives five feature directions, independent Gaussian loadings give the slopes, and a Gaussian row scaled by 0.08 gives the intercept. A common normalization makes the largest absolute coefficient one. The slope matrix has rank five, and the conditional mean is strictly interior in the fitted coefficient box, at distance one from its boundary.

Boundary centers are projections onto facets at which two paths attain the minimum predicted cost. The easy centers are selected from radius-eight directions with large distances to degeneracy. The teacher matrix and component arrays are fixed before final sampling. Costs follow \eqref{eq:numerical_grid3_cost}, with independent uniform innovations across edges and observations. The conditional covariance is
\[
\operatorname{Cov}(c\mid X)=\frac{\epsilon^2}{3}
\operatorname{diag}\bigl(\mu_{c,1}(X)^2,\ldots,\mu_{c,12}(X)^2\bigr).
\]
The conditional noise is centrally symmetric, but its absolute scale varies with $X$. Gaussian features and costs are not clipped.

\paragraph{The $5\times5$ grid.}\label{appendix:grid5_setup}
The feature mixture has one easy component of probability 0.85, center zero, and covariance $0.03^2I_5$, and 32 boundary components of equal probability $0.15/32$. To obtain bounded support with the specified covariance, let $Z$ be uniform on the five-dimensional ball of radius $\sqrt7$, so $\bbE[ZZ^T]=I_5$. Conditional on component $k$, set
\[
X=\mu_k+L_kZ,\qquad L_kL_k^T=\Sigma_k.
\]
For a boundary component with unit normal $u_k$, use
\[
\Sigma_k=0.3^2(I_5-u_ku_k^T)+0.035^2u_ku_k^T.
\]
Thus the easy component has small feature variation, while boundary components vary more along the facet and remain concentrated across its normal.

Using coefficient seed 62010107, we draw a $5\times40$ Gaussian slope matrix and normalize its largest absolute entry to 0.22. The intercept is $0.18(\mathbf1_{40}-2w_0)$, where $w_0$ is the first enumerated path. The slope matrix has rank five, and the largest absolute entry of $\Theta^*$ is 0.22. The conditional mean is therefore strictly interior in the unchanged coefficient box, at distance 0.28 from its boundary. This teacher does not restrict the fitted class to its span.

We project feature proposals onto genuine optimal-path facets, require the projected feature norm to lie between 0.45 and 15, and first retain facets that expand the span of their path-incidence contrasts. We then add separated centers until there are 32 components. Their contrasts have rank 16. The coefficient seed is chosen from a finite geometry-only preflight before any fitted outcomes. These components span the structural decision contrasts and provide information in directions relevant to path choices. Costs follow \eqref{eq:numerical_grid5_cost}; the Gaussian noise has the same variance in every component and is independent of the features. No clipping or shared random multiplier is applied.

For both grids, numerical training subtracts the common cost offset 2. Regression predictions receive the same offset. Because all paths in a grid have the same number of edges, this translation preserves decisions, distances to degeneracy, and the SPO+ objective; it also preserves regression residuals. The initial classes contain the respective conditional means, but initial membership does not imply retention in every empirical confidence set.
\paragraph{Selected information.}
We evaluate information at the fixed conditional-mean predictor
$\mu_c$. For a radius $r>0$ with positive acquisition probability,
define
\[
 \mathcal A_r:=\{x:\nu_S(\mu_c(x))<r\},\qquad
 \Sigma:=\bbE[\widetilde X\widetilde X^T],\qquad
 \Sigma_r:=\bbE[\widetilde X\widetilde X^T\mid X\in\mathcal A_r],
\]
where $\widetilde X=(X^T,1)^T$ includes the fitted intercept.
With $\eta:=\lambda_{\min}(\Sigma)$, define the pointwise
feature-information diagnostic by
\[
 a_{\rm pt}(r)
 :=\eta\,\lambda_{\min}\!\left(\Sigma^{-1/2}\Sigma_r\Sigma^{-1/2}\right).
\]
Thus $a_{\rm pt}(r)/\eta$ is the largest scalar $\lambda$ such that
$\Sigma_r\succeq\lambda\Sigma$. The slope-only diagnostic is computed
in the same way after omitting the constant feature.

For each training loss, let $H,V$ and $H_{\mathcal A_r},V_{\mathcal A_r}$
be the population coefficient Hessians and score covariances in
\eqref{eq:convex-hessian-variance}, evaluated at the conditional-mean
coefficient under the full and selected feature laws. When the
Hessians are positive definite, define
\[
 \begin{aligned}
 J_{\ell,\rm pt}(r)
 &:=
 \lambda_{\max}\!\left(
 H_{\mathcal A_r}^{-1/2}H H_{\mathcal A_r}^{-1/2}\right)
 \lambda_{\max}\!\left(
 H_{\mathcal A_r}^{-1/2}V_{\mathcal A_r}
 H_{\mathcal A_r}^{-1/2}\right),\\
 v_{\rm SL}
 &:=
 \lambda_{\min}\!\left(H^{-1/2}VH^{-1/2}\right).
 \end{aligned}
\]
The ratio $v_{\rm SL}/J_{\ell,\rm pt}(r)$ is the pointwise counterpart
of the information ratio in \eqref{eq:spop-information-comparison}.
Unlike the theoretical constants, these diagnostics concern one
reference predictor and one radius; they do not establish uniform
bounds over a neighborhood of predictors.

For the $3\times3$ law, four independent feature replicates contain
500,000 observations each. We use the exact full feature second
moment and Monte Carlo selected second moments. The mean estimates
of $a_{\rm pt}(r)/\eta$ are 0.01561 and 0.01558 at $r=0.025$ and
$0.05$, respectively, with $\eta=0.022534$. Expected squared- and
Huber-loss Hessians and score covariances integrate the conditional
uniform noise analytically. The corresponding mean estimates of
$v_{\rm SL}/J_{\ell,\rm pt}(r)$ are $1.047\times10^{-4}$ and
$1.051\times10^{-4}$ for squared loss, and $8.055\times10^{-5}$
and $8.043\times10^{-5}$ for Huber loss. None of these
full-coefficient diagnostics exceeds one. The unbounded feature law,
varying absolute noise scale, and lack of uniform neighborhood
calculations further distinguish this downstream experiment from
the conditions of the finite surrogate comparisons.

For SPO+, decision-invariant coefficient directions make the full
Hessian singular. We therefore compute the curvature and score
diagnostics on the fixed structural decision-contrast subspace.
Let $Q$ have orthonormal columns spanning that subspace. In the
preceding formulas, replace each matrix
$M\in\{H,V,H_{\mathcal A_r},V_{\mathcal A_r}\}$ by $Q^TMQ$,
using the same $Q$ throughout and requiring the restricted Hessians
to be positive definite. The reported curvature gain is
$\lambda_{\min}(H^{-1/2}H_{\mathcal A_r}H^{-1/2})$
on this subspace. This projection is used only for the diagnostics;
it does not restrict the fitted coefficient class.

The larger-grid generator is screened using feature second moments and approximate Hessians of expected SPO+ loss, rather than samplewise Hessians of its piecewise-linear loss. For a margin radius $r$, we compare selected and full population matrices on the fixed structural decision-contrast space. Two stratified Monte Carlo replicates use 128 feature draws for the easy component, 48 per boundary component, and 32 conditional Gaussian draws for each active-facet calculation, truncating gaps beyond six noise standard deviations. These calculations are pointwise diagnostics, not uniform bounds.

For the selected $5\times5$ law at $r=0.025$, the minimum generalized feature eigenvalue is 0.891, the slope-only counterpart is 6.578, and the minimum curvature gain on the 96 identifiable coefficient directions is 5.824. The computed pointwise counterpart of $v_{\rm SL}/J_\ell$ is 1.263; at $r=0.05$ it is 1.446. A free intercept prevents the full-feature information gain from exceeding one as a population bound. Moreover, decision-invariant directions make the full coefficient Hessian singular. The diagnostic contrast projection leaves the fitted class unchanged and does not establish the positive-definiteness or uniform selected-information conditions for that unrestricted class. We therefore interpret the numerical gains as downstream decision improvements, rather than a verification of the surrogate comparison theorems.

\paragraph{Surrogate losses and optimization.}
Besides SPO+, we use the coordinate sums
\begin{align*}
\ell_{\mathrm{sq}}(\hat c,c)&=\sum_{i=1}^d(\hat c_i-c_i)^2,
&\ell_{\mathrm{MAE}}(\hat c,c)&=\sum_{i=1}^d|\hat c_i-c_i|,\\
\ell_{\mathrm{Huber}}(\hat c,c)&=\sum_{i=1}^d\rho_{\delta_H}(\hat c_i-c_i),
&\rho_{\delta_H}(u)&=\begin{cases}\tfrac12u^2,&|u|\leq\delta_H,\\
\delta_H|u|-\tfrac12\delta_H^2,&|u|>\delta_H.
\end{cases}
\end{align*}
There is no division by $d$. We calibrate the grid Huber cutoffs from pooled absolute residuals of development warm-up squared-loss fits, with floor $10^{-4}$, and hold them fixed in final and sensitivity runs. The selected cutoffs are 0.0116452 and 0.0133885 for the two grids.

The $3\times3$ SPO+ oracle uses a finite-path epigraph linear program, and the $5\times5$ oracle uses a network-potential formulation, solved through SciPy's HiGHS interface. The smaller-grid solver uses a bounded $10^{-9}$ weighted $\ell_1$ reference term and includes its maximum bias in the pure-loss certificate. The larger-grid solver first minimizes the pure objective and then selects a nearby predictor with a secondary $\ell_1$ objective. Supervised observation constraints are added until all omitted path-epigraph inequalities are satisfied, preserving the full-prefix objective. Larger-grid objectives and historical constraint rows are scaled by their full-batch denominators for conditioning.

Squared, MAE, and Huber losses use convex formulations in CVXPY and Clarabel. Violated historical confidence constraints are added iteratively, and every retained historical loss is checked on its full data before accepting a fit. For each grid, the pure-loss gap and historical violations must not exceed the smaller of $2\times10^{-5}$ and 5\% of the applicable scheduled slack, with a $10^{-12}$ arithmetic allowance. These are numerical convex minimizers, and both methods use the same class and grid-specific oracle.

\subsection{Warm-up, Rejection Threshold, and Confidence-Set Slack}\label{appendix:setting_parameters}

For a completed batch of size $m$, all settings use
\begin{equation}\label{eq:numerical_fixed_schedules}
u(m)=\ln(16m^2/\delta),\qquad
s_{2m-1}=s_0\sqrt{\frac{u(m)}m},\qquad
b_{2m-1}=b_0\left(\frac{u(2m)}{2m}\right)^{1/4}.
\end{equation}
Here the confidence input is $\delta=0.05/72$, $0.05/240$, and $0.05/63$ for the smaller grid, larger grid, and pricing class, respectively. These inputs reproduce the recorded thresholds and slacks without changing the scales. The first post-warm-up batch uses $b_0$ itself. These fixed orders are common across losses and sensitivity arms. The scales $s_0$ and $b_0$ are numerical inputs selected on development streams, rather than the statistical batch-error bounds in the analysis or quantiles recalculated on final warm-ups.

For the $3\times3$ generator, a finite catalogue includes the preceding data law and mixtures with different easy-component masses, separations, and coefficient ranks. Three-stream SPO+ screening retains at most three profiles, including an eligible control. Two independent calibration streams per profile and loss determine anchors from median warm-up margins and losses. We screen nine pairs of threshold and slack multipliers in $\{0.5,1,2\}$ on two development streams, using the minimum regret ratio at 1,000, 2,000, and 3,000 labels. Five new validation streams select the generator by the minimum ratio with SPO+ training at 2,000, 3,000, and 5,000 labels. Ties use fixed identifiers.

The $5\times5$ setting uses a finite bounded-feature catalogue, together with preceding-law controls. Approximate population-matrix screening retains the control and one candidate at each easy mass 0.85 and 0.95; mass 0.98 is diagnostic only. Two-stream calibration and a nine-pair scale search use three fresh development streams. Eight new validation streams per profile and loss require complete four-loss coverage and select the generator by its minimum ratio with SPO+ training at 2,000, 3,000, and 5,000 labels. The selected generator and all loss-specific scales are fixed before thirty fresh final streams. No final outcome is used for statistical retuning.

\begin{table}[ht]
\centering\small\setlength{\tabcolsep}{5pt}
\begin{tabular}{lrrrr}
\toprule
Training loss & $n_0$ & $b_0$ & $s_0$ & $\delta_H$ \\
\midrule
SPO+ & 100 & 0.631675 & 0.05 & -- \\
Squared & 100 & 2.22757 & 0.00767514 & -- \\
MAE & 100 & 2.22805 & 0.271243 & -- \\
Huber & 100 & 2.22742 & 0.00222361 & 0.0116452 \\
\bottomrule
\end{tabular}

\caption{Fixed initialization for the $3\times3$ independent-coordinate-noise setting. All losses use the schedules in \eqref{eq:numerical_fixed_schedules}. The Huber cutoff and scales are calibrated on development samples; full-precision inputs are retained in the numerical files.}
\label{tab:grid3_parameters}
\end{table}

\begin{table}[ht]
\centering\small\setlength{\tabcolsep}{5pt}
\begin{tabular}{lrrrr}
\toprule
Training loss & $n_0$ & $b_0$ & $s_0$ & $\delta_H$ \\
\midrule
SPO+ & 200 & 0.208528 & 0.025 & -- \\
Squared & 200 & 0.0890489 & 0.0234022 & -- \\
MAE & 200 & 0.0891542 & 0.93674 & -- \\
Huber & 200 & 0.0891638 & 0.00819009 & 0.0133885 \\
\bottomrule
\end{tabular}

\caption{Fixed initialization for the bounded-feature $5\times5$ setting with additive Gaussian cost noise. The prediction class and fixed schedule orders are unchanged across training losses.}
\label{tab:grid5_parameters}
\end{table}

All four losses for a grid share thirty final raw streams and their test observations. The final seed namespaces are disjoint from that grid's generator screening, calibration, scale development, and validation. We use 100,000 test observations per smaller-grid trial and 50,000 per larger-grid trial. Each main SPO+ figure and its appendix counterpart read the same fitted trajectories.

Regression comparisons evaluate at warm-up and 500, 1,000, 2,000, 3,000, 4,000, and 5,000 labels. SPO+ comparisons additionally evaluate every 250 labels through 3,000. At a checkpoint, we use the model available at the arrival acquiring that label; a within-batch label cannot use the later completed-batch fit. Supplemental supervised fits retain the preceding coarse-checkpoint predictor as numerical reference and do not propagate into subsequent coarse references. Thus denser reporting does not change the fitted trajectories.

The runs use at most twenty numerical workers, with one numerical-library thread per worker. Solver-call and paired-trial limits are 180 and 1,200 seconds for the $3\times3$ grid, and 900 and 3,600 seconds for the revised $5\times5$ grid. Data, scales, tolerances, and fallback policies are frozen and checksummed. Exact zero regrets remain in tables and are omitted from logarithmic axes without a floor. An incomplete planned cohort has no primary aggregate; we retain its failure and completed-pair count without replacing seeds or averaging survivors.

\subsection{Full Shortest Path Trajectories and Risk Levels}\label{appendix:numerical_experiments_results}

Tables~\ref{tab:grid3_budget_ratios} and~\ref{tab:grid5_budget_ratios} report SPO regret ratios at all six coarse label budgets. Figures~\ref{fig:complete_grid3_spop}--\ref{fig:complete_grid5_huber} retain the full trajectories for each training loss. Tables~\ref{tab:grid3_risk_levels} and~\ref{tab:grid5_risk_levels} separately give absolute SPO regret and full realized-cost SPO risk. Every main cohort contains all thirty planned pairs. The realized-cost risk includes the irreducible contribution of cost noise and is distinct from the regret in \eqref{eq:numerical_spo_regret}.

At 3,000 labels, regret ratios for squared, MAE, and Huber training are 4.532, 3.375, and 3.017 in the $3\times3$ grid, and 2.328, 2.321, and 2.342 in the $5\times5$ grid. At 5,000 labels, they are 5.232, 4.605, and 4.051 in the smaller grid, and 1.653, 1.837, and 1.696 in the larger grid. All corresponding pointwise intervals lie above one. The SPO+ training ratios are those described in Section~\ref{sec:shortest_path}. These relative gains do not rank training losses by absolute decision accuracy: their absolute regrets differ, and every loss has its own developed scale pair.

\clearpage
\begin{table}[ht]
\centering\small\setlength{\tabcolsep}{5pt}
\begin{tabular}{rcccc}
\toprule
Labels & SPO+ & Squared & MAE & Huber \\
\midrule
500 & \shortstack{21.824\\{\scriptsize [14.488, 34.395]}} & \shortstack{5.093\\{\scriptsize [3.665, 7.089]}} & \shortstack{3.909\\{\scriptsize [3.202, 4.796]}} & \shortstack{4.169\\{\scriptsize [3.225, 5.454]}} \\
1,000 & \shortstack{29.828\\{\scriptsize [16.942, 56.591]}} & \shortstack{7.022\\{\scriptsize [5.471, 9.067]}} & \shortstack{6.621\\{\scriptsize [5.282, 8.210]}} & \shortstack{6.402\\{\scriptsize [4.957, 8.229]}} \\
2,000 & \shortstack{96.884\\{\scriptsize [69.964, 129.619]}} & \shortstack{6.776\\{\scriptsize [5.177, 9.048]}} & \shortstack{5.338\\{\scriptsize [4.191, 6.848]}} & \shortstack{5.173\\{\scriptsize [4.061, 6.590]}} \\
3,000 & \shortstack{62.447\\{\scriptsize [48.875, 81.168]}} & \shortstack{4.532\\{\scriptsize [3.498, 5.884]}} & \shortstack{3.375\\{\scriptsize [2.806, 4.102]}} & \shortstack{3.017\\{\scriptsize [2.424, 3.818]}} \\
4,000 & \shortstack{110.062\\{\scriptsize [80.980, 144.559]}} & \shortstack{6.058\\{\scriptsize [4.519, 8.146]}} & \shortstack{5.470\\{\scriptsize [4.079, 7.495]}} & \shortstack{4.780\\{\scriptsize [3.718, 6.197]}} \\
5,000 & \shortstack{108.419\\{\scriptsize [79.182, 145.766]}} & \shortstack{5.232\\{\scriptsize [3.778, 7.262]}} & \shortstack{4.605\\{\scriptsize [3.348, 6.347]}} & \shortstack{4.051\\{\scriptsize [3.103, 5.347]}} \\
\bottomrule
\end{tabular}

\caption{SPO regret ratios for the $3\times3$ grid over the full label range. Columns identify training losses. Each cell uses all thirty pairs and divides mean supervised-learning SPO regret by mean MBAL-SPO regret, with its pointwise 90\% paired bootstrap interval.}
\label{tab:grid3_budget_ratios}
\end{table}

\begin{figure}[ht]
\centering
\includegraphics[width=0.7\columnwidth]{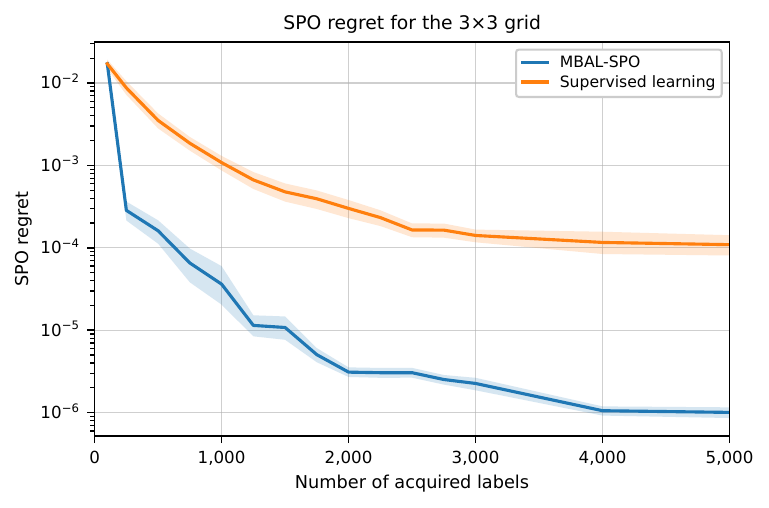}
\caption{SPO regret for the $3\times3$ shortest path problem under SPO+ training. Curves retain all thirty paired trials and show means at warm-up, every 250 labels through 3,000, and at 4,000 and 5,000, using the latest completed-batch predictor. Shading gives pointwise 90\% bootstrap intervals. Counts include warm-up and match within each method pair. The label axis is linear and the regret axis is logarithmic.}
\label{fig:complete_grid3_spop}
\end{figure}

\begin{figure}[ht]
\centering
\includegraphics[width=0.7\columnwidth]{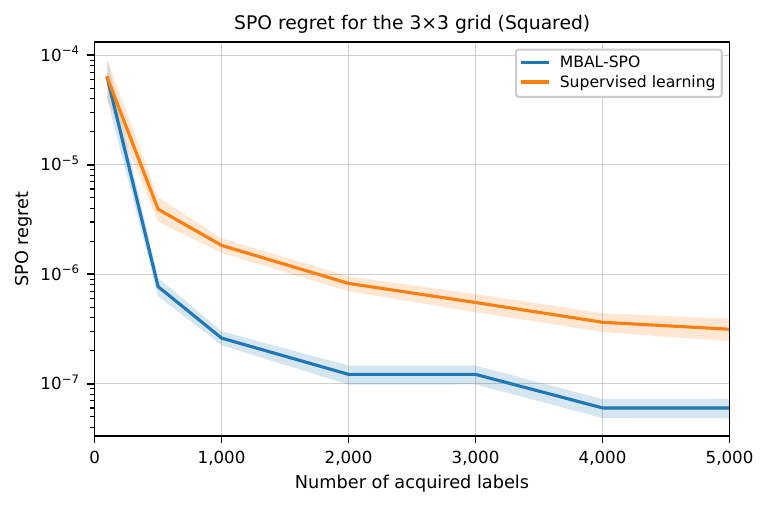}
\caption{SPO regret for the $3\times3$ shortest path problem under Squared training. Curves retain all thirty paired trials and show means at warm-up and 500, 1,000, 2,000, 3,000, 4,000, and 5,000 labels, using the latest completed-batch predictor. Shading gives pointwise 90\% bootstrap intervals. Counts include warm-up and match within each method pair. The label axis is linear and the regret axis is logarithmic.}
\label{fig:complete_grid3_squared}
\end{figure}

\begin{figure}[ht]
\centering
\includegraphics[width=0.7\columnwidth]{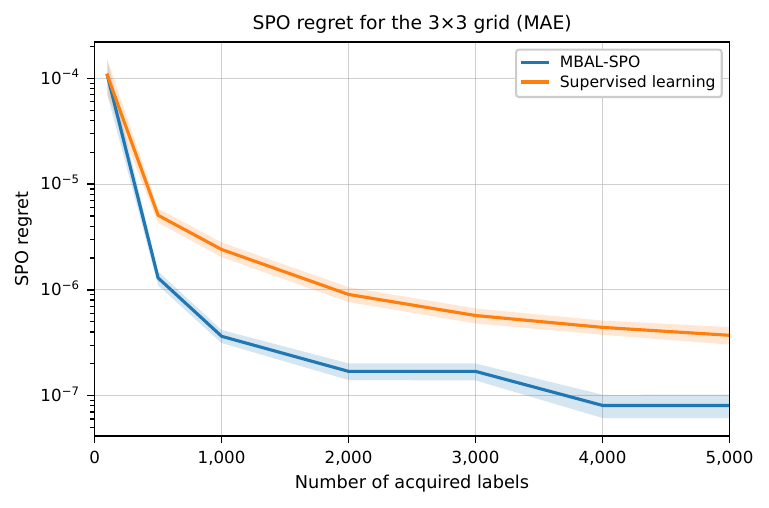}
\caption{SPO regret for the $3\times3$ shortest path problem under MAE training. Curves retain all thirty paired trials and show means at warm-up and 500, 1,000, 2,000, 3,000, 4,000, and 5,000 labels, using the latest completed-batch predictor. Shading gives pointwise 90\% bootstrap intervals. Counts include warm-up and match within each method pair. The label axis is linear and the regret axis is logarithmic.}
\label{fig:complete_grid3_mae}
\end{figure}

\begin{figure}[ht]
\centering
\includegraphics[width=0.7\columnwidth]{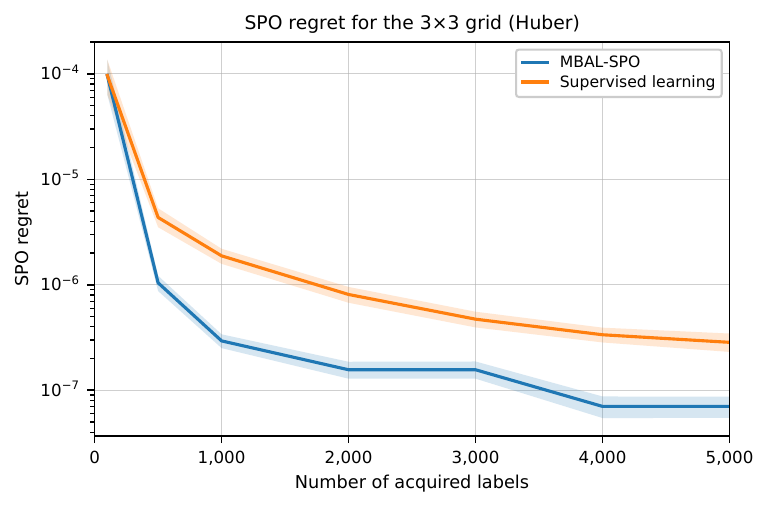}
\caption{SPO regret for the $3\times3$ shortest path problem under Huber training. Curves retain all thirty paired trials and show means at warm-up and 500, 1,000, 2,000, 3,000, 4,000, and 5,000 labels, using the latest completed-batch predictor. Shading gives pointwise 90\% bootstrap intervals. Counts include warm-up and match within each method pair. The label axis is linear and the regret axis is logarithmic.}
\label{fig:complete_grid3_huber}
\end{figure}

\begin{table}[ht]
\centering\small\setlength{\tabcolsep}{5pt}
\begin{tabular}{lrrrrr}
\toprule
Training loss & Labels & MBAL regret & SL regret & MBAL risk & SL risk \\
\midrule
SPO+ & 3,000 & $2.259\times10^{-6}$ & $1.411\times10^{-4}$ & $1.457\times10^{-5}$ & $1.519\times10^{-4}$ \\
SPO+ & 5,000 & $1.006\times10^{-6}$ & $1.091\times10^{-4}$ & $1.304\times10^{-5}$ & $1.215\times10^{-4}$ \\
Squared & 3,000 & $1.217\times10^{-7}$ & $5.516\times10^{-7}$ & $1.262\times10^{-5}$ & $1.300\times10^{-5}$ \\
Squared & 5,000 & $6.012\times10^{-8}$ & $3.145\times10^{-7}$ & $1.232\times10^{-5}$ & $1.274\times10^{-5}$ \\
MAE & 3,000 & $1.694\times10^{-7}$ & $5.716\times10^{-7}$ & $1.258\times10^{-5}$ & $1.288\times10^{-5}$ \\
MAE & 5,000 & $8.070\times10^{-8}$ & $3.716\times10^{-7}$ & $1.222\times10^{-5}$ & $1.267\times10^{-5}$ \\
Huber & 3,000 & $1.565\times10^{-7}$ & $4.723\times10^{-7}$ & $1.257\times10^{-5}$ & $1.291\times10^{-5}$ \\
Huber & 5,000 & $7.032\times10^{-8}$ & $2.849\times10^{-7}$ & $1.227\times10^{-5}$ & $1.270\times10^{-5}$ \\
\bottomrule
\end{tabular}

\caption{Mean SPO regret and full realized-cost SPO risk for the $3\times3$ grid. The first column identifies the training loss; every regret column uses \eqref{eq:numerical_spo_regret}. All thirty pairs are retained. Full SPO risk includes the irreducible cost-noise contribution and differs from SPO regret.}
\label{tab:grid3_risk_levels}
\end{table}

\clearpage
\begin{table}[ht]
\centering\small\setlength{\tabcolsep}{5pt}
\begin{tabular}{rcccc}
\toprule
Labels & SPO+ & Squared & MAE & Huber \\
\midrule
500 & \shortstack{1.559\\{\scriptsize [0.851, 3.074]}} & \shortstack{1.139\\{\scriptsize [1.016, 1.289]}} & \shortstack{1.289\\{\scriptsize [1.179, 1.416]}} & \shortstack{1.228\\{\scriptsize [1.082, 1.395]}} \\
1,000 & \shortstack{1.754\\{\scriptsize [1.371, 2.312]}} & \shortstack{1.571\\{\scriptsize [1.393, 1.776]}} & \shortstack{1.540\\{\scriptsize [1.391, 1.699]}} & \shortstack{1.647\\{\scriptsize [1.481, 1.828]}} \\
2,000 & \shortstack{1.328\\{\scriptsize [1.018, 1.740]}} & \shortstack{1.731\\{\scriptsize [1.540, 1.936]}} & \shortstack{1.710\\{\scriptsize [1.510, 1.932]}} & \shortstack{1.707\\{\scriptsize [1.522, 1.914]}} \\
3,000 & \shortstack{1.850\\{\scriptsize [1.752, 1.957]}} & \shortstack{2.328\\{\scriptsize [2.013, 2.693]}} & \shortstack{2.321\\{\scriptsize [2.026, 2.651]}} & \shortstack{2.342\\{\scriptsize [2.005, 2.730]}} \\
4,000 & \shortstack{1.519\\{\scriptsize [1.456, 1.589]}} & \shortstack{1.850\\{\scriptsize [1.597, 2.144]}} & \shortstack{1.858\\{\scriptsize [1.604, 2.147]}} & \shortstack{1.812\\{\scriptsize [1.523, 2.145]}} \\
5,000 & \shortstack{1.664\\{\scriptsize [1.548, 1.789]}} & \shortstack{1.653\\{\scriptsize [1.421, 1.927]}} & \shortstack{1.837\\{\scriptsize [1.617, 2.077]}} & \shortstack{1.696\\{\scriptsize [1.446, 1.985]}} \\
\bottomrule
\end{tabular}

\caption{SPO regret ratios for the $5\times5$ grid over the full label range. Columns identify training losses. Each cell uses all thirty pairs and divides mean supervised-learning SPO regret by mean MBAL-SPO regret, with its pointwise 90\% paired bootstrap interval.}
\label{tab:grid5_budget_ratios}
\end{table}

\begin{figure}[ht]
\centering
\includegraphics[width=0.7\columnwidth]{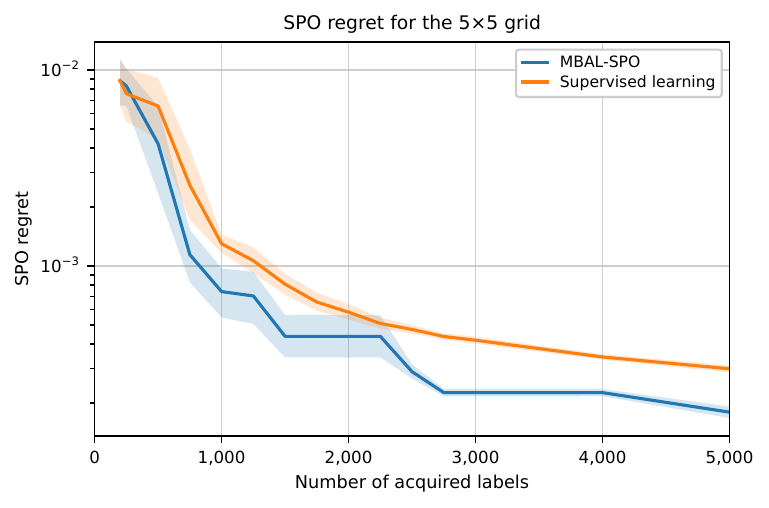}
\caption{SPO regret for the $5\times5$ shortest path problem under SPO+ training. Curves retain all thirty paired trials and show means at warm-up, every 250 labels through 3,000, and at 4,000 and 5,000, using the latest completed-batch predictor. Shading gives pointwise 90\% bootstrap intervals. Counts include warm-up and match within each method pair. The label axis is linear and the regret axis is logarithmic.}
\label{fig:complete_grid5_spop}
\end{figure}

\begin{figure}[ht]
\centering
\includegraphics[width=0.7\columnwidth]{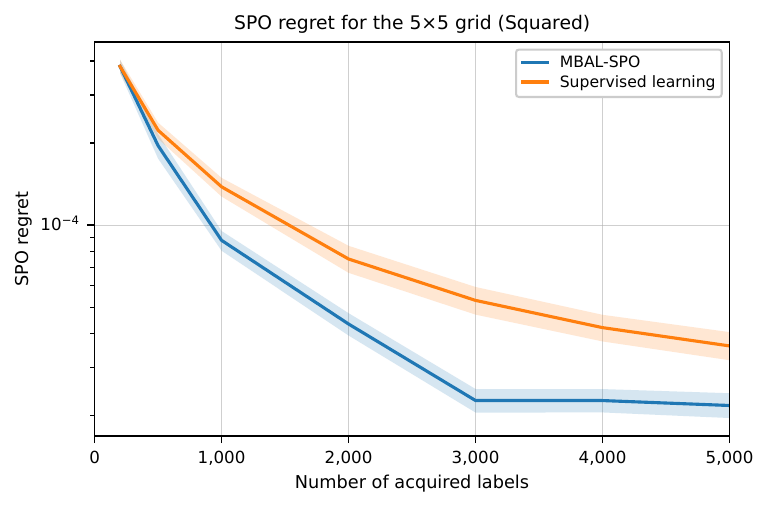}
\caption{SPO regret for the $5\times5$ shortest path problem under Squared training. Curves retain all thirty paired trials and show means at warm-up and 500, 1,000, 2,000, 3,000, 4,000, and 5,000 labels, using the latest completed-batch predictor. Shading gives pointwise 90\% bootstrap intervals. Counts include warm-up and match within each method pair. The label axis is linear and the regret axis is logarithmic.}
\label{fig:complete_grid5_squared}
\end{figure}

\begin{figure}[ht]
\centering
\includegraphics[width=0.7\columnwidth]{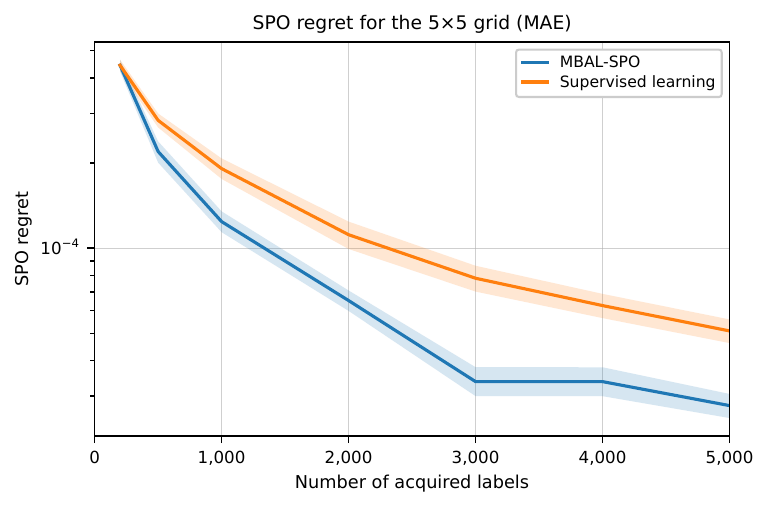}
\caption{SPO regret for the $5\times5$ shortest path problem under MAE training. Curves retain all thirty paired trials and show means at warm-up and 500, 1,000, 2,000, 3,000, 4,000, and 5,000 labels, using the latest completed-batch predictor. Shading gives pointwise 90\% bootstrap intervals. Counts include warm-up and match within each method pair. The label axis is linear and the regret axis is logarithmic.}
\label{fig:complete_grid5_mae}
\end{figure}

\begin{figure}[ht]
\centering
\includegraphics[width=0.7\columnwidth]{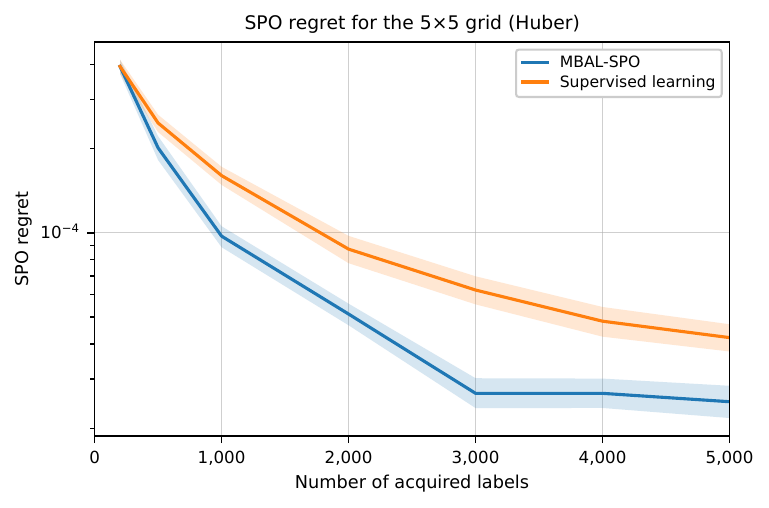}
\caption{SPO regret for the $5\times5$ shortest path problem under Huber training. Curves retain all thirty paired trials and show means at warm-up and 500, 1,000, 2,000, 3,000, 4,000, and 5,000 labels, using the latest completed-batch predictor. Shading gives pointwise 90\% bootstrap intervals. Counts include warm-up and match within each method pair. The label axis is linear and the regret axis is logarithmic.}
\label{fig:complete_grid5_huber}
\end{figure}

\begin{table}[ht]
\centering\small\setlength{\tabcolsep}{5pt}
\begin{tabular}{lrrrrr}
\toprule
Training loss & Labels & MBAL regret & SL regret & MBAL risk & SL risk \\
\midrule
SPO+ & 3,000 & $2.262\times10^{-4}$ & $4.185\times10^{-4}$ & 0.0024831 & 0.0026841 \\
SPO+ & 5,000 & $1.798\times10^{-4}$ & $2.993\times10^{-4}$ & 0.0024398 & 0.00256 \\
Squared & 3,000 & $2.276\times10^{-5}$ & $5.298\times10^{-5}$ & 0.0022814 & 0.0023082 \\
Squared & 5,000 & $2.183\times10^{-5}$ & $3.609\times10^{-5}$ & 0.002282 & 0.0022925 \\
MAE & 3,000 & $3.372\times10^{-5}$ & $7.825\times10^{-5}$ & 0.002296 & 0.0023339 \\
MAE & 5,000 & $2.774\times10^{-5}$ & $5.095\times10^{-5}$ & 0.0022905 & 0.0023093 \\
Huber & 3,000 & $2.664\times10^{-5}$ & $6.238\times10^{-5}$ & 0.0022885 & 0.0023209 \\
Huber & 5,000 & $2.486\times10^{-5}$ & $4.216\times10^{-5}$ & 0.0022872 & 0.0022974 \\
\bottomrule
\end{tabular}

\caption{Mean SPO regret and full realized-cost SPO risk for the $5\times5$ grid. The first column identifies the training loss; every regret column uses \eqref{eq:numerical_spo_regret}. All thirty pairs are retained. Full SPO risk includes the irreducible cost-noise contribution and differs from SPO regret.}
\label{tab:grid5_risk_levels}
\end{table}

\clearpage
\subsection{Controlled Sensitivity Analyses}\label{appendix:threshold_sensitivity}

We evaluate four training losses on ten new common streams per arm, distinct from the final and development cohorts. For each grid, we vary the rejection scale by multiplying its selected $b_0$ by $\{0.5,1,2\}$, keeping $s_0$ and the two schedule orders fixed. We also vary $\epsilon\in\{0.005,0.01,0.03\}$ for the smaller grid and the additive noise standard deviation $\sigma_c\in\{0.01,0.02,0.06\}$ for the larger grid. Warm-up sizes are $\{50,100,200\}$ and $\{100,200,500\}$, respectively. The feature law, conditional mean, coefficient box, Huber cutoff, and remaining scales are held fixed. Three sweeps share one baseline, giving seven distinct configurations and 280 planned pairs per grid.

The grid comparisons use acquired-label checkpoints of 500, 1,000, 2,000, and 3,000. The following tables report 3,000 labels; earlier checkpoints remain in the numerical files. We use 100,000 test observations per smaller-grid trial and 50,000 per larger-grid trial. Each noise arm uses the same underlying feature and innovation streams, with the changed amplitude affecting costs and acquisition. These are pointwise ten-pair comparisons, rather than simultaneous robustness guarantees.

All 280 smaller-grid pairs complete. At 3,000 labels, every ratio interval for every tested loss and arm lies above one. With SPO+ training, the point ratios range from 31.370 to 102.721. The regression-loss ratios also favor MBAL-SPO in this sensitivity cohort, though their magnitudes differ from the thirty-trial final comparisons.

For the larger grid, the rejection-scale ratios with SPO+ training are 0.604 $[0.298,1.561]$, 1.581 $[1.177,2.152]$, and 1.744 $[1.539,2.009]$ at multipliers 0.5, 1, and 2. Halving the threshold therefore does not preserve the favorable baseline comparison. Increasing the noise standard deviation to 0.06 gives ratio 1.107 $[0.916,1.349]$, whose interval includes one. These arms retain their fixed slack rather than being recalibrated after the sensitivity outcomes.

Of the 280 larger-grid pairs, 278 complete. Squared-loss training at $\sigma_c=0.06$ has one native conic-certificate failure, and SPO+ training with $n_0=500$ has one supervised-prefix LP time limit. Each affected cell is unavailable and displays nine of ten completed pairs; no survivor mean or replacement seed is used. Every other complete regression-loss interval lies above one. These numerical failures remain separate from the complete, statistically unfavorable threshold arm.

\begin{table}[ht]
\centering\small\setlength{\tabcolsep}{5pt}
\begin{tabular}{rcccc}
\toprule
$b_0$ multiplier & SPO+ & Squared & MAE & Huber \\
\midrule
0.5 & \shortstack{102.721\\{\scriptsize [53.818, 194.028]}} & \shortstack{11.436\\{\scriptsize [6.891, 18.573]}} & \shortstack{5.361\\{\scriptsize [2.669, 10.263]}} & \shortstack{7.749\\{\scriptsize [5.081, 11.390]}} \\
1 & \shortstack{35.454\\{\scriptsize [24.590, 45.799]}} & \shortstack{7.756\\{\scriptsize [4.480, 14.357]}} & \shortstack{4.041\\{\scriptsize [2.314, 6.807]}} & \shortstack{6.855\\{\scriptsize [4.118, 11.122]}} \\
2 & \shortstack{33.322\\{\scriptsize [21.770, 44.669]}} & \shortstack{8.922\\{\scriptsize [5.563, 13.809]}} & \shortstack{4.778\\{\scriptsize [2.675, 8.599]}} & \shortstack{7.045\\{\scriptsize [4.150, 12.164]}} \\
\bottomrule
\end{tabular}

\caption{Sensitivity to the initial rejection scale for the $3\times3$ grid at 3,000 labels. All other loss-specific settings are fixed. Each complete cell reports the SPO regret ratio over ten pairs and its pointwise 90\% paired bootstrap interval. Unavailable cells display the completed-pair count.}
\label{tab:grid3_threshold_sensitivity}
\end{table}

\begin{table}[ht]
\centering\small\setlength{\tabcolsep}{5pt}
\begin{tabular}{rcccc}
\toprule
$b_0$ multiplier & SPO+ & Squared & MAE & Huber \\
\midrule
0.5 & \shortstack{0.604\\{\scriptsize [0.298, 1.561]}} & \shortstack{2.123\\{\scriptsize [1.748, 2.587]}} & \shortstack{1.982\\{\scriptsize [1.488, 2.614]}} & \shortstack{2.182\\{\scriptsize [1.728, 2.739]}} \\
1 & \shortstack{1.581\\{\scriptsize [1.177, 2.152]}} & \shortstack{2.382\\{\scriptsize [1.996, 2.892]}} & \shortstack{2.178\\{\scriptsize [1.712, 2.723]}} & \shortstack{2.363\\{\scriptsize [1.923, 2.910]}} \\
2 & \shortstack{1.744\\{\scriptsize [1.539, 2.009]}} & \shortstack{2.424\\{\scriptsize [2.032, 2.933]}} & \shortstack{2.204\\{\scriptsize [1.730, 2.746]}} & \shortstack{2.400\\{\scriptsize [1.956, 2.925]}} \\
\bottomrule
\end{tabular}

\caption{Sensitivity to the initial rejection scale for the $5\times5$ grid at 3,000 labels. All other loss-specific settings are fixed. Each complete cell reports the SPO regret ratio over ten pairs and its pointwise 90\% paired bootstrap interval. Unavailable cells display the completed-pair count.}
\label{tab:grid5_threshold_sensitivity}
\end{table}

\begin{table}[ht]
\centering\small\setlength{\tabcolsep}{5pt}
\begin{tabular}{rcccc}
\toprule
$\epsilon$ & SPO+ & Squared & MAE & Huber \\
\midrule
0.005 & \shortstack{92.976\\{\scriptsize [38.133, 180.305]}} & \shortstack{5.458\\{\scriptsize [3.233, 9.841]}} & \shortstack{3.116\\{\scriptsize [1.873, 5.909]}} & \shortstack{5.626\\{\scriptsize [2.972, 11.873]}} \\
0.01 & \shortstack{35.454\\{\scriptsize [24.590, 45.799]}} & \shortstack{7.756\\{\scriptsize [4.480, 14.357]}} & \shortstack{4.041\\{\scriptsize [2.314, 6.807]}} & \shortstack{6.855\\{\scriptsize [4.118, 11.122]}} \\
0.03 & \shortstack{66.660\\{\scriptsize [51.379, 83.149]}} & \shortstack{5.209\\{\scriptsize [3.772, 7.280]}} & \shortstack{3.826\\{\scriptsize [2.931, 5.159]}} & \shortstack{3.705\\{\scriptsize [2.867, 4.871]}} \\
\bottomrule
\end{tabular}

\caption{Sensitivity to independent multiplicative noise amplitude for the $3\times3$ grid at 3,000 labels. All other loss-specific settings are fixed. Each complete cell reports the SPO regret ratio over ten pairs and its pointwise 90\% paired bootstrap interval. Unavailable cells display the completed-pair count.}
\label{tab:grid3_epsilon_sensitivity}
\end{table}

\begin{table}[ht]
\centering\small\setlength{\tabcolsep}{5pt}
\begin{tabular}{rcccc}
\toprule
$\sigma_c$ & SPO+ & Squared & MAE & Huber \\
\midrule
0.01 & \shortstack{2.369\\{\scriptsize [1.927, 2.950]}} & \shortstack{2.389\\{\scriptsize [2.026, 2.844]}} & \shortstack{2.232\\{\scriptsize [1.783, 2.774]}} & \shortstack{2.393\\{\scriptsize [1.986, 2.928]}} \\
0.02 & \shortstack{1.581\\{\scriptsize [1.177, 2.152]}} & \shortstack{2.382\\{\scriptsize [1.996, 2.892]}} & \shortstack{2.178\\{\scriptsize [1.712, 2.723]}} & \shortstack{2.363\\{\scriptsize [1.923, 2.910]}} \\
0.06 & \shortstack{1.107\\{\scriptsize [0.916, 1.349]}} & \shortstack{Unavailable\\{\scriptsize 9/10 pairs}} & \shortstack{1.464\\{\scriptsize [1.232, 1.748]}} & \shortstack{1.507\\{\scriptsize [1.280, 1.774]}} \\
\bottomrule
\end{tabular}

\caption{Sensitivity to additive Gaussian cost-noise standard deviation for the $5\times5$ grid at 3,000 labels. All other loss-specific settings are fixed. Each complete cell reports the SPO regret ratio over ten pairs and its pointwise 90\% paired bootstrap interval. Unavailable cells display the completed-pair count.}
\label{tab:grid5_epsilon_sensitivity}
\end{table}

\begin{table}[ht]
\centering\small\setlength{\tabcolsep}{5pt}
\begin{tabular}{rcccc}
\toprule
$n_0$ & SPO+ & Squared & MAE & Huber \\
\midrule
50 & \shortstack{57.635\\{\scriptsize [27.770, 95.377]}} & \shortstack{7.720\\{\scriptsize [4.501, 14.408]}} & \shortstack{3.887\\{\scriptsize [2.236, 6.530]}} & \shortstack{6.897\\{\scriptsize [4.175, 11.226]}} \\
100 & \shortstack{35.454\\{\scriptsize [24.590, 45.799]}} & \shortstack{7.756\\{\scriptsize [4.480, 14.357]}} & \shortstack{4.041\\{\scriptsize [2.314, 6.807]}} & \shortstack{6.855\\{\scriptsize [4.118, 11.122]}} \\
200 & \shortstack{31.370\\{\scriptsize [20.175, 49.107]}} & \shortstack{7.366\\{\scriptsize [4.280, 13.658]}} & \shortstack{4.221\\{\scriptsize [2.507, 6.872]}} & \shortstack{6.391\\{\scriptsize [3.765, 10.714]}} \\
\bottomrule
\end{tabular}

\caption{Sensitivity to the warm-up size for the $3\times3$ grid at 3,000 labels. All other loss-specific settings are fixed. Each complete cell reports the SPO regret ratio over ten pairs and its pointwise 90\% paired bootstrap interval. Unavailable cells display the completed-pair count.}
\label{tab:grid3_warmup_sensitivity}
\end{table}

\begin{table}[ht]
\centering\small\setlength{\tabcolsep}{5pt}
\begin{tabular}{rcccc}
\toprule
$n_0$ & SPO+ & Squared & MAE & Huber \\
\midrule
100 & \shortstack{1.976\\{\scriptsize [1.724, 2.295]}} & \shortstack{2.381\\{\scriptsize [1.976, 2.920]}} & \shortstack{2.283\\{\scriptsize [1.840, 2.796]}} & \shortstack{2.420\\{\scriptsize [2.014, 2.941]}} \\
200 & \shortstack{1.581\\{\scriptsize [1.177, 2.152]}} & \shortstack{2.382\\{\scriptsize [1.996, 2.892]}} & \shortstack{2.178\\{\scriptsize [1.712, 2.723]}} & \shortstack{2.363\\{\scriptsize [1.923, 2.910]}} \\
500 & \shortstack{Unavailable\\{\scriptsize 9/10 pairs}} & \shortstack{2.166\\{\scriptsize [1.769, 2.723]}} & \shortstack{1.950\\{\scriptsize [1.499, 2.504]}} & \shortstack{1.926\\{\scriptsize [1.497, 2.500]}} \\
\bottomrule
\end{tabular}

\caption{Sensitivity to the warm-up size for the $5\times5$ grid at 3,000 labels. All other loss-specific settings are fixed. Each complete cell reports the SPO regret ratio over ten pairs and its pointwise 90\% paired bootstrap interval. Unavailable cells display the completed-pair count.}
\label{tab:grid5_warmup_sensitivity}
\end{table}

\paragraph{Pricing sensitivity.}
We vary only the selected $b_0$ by multipliers $\{0.5,1,2\}$ on ten fresh common streams per loss at $T=2047$. The exponential demand law, coefficient class, cost units, and remaining scales are fixed. All 120 pairs complete. SPO+, MAE, and Huber training give ratio intervals above one at all three multipliers. Squared-loss training gives point ratios 1.211, 1.000, and 1.260; the first interval includes one before rounding and the baseline is effectively tied, whereas the doubled-threshold interval lies above one. These sensitivity results are not pooled with the thirty-trial final comparison, whose squared-loss interval spans one.

\begin{table}[ht]
\centering\small\setlength{\tabcolsep}{5pt}
\begin{tabular}{rcccc}
\toprule
$b_0$ multiplier & SPO+ & Squared & MAE & Huber \\
\midrule
0.5 & \shortstack{15.678\\{\scriptsize [6.550, 66.369]}} & \shortstack{1.211\\{\scriptsize [1.000, 1.469]}} & \shortstack{2.614\\{\scriptsize [1.109, 13.491]}} & \shortstack{2.791\\{\scriptsize [1.629, 7.013]}} \\
1 & \shortstack{38.608\\{\scriptsize [24.746, 69.717]}} & \shortstack{1.000\\{\scriptsize [1.000, 1.000]}} & \shortstack{3.213\\{\scriptsize [1.319, 33.715]}} & \shortstack{2.783\\{\scriptsize [1.532, 6.202]}} \\
2 & \shortstack{6.554\\{\scriptsize [4.215, 12.954]}} & \shortstack{1.260\\{\scriptsize [1.002, 1.725]}} & \shortstack{7.296\\{\scriptsize [2.900, 23.275]}} & \shortstack{2.783\\{\scriptsize [1.353, 10.026]}} \\
\bottomrule
\end{tabular}

\caption{Sensitivity to the initial rejection scale in personalized pricing at $T=2047$. Multipliers apply to each loss's frozen $b_0$; the fourth-root threshold, square-root slack, and other parameters remain fixed. Each cell reports the SPO regret ratio over ten pairs and its pointwise 90\% paired bootstrap interval.}
\label{tab:pricing_threshold_sensitivity}
\end{table}

\subsection{Data Generation and Implementation for Personalized Pricing}\label{appendix:pricing}

We retain the exponential-demand generator. For product $j\in\{1,2,3\}$ and candidate price $p_i\in\{60,80,90\}$, demand and cost are
\begin{align*}
d_j(p_i)&=\epsilon_{ij}\exp\!\left(g(B_j^TX+p_iA_j^TX)\right),
&c_{ij}&=-p_i d_j(p_i),\\
g(u)&=\min\{20,\max\{-30,u\}\},
&\epsilon_{ij}&\overset{\mathrm{iid}}{\sim}\operatorname{Unif}[0.98,1.02].
\end{align*}
The nine multiplicative variables are independent of $X$ and one another. The conditional mean replaces $\epsilon_{ij}$ by one and retains $g$. Writing the demand vectors as rows gives
\begin{align*}
A&=\begin{bmatrix}
0&0&0&1&0&1\\1&1&1&1&0&1\\0&0&0&1&0&1
\end{bmatrix},
&B&=\begin{bmatrix}
0&0&0&0&1&1\\0&0&1&0&0&1\\0&1&0&1&1&1
\end{bmatrix}.
\end{align*}

The features follow a six-component Gaussian mixture with covariance $3\times10^{-4}I_6$ and weights $(0.12,0.12,0.40,0.12,0.12,0.12)$. Let $a_1=-0.0202733$, $a_2=-0.0133531$, $a_3=-0.00540672$, $b_1=-1.19155$, $b_2=-1.45748$, and $b_3=-1.22819$, and define
\begin{align*}
m_1&=(b_1,b_1,b_1,a_1,a_1,a_1)^T,&m_2&=(b_1,b_2,b_3,a_1,a_2,a_3)^T,\\
m_3&=(b_3,b_3,b_3,a_3,a_3,a_3)^T,&m_4&=(b_2,b_2,b_2,a_2,a_2,a_2)^T,\\
m_5&=(b_1,b_1,b_2,a_1,a_1,a_2)^T,&m_6&=(b_2,b_3,b_3,a_2,a_3,a_3)^T.
\end{align*}
With $\bar m=6^{-1}\sum_{k=1}^6m_k$, the actual means are
\[
\mu_k=\bar m+0.75(m_k-\bar m)+0.002(0,0,0,1,1,1)^T.
\]

The binary variable $w_{ij}$ indicates selection of price $i$ for item $j$. The constraints are
\begin{align}
\sum_{i=1}^3w_{ij}&=1,\quad j=1,2,3,\label{cons:n1}\\
w_{2,1}&\leq w_{2,2}+w_{3,2},\label{cons:n2}\\
w_{3,1}&\leq w_{3,2},\label{cons:n3}\\
w_{2,2}&\leq w_{2,3}+w_{3,3},\label{cons:n4}\\
w_{3,2}&\leq w_{3,3},\label{cons:n5}
\end{align}
with $w_{ij}\in\{0,1\}$. The first constraint chooses one price per item; the others impose $p^1\leq p^2\leq p^3$. There are ten feasible price decisions.

We generate 4,000 observations per trial, with the first 3,000 available for training and the last 1,000 held out for testing. MBAL-SPO uses 200 warm-up labels and 2,047 further arrivals, ending at a completed batch. Each loss uses a signed $7\times9$ affine coefficient matrix, including the intercept. The conditional mean is nonlinear, so the pricing model remains deliberately misspecified.

\begin{table}[ht]
\centering\small\setlength{\tabcolsep}{5pt}
\begin{tabular}{lrrrrr}
\toprule
Training loss & $n_0$ & $b_0$ & $s_0$ & Cost divisor & Coefficient bound \\
\midrule
SPO+ & 200 & 0.17222 & 2.22352 & 100 & 8 \\
Squared & 200 & 0.0113622 & 1.6844 & 1000 & 0.8 \\
MAE & 200 & 0.00101832 & 14.8865 & 100 & 8 \\
Huber & 200 & 0.00119286 & 0.0740007 & 100 & 8 \\
\bottomrule
\end{tabular}

\caption{Fixed pricing initialization and numerical cost units. The initial scales and coefficient bound are expressed on the indicated normalized training-cost scale. Every loss uses the fixed schedules in \eqref{eq:numerical_fixed_schedules}.}
\label{tab:pricing_parameters}
\end{table}

The squared-loss costs and coefficients have an additional factor-of-ten rescaling for conditioning. After returning to original revenue units, every loss has coefficient bound 800. All nine coordinate losses are summed, and pricing Huber loss uses its conventional half-square definition with $\delta_H=0.01$ on the normalized scale. Reported SPO regrets are evaluated in original revenue units, using the true exponential conditional mean.

SPO+ training uses a finite-decision epigraph linear program, solved by HiGHS, followed by a secondary $\ell_1$ movement minimization. Regression training uses CVXPY and Clarabel, with historical constraints added when their evaluated losses violate the caps. A small $\ell_1$ movement term selects a numerical minimizer. Absolute acceptance tolerances are $2\times10^{-4}$ for squared and MAE training and $3\times10^{-5}$ for SPO+ and Huber training; objective gaps and historical violations must also be at most 5\% of the applicable scheduled slack. The fixed conditioning-only fallback preserves the objective, feasible set, tolerances, and seed. Solver-call and paired-trial limits are 180 and 300 seconds. Supervised learning minimizes the same surrogate over the same box and an equally sized fully labeled prefix, without the historical confidence constraints.

We compare the original exponential law with two exponential boundary-mixture candidates selected by approximate population-matrix screening. Two-stream warm-up calibration and nine threshold/slack scale pairs use three development streams per profile and loss. The developed scales are then frozen. Eight new validation streams per candidate and loss give 96 complete pairs. Requiring complete four-loss coverage, we select the candidate with the largest final SPO regret ratio under SPO+ training; ties use the profile identifier. The original exponential law is eligible and wins this comparison. Thirty new final streams and ten new sensitivity streams are disjoint from all pricing development and validation samples.

Curvature at the nonlinear conditional mean can guide candidate screening, but it is not the expected Hessian at a common affine population minimizer. Feature selection can change that minimizer under misspecification. We therefore do not use these diagnostics to assert the selected-information conditions for pricing. The empirical comparison retains the exponential demand family and tests downstream behavior under the stated misspecification.

\begin{table}[ht]
\centering\small\setlength{\tabcolsep}{5pt}
\begin{tabular}{lrrrc}
\toprule
Training loss & Mean labels & MBAL regret & SL regret & SL/MBAL \\
\midrule
SPO+ & 1,050.5 & 0.017644 & 0.16048 & \shortstack{9.096\\{\scriptsize [3.657, 40.637]}} \\
Squared & 288.0 & 3.0981 & 2.4911 & \shortstack{0.804\\{\scriptsize [0.463, 1.512]}} \\
MAE & 408.5 & 21.384 & 108.1 & \shortstack{5.055\\{\scriptsize [2.648, 13.233]}} \\
Huber & 308.9 & 28.94 & 104.2 & \shortstack{3.600\\{\scriptsize [1.848, 11.667]}} \\
\bottomrule
\end{tabular}

\caption{Mean SPO regret in original revenue units for pricing at $T=2047$, with the SPO regret ratio and its pointwise 90\% paired bootstrap interval. Each loss uses all thirty paired trials. Counts include warm-up and match within each method pair.}
\label{tab:pricing_spo_regret}
\end{table}

All 120 pricing final pairs complete and pass acquisition, schedule, historical-cap, and SPO-regret replay. Their saved features, realized costs, and true conditional means also agree exactly with regeneration under the exponential law. Mean SPO regret favors MBAL-SPO for SPO+, MAE, and Huber training, with ratios 9.096, 5.055, and 3.600. Squared-loss training gives ratio 0.804 with an interval spanning one. Figures~\ref{fig:pricing_squared}--\ref{fig:pricing_huber} retain the earlier batch reversals and large individual regrets rather than showing only the favorable terminal comparisons.

Demand intensities are not purchase probabilities. Moreover, the lower clipping threshold is reached for all three prices of item 2 in these final cohorts, giving that item negligible conditional-mean revenue. We retain this generator and do not interpret it as a model with three equally informative products. The results are confined to the stated exponential-demand benchmark and do not establish a general benefit under misspecification.

\clearpage
\begin{figure}[ht]
\centering
\includegraphics[width=0.7\columnwidth]{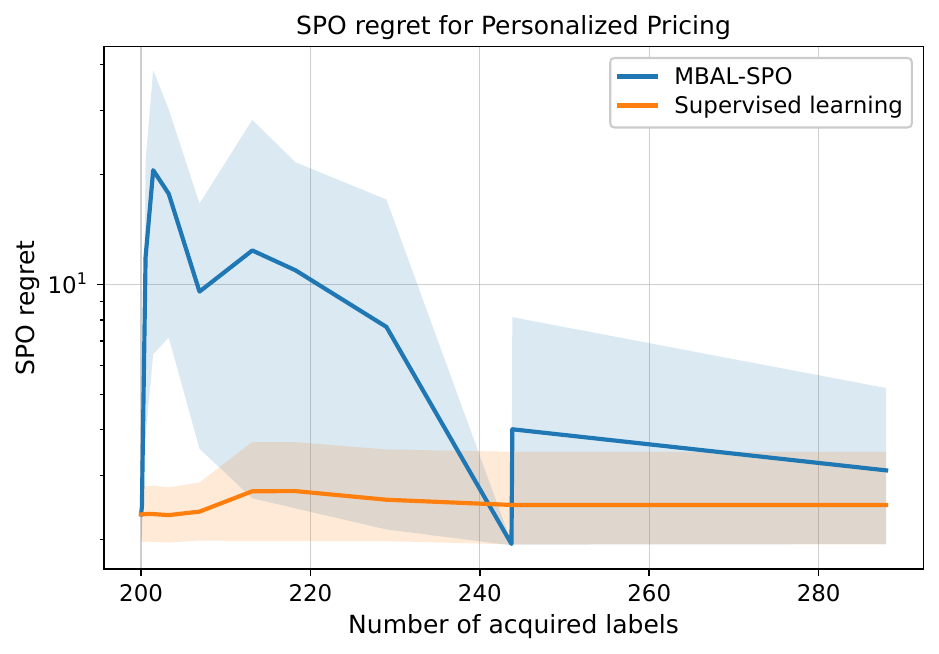}
\caption{SPO regret for personalized pricing under Squared training. Curves show means over thirty paired trials at warm-up and the eleven completed batches, with pointwise 90\% bootstrap intervals. The horizontal coordinate is mean acquired labels, including 200 warm-up labels, and counts match within each method pair. The regret axis is logarithmic. Every checkpoint and large individual regret is retained.}
\label{fig:pricing_squared}
\end{figure}

\clearpage
\begin{figure}[ht]
\centering
\includegraphics[width=0.7\columnwidth]{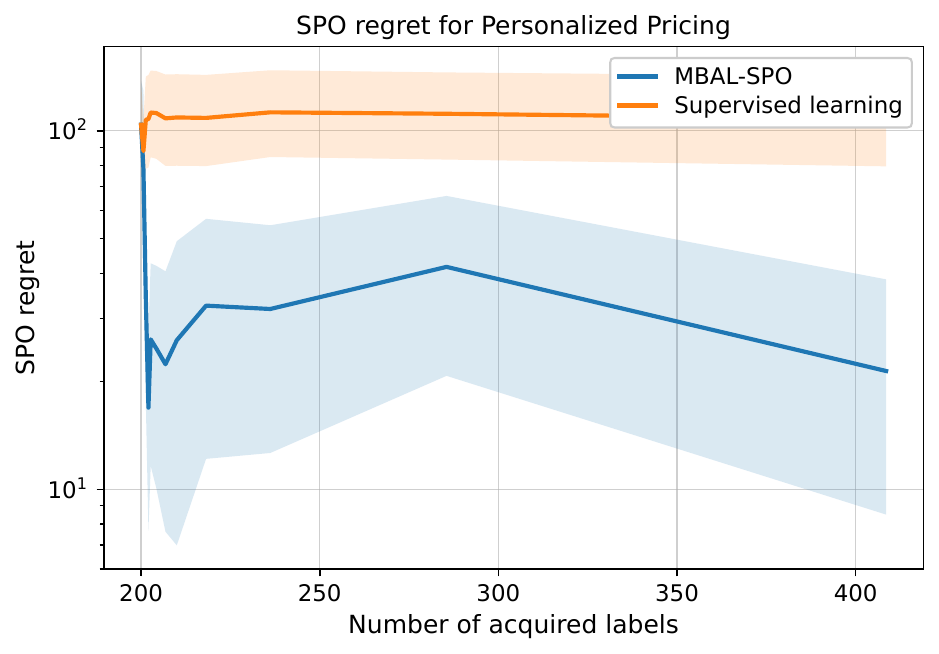}
\caption{SPO regret for personalized pricing under MAE training. Curves show means over thirty paired trials at warm-up and the eleven completed batches, with pointwise 90\% bootstrap intervals. The horizontal coordinate is mean acquired labels, including 200 warm-up labels, and counts match within each method pair. The regret axis is logarithmic. Every checkpoint and large individual regret is retained.}
\label{fig:pricing_mae}
\end{figure}

\clearpage
\begin{figure}[ht]
\centering
\includegraphics[width=0.7\columnwidth]{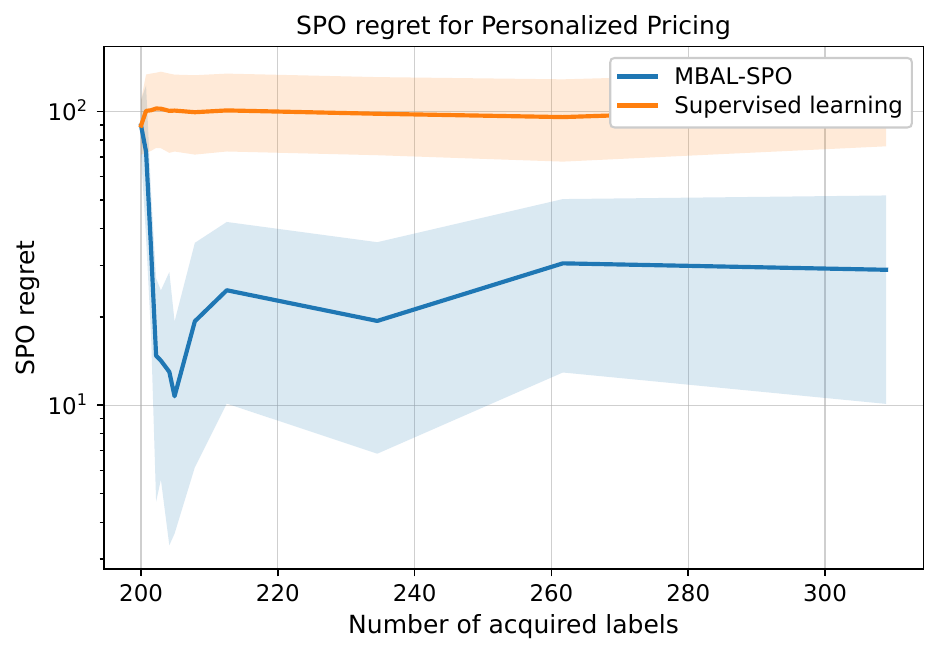}
\caption{SPO regret for personalized pricing under Huber training. Curves show means over thirty paired trials at warm-up and the eleven completed batches, with pointwise 90\% bootstrap intervals. The horizontal coordinate is mean acquired labels, including 200 warm-up labels, and counts match within each method pair. The regret axis is logarithmic. Every checkpoint and large individual regret is retained.}
\label{fig:pricing_huber}
\end{figure}

\paragraph{Running time.}
We record the wall-clock time of each complete paired trial,
including process startup, data preparation, MBAL-SPO and matched
supervised fits, checkpoint evaluation, and result output.
For SPO+, squared, MAE, and Huber training, respectively, the
mean runtimes $\pm$ sample standard deviations over thirty final
paired trials, in seconds, are
$47.42\pm2.57$, $1.81\pm0.04$, $6.78\pm0.08$, $5.81\pm0.06$
for the $3\times3$ grid;
$477.97\pm90.52$, $2.39\pm0.13$, $22.37\pm1.19$, $18.71\pm1.08$
for the $5\times5$ grid; and
$4.82\pm0.45$, $1.31\pm0.07$, $1.64\pm0.35$, $1.56\pm0.25$
for personalized pricing.
Grid trials end at 5,000 acquired labels, whereas pricing trials
end at the completed horizon $T=2047$.
The grid SPO+ comparisons include fifteen evaluation checkpoints,
the other grid losses seven, and pricing twelve, including warm-up.

\end{APPENDICES}

\end{document}